\documentclass{article}

\usepackage{microtype}
\usepackage{graphicx}
\usepackage{booktabs} % for professional tables

\usepackage{hyperref}

\usepackage[table,xcdraw,usenames,dvipsnames]{xcolor}
\usepackage[preprint]{icml2025}

\usepackage{amsmath}
\usepackage{amssymb}
\usepackage{mathtools}
\usepackage{amsthm}

\usepackage[capitalize,noabbrev]{cleveref}

\theoremstyle{plain}
\newtheorem{theorem}{Theorem}[section]

\newtheorem{corollary}[theorem]{Corollary}
\theoremstyle{definition}
\newtheorem{definition}[theorem]{Definition}

\theoremstyle{remark}

\usepackage[textsize=tiny]{todonotes}
\usepackage{subcaption}

\usepackage{multirow}
\usepackage{framed}
\usepackage{fancybox}
\usepackage{arydshln}
\usepackage{soul}
\definecolor{mygray}{gray}{0.95}
\colorlet{shadecolor}{mygray}
\colorlet{framecolor}{white}

{\endMakeFramed}

{\endMakeFramed}

\newcounter{cfinding}[section]
\newenvironment{cfinding}[1][]{\refstepcounter{cfinding}\par\medskip
   \noindent\textsc{Takeaway~\thecfinding. #1} \rmfamily}{\medskip}

\usepackage[most]{tcolorbox}

\usepackage[textsize=tiny]{todonotes}
\newcommand{\edit}[1]{\textcolor{black}{#1}}

\newcommand{\probname}{Continual Adaptive Robustness}

\DeclareMathOperator*{\argmax}{arg\,max}

\usepackage{enumitem}
\setlist[itemize]{leftmargin=*, noitemsep, topsep=0pt}

\def\BibTeX{{\rm B\kern-.05em{\sc i\kern-.025em b}\kern-.08em
    T\kern-.1667em\lower.7ex\hbox{E}\kern-.125emX}}

\crefformat{section}{\S#2\color{blue}#1#3} % see manual of cleveref, section 8.2.1
\crefformat{subsection}{\S#2\color{blue}#1#3}
\crefformat{subsubsection}{\S#2#1#3}

\icmltitlerunning{Adapting to Evolving Adversaries with Regularized Continual Robust Training}

\begin{document}

\tcbset{
    myboxstyle/.style={
        colback=gray!10, % Light grey background
        colframe=white, % Default white border
        fonttitle=\bfseries, % Bold title font
        coltitle=black, % Title text color
        left=1mm, % Padding inside the box
        right=1mm,
        top=0.5mm,
        bottom=0.5mm,
        enhanced, % Enables advanced features
        borderline west={0.2mm}{0pt}{Green}, % Adds a 2mm red border to the left
        boxrule=0mm, % No outer border
        sharp corners, % Square corners
        before skip=5pt,
        after skip=5pt, 
    }
}

\twocolumn[
\icmltitle{Adapting to Evolving Adversaries with Regularized Continual Robust Training}

% It is OKAY to include author information, even for blind
% submissions: the style file will automatically remove it for you
% unless you've provided the [accepted] option to the icml2025
% package.

% List of affiliations: The first argument should be a (short)
% identifier you will use later to specify author affiliations
% Academic affiliations should list Department, University, City, Region, Country
% Industry affiliations should list Company, City, Region, Country

% You can specify symbols, otherwise they are numbered in order.
% Ideally, you should not use this facility. Affiliations will be numbered
% in order of appearance and this is the preferred way.
\icmlsetsymbol{equal}{*}

\begin{icmlauthorlist}
\icmlauthor{Sihui Dai}{equal,princeton}
\icmlauthor{Christian Cianfarani}{equal,uchicago}
\icmlauthor{Arjun Nitin Bhagoji}{uchicago}
\icmlauthor{Vikash Sehwag}{comp}
\icmlauthor{Prateek Mittal}{princeton}
%\icmlauthor{}{sch}
%\icmlauthor{}{sch}
\end{icmlauthorlist}

\icmlaffiliation{princeton}{Department of Electrical and Computer Engineering, Princeton University}
\icmlaffiliation{uchicago}{Department of Computer Science, University of Chicago}
\icmlaffiliation{comp}{Google Deepmind}

\icmlcorrespondingauthor{Sihui Dai}{sihuid@princeton.edu}
\icmlcorrespondingauthor{Christian Cianfarani}{crc@uchicago.edu}

% You may provide any keywords that you
% find helpful for describing your paper; these are used to populate
% the "keywords" metadata in the PDF but will not be shown in the document
\icmlkeywords{Machine Learning, ICML}

\vskip 0.3in
]

% this must go after the closing bracket ] following \twocolumn[ ...

% This command actually creates the footnote in the first column
% listing the affiliations and the copyright notice.
% The command takes one argument, which is text to display at the start of the footnote.
% The \icmlEqualContribution command is standard text for equal contribution.
% Remove it (just {}) if you do not need this facility.

%\printAffiliationsAndNotice{}  % leave blank if no need to mention equal contribution
\printAffiliationsAndNotice{\icmlEqualContribution} % otherwise use the standard text.

\begin{abstract}
% New abstract
%Adversarial training typically only provides protection against adversarial examples defined with respect to a given threat model. Attacks generated using different constraints can easily evade adversary-specific defenses. Several techniques have been developed to train models that are robust to multiple types of adversarial attacks, a quality we call multi-robustness for short. However, they offer little guidance on how to adapt models in the face of new adversaries.
%Adversarial training typically provides robustness only against adversarial examples defined by a specific attack type, such as an $\ell_p$ attack with a fixed budget. However, over time, defenders may encounter new attack types that were not accounted for during training. 
Robust training methods typically defend against specific attack types, such as $\ell_p$ attacks with fixed budgets, and rarely account for the fact that defenders may encounter new attacks over time.  A natural solution is to adapt the defended model to new adversaries as they arise via fine-tuning, a method which we call continual robust training (CRT).  However, when implemented naively, fine-tuning on new attacks degrades robustness on previous attacks.  This raises the question: \textit{how can we improve the initial training and fine-tuning of the model to simultaneously achieve robustness against previous and new attacks?} We present theoretical results which show that the gap in a model's robustness against different attacks is bounded by how far each attack perturbs a sample in the model's logit space, suggesting that regularizing with respect to this logit space distance can help maintain robustness against previous attacks.
Extensive experiments on 3 datasets (CIFAR-10, CIFAR-100, and ImageNette) and over 100 attack combinations demonstrate that the proposed regularization improves robust accuracy with little overhead in training time. Our findings and open-source code\footnote{Our code is available at: \url{https://github.com/inspire-group/continual_robust_training/}} lay the groundwork for the deployment of models robust to evolving attacks.
\end{abstract}

\section{Introduction}
%A large body of work has shown that image classification systems can easily be fooled by small perturbations \cite{szegedy2013intriguing,goodfellow2015,XiaoZ0HLS18,xiao2018generating,kurakin2018adversarial,kaufmann2019testing}. These perturbed inputs, known as adversarial examples, are typically bounded by small $\ell_p$ perturbations and imperceptible to the human eye~\citep{goodfellow2015}. The most promising line of defense against such attacks is known as adversarial training, which incorporates adversarial examples into the training process~\citep{madry2017towards, zhang2019theoretically, croce2020robustbench, gowal2020uncovering, cohen2019certified,  gowal2021improving, sehwag2021robust}. 

\begin{figure}[th]
    \centering
    \includegraphics[width=0.7\linewidth]{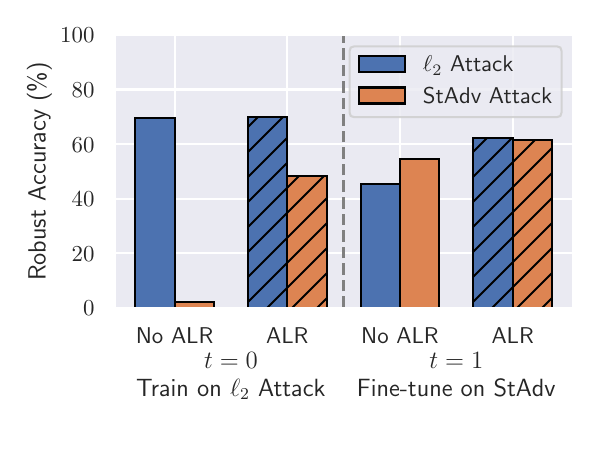}
        \vspace{-20pt}
    \caption{\textbf{Impact of our proposed regularization term (ALR) in both training and fine-tuning on CIFAR-10.}  Adversarial $\ell_2$ regularization (ALR) significantly improves generalization to the unforeseen StAdv attack when performing adversarial training for $\ell_2$ robustness.  Using ALR when subsequently fine-tuning with only StAdv attack also decreases the drop in $\ell_2$ robustness.}
    \label{fig:summary}
    \vspace{-20pt}
\end{figure}

For safety critical applications, it is important to defend machine learning (ML) models against test-time attacks.  However, many existing defenses \citep{madry2017towards, zhang2019theoretically, croce2020robustbench} assume that the adversary is restricted to a narrow threat model such as an $\ell_p$ ball of fixed radius around the input.
When this assumption is violated, the robustness of adversarially trained models can significantly degrade~\citep{dai2023multirobustbench, kaufmann2019testing}. Additionally, due to rapid development of new types of attacks \citep{XiaoZ0HLS18, LaidlawF19, laidlaw2020perceptual, kaufmann2019testing}, it is difficult to anticipate all types of attacks in advance.  This raises the question: \emph{how can we defend models as new attacks emerge?}%Several variants of adversarial training have been developed to train models robust to sets of attacks -- a quality we refer to as \textit{multi-robustness}~\citep{MainiWK20, TB19, madaan2020learning, Croce020,croce2022adversarial}. However, given the rapid development of new types of attacks, even a multi-robust model could still be fooled by an adversary not accounted for during training. This raises the question: \emph{how can we defend models as new adversaries emerge?}
%Prior work in the field of multi-robust training provides possible starting points. One approach would be train a model using a method that ensures robustness against all types of attacks, including those not known at training time~\cite{laidlaw2020perceptual,dai2022formulating}. However, these approaches inevitably make assumptions about the kinds of attacks that will be seen in the future. For example, \citet{laidlaw2020perceptual} provide a method to defend against a wide spectrum of imperceptible attacks. In practice, though, attacks against image classifiers need not be imperceptible, so long as they do not change the semantic meaning of an image~\cite{kaufmann2019testing}. This limitation makes it difficult for these techniques to provide strong guarantees against adaptive adversaries. Another approach would be to train a new multi-robust model from scratch on a set of known attacks each time a new attack is introduced~\cite{MainiWK20,TB19, madaan2020learning, Croce020, croce2022adversarial}. However, these techniques are computationally intensive, and do not make use of the information that was learned in previous iterations of model training. An ideal approach would be able to rapidly learn defenses against new attacks while not forgetting information already learned about previous attacks.

For long-term robustness, models must quickly adapt to new attacks without sacrificing robustness to previous ones, a goal known as continual adaptive robustness (CAR) \citep{dai2024position} (\S \ref{sec: setup}). A natural approach is to apply adversarial training on known attacks and fine-tune when new ones emerge, a process we call continual robust training (CRT). However, adversarial training provides poor generalization to unseen attacks, leading to suboptimal starting points for fine-tuning, and fine-tuning itself can degrade robustness against past attacks (\cref{fig:summary}).

We theoretically show that the robustness gap between attacks is linked to logit-space distances between perturbed and clean inputs and regularizing these distances can improve generalization to new attacks and reduce drops in robustness on previous attacks. Extensive experiments confirm these findings. Our key contributions are as follows:

\begin{figure*}[!t]
	\centering
	\includegraphics[width=0.9\textwidth]{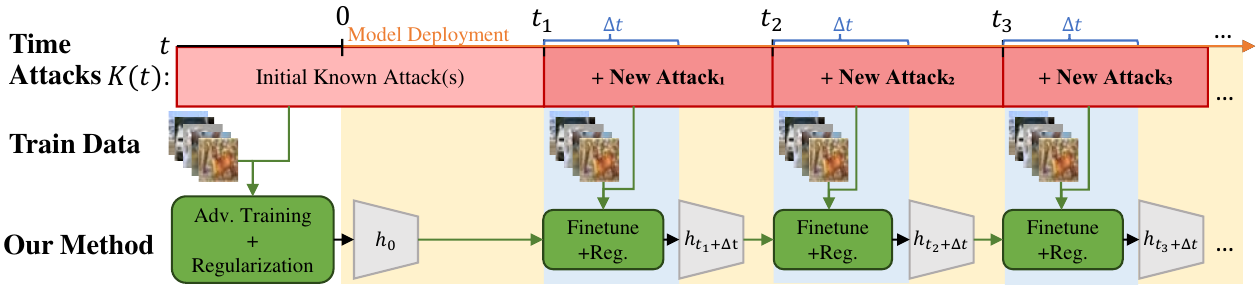}
	\caption{\textbf{An overview of the problem of adapting to new adversaries (continual adaptive robustness) and our solution framework (Regularized Continual Robust Training).} The defender learns about the existence of new attacks sequentially, and at time $t$ aims to achieve robustness against $K(t)$, the set of attacks known at times $\le t$.  A model $h_0$ is deployed at time $0$ to be robust against an initial set of known attacks, and new attacks are introduced at times $t_1$, $t_2$, and $t_3$. We propose performing regularized initial robust training \edit{on the initially known attack(s)} and then using regularized fine-tuning to adapt the model against future attacks within time $\Delta t$, leading to a sequence of models $h_0, h_{t_1 + \Delta t}, h_{t_3 + \Delta t}, h_{t_3 + \Delta t}$.}
	\label{fig:overview}
	\vspace{-10pt}
\end{figure*}

\noindent \textbf{Regularized Continual Robust Training for Adapting to New Adversaries} (\cref{sec: theory_methods}) . To enhance CRT, we analyze the difference in robust losses between attacks and show it is upper bounded by the sum of the maximal $\ell_2$ distance between clean and perturbed logits for both attacks. Minimizing this bound can thus improve generalization to new attacks and preserve robustness against previous ones. This motivates our proposed \textit{adversarial $\ell_2$ regularization (ALR)}, which penalizes the $\ell_2$ distance between adversarial and benign logits.
%To improve the performance of CRT, we analyze the difference in robust losses for a pair of attacks.  We prove this value is upper bounded by the maximal $\ell_2$ distance between the logits of the clean and a perturbed versions of a data point. Reducing this upper bound can thus improve generalization to unforeseen attacks and reduce degradation in robustness on previous attacks when used during fine-tuning.  This leads us to introduce a regularization term which we call \textit{adversarial $\ell_2$ regularization} (ALR). ALR induces a penalty to the loss proportional to the $\ell_2$ distance between adversarial and benign logits.

\noindent \textbf{Empirical Validation on Sequentially Introduced Attacks} (\cref{sec:car_reg}).
We conduct experiments on 2 sequences of 4 attacks across 3 datasets (CIFAR-10, CIFAR-100, and Imagenette). Our results show that ALR improves robustness in CRT with a 5.48\% gain in Union accuracy (worst-case across all attacks) across $\ell_2$, StAdv \citep{XiaoZ0HLS18}, and Recolor attacks \citep{LaidlawF19} over its unregularized counterpart. \cref{fig:summary} visualizes improvements brought through ALR for a sequence of 2 attacks.  %We find that when applied on a sequence of $\ell_2$, StAdv \citep{XiaoZ0HLS18}, and Recolor attacks \citep{LaidlawF19} on CIFAR-10, regularized CRT achieves 5.48\% higher average robustness compared to its unregularized counterpart.

\noindent \textbf{Impact of ALR and Efficient Approximations in Training and Fine-Tuning} (\cref{sec:init_train_impact},\cref{sec:fine-tuning_impact}). We conduct ablations using over $100$ attack combinations (12 attack types, with 9 non-$\ell_p$) to study ALR’s role in different stages of CRT. We also explore random noise-based regularization as a more efficient alternative.  We find that while noise-based regularization improves generalization in initial training, ALR is essential for maintaining robust performance during fine-tuning and improves Union accuracy by up to 7.85\%.

\section{Setup: \probname}
\label{sec: setup}
In this section, we introduce the problem of continual adaptive robustness (CAR)~\cite{dai2024position}, which aims to achieve robustness against new attacks as they are sequentially discovered. We survey existing approaches to this problem, with additional related work included in \cref{appsec: add_rel_work}.  CAR is visualized in Figure \ref{fig:overview}.

\subsection{A Motivating Example}
Consider an entity that wants to deploy a robust ML system. The entity uses recent techniques (\emph{e.g.} adversarial training) to defend their model  against existing attack types (such as $\ell_p$ perturbations) and deploys their model at time $t=0$. At a later time $t_1$, a research group publishes a paper about a new attack type (\emph{e.g.} spatial perturbations \citep{XiaoZ0HLS18}) against which the entity's model is not robust.  Since the ML system has been deployed, the entity would want to \textit{quickly} modify the model to be robust against the new attack while maintaining robustness against previous attacks.  Having a quick update procedure would minimize the time that an attacker can exploit this vulnerability.  Quick adaptation to new attacks is the foundation of continual adaptive robustness (CAR), a problem setting first introduced in a position paper \citep{dai2024position}.  We propose and analyze the first dedicated defense for CAR in this paper.%, and any adversary can take advantage of the new attack type to break the defended model. Thus, the entity needs to develop a new defense.
%Consider an entity which has deployed an ML system involving an image classifier and wants their classifier to predict robustly even in the presence of an adversary. This entity follows the adversarial robustness literature and uses a technique such as adversarial training to secure their model against $\ell_p$ perturbations and deploys their model at time $t=0$.  At some later time $t_1$, a research group publishes a paper about a new attack type (for example, spatial perturbations \citep{XiaoZ0HLS18}) that the entity's image classifier is not robust against.  In this case, what should the entity do?

%The commonly used definition of adversarial robustness is unable to capture these goals since it does not consider the dimension of time in its formulation. \citet{dai2024position} \edit{is a position paper which} introduces a problem setting called continual adaptive robustness (CAR) which models a dynamic setting of robustness where new attack types are \emph{revealed} to the defender sequentially and the defender has the ability to update their model with information of the new attack.  We now provide the problem formulation for CAR and demonstrate how it models the example described in this section.  \edit{In later sections, we will introduce a novel defense framework for this problem which we call continual robust training (CRT) and theoretically and empirically demonstrate how regularized CRT can help improve performance in CAR.}%. Other possible models are discussed in \cref{sec: discussion}.

\subsection{Problem Formulation}
\noindent
\textbf{Notation:} $\mathcal{D} = X \times Y$ denotes a data distribution where $X$ and $Y$ are the support of inputs and labels, respectively.  $\mathcal{H}$ denotes the hypothesis class.  We use $C:X \to \tilde{X}$ to define an adversarial constraint where $\tilde{X}$ is the space of adversarial examples.  $\ell: Y \times Y \to \mathbb{R}$ denotes the loss function.

\noindent
\textbf{Attack sequences:} In CAR~\cite{dai2024position}, different test-time attacks are introduced sequentially (Figure \ref{fig:overview}). Each attack (perturbed input) can be formulated as the maximizer of the loss $P_C(x, y, h) = \argmax_{x' \in C(x)} \ell(h(x), y)$ (within the constraint $C$) and has a corresponding time $T(P_C)$ at which it is discovered by the defender.  We call the set of attacks known by the defender at a given time $t$ the \textit{knowledge set} at time $t$: $K(t) = \{P\ | T(P) \le t\}$. The expansion of $K$ over time models settings such as research groups or a security team discovering new attack types. 

\noindent
\textbf{Goals in CAR:}  A defender in CAR uses a defense algorithm $\mathcal{A}_{\text{CAR}}$ to deploy a model $h_t = \mathcal{A}_{\text{CAR}}(\mathcal{D},K(t),\mathcal{H})$ at each time step $t$. Performance at time $t$ is measured by Union robust loss 
% \anote{Say that this is the Union loss quantity we care about in general, good robustness on all known attack} a
across the knowledge set:
$\mathcal{L}(h, t) = \mathbb{E}_{(x,y)\sim \mathcal{D}} \max_{P \in K(t)} [\ell(P(x,y, h), y)]$.%Overall, there are 3 goals of the defender: (1) achieve good robustness on the set of known attacks, (2) achieve some robustness on previously unseen attacks, and (3) recover quickly from recently introduced attacks.  We now provide a more formal description of what it means for a defense to achieve CAR: \anote{The 3 goals are defined again below the definition, also are these and the definition taken from the position paper? If yes, then clarify}

\begin{definition}[Continual Adaptive Robustness \citep{dai2024position}] Given loss tolerances $\delta_{\text{known}}$ and $\delta_{\text{unknown}}$ with $0 < \delta_{\text{known}} < \delta_{\text{unknown}}$ and grace period $\Delta t$ for recovering from a new attack, a defense algorithm $\mathcal{A}_{\text{CAR}}$ achieves CAR if for all $t > 0$:
\begin{itemize}
    \item When $t - T(P) < \Delta t$ for any attack $P$ and $T(P) < t$, $h_t$ satisfies $\mathcal{L}(h_t, t) \le \delta_{\text{unknown}}$
    \item Otherwise, $\mathcal{L}(h_t, t) \le \delta_{\text{known}}$.
\end{itemize}
\end{definition}

These criteria capture 3 distinct goals for the defender:  (1) The model at time $t$ must \textit{achieve good robustness if no attacks have been introduced recently} (within $\Delta t$ time). This is due to the $\delta_{\text{known}}$ threshold on the robust loss in the second criterion;  (2) If a new attack has occurred within $\Delta t$ period before the current time $t$, the model at time $t$ must \textit{achieve some robustness against the new attack}.  This is modeled by the $\delta_{\text{unknown}}$ threshold in the first criterion.  Since $0 < \delta_{\text{known}} < \delta_{\text{unknown}}$, CAR tolerates a degradation in robustness between the 2 cases; (3) The defense is expected to \textit{recover robustness quickly after new attacks}.  This is modeled by the $\Delta t$ time window; $\Delta t$ time after the introduction of a new attack, the loss threshold changes from $\delta_{\text{unknown}}$ to $\delta_{\text{known}}$.

\subsection{Baseline approaches to CAR}
\noindent
\textbf{CAR through multiattack robustness (MAR). }Prior works for multiattack robustness often involve training with multiple attacks simultaneously \citep{TB19, MainiWK20}, which can be computationally expensive. A trivial (but expensive) defense algorithm for CAR is to use these training-based techniques and retrain a model from scratch on $K(t)$ every time it changes.  However, this would require us to tolerate larger values of $\Delta t$. % For CAR, inefficiency in updating a model with each new attack is harmful since the adversary can continue to exploit the new attack while the model is being updated.% \arjun{This paragraph should say that sMAR is the framework that all previous work has basically operated in}

\noindent
\textbf{CAR through unforeseen attack robustness (UAR). } Defenses for unforeseen attack robustness (UAR) aim to attacks outside of those used in the design of the defense. Another trivial defense algorithm for CAR is to use a UAR defense \citep{laidlaw2020perceptual,dai2022formulating} to get a model $h$ and use $h$ for all time steps.  This approach is efficient as no time is spent updating the model but would require much higher values of $\delta_{\text{known}}$ as these methods do not obtain high robustness across attacks \citep{dai2023multirobustbench}.%  Prior work  demonstrates that current techniques for UAR performs poorly when evaluated on a wide variety of attacks with best performing approaches achieving only 3\% when considering the worst-case attack of the set. 
% \arjun{Worried about the phrasing here, we might get asked why we don't compare to UAR based defenses}
%\arjun{There needs to be more motivation either here or in the next section about why we ask the theoretical question we do ask. }

\section{Theoretical Motivation and Methods}
\label{sec: theory_methods}
In this section, we introduce continual robust training (CRT) and provide theoretical results to demonstrate that adding a regularization term bounding adversarial logit distances can help balance performance across a set of adversaries.
\subsection{Continual Robust Training}
We now introduce continual robust training (CRT).  CRT consists of 2 parts, \textit{initial training} and \textit{iterative fine-tuning} (Figure \ref{fig:overview}).  The output of initial training is deployed at $t=0$ while fine-tuning is used as new attacks are introduced.

% \anote{Okay but why can CRT meet the defender's goals? The fine-tuning reduces $\Delta t$, and can satisfy the $\delta_{\text{known}}$ condition given enough time.}

At time $t=0$, the goal of the defender is to minimize the initial training objective: $\mathcal{L}(h,0) = \frac{1}{m}\sum_{i=1}^m\ell(h(P_{C_\text{init}}(x_i, y_i, h)), y_i)$
where $\{(x_i, y_i)\}_{i=1}^m$ is the training dataset and $P_{C_\text{init}}$ is the initial attack. 
Notably, using standard training in this stage yields a high $\delta_\text{unknown}$.
% The goal of the defender is to minimize $L_{\text{init}}$ in initial training.

At $t>0$, as new attacks are introduced, we use a fine-tuning strategy $F$ to select the attack from $K(t)$ to use for each example.  Specifically, we formulate this as:
$    \mathcal{L}(h, t) = \frac{1}{m}\sum_{i=1}^m \ell(h(P_{C}(x_i, y_i, h)), y_i)$ where $P_{C}= F(K(t), (x_i, y_i))$.
Fine-tuning strategies include picking the attack that maximizes $\ell(x_i, y_i)$, randomly sampling from $K(t)$, and using the newest attack.  A good fine-tuning strategy would be able to quickly adapt the model to new attacks, allowing it to satisfy a small $\Delta t$ threshold.
However, naive fine-tuning does not guarantee good performance across all attacks and may require large values of $\delta_\text{known}$. 
% \anote{Cast this in terms of the requirements: too much degradation will cause an unacceptable $\delta_{\text{known}}$} 
As illustrated in \cref{fig:summary}, a model may lose robustness to the initial attack after the fine-tuning stage. We now discuss how such degradation can be addressed through regularization.
% \sophie{point to \cref{fig:summary}, naively doing CRT doesn't guarantee good robust performance for CAR - poor generalization to unforeseen attacks, forgetting of initial attack. Then connect with next section.}

\subsection{Bounding the difference in adversarial losses}\label{subsec: theory}
% \anote{Motivate the terms we consider: we want L1-L2 to be small for two reasons: some unforeseen robustness, which helps meet the $\delta_{\text{unknown}}$ requirement; and prevent degradation on known attacks, which helps meet the $\delta_{\text{known}}$ requirement}

In order for CRT to be a practical solution for CAR, it is important that CRT enhances robustness to new attacks without losing robustness against attacks we have already learned. 
We now relate the gap in robust loss between two attacks to how far each attack can perturb the logits for a given example,
which suggests that regularization in CRT may improve decrease the drop in robustness across attacks. 
%Such a bound can be useful in designing new defenses, as it could correspond to an upper bound on the increase in adversarial loss on a previously learned attack after a new attack is introduced.
%Ideally, we would want to ensure not only that the gap between the losses is minimal, but also that the individual losses are small in absolute terms~\cite{yin2019rademacher}. While our results do not directly ensure low individual losses, they hold for any model that follows our assumptions, including those that perform well against both attacks. Moreover, they suggest that a model whose \edit{logits} are highly sensitive to perturbations will not perform well against both attacks.

% Let $h:\mathbb{R}^d \rightarrow \mathbb{R}^c$ be a $c$ class neural network classification model with a final linear layer (i.e. $h(x) = Wg(x)$, where $g: \mathbb{R}^d \rightarrow \mathbb{R}^r$ is a representation function and $W \in \mathbb{R}^{c \times r}$). 
\edit{Let $h:\mathbb{R}^d \rightarrow \mathbb{R}^k$ be a $k$ class neural network classification model.}
To simplify the problem, we focus on the state of the model when attacks $P_{C_1}$ and $P_{C_2}$ (with corresponding adversarial constraints $C_1$ and $C_2$) are known to the defender. %, 
%although the bounds we derive hold for any model of the above form.
Consider the following two adversarial loss functions:
$\mathcal{L}_1(h) := \mathbb{E}_{\mathcal{D}}\left[\ell(h(P_{C_1}(x,y)),y)\right]$
and
$\mathcal{L}_2(h) := \mathbb{E}_{\mathcal{D}}\left[\ell(h(P_{C_2}(x,y)),y)\right].$
Without loss of generality, assume that $\mathcal{L}_1(h)
\geq \mathcal{L}_2(h)$. We can then bound the difference between $\mathcal{L}_1(h)$ and $\mathcal{L}_2(h)$, \edit{adapting a result from}~\citet{nern2023transfer}, as follows\footnote{\edit{As stated, these results hold for loss functions that are Lipschitz with respect to the $\ell_2$ norm. We note that similar bounds can be derived for other norms by applying a constant scaling factor to the first term of the bound (i.e. for losses Lipschitz with respect to the $\ell_1$ norm, the scaling factor would be $\sqrt{c}$).}}:
\vspace{-3pt}
%We arrive at the following theorem, building on Nern \textit{et. al.}~\cite{nern2023transfer}, bounding the difference between $\mathcal{L}_1(h)$ and $\mathcal{L}_2(h)$:

%In order for CRT to be effective, it must be possible for a model to achieve acceptable loss on multiple attacks. Therefore, we demonstrate theoretically that the loss of a model on additional attacks is bounded by an objective that can be directly minimized. 

%For ease of notation, we will refer to $P_1$ as $\psi$ and $P_2$ as $\omega$.

% , which mirrors Theorem 1 from \cite{nern2023transfer} \sophie{do we need citation here?}:
\begin{theorem}
    \label{thm:robustness}
    Assume that loss $\ell(\hat{y},y)$ is $M_1$-Lipschitz in $\|\cdot\|_2$, for $\hat{y} \in h(X)$ with $M_1 > 0$ and bounded by $M_2 > 0$ \footnote{We note that surrogate losses such as the cross-entropy used during training are not bounded, but the $0-1$ loss which is often the key quantity of interest \emph{is bounded}.}, i.e. $0 \leq \ell(\hat{y},y) \leq M_2$ $\forall \hat{y} \in h(X)$. Then, for a subset $\mathbb{X} = \{x_i\}_{i=1}^n$ independently drawn from $\mathcal{D}$, the following holds with probability at least $1-\rho$:
    \begin{align*}
        \mathcal{L}_1(h) - \mathcal{L}_2(h) \leq \;&M_1 \frac{1}{n}\sum_{i=1}^n\biggl(\max_{x' \in C_1(x_i)}\|h(x') - h(x_i)\|_2 \\
        &+ \max_{x' \in C_2(x_i)}\|h(x') - h(x_i)\|_2\biggl) + D,
    \end{align*}
    where $D = M_2\sqrt{\frac{\log(\rho/2)}{-2n}}$. 
\end{theorem}
This result suggests that regularization with respect to a single attack (say, in pre-training) will give the model greater resiliency against unforeseen attacks and help meet the $\delta_\text{unknown}$ threshold. Using regularization when fine-tuning on a new attack could also prevent degradations in robustness against previously seen attacks, helping to meet the $\delta_\text{known}$ threshold.
Using similar reasoning, we can also bound the gap between Union and clean loss:
\begin{corollary}
\label{thm:corollary}
Let $\mathcal{L}_{1,2}(h) := \mathbb{E}_{\mathcal{D}}\left[\max{(\ell(h(P_{C_1}(x,y,h)),y),\ell(h(P_{C_2}(x,y,h)),y)})\right]$. Then, with probability at least $1 - \rho$,
\begin{align*}
        \mathcal{L}_{1,2}&(h) - \mathcal{L}(h) 
        \leq M_1 \frac{1}{n}\sum_{i=1}^n\biggl(\max_{x' \in C_1(x_i)}\|h(x') - h(x_i)\|_2 \\
        &+ \max_{x' \in C_2(x_i)}\|h(x') - h(x_i)\|_2\biggl) + D.
    \end{align*}
\end{corollary}
This corollary helps characterize the trade-off between clean and robust loss in our setting. Proofs of Theorem \ref{thm:robustness} and Corollary \ref{thm:corollary} are present in Appendix \ref{sec:proof}.

\noindent \textbf{Comparison to \citet{dai2022formulating}: } We note that \citet[Theorem 4.2]{dai2022formulating} derive a related bound on the adversarial loss gap between two attacks in the context of UAR. However, their formulation assumes that the constraint set of the target attack is a strict superset of that of the source attack, whereas we make no assumptions about the relationship between the two constraint sets.

\subsection{Regularization Methods}
\label{subsec: methods}
\label{sec:regularization}
Theorem \ref{thm:robustness} suggests that reducing the sensitivity of logits to \textit{either} attack has the potential to reduce the performance gap between attacks (see \cref{fig:loss_gap} in the Appendix for an empirical validation of this effect). To this end, we propose incorporating regularization into both training stages.  Specifically, we adopt modified training objective $\mathcal{L}_{\text{reg}}(h,t) = \mathcal{L}(h,t) + \lambda R(h, K(t))$, 
% andx\anote{I think this is now $\mathcal{L}(h,0)$, and respectively for the fine-tuning loss} $L_{\text{fine-tune-reg}}(h, t) = L_{\text{fine-tune}}(h) + \lambda R(h, K(t))$, 
where $\lambda$ is the regularization strength and $R(h)$ is the regularization term used. We will now discuss several forms of regularization.

%We adopt a modified training objective $L(h) = L_{\text{init}}(h) + \lambda R(h)$,
%where $\lambda$ is the regularization strength and $R(h)$ is the regularization term used. 
%\edit{\st{Following the notation of our theoretical results, we define $h$ to be a composition of a feature extractor and a final linear layer (i.e. $h(x) = Wg(x)$).}}
% Notably, we introduce regularization terms in both pretraining and fine-tuning, which is supported both by our theoretical results and prior work showing that regularization can improve model performance on unseen attacks \cite{dai2022formulating}. 
%As regularization with respect to \textit{either} attack has the potential to directly reduce the upper bound in \ref{thm:robustness}, we would expect the modified objective to be beneficial in both the pre-training stage (when only the initial attack is known) and the fine-tuning stage (when both attacks are known).

\noindent
\textbf{Adversarial $\ell_2$ regularization. (ALR)}
Driven by our theoretical results, we first introduce adversarial $\ell_2$ regularization: $R_{\text{ALR}}(h, K(t)) =  \frac{1}{m} \sum_{i=1}^m \max_{x' \in C(x_i)} \|h(x') - h(x_i)\|_2$ where $C = C_{\text{init}}$ in initial training and corresponds to attack $P_C = F(K(t), (x_i, y_i))$ chosen by the fine-tuning strategy.  $\ell_2$ regularization penalizes the maximum distance between a sample's \edit{logits} and the furthest adversarially perturbed \edit{logits} within that sample's neighborhood. Using this regularization term would directly minimize the upper bounds in Theorem~\ref{thm:robustness} and Corollary~\ref{thm:corollary}.
We note that while ALR is similar in form to  TRADES~\cite{zhang2019theoretically}, it uses a Euclidean distance instead of the KL-divergence. Our paper is the first to show that this form of regularization is beneficial for CAR.

\noindent
\textbf{Efficiently approximating ALR. } Computing ALR uses multi-step optimization which can be costly to compute in practice. To improve efficiency in experiments, we consider (1) using single step optimization for ALR and (2) using randomly sampled, unoptimized perturbations can help with CAR. For (2), we consider Gaussian noise regularization (GR) and Uniform noise regularization (UR), specifically:
$R_{\text{GR}}(h, K(t)) = \frac{1}{m} \sum_{i=1}^m \|\edit{h}(x') - \edit{h}(x_i)\|_2$ where $x' \sim \mathcal{N}(0, \sigma^2)$ and 
$R_{\text{UR}}(h, K(t)) = \frac{1}{m} \sum_{i=1}^m \|\edit{h}(x') - \edit{h}(x_i)\|_2$ where $x' \sim \mathcal{U}(-\sigma, \sigma)$.

\noindent
\textbf{Other Baselines.} We compare to variation regularization (VR), which has been shown to improve generalization to unforeseen attacks \citep{dai2022formulating}. VR is defined as: $R_{\text{VR}}(h, K(t)) = \frac{1}{m} \sum_{i=1}^m \max_{x', x'' \in C(x_i)} \|\edit{h}(x') - \edit{h}(x'')\|_2$ where $C = C_{\text{init}}$ in initial training.  We also consider VR in finetuning with $C$ corresponding to attack $P_C = F(K(t), (x_i, y_i))$.
The link between VR and ALR is discussed in \cref{app:alr_vr}.
% \sophie{add link to Appendix, mention that the terms are connected.}
%\edit{We note that since VR maximizes over 2 perturbations and $x_i \in C(x_i)$, we can lower bound VR with ALR:} $R_\text{ALR}(h) \leq R_\text{VR}(h)$.  \edit{Additionally, by adding and subtracting $h(x_i)$ from within the distance computation and applying triangle inequality, we have that }$R_\text{VR}(h) \le 2R_\text{ALR}(h)$.
%Since during the initial training phase, the CAR problem reduces to an unforeseen robustness problem, variation regularization may provide a good starting point in terms of robustness.
%Algorithmically, \citet{dai2022formulating} estimate the maximization within $R_{\text{VR}}(h)$ by optimizing both $x'$ and $x''$ simultaneously with PGD.

\section{Experimental Results}
\label{sec: results}
In this section, we empirically demonstrate that using regularization in CRT helps improve robustness when attacks are introduced sequentially.  This section is organized as follows: (i) experimental setup \S (\ref{sec:exp_setup}), (ii) overall results for using regularization in CRT (\S \ref{sec:car_reg}), (iii) ablations in initial training (\S \ref{sec:init_train_impact}) and (iv) ablations in fine-tuning (\S \ref{sec:fine-tuning_impact}).% of the CRT pipeline.  \anote{shorten by using inline bullets like i) experimental setup, ii) overall results using RCRT, ...}
%We begin by looking at the impact of regularization on different parts of the CRT pipeline; we discuss impact on initial training in \S\ref{sec:init_train_impact} and iterative fine-tuning in \S\ref{sec:fine-tuning_impact}.  We then analyze the performance of fine-tuning with existing approaches \citep{croce2022adversarial,TB19} within a CRT framework, and investigate how regularization within both stages of CRT performs on sequences of up to 4 attacks.
\vspace{-5pt}
\subsection{Experimental setup}
\label{sec:exp_setup}

\noindent
\textbf{Datasets. } We experiment with CIFAR-10, CIFAR-100 \citep{krizhevsky2009learning}, and ImageNette \citep{howardimagenette}, a 10-class subset of ImageNet \citep{deng2009imagenet}.%and 10-class and 100-class ImageNet subsets\citep{deng2009imagenet} (ImageNette \citep{howardimagenette} and ImageNet-100 respectively).

\noindent
\textbf{Architectures. } For CIFAR-10 and CIFAR-100, we use WideResnet-28-10 (WRN-28-10) architecture \citep{zagoruyko2016wide} and ResNet-18 for ImageNette. % For Imagenet subsets, we use ResNet architectures \citep{he2016deep}; we use ResNet-18 for ImageNette and Resnet-50 for ImageNet-100.

\noindent
\textbf{Attacks. }We include results for $\ell_2$, $\ell_\infty$, StAdv \citep{XiaoZ0HLS18}, ReColor attacks \citep{LaidlawF19}, and the 8 core attacks of Imagenet-UA \citep{kaufmann2019testing}. For $\ell_2$ attacks, we use a bound $\epsilon= 0.5$ for CIFAR datasets and $\epsilon=1$ for ImageNette.  For $\ell_2$ attacks, we use $\epsilon = \frac{8}{255}$, and for StAdv and ReColor attacks, we use the same bounds as used in their original papers \citet{XiaoZ0HLS18} ($\epsilon=0.05$) and \citet{LaidlawF19} ($\epsilon=0.06$) respectively. For ImageNet-UA attacks, we use the \emph{medium distortion} strength bounds used by \citet{kaufmann2019testing}.  For experiments investigating the impact of regularization in the fine-tuning step of CRT (\S\ref{sec:fine-tuning_impact}), we include results for fine-tuning to the same attack type but with larger attack bounds.  For these experiments, the larger bounds are given by $\epsilon = 1$ for $\ell_2$, $\epsilon=\frac{12}{255}$ for $\ell_\infty$, $\epsilon=0.07$ for StAdv, $\epsilon=0.08$ for ReColor, and high distortion strength bounds for ImageNet-UA attacks. 

\noindent\textbf{Training from scratch baselines. }We consider the following baselines for training from scratch:
\begin{itemize}
    %\item \textit{Adversarial Training (AT)}  \citep{madry2017towards}: when only a single attack is known, we consider adversarial training a baseline for comparison.  We also consider using this AT model as the initialization for different fine-tuning techniques within a CRT framework.
    \item \textit{Training with AVG and MAX objectives} \citep{TB19}: \citet{TB19} propose two different training objectives, AVG ($L_{\text{AVG}}(h, t) = \frac{1}{m|K(t)|} \sum_{i=1}^{m} \sum_{P_C \in K(t)} \ell(h(P_C(x_i, y_i)), y_i)$) and MAX ($L_{\text{MAX}}(h, t) = \frac{1}{m} \sum_{i=1}^{m} \max_{P_C \in K(t)} \ell(h(P_C(x_i, y_i)), y_i)$), for robustness against multiple known attacks.
    \item \textit{Randomly sampling attacks} \citep{madaan2020learning}: 
    AVG and MAX require generating adversarial examples with all attacks for each image. For a more efficient baseline, we consider randomly sampling an attack for each batch for use in adversarial training.
\end{itemize}

\noindent
\textbf{CRT Baselines. }For CRT, we use PGD adversarial training (AT) \citep{madry2017towards} for initial training and then fine-tune the model using several different fine-tuning strategies:
\begin{itemize}
    \item \textit{MAX objective fine-tuning} (FT-MAX) \citep{TB19}: We use the MAX objective for fine-tuning when a new attack is introduced.
    \item \textit{\citet{croce2022adversarial} fine-tuning} (FT Croce): \citet{croce2022adversarial} introduce a fine-tuning technique for use with $\ell_{\infty}$ and $\ell_1$ attacks which we \emph{generalize to training with arbitrary attacks}. This approach samples a single attack per batch. The probability that an attack $P_C$ is sampled is given by $\frac{\text{err}(P_C)}{\sum_{P \in K(t)} \text{err}(P)}$ where $\text{err}(P)$ denotes the running average of robust loss with respect to attack $P$ computed across batches of each attack.%sampling of different attack types during training such that only 1 attack is used per batch, making it more efficient than MAX fine-tuning for which all attacks are generated per batch.  The sampling procedure is based on computing a running average across batches of errors with respect to each attack.  Each attack $P_C$ is sampled with probability $\frac{\text{err}(P_C)}{\sum_{P \in K(t)} \text{err}(P)}$ where $\text{err}(P)$ denotes the running average of loss computed across batches of each attack.
    \item \textit{Single attack fine-tuning} (FT Single): We also consider fine-tuning with \emph{only the newly introduced attack}, allowing us to determine the extent to which previous attacks are forgotten. The previous two fine-tuning techniques involve replaying previous attacks.
\end{itemize}
We then investigate incorporating regularization into the initial training and fine-tuning phases of CRT.

\noindent
\textbf{Training and Fine-tuning Procedures. } During training, we use 10-step Projected Gradient Descent~\cite{madry2017towards} to generate adversarial examples. For the regularization terms (\cref{subsec: methods}), VR and ALR use single step optimization to reduce time overhead, while UR and GR use $\sigma=2$ and $\sigma=0.2$, respectively.  Results for additional values of $\sigma$ are in \cref{app:random_noise_var_abl}. We train models for 100 epochs for initial training and 10 epochs  for fine-tuning (results with 25 epochs in \cref{app:exp_seq}). We include additional details about the training procedure in \cref{app:exp_setup}.%We train models using adversarial training with and without regularization (\S\ref{sec:regularization}) with 10-step PGD adversarial training for $\ell_p$ attacks \citep{madry2017towards}.  We similarly use 10-steps to generate adversarial examples for adversarial training with non-$\ell_p$ attack types. For variation regularization and adversarial $\ell_2$ regularization, we optimize the objective using iterative optimization (in the same manner in which adversarial examples are generated for each attack type) with a single iteration to reduce the time overhead of training.  For uniform regularization and gaussian regularization, we use $\sigma=2$ and $\sigma=0.2$, respectively.  We include results for additional $\sigma$ in Appendix \ref{app:random_noise_var_abl}. When performing initial training, we train models for 100 epochs and when performing fine-tuning, we train models for 10 epochs (we also include results for other numbers of epochs in Appendix \ref{app:exp_seq}. We include additional details about the training procedure in Appendix \ref{app:exp_setup}.

%For the purposes of our experiments, we modified the regularization terms from Section~\ref{sec: theory_methods} to penalize distance in the \textit{logit} space (i.e. the output of $h$) rather than the \textit{feature} space (i.e. the output of $g$). We found that this change led to more stable training and better robust performance, likely owing to the high-dimensional nature of the internal representations. See Table~\ref{tab:main_results_cifar_epochs} in the Appendix for an ablation on the choice of layer used in regularization.

\noindent
\textbf{Evaluation Attacks and Metrics. } Our main results in Table \ref{tab:main_results_cifar} and additional ones in Appendix \ref{app:exp_seq} use full AutoAttack \citep{croce2020reliable} for evaluating $\ell_p$ robustness. For ablations, we restrict to APGD-T and FAB-T from AutoAttack to reduce evaluation time.  We use 20-step optimization when evaluating StAdv and ReColor attacks and the default evaluation hyperparameters for ImageNet-UA attacks in \citet{kaufmann2019testing}. We report \textit{accuracy on each attack}, \textit{Union accuracy} (overall accuracy when the worst case attack is chosen for each test example), \textit{Average accuracy} (average over accuracy on each attack), and \textit{training time} (in hours). Metrics are reported for the epoch $E^*$ with best performance on the set of known attacks. For training from scratch, the reported training time is scaled by fraction of training for the best epoch (\textit{i.e.} we report $\frac{E^*}{100} \times \text{training time for 100 epochs}$). For fine-tuning we report training time for the full 10 epochs. This allows us to see how much faster fine-tuning is to optimal early stopping when re-training from scratch.

%Beyond loss on each attack, we also report \textit{union robust accuracy} (\edit{overall accuracy if for each test example, we chose the worst case attack from the set of attacks}) and \textit{average robust accuracy} (average over individual robust accuracies). \edit{We note that union accuracy is equal to $1 - L_{MAX}$ and average accuracy is equal to $1 - L_{AVG}$ when $\ell$ is taken to be the 0-1 loss in Equations \ref{eq:avg} and \ref{eq:max}.}
%Since the time that is taken for fine-tuning is also important, we report training and fine-tuning times in hours.  

\begin{table*}[ht]
\centering
\scalebox{0.77}{
\begin{tabular}{|c|c|l|l|c|cccc|cc|cc|c|}
\hline
\multicolumn{1}{|c|}{\begin{tabular}[c]{@{}c@{}}Time\\ Step \end{tabular}} & &Procedure & Threat Models & \multicolumn{1}{c|}{Clean} & \multicolumn{1}{c}{$\ell_2$} & \multicolumn{1}{c}{StAdv} & \multicolumn{1}{c}{$\ell_\infty$} & \multicolumn{1}{c|}{Recolor} & \multicolumn{1}{c}{\begin{tabular}[c]{@{}c@{}}Avg\\ (known)\end{tabular}} & \multicolumn{1}{c|}{\begin{tabular}[c]{@{}c@{}}Union\\ (known)\end{tabular}} & \multicolumn{1}{c}{\begin{tabular}[c]{@{}c@{}}Avg\\ (all)\end{tabular}} & \multicolumn{1}{c|}{\begin{tabular}[c]{@{}c@{}}Union\\ (all)\end{tabular}} & \multicolumn{1}{c|}{\begin{tabular}[c]{@{}c@{}}Time\\ (hrs)\end{tabular}} \\ \hline
\multirow{ 2}{*}{0} & \parbox[t]{2mm}{\multirow{2}{*}{\rotatebox[origin=c]{90}{Init}}}& AT & $\ell_2$ & \textbf{91.17} & \cellcolor[HTML]{B7E1CD}69.7 & 2.08 & 28.41 & 44.94 & 69.7 & 69.7 & 36.28 & 1.24 & 8.68 \\ %16.69
& & AT + ALR ($\lambda=1$) & $\ell_2$ & 89.43 & \cellcolor[HTML]{B7E1CD}\textbf{69.84} & \textbf{48.23} & \textbf{34.00} & \textbf{65.46} & \textbf{69.84} & \textbf{69.84} & \textbf{54.38} & \textbf{31.27} & 17.17\\ %22.29
\hline 
%\multirow{ 8}{*}{1} & \parbox[t]{2mm}{\multirow{3}{*}{\rotatebox[origin=c]{90}{Scratch}}} & AVG & $\ell_2$, StAdv & \bf{87.74} & \cellcolor[HTML]{B7E1CD}62.17 & \cellcolor[HTML]{B7E1CD}50.92 & 17.17 & 45.47 & 56.55 & 47.55 & 43.93 & 15.92 & 23.72\\ %47.44
%& & MAX & $\ell_2$, StAdv & 86.18 & \cellcolor[HTML]{B7E1CD}58.65 & \cellcolor[HTML]{B7E1CD}57.21 & 11.21 & 43.07 & 57.93 & 51.72 & 42.54 & 11.03 &  23.69 \\ %47.37
%& & Random & $\ell_2$, StAdv & 84.91 & \cellcolor[HTML]{B7E1CD}57.77 & \cellcolor[HTML]{B7E1CD}59.74 & 14.05 & 44.88 & 58.76 & 52.15 & 44.11 & 13.68 & 10.92 \\ %9.36
%\cdashline{2-14}
\multirow{ 5}{*}{1} & \parbox[t]{2mm}{\multirow{5}{*}{\rotatebox[origin=c]{90}{Finetune}}}& FT MAX &  $\ell_2$, StAdv & 83.73 & \cellcolor[HTML]{B7E1CD}57.07 & \cellcolor[HTML]{B7E1CD}58.67 & 12.51 & 49.03 & 57.87 & 51.32 & 44.32 & 12.36 & 4.00 \\
 %& FT MAX (25 ep) & $\ell_2$, StAdv & 84.85 & \cellcolor[HTML]{B7E1CD}56.44 & \cellcolor[HTML]{B7E1CD}\textbf{\underline{61.34}} & 10.35 & 48.08 & 58.89 & 52.52 & 44.05 & 10.24 & 10 \\ %
 
 & & FT Single & $\ell_2$, StAdv &  80.89
& \cellcolor[HTML]{B7E1CD}45.45& \cellcolor[HTML]{B7E1CD}54.5
& 6.09 & 41.98 & 49.98 & 41.05 & 37.0 & 5.87 & 2.78 \\
 & & FT Croce & $\ell_2$, StAdv & 84.7 & \cellcolor[HTML]{B7E1CD}57.88 & \cellcolor[HTML]{B7E1CD}54.27 & 14.38 & 51.08 & 56.07 & 48.13 & 44.4 & 13.8 & 2.40\\

 & & FT Single + ALR & $\ell_2$, StAdv & \textbf{87.24} & \cellcolor[HTML]{B7E1CD}\textbf{62.22}& \cellcolor[HTML]{B7E1CD}61.5
& \textbf{21.4} & \textbf{70.87} & 61.86 & 55.04 & \textbf{54.0} & \textbf{21.14} & 4.24 \\
   & & FT Croce + ALR & $\ell_2$, StAdv & 86.03 & \cellcolor[HTML]{B7E1CD}59.18 &\cellcolor[HTML]{B7E1CD}\textbf{65.14} & 15.36  & 63.31 & \textbf{62.16}& \textbf{55.83} & 50.75 & 15.29 & 3.47 \\

\hline

%\multirow{ 8}{*}{2}&\parbox[t]{2mm}{\multirow{3}{*}{\rotatebox[origin=c]{90}{Scratch}}} & AVG & $\ell_2$, StAdv, $\ell_\infty$ & 85.98 & \cellcolor[HTML]{B7E1CD}67.60 & \cellcolor[HTML]{B7E1CD}45.81 & \cellcolor[HTML]{B7E1CD}\textbf{42.39} & 62.43 & 51.93 & 34.05 & 54.56 & 33.39 & 33.12 \\
%& & MAX & $\ell_2$, StAdv, $\ell_\infty$ & 84.54 & \cellcolor[HTML]{B7E1CD}54.87 & \cellcolor[HTML]{B7E1CD}52.33 & \cellcolor[HTML]{B7E1CD}38.23 & 55.90 & 48.48 & 35.25 & 50.33 & 34.08 & 45.84 \\
%& &Random & $\ell_2$, StAdv, $\ell_\infty$ & 84.63 & \cellcolor[HTML]{B7E1CD}67.46 & \cellcolor[HTML]{B7E1CD}47.35 & \cellcolor[HTML]{B7E1CD}42.12 & 63.61 & 52.31 & 35.46 & 55.13 & 34.79 & 11.03 \\ \cdashline{2-14}
 \multirow{ 5}{*}{2}&\parbox[t]{2mm}{\multirow{5}{*}{\rotatebox[origin=c]{90}{Finetune}}} &FT MAX & $\ell_2$, StAdv, $\ell_\infty$ & 83.16 &  \cellcolor[HTML]{B7E1CD}65.63 &  \cellcolor[HTML]{B7E1CD}56.68 &  \cellcolor[HTML]{B7E1CD}36.9 & 65.69 & 53.07 & 35.18 & 56.23 & 34.83 &	5.62 \\
  & &FT Single & $\ell_2$, StAdv, $\ell_{\infty}$ & 87.99 & \cellcolor[HTML]{B7E1CD}\textbf{70.53} & \cellcolor[HTML]{B7E1CD}11.17
 & \cellcolor[HTML]{B7E1CD}41.63 & 63.46 &41.11 & 7.95 & 46.7 & 7.74  & 1.57 \\
 & &FT Croce & $\ell_2$, StAdv, $\ell_{\infty}$ & 85.05 &  \cellcolor[HTML]{B7E1CD}67.3 &  \cellcolor[HTML]{B7E1CD}48.07 &  \cellcolor[HTML]{B7E1CD}33.38 & 62.52 & 49.58 & 28.96 & 52.82 & 28.63 & 2.27 \\

 & &FT Single + ALR & $\ell_2$, StAdv, $\ell_{\infty}$ & \textbf{88.74} & \cellcolor[HTML]{B7E1CD}69.15 & \cellcolor[HTML]{B7E1CD}47.33 & \cellcolor[HTML]{B7E1CD}\textbf{42.08} & 68.62 & 52.85 & \textbf{36.66} & 56.8 & \textbf{36.62} & 2.26 \\
 & &FT Croce + ALR & $\ell_2$, StAdv, $\ell_{\infty}$ & 86.57 & \cellcolor[HTML]{B7E1CD}67.99 & \cellcolor[HTML]{B7E1CD}\textbf{61.55} & \cellcolor[HTML]{B7E1CD}36.59 & \textbf{72.16} & \textbf{55.38} & 35.68 & \textbf{59.57} & 35.52 & 2.87 \\

 \hline

%\multirow{ 8}{*}{3} &\parbox[t]{2mm}{\multirow{3}{*}{\rotatebox[origin=c]{90}{Scratch}}}& AVG & $\ell_2$, StAdv, $\ell_\infty$, Recolor & 87.77 & \cellcolor[HTML]{B7E1CD}\textbf{68.55} & \cellcolor[HTML]{B7E1CD}39.55 & \cellcolor[HTML]{B7E1CD}\textbf{41.97} & \cellcolor[HTML]{B7E1CD}67.93 & 54.5 & 30.39 & 54.5 & 30.39 & 51.55 \\
%& &MAX & $\ell_2$, StAdv, $\ell_\infty$, Recolor & 84.3 & \cellcolor[HTML]{B7E1CD}57.62 & \cellcolor[HTML]{B7E1CD}52.3 & \cellcolor[HTML]{B7E1CD}41.69 & \cellcolor[HTML]{B7E1CD}65.1 & 54.18 & \textbf{37.44} & 54.18 & \textbf{37.44} & 61.09 \\
%& &Random & $\ell_2$, StAdv, $\ell_\infty$, Recolor & 86.32 & \cellcolor[HTML]{B7E1CD}65.87 & \cellcolor[HTML]{B7E1CD}47.82 & \cellcolor[HTML]{B7E1CD}35.04 & \cellcolor[HTML]{B7E1CD}68.35 & 54.27 & 30.76 & 54.27 & 30.76 & 13.15 \\ \cdashline{2-14}
\multirow{ 5}{*}{3}&\parbox[t]{2mm}{\multirow{5}{*}{\rotatebox[origin=c]{90}{Finetune}}}&FT MAX & $\ell_2$, StAdv, $\ell_{\infty}$, Recolor & 83.64 & \cellcolor[HTML]{B7E1CD}66.21 & \cellcolor[HTML]{B7E1CD}57.53 & \cellcolor[HTML]{B7E1CD}\textbf{37.77} & \cellcolor[HTML]{B7E1CD}69.32 & 57.71 & \textbf{36.02} & 57.71 & \textbf{36.02} & 8.45 \\
   & &FT Single & $\ell_2$, StAdv, $\ell_{\infty}$, Recolor &  90.41& \cellcolor[HTML]{B7E1CD}66.47 & \cellcolor[HTML]{B7E1CD}3.93 & \cellcolor[HTML]{B7E1CD}29.6& \cellcolor[HTML]{B7E1CD}69.03 & 42.26 & 2.49 & 42.26 & 2.49 & 3.11 \\
& &FT Croce & $\ell_2$, StAdv, $\ell_{\infty}$, Recolor & 86.64 & \cellcolor[HTML]{B7E1CD}\textbf{68.76} & \cellcolor[HTML]{B7E1CD}44.81 & \cellcolor[HTML]{B7E1CD}36.02 & \cellcolor[HTML]{B7E1CD}68.05 & 54.41 & 29.44 & 54.41 & 29.44 & 2.34 \\
 & &FT Single + ALR & $\ell_2$, StAdv, $\ell_{\infty}$, Recolor & \textbf{90.45} & \cellcolor[HTML]{B7E1CD}61.58 & \cellcolor[HTML]{B7E1CD}25.77 & \cellcolor[HTML]{B7E1CD}27.43 & \cellcolor[HTML]{B7E1CD}69.26 & 46.01 & 19.2 & 46.01 & 19.2 & 4.24 \\
  & &FT Croce + ALR & $\ell_2$, StAdv, $\ell_{\infty}$, Recolor &87.62 & \cellcolor[HTML]{B7E1CD}68.14 & \cellcolor[HTML]{B7E1CD}\textbf{58.5} & \cellcolor[HTML]{B7E1CD}36.39 & \cellcolor[HTML]{B7E1CD}\textbf{72.35} & \textbf{58.85} & 34.92 & \textbf{58.85} & 34.92 & 3.35 \\
\hline
\end{tabular}}
\vspace{-5pt}
\caption{\textbf{Continual Robust Training on CIFAR-10.} Best performance for each time step are \textbf{bolded}. The defender initially knows about $\ell_2$ attacks and over time, is sequentially introduced to StAdv, $\ell_\infty$, and ReColor attacks. We report clean accuracy, accuracy on individual attacks, and average and union accuracies.  The ``Threat Models" column specifies known attacks at the current time step, and accuracies on these attacks are in {\color[HTML]{B7E1CD} green cells}. Initial adversarial training occurs at time step 0, and the model is updated through fine-tuning the model from the previous time step.  ``Avg (known)" and ``Union (known)" columns represent average and union accuracies on known attacks while ``Avg (all)" and ``Union (all)" columns report performance across all four attacks.  We report training time for each time step in the ``Time" column.}
\label{tab:main_results_cifar}
\vspace{-10pt}
\end{table*}

\begin{table}[ht]
    \centering
    \scalebox{0.85}{
    \begin{tabular}{|l|c|cc|c|}
    \hline 
    Procedure & Clean & Avg & Union & Time \\ \hline
       MAX & 84.3 &  54.18 & 37.44 &61.09 \\
       AVG  & 87.77 &  54.5 & 30.39 & 51.55\\
       Random & 86.32 & 54.27 & 30.76 & 13.15\\\hline
       CRT + ALR & 87.62 & 58.85 & 34.92 & 26.86\\ \hline
    \end{tabular}}
    \caption{Regularized CRT (using \citet{croce2020reliable} fine-tuning strategy) compared to training from scratch on $\ell_2$, StAdv, $\ell_{\infty}$, and Recolor attacks on CIFAR-10.}
    \label{tab:training_from_scratch}
    \vspace{-20pt}
\end{table}

\begin{table*}[]
\centering
{\renewcommand{\arraystretch}{1.2}
\scalebox{0.7}{
\begin{tabular}{|l|c|c|c|cccccccccccc|cc|}
\hline
\begin{tabular}[c]{@{}l@{}}Initial\\Attack\end{tabular} & \begin{tabular}[c]{@{}c@{}}Reg\\Type\end{tabular} & $\lambda$ & Clean & $\ell_2$ & $\ell_\infty$ & StAdv & ReColor & Gabor & Snow & Pixel & JPEG & Elastic & Wood & Glitch & \begin{tabular}[c]{@{}c@{}}Kaleid-\\oscope\end{tabular} & Avg & Union \\ \hline
$\ell_2$ & None & 0 & 91.08 & 70.02 & 29.38 & 0.79 & 33.69 & 66.93 & 24.59 & 14.99 & 64.22 & 45.13 & 70.85 & 80.30 & 30.08 & 44.25 & 0.10 \\
$\ell_2$ & VR & 0.2 & \cellcolor[HTML]{F4C7C3}89.99 & 70.38 & \cellcolor[HTML]{B7E1CD}34.56 & \cellcolor[HTML]{B7E1CD}13.41 & \cellcolor[HTML]{B7E1CD}48.99 & 67.64 & \cellcolor[HTML]{B7E1CD}29.09 & \cellcolor[HTML]{B7E1CD}22.57 & \cellcolor[HTML]{B7E1CD}66.64 & \cellcolor[HTML]{B7E1CD}48.38 & \cellcolor[HTML]{B7E1CD}73.31 & 80.07 & \cellcolor[HTML]{B7E1CD}32.33 & \cellcolor[HTML]{B7E1CD}48.94 & \cellcolor[HTML]{B7E1CD}5.40 \\ 
$\ell_2$ & ALR & 0.5 & \cellcolor[HTML]{F4CCCC}89.57 & 70.29 & \cellcolor[HTML]{B7E1CD}34.16 & \cellcolor[HTML]{B7E1CD}17.44 & \cellcolor[HTML]{B7E1CD}51.04 & \cellcolor[HTML]{F4CCCC}65.63 & \cellcolor[HTML]{B7E1CD}28.71 & \cellcolor[HTML]{B7E1CD}22.50 & \cellcolor[HTML]{B7E1CD}66.76 & \cellcolor[HTML]{B7E1CD}48.80 & \cellcolor[HTML]{B7E1CD}73.24 & 79.66 & \cellcolor[HTML]{F4CCCC}28.83 & \cellcolor[HTML]{B7E1CD}48.92 & \cellcolor[HTML]{B7E1CD}5.94 \\ 
$\ell_2$ & UR & 5  & \cellcolor[HTML]{F4CCCC}88.34 & \cellcolor[HTML]{F4CCCC}66.66   & \cellcolor[HTML]{F4CCCC}27.41   & \cellcolor[HTML]{B7E1CD}26.22 & \cellcolor[HTML]{B7E1CD}60.22 & \cellcolor[HTML]{B7E1CD}69.16   & \cellcolor[HTML]{B7E1CD}26.67 & \cellcolor[HTML]{B7E1CD}22.57 & 64.08  & \cellcolor[HTML]{B7E1CD}46.83   & 71.14  & \cellcolor[HTML]{F4CCCC}77.60    & \cellcolor[HTML]{B7E1CD}31.36    & \cellcolor[HTML]{B7E1CD}49.16 & \cellcolor[HTML]{B7E1CD}6.23 \\
$\ell_2$ & GR & 0.5& \cellcolor[HTML]{F4CCCC}86.89 & \cellcolor[HTML]{F4CCCC}68.19   & \cellcolor[HTML]{B7E1CD}32.02 & \cellcolor[HTML]{B7E1CD}16.54 & \cellcolor[HTML]{B7E1CD}58.32 & \cellcolor[HTML]{B7E1CD}74.85 & \cellcolor[HTML]{B7E1CD}25.69 & \cellcolor[HTML]{B7E1CD}21.26 & \cellcolor[HTML]{B7E1CD}65.32   & \cellcolor[HTML]{B7E1CD}46.82 & \cellcolor[HTML]{B7E1CD}74.08   & \cellcolor[HTML]{F4CCCC}76.99   & \cellcolor[HTML]{B7E1CD}31.93    & \cellcolor[HTML]{B7E1CD}49.33 & \cellcolor[HTML]{B7E1CD}4.18 \\
\hline
$\ell_\infty$ & None & 0 & 85.53 & 59.36 & 50.98 & 6.34 & 56.27 & 68.94 & 36.79 & 20.57 & 54.02 & 51.00 & 64.24 & 75.94 & 39.44 & 48.66 & 1.31 \\
$\ell_\infty$ & VR & 0.2 & \cellcolor[HTML]{F4C7C3}82.58 & 58.36 & 51.53 & \cellcolor[HTML]{B7E1CD}18.98 & \cellcolor[HTML]{B7E1CD}62.12 & \cellcolor[HTML]{F4C7C3}67.18 & \cellcolor[HTML]{B7E1CD}39.22 & \cellcolor[HTML]{B7E1CD}23.62 & 54.73 & \cellcolor[HTML]{B7E1CD}52 & 63.35 & \cellcolor[HTML]{F4C7C3}71.72 & \cellcolor[HTML]{B7E1CD}43.18 & \cellcolor[HTML]{B7E1CD}50.50 & \cellcolor[HTML]{B7E1CD}5.08 \\ 
$\ell_\infty$ & ALR &  0.5 & \cellcolor[HTML]{F4C7C3}83.18 & \cellcolor[HTML]{F4C7C3}58.21 & 51.47 & \cellcolor[HTML]{B7E1CD}19.50 & \cellcolor[HTML]{B7E1CD}61.02 & 68.75 & \cellcolor[HTML]{B7E1CD}37.94 & \cellcolor[HTML]{B7E1CD}22.78 & 53.89 & \cellcolor[HTML]{F4C7C3}49.82 & 63.47 & \cellcolor[HTML]{F4C7C3}73.57 & 39.88 & \cellcolor[HTML]{B7E1CD}50.02 & \cellcolor[HTML]{B7E1CD}5.52 \\ 
 $\ell_\infty$    & UR  & 5  & \cellcolor[HTML]{F4C7C3}78.04 & 60.28  & \cellcolor[HTML]{F4C7C3}40.59   & \cellcolor[HTML]{B7E1CD}42.25 & \cellcolor[HTML]{B7E1CD}70.00    & \cellcolor[HTML]{F4C7C3}67.06   & \cellcolor[HTML]{F4C7C3}33.40    & \cellcolor[HTML]{B7E1CD}26.57 & \cellcolor[HTML]{B7E1CD}60.07   & \cellcolor[HTML]{F4C7C3}49.21   & 64.61  & \cellcolor[HTML]{F4C7C3}67.08   & \cellcolor[HTML]{F4C7C3}38.43      & \cellcolor[HTML]{B7E1CD}51.63 & \cellcolor[HTML]{B7E1CD}8.36 \\ 
$\ell_\infty$      & GR & 0.5& \cellcolor[HTML]{F4C7C3}80.65 & 59.74  & \cellcolor[HTML]{F4C7C3}46.12   & \cellcolor[HTML]{B7E1CD}34.57 & \cellcolor[HTML]{B7E1CD}70.49 & 68.33  & 35.80   & \cellcolor[HTML]{B7E1CD}26.04 & \cellcolor[HTML]{B7E1CD}57.28   & 51.98  & \cellcolor[HTML]{B7E1CD}65.46   & \cellcolor[HTML]{F4C7C3}70.73   & \cellcolor[HTML]{F4C7C3}38.21      & \cellcolor[HTML]{B7E1CD}52.06 & \cellcolor[HTML]{B7E1CD}6.28 \\ \hline
\end{tabular}
}}
\vspace{-5pt}
\caption{\textbf{Impact of Regularization on Unforeseen Robustness.} We consider the setting where the defender is only aware of a single attack and performs training with and without different types of regularization: variation regularization (VR), adversarial $\ell_2$ regularization (ALR), uniform regularization (UR), and Gaussian regularization (GR) at regularization strength $\lambda$.  We report clean accuracy and robust accuracies on a range of attacks. {\color[HTML]{B7E1CD} Green cells} represent an improvement of at least 1\%  while {\color[HTML]{F4C7C3} red cells} represent a drop of at least 1\% in comparison to no regularization.}
\label{tab:reg_unforeseen_rob}
\vspace{-10pt}
\end{table*}

\subsection{Improving CRT with Regularization}
\label{sec:car_reg}
We now analyze the robustness of models trained using CRT with and without regularization. For simplicity, we focus on ALR with other methods analyzed in \cref{sec:init_train_impact}.  To model a CAR setting, we consider a sequence of 4 attacks: $\ell_2 \to$ StAdv $\to \ell_{\infty} \to$ Recolor.  The first attack is the initially known attack while other attacks are introduced at later time steps. %The defender learns about these attacks at different time steps, with the first attack of the sequence known prior to deployment while subsequent attacks are sequentially discovered post-deployment. 
We present results for CIFAR-10 in Table \ref{tab:main_results_cifar}.  We include results in Appendix \ref{app:exp_seq} for Imagenette and CIFAR-100 as well as additional results for longer duration of fine-tuning (25 epochs) and a separate sequence of attacks: $\ell_\infty \to$ StAdv $\to$ Recolor $\to \ell_2$.  For these experiments, we use $\lambda=0.5$ unless specified otherwise.

\noindent\textbf{Regularization reduces degradation on previous attacks. } From Table \ref{tab:main_results_cifar}, we observe that fine-tuning with only the new attack (FT Single) can lead to degradation of robustness against previous attacks.  The incorporation of ALR significantly decreases this drop in robustness.  For example, when fine-tuning from an $\ell_2$ robust model with StAdv attacks (time step 1 in Table \ref{tab:main_results_cifar}), FT Single incurs a 24.25\% drop (from 69.7\% to 45.45\%) in $\ell_2$ accuracy from the initial checkpoint (AT at time step 0).  Meanwhile FT Single + ALR only experiences a 7.62\% drop (from 69.84\% to 62.22\%) in $\ell_2$ accuracy from the initial checkpoint (AT + ALR at time step 0).  Similarly, after the introduction of $\ell_\infty$ attack at time step 2, the accuracy of FT Single on StAdv attacks drops 43.42\% (from 54.5\% to 11.17\%) while FT Single + ALR only experiences a 14.17\% drop (from 61.5\% to 47.33\%). These results align with Theorem~\ref{thm:robustness}: when incorporating ALR into training, the gap in loss on the two attacks is lessened.
%For example, we find that FT MAX and our approach is able to consistently outperform training from scratch with AVG, MAX, and Random procedures in terms of average and union accuracy.  This suggests that representations learned on the attack used in initial training ($\ell_2$ attack in Table \ref{tab:main_results_cifar}) can be a useful starting point for robustness on other attacks.  This also suggests that existing algorithms for achieving simultaneous multi-robustness may be suboptimal since we would expect the performance of these methods to serve as an upper bound for CAR.

\begin{figure*}[t!]
    \centering
    \begin{subfigure}[t]{0.45\textwidth}
        \centering
        \includegraphics[width=0.95\textwidth]{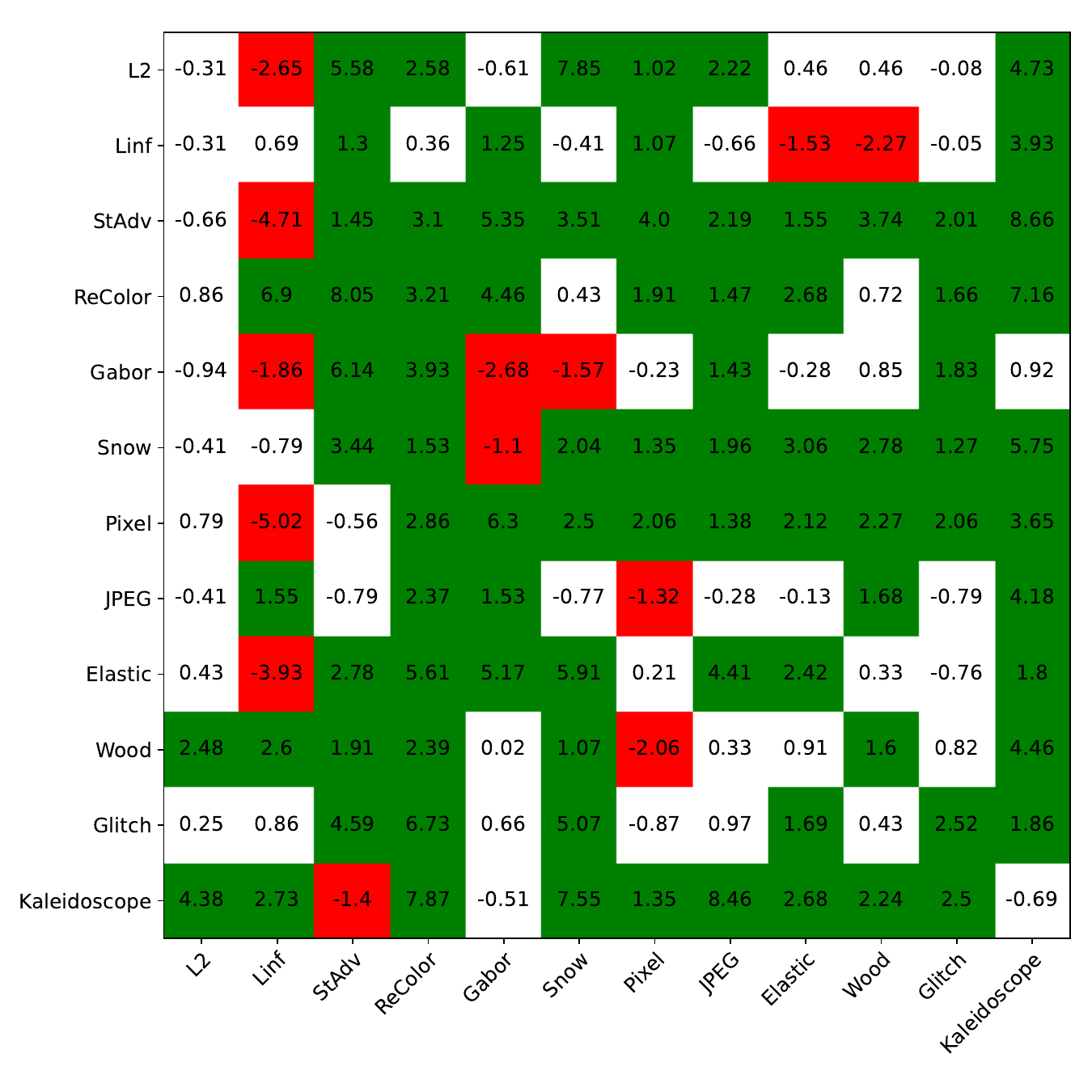}
        \vspace{-15pt}
        \caption{Adversarial $\ell_2$ regularization ($\lambda=0.5$)}
    \end{subfigure}%
    \begin{subfigure}[t]{0.45\textwidth}
        \centering
        \includegraphics[width=0.95\textwidth]{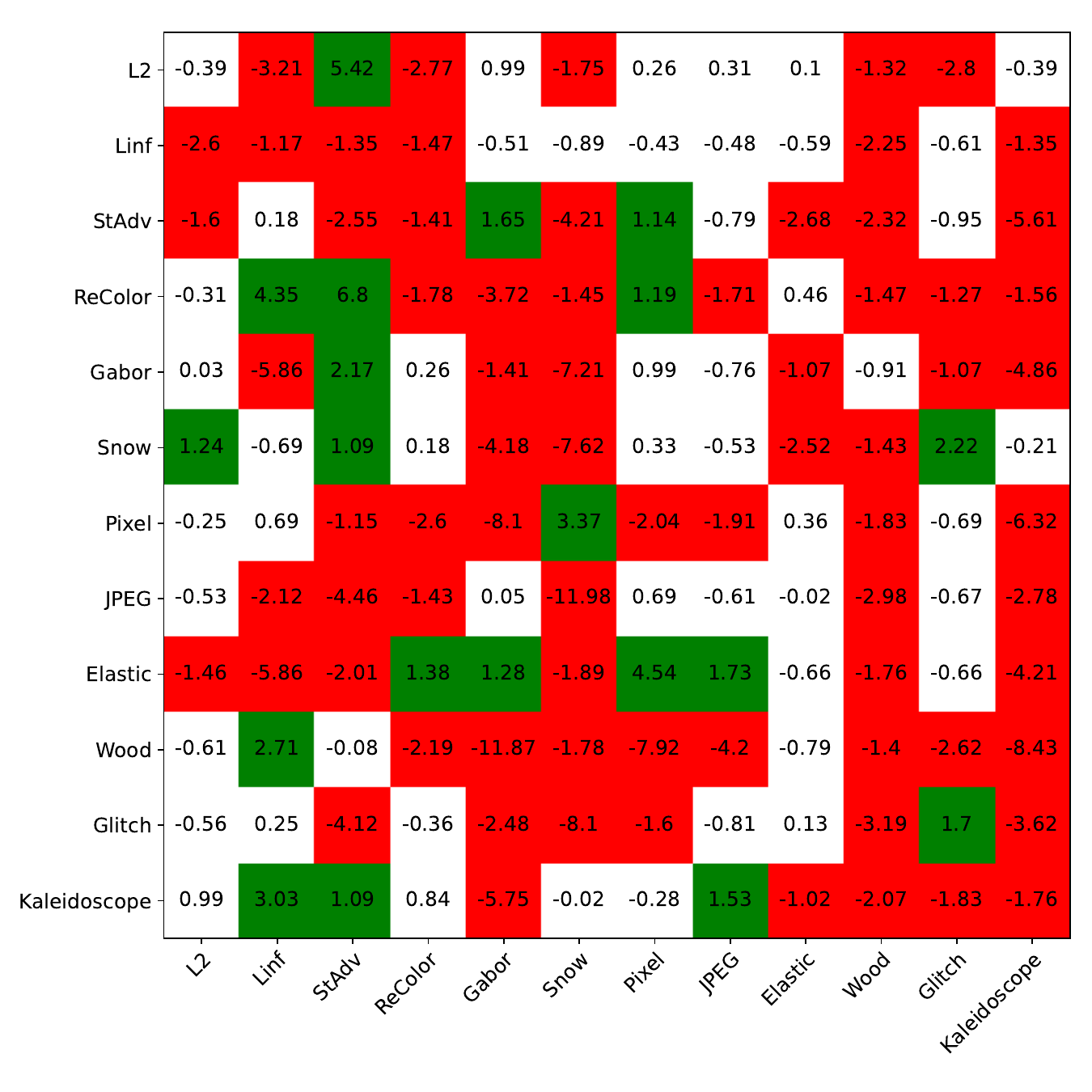}
        \vspace{-15pt}
        \caption{Uniform Regularization
        ($\sigma=2,\lambda=1$)}
    \end{subfigure}
    %\begin{subfigure}[t]{0.5\textwidth}
    %    \centering
    %    \includegraphics[width=0.95\textwidth]{figures/imagenette_finetune_varreg_union.pdf}
    %    \vspace{-10pt}
    %    \caption{Variation Regularization\\ ($\lambda = 0.5$)}
    %\end{subfigure}%
    %\begin{subfigure}[t]{0.5\textwidth}
    %    \centering
    %    \includegraphics[width=0.95\textwidth]{figures/imagenette_finetune_gaussian_union.pdf}
     %   \vspace{-10pt}
     %   \caption{Gaussian Regularization\\ ($\sigma=0.2, \lambda = 0.5$)}
    %\end{subfigure}
    \vspace{-10pt}
    \caption{\textbf{Ablation 2: Change in union robust accuracy after fine-tuning with regularization (initial model does not use regularization).}  We fine-tune models on Imagenette across 144 pairs of initial attack and new attack.  The initial attack corresponds to the row of each grid and new attack corresponds to each column.  Values represent differences between the accuracy measured on a model \emph{fine-tuned with and without regularization}.  Gains in accuracy of at least 1\% are highlighted in {\color[HTML]{0E7003} green}, while drops in accuracy of at least 1\% in {\color[HTML]{FC0006} red}. Further results are in Appendix~\ref{app:fine-tuning}.}
    \label{fig:fine-tune_abl_main_text}
    \vspace{-15pt}
\end{figure*}

\noindent\textbf{Regularization improves performance on held out (unforeseen) attacks. }We observe that regularized CRT leads to higher robustness on attacks held out from training.  For example, at time step 1 in Table \ref{tab:main_results_cifar}, which trains with $\ell_2$ and StAdv attacks, the best accuracy on Recolor attacks out of unregularized CRT methods is 51.08\%, while FT Single + ALR achieves 70.87\% accuracy on Recolor attacks and FT Croce + ALR achieves 63.31\% accuracy on Recolor attacks.  The improvement in robustness on unforeseen attacks aligns with \cref{thm:corollary} as regularization helps decrease the drop in accuracy between clean inputs and perturbed inputs. This also aligns with CAR's goal of having a small $\delta_{\text{unknown}}$.%\sophie{does this make sense?}

%This suggests that regularized CRT can also lead to improvements in unforeseen robustness. \anote{last sentence does not add any information, what is the intuition? The reg. brings the logits close to benign ones, so helps against unseen attacks as well...}% as well as robustness on the set of known attacks.% This gain also aligns with the goal of also having some unforeseen robustness for when a new attack is just introduced and time is taken to fine-tune the model before it can be redeployed for use.

\noindent\textbf{Regularization balances performance and efficiency. } Our proposed regularization term adds a small computational overhead over other FT approaches but generally improves union performance on the set of known attacks. For example, when considering the sequence of $\ell_2$ and StAdv attacks (time step 1 in Table \ref{tab:main_results_cifar}), FT Croce + ALR improves union accuracy over FT Croce by 7.7\% while adding a time overhead of 1.07 hours.  Additionally, when considering the sequence of 3 attacks ($\ell_2$, StAdv, and $\ell_\infty$ attacks), FT Croce + ALR improves union accuracy over FT Croce by 6.72\% while adding a time overhead of 0.6 hours.  This increase in time complexity is much smaller than FT MAX which takes 1.6 hours longer than FT Croce for $\ell_2$ and StAdv and 3.35 hours longer for $\ell_2$, StAdv, and $\ell_\infty$. With respect to goals in CAR, regularization balances $\delta_{\text{known}}$ and $\Delta t$.

\noindent\textbf{Comparison to training from scratch.} In Table \ref{tab:training_from_scratch}, we report clean, average, and union accuracies along with total training times for using training from scratch on all 4 attacks compared to training sequentially with regularized CRT on CIFAR-10.  We observe that regularized CRT is significantly more efficient than MAX and AVG training (taking a total of 26.86 hours while AVG and MAX take over 50 hours of training time).  Surprisingly, we find that on CIFAR-10, regularized CRT can outperform training from scratch methods, achieving 4.35\% higher average accuracy compared to the best achieved by training from scratch.  This suggests that transferable robustness between carefully chosen attacks can improve MAR as a whole.  However, we note that the ability to outperform training from scratch seems to be specific to CIFAR-10; for ImageNette and CIFAR-100 (Appendix \ref{app:exp_seq}) training from scratch outperforms using fine-tuning in CAR.

\noindent\textbf{Impact of dataset and attack sequence. }In Appendix \ref{app:exp_seq}, we provide results on ImageNette and CIFAR-100 as well as for attack sequence $\ell_\infty \to$ StAdv $\to$ Recolor $\to \ell_2$.  Overall, we observe that trends such as improved robustness to unforeseen and the union of attacks are generally consistent. However, but the extent to which regularization improves performance over FT Croce varies.  The choice of the initial attack seems to play a role in subsequent robustness, and if defenders are aware of multiple attacks, choosing the right one to start with is an interesting open question.  %We discuss this direction more in Section \ref{sec: discussion}.} 
%Overall, we find that the degree to which ALR helps in CRT to be dependent on dataset and attack sequence.  For ImageNette and sequence starting with $\ell_2$, we find that FT Croce + ALR can lead to an increase of at least 2.58\% in union (known attack) accuracy over FT Croce across time steps in the sequence.  However, for ImageNette with $\ell_{\infty}$ sequence, we find that the gain in performance of FT Croce + ALR over FT Croce can be quite small after the first time step; At time step 1, there is a gain of 3.19\%, 0.81\% at time step 2, and 1.4\% at time step 3. Meanwhile for CIFAR-100 with the $\ell_2$ attack sequence, we find that the gain in performance of FT Croce + ALR over FT Croce to be 2.6\% at time step 1 and 2.13\% at time step 2, but there is a 2.46\% drop in accuracy over FT Croce at time step 3.  Additionally for the $\ell_\infty$ attack sequence on CIFAR-100, improvements in through ALR can also be quite small, with a 1.21\% increase at time step 1, 1.26\% at time step 2, and 0.13\% at time step 3.  

%\begin{froval}
\begin{tcolorbox}[myboxstyle]
\begin{cfinding}
CRT+ALR improves robustness on both known and unforeseen attacks, and reduces drop in robustness on previous attacks with only a small overhead in fine-tuning time compared to unregularized CRT.\end{cfinding}
\end{tcolorbox}
%\end{froval}

\subsection{Ablation 1: Regularization in Initial Training}
\label{sec:init_train_impact}
We now study the impact of regularization \textit{only} in the initial training phase of CRT.  In Table \ref{tab:reg_unforeseen_rob}, we present results for robust accuracies of models initially trained on $\ell_2$ and $\ell_{\infty}$ attacks with different forms of regularization.  We present results for different regularization strengths and initial attack choices in \cref{app:init_train_different_attack_types}.

\noindent\textbf{Regularization improves robustness on unforeseen attacks.} Interestingly, we find that all regularization types including random noise-based regularization can improve unforeseen robustness.  For example, at $\lambda=5$, UR improves union accuracy across all attacks by 6.13\% for $\ell_2$ initial attack and by 7.05\% for $\ell_\infty$ initial attack compared to the model trained without regularization.  Improved unforeseen robustness provides a better starting point for fine-tuning, which we demonstrate experimentally in Appendix \ref{app:fine-tuning_pairs_init_train}.

\noindent\textbf{Trade-offs for clean and different attack accuracies. }We observe that all regularization types generally exhibit a trade-off with clean accuracy and trade-offs with a few attack types such as Glitch.  This trade-off aligns with \cref{thm:corollary} which states that the gap between clean loss and loss over the union of attacks is decreased via regularization. We also find that random noise based regularization (UR and GR) generally exhibits trade-off with the robust accuracy on the initial attack.  This is generally not the case for adversarial regularization (ALR and VR) which maintains performance on the initial attack. 

\noindent\textbf{Regularized initial models are better starting points for fine-tuning. } In Appendix \ref{app:fine-tuning_pairs_init_train}, we present results for fine-tuning with a new attack from models using regularization in only initial training.  We observe that for all regularization types, regularization in initial training can improve the robustness on the union of attacks after fine-tuning, but this trend is more consistent with adversarial regularization types (ALR and VR) compared to random regularization types (UR and GR).

\begin{tcolorbox}[myboxstyle]
    \begin{cfinding}\label{cfind: justification} Adversarial and random noise regularization in initial training improves performance on unforeseen attacks.  Fine-tuning on a new attack from a regularized model boosts resulting Union accuracy.
    \end{cfinding}
\end{tcolorbox}

\subsection{Ablation 2: Regularization during Fine-tuning}
\label{sec:fine-tuning_impact}

We now investigate whether regularization within just the the fine-tuning phase can improve CAR.  We initially train models on a single initial attack using adversarial training (\emph{without regularization}) and then fine-tune with \citet{croce2022adversarial}'s fine-tuning approach both with and without regularization on a new attack.  In Figure \ref{fig:fine-tune_abl_main_text}, we present grids representing differences in Union accuracy between regularized and unregularized fine-tuning.  Rows represent the initial attack used to adversarially train the model (without regularization), columns represent the new attack.  We provide corresponding plots detailing differences in average accuracy, initial attack accuracy, new attack accuracy, and clean accuracy in Appendix \ref{app:fine-tuning_pairs}.  

\noindent\textbf{Adversarial regularization can improve union accuracy in fine-tuning. } We find that across different initial and new attack pairs, using ALR in fine-tuning generally improves union accuracy as most cells in Figure \ref{fig:fine-tune_abl_main_text}(a) are green.  These increases in robustness can be quite large; for example, when the initial attack is StAdv \citep{XiaoZ0HLS18} and the new attack is Kaleidoscope \citep{kaufmann2019testing}, ALR improves robustness on the union by 8.66\%.  Additionally, when the initial attack is $\ell_2$ and the new attack is Snow \citep{kaufmann2019testing}, ALR improves robustness on the Union of both attacks by 7.85\%.  We find same trend holds for VR (\cref{app:fine-tuning_pairs}).

\noindent\textbf{Random noise based regularization is harmful when used in fine-tuning. }Although random noise based regularization can improve robustness when used in the initial training phase of CRT, Figure \ref{fig:fine-tune_abl_main_text}(b) demonstrates that UR in fine-tuning hurts union accuracy for many initial and new attack pairs (corresponding results for GR are present in \cref{app:fine-tuning_pairs}). This suggests that while random noise based regularization can be used to perform initial training more efficiently, they should not be used during fine-tuning.  Since we found that UR and GR trade off accuracy on the initial attack when used in initial training in \cref{sec:init_train_impact}, this suggests that UR and GR generally trade off performance on attacks that are used in training or fine-tuning.

\begin{tcolorbox}[myboxstyle]
    \begin{cfinding}\label{cfind: justification} In fine-tuning, adversarial regularization (ALR and VR) can improve Union accuracy significantly (up to $\sim 7\%$) while random noise-based regularization hurts Union accuracy. %Adversarial regularization (ALR and VR) can help improve union accuracy when used in fine-tuning. In contrast, we find that random noise based regularization can harm performance when used in fine-tuning across many pairs attack types.
    \end{cfinding}
\end{tcolorbox}

\section{Discussion and Related Work}
\label{sec: discussion}
This work makes early progress towards deployable defenses that mitigate model obsolescence in the face of evolving adversaries. Such approaches could promote the adoption of robust models, as they allow model trainers to `patch' against vulnerabilities without training from scratch.

% towards creating defenses which can be adapted to new attack types. We propose a defense framework called continual robust training (CRT) and a theoretically-motivated regularization term for improving performance against an evolving set of attacks. Our method achieves good performance across known attacks while being more efficient than training from scratch.% The effectiveness of regularization is motivated by our theoretical results relating the change in robust loss for two attacks to the distances between their corresponding adversarially perturbed logits.

\noindent\textbf{Related Work:} Prior works investigate multiattack robustness (MAR) \citep{MainiWK20,TB19,madaan2020learning,Croce020} and unforeseen attack robustness \citep{laidlaw2020perceptual,zhang2018lpips,dai2022formulating,jin2020manifold,dai2023multirobustbench}. Unlike these methods, we assume that the defender may not know all attacks \emph{a priori} but adjusts their model as new attacks emerge. \citet{croce2022adversarial} propose a fine-tuning method for MAR on unions of $\ell_p$ attacks. Our work differs by exploring additional attack types (\emph{e.g.} spatial attacks \citep{XiaoZ0HLS18} and color shifts \citep{LaidlawF19}) and improvements to the initial training stage prior to fine-tuning. Additional related work is discussed in \cref{appsec: add_rel_work}.

\noindent\textbf{Limitations:} 
% \sophie{make limitations more clear, not always the best performing out of prior work, understanding connections between threat models.} In this work, we experimented with using existing robust training and fine-tuning techniques in CRT.  Further research on robust training methods can improve the performance of our framework. Our thoughts on how to address these limitations are in Appendix \ref{appsec:future_directions}.
More work is needed to improve the performance of RCRT, as our approach does not outperform existing baselines in all settings. It also remains unclear whether training from scratch with all attacks or fine-tuning on new attacks is optimal from both a theoretical and empirical perspective. Future theoretical work could characterize the convergence rates of each approach, as well as the gap in robustness between models at different stages in CRT. Further limitations and future directions are discussed in Appendix~\ref{appsec:future_directions}.

\section*{Impact Statement}
 The defense framework proposed can be useful for safety in practical, high-risk applications of supervised machine learning such as autonomous vehicles \citep{tencent2019experimental, 272270, 291108}, content moderation \citep{ye2023noisyhate,schaffner2024community}, and face authentication \citep{komkov2021advhat, wei2022adversarial} and provides first steps towards training and updating models in order to maintain robustness over time.  However, there are cases in which adversarial examples are used for good (\textit{e.g.} defending against website fingerprinting \citep{rahman2020mockingbird,shan2021patch}) which may be adversely affected by models robust to adversarial examples, including our proposed approach.

% In the unusual situation where you want a paper to appear in the
% references without citing it in the main text, use \nocite

\bibliography{multi_robust}
\bibliographystyle{icml2025}

%%%%%%%%%%%%%%%%%%%%%%%%%%%%%%%%%%%%%%%%%%%%%%%%%%%%%%%%%%%%%%%%%%%%%%%%%%%%%%%
%%%%%%%%%%%%%%%%%%%%%%%%%%%%%%%%%%%%%%%%%%%%%%%%%%%%%%%%%%%%%%%%%%%%%%%%%%%%%%%
% APPENDIX
%%%%%%%%%%%%%%%%%%%%%%%%%%%%%%%%%%%%%%%%%%%%%%%%%%%%%%%%%%%%%%%%%%%%%%%%%%%%%%%
%%%%%%%%%%%%%%%%%%%%%%%%%%%%%%%%%%%%%%%%%%%%%%%%%%%%%%%%%%%%%%%%%%%%%%%%%%%%%%%
\newpage
\appendix
\onecolumn

This appendix is organized as follows:
\begin{enumerate}
    \item Additional related work (\cref{appsec: add_rel_work})
    \item Future directions (\cref{appsec:future_directions})
        \item Proofs (\cref{sec:proof})
    \item Connection to variation regularization (\cref{app:alr_vr})

        \item Experimental verification of theoretical results (\cref{appsec:experimental_verification_of_theory})
    \item Additional experimental setup details (training and attack parameters, model selection, regularization setup) (\cref{app:exp_setup})
    \item Additional experiments
    \begin{itemize}
        \item Longer attack sequences and different datasets (CIFAR-100 and ImageNette) (\cref{app:exp_seq})
        \item Ablations on initial training (attack choice, regularization parameters) (\cref{app:init_train})
        \item Ablations on fine-tuning (attack choice, regularization parameters) ((\cref{app:fine-tuning}))
    \end{itemize}
\end{enumerate}

\section{Additional Related Work}
\label{appsec: add_rel_work}

\textbf{Adversarial Attacks and Defenses:} ML models are vulnerable to input-space perturbations known as adversarial examples \citep{szegedy2013intriguing}.  These attacks come in different formulations including $\ell_p$-norm bounded attacks \citep{madry2017towards, carlini2017towards}, spatial transformations \citep{XiaoZ0HLS18}, color shifts \citep{LaidlawF19}, JPEG compression and weather changes \citep{kaufmann2019testing}, bounded Wasserstein distance \citep{wasserstein_attacks, wu2020stronger} as well as attacks based on distances that are more aligned with human perception such as SSIM \citep{GRAGNANIELLO2021142} and LPIPS distances \citep{laidlaw2020perceptual, ghazanfari2023r}.

Despite the wide variety of attacks that have been introduced, defenses against adversarial examples focus mainly on $\ell_{\infty}$ or $\ell_2$-norm bounded perturbations \citep{cohen2019certified, zhang2020towards, madry2017towards, zhang2019theoretically, croce2020robustbench}.  Of existing defenses, adversarial training \citep{madry2017towards}, an approach that uses adversarial examples generated by the attack of interest during training, can most easily be adjusted to different attacks.  In our work, we build off of adversarial training in order to adapt to new adversaries.

\textbf{Training Techniques for Multi-Robustness:}A few prior works have studied the problem of achieving robustness against multiple attacks, under the assumption that all attacks are known a priori.  These include training based approaches \citep{MainiWK20, TB19, madaan2020learning} which incorporate adversarial examples from the threat models of interest (usually the combination of $\ell_1$, $\ell_2$, and $\ell_\infty$ norm bounded attacks) during training.  \citet{Croce020} provides a robustness certificate of all $\ell_p$ norms given certified robustness against $\ell_\infty$ and $\ell_1$ attacks.

Another line of works has looked at defending against attacks that are not known by the defender, which is a problem known as unforeseen robustness.  These techniques are all training-based and include \citet{laidlaw2020perceptual} which proposes training based on LPIPS \citep{zhang2018lpips}, a metric more aligned with human perception than $\ell_p$ distances, and \citet{dai2022formulating, jin2020manifold} which use regularization during training in order to obtain better generalization to unforeseen attacks.  \citet{dai2023multirobustbench} provides a comprehensive leaderboard for the performance of existing defenses against a large variety of attacks at different attack strengths.

Our work differs from these lines of works since we assume that while the defender may not know all attacks a priori, they are allowed to adjust their defense when they become aware of new attacks.  The work most similar to ours is \citet{croce2022adversarial}, which proposes fine-tuning a model robust against one $\ell_p$ attack to be robust against the union of $\ell_p$ attacks.  Specifically, they demonstrate that we can achieve simultaneous multiattack robustness for the union of $\ell_p$ attacks by obtaining robustness against $\ell_1$ and $\ell_{\infty}$ attacks, and thus propose fine-tuning with $\ell_1$ and $\ell_{\infty}$ attacks to achieve this efficiently.  Our work differs from this work since we explore adapting to attacks outside of $\ell_p$ attacks, investigate ways of improving the initial state of the model prior to fine-tuning, and consider adapting to sequences of attacks.

\textbf{Continual Learning:} A similar direction of research is continual learning (CL) in which a set of tasks are learned sequentially with the goal of performing as well as if they were learned simultaneously \citep{wang2023comprehensive}.  Few works have studied CL in conjunction with adversarial ML.  Of these, most works focus on evaluating or improving the robustness of models trained in the CL framework \citep{bai2023towards,9892970, khan2022susceptibility}. The most similar to our work is \citet{wang2023continual} which treats different attacks as tasks and uses approaches in CL in order to sequentially adapt a model against new attacks.  The attacks they consider follow the same threat model (ie. $\ell_{\infty}$ attacks using different optimization procedures to find the adversarial example).  In our work, we investigate adapting to new threat models.

%\textbf{Representation Similarity } \sophie{depending on how large of a portion CKA results make up the experimental section and how much space we have, can move this section into the appendix}
%\citep{cianfarani2022understanding}

\section{Future Directions}
\label{appsec:future_directions}
We now discuss a few directions for future work in depth.

\noindent\edit{\textbf{Choice of initial attacks and attack similarities. } In this work, we looked at $\ell_2$ and $\ell_\infty$ attacks as the initial attack in the CAR problem.  However, in practice, we would like to choose an initial attack that is the most representative of the attacks we want to be robust against, in order to generalize to downstream new attacks.  Further research on understanding and improving the initial attack can improve the accuracies achieved through training with CRT.  Additionally, having ways of measuring attack similarity between the known attacks and new attacks can help allow us to decide whether using CRT is sufficient for achieving good robustness or whether we need to train from scratch or combine the model with other defenses tailored towards the new attack.}

\noindent\edit{\textbf{Attack Monitoring. }}One assumption of CAR is that the defender is able to discover when a new attack exists.  While this is clear in cases such as a research group publishing a paper with a new attack or a company's security team finding a vulnerabilities, in practice, we would also be interested in recovering after an adversary discovers a new, unknown attack and successfully attacks the model. \edit{In this case, we would need a good monitoring system for detecting and synthesizing these new attacks for use with CRT.}

\noindent\edit{\textbf{Reducing catastrophic forgetting. } In CAR, since attacks are introduced sequentially, catastrophic forgetting is an important problem.  In our work, we utilized replay via \citet{croce2022adversarial}'s fine-tuning approach and also found that ALR reduces catastrophic forgetting to some extent.  Future work on reducing catastrophic forgetting can help improve the effectiveness of updating the model with CRT.}

\noindent\edit{\textbf{Training and fine-tuning efficiency.}} \edit{In our experiments,} we combine regularization with \citet{croce2022adversarial}'s fine-tuning approach due to the effectiveness and efficiency of that approach.  Further research on developing better and more efficient fine-tuning techniques for achieving robustness to new attacks (while maintaining robustness against previous attacks) can improve our CRT framework.

\noindent\edit{\textbf{Model capacity.} Current works in adversarial robustness literature show that adversarially robust models need higher model capacity~\citep{madry2017towards, gowal2020uncovering,cianfarani2022understanding}. As we increase the space of attacks to defend against, we may need to increase the capacity of the model in order to achieve multi-robustness~\citep{dai2024characterizing}. An interesting future direction is looking at the connection between model capacity and CAR and seeing if adding more parameters to the network during fine-tuning (such as using adapters~\citep{rebuffirevisiting}) can be used to address the issue of model capacity.}

\noindent\edit{\textbf{Theory.}}  We believe further work is necessary to extend the theory of CAR. Our results focus on the relationship between robust loss and \edit{logit} distance between attacks for a \emph{single model}. However, we do not extend them to comparisons between loss under different attacks for \emph{different} models, such as the initial robust model and the one at the end of fine-tuning. Additionally, the CAR framework could be extended to the multi-task setting, as is the case in multi-task representation learning \cite{watkinsadversarially,tripuraneni2020theory}. These prior works connect the ability of a class of models to learn a set of tasks to the complexity of that class (measured using Gaussian or Rademacher complexity, for example). Similar methods may also be useful for proving a model's ability to defend against multiple adversaries.

\section{Proofs}
\label{sec:proof}
\noindent\subsection{Proof of Theorem~\ref{thm:robustness}}
The proof of Theorem~\ref{thm:robustness} adapts that of Theorem I from \citet{nern2023transfer} by considering multiple attacks compared to the single one considered there.
\begin{proof}
    Define independent random variables $D_1,\ldots,D_n$ as
    \[D_i = \max_{x_i' \in C_1(x_i)}\ell(h(x_i'),y_i) - \max_{x_i'' \in C_2(x_i)}\ell(h(x_i''),y_i),\]
    based on independently drawn data points with probability distribution $\mathcal{P}(X)$. Using Hoeffding’s inequality, we get
    \begin{align*}
    &\mathbb{P}\left(\left| \sum_{i=1}^n D_i - n\mathbb{E}[D] \right| \geq t\right) \leq 2\cdot\exp\left( \frac{-2t^2}{nM_2^2} \right) \\ 
    \implies &\mathbb{P}\left(\left| \frac{1}{n}\sum_{i=1}^n D_i - \mathbb{E}[D] \right| \leq M_2\sqrt{\frac{\log(\rho/2)}{-2n}}\right) \geq 1 - \rho. \end{align*}
    Thus, with probability at least $1-\rho$ it holds that
    \begin{align}
    \nonumber\mathbb{E}[D] &= \left| \mathcal{L}_1(h) - \mathcal{L}_2(h) \right| \\
    \nonumber&= \left| \mathbb{E}_{(x,y)}\left[\max_{x' \in C_1(x)}\ell(h(x'),y) - \max_{x'' \in C_2(x)}\ell(h(x''),y)\right]\right|\\ 
    &\leq \left| \frac{1}{n}\sum_{i=1}^n\max_{x' \in C_1(x)}\ell(h(x'),y_i) - \max_{x'' \in C_2(x)}\ell(h(x''),y_i)\right| + M_2\sqrt{\frac{\log(\rho/2)}{-2n}}.\label{eq1}\end{align}
    We can further bound the first term on the right hand side, since the loss function $\ell(r,y)$ is $M_1$-Lipschitz in $\|\cdot\|_2$ for $r \in h(X)$:
    \begin{align}
    \nonumber&\left|\frac{1}{n}\sum_{i=1}^{n}\max_{x' \in C_1(x)}\ell(h(x'),y_i) - \max_{x'' \in C_2(x)}\ell(h(x''),y_i)\right|\\
    \nonumber&\leq \left|\frac{1}{n}\sum_{i=1}^{n}|\ell(h(x_i'),y_i) - \ell(h(x_i''),y_i)|\right| \\
    &\leq M_1 \frac{1}{n} \sum_{i=1}^{n}\|h(x_i') - h(x_i'')\|_2, \label{eq2}
    \end{align}
    where $x_1',\ldots,x_n'$ with $x_i' \in C_1(x_i)$ and $x_1'',\ldots,x_n''$ with $x_i'' \in C_2(x_i)$ are chosen to maximize $\ell(h(\cdot),y_i)$ for each $i$. 
    The perturbed samples represented in this inequality might not maximize the distance between the logits, but that distance can be bounded by the maximally distant perturbations within each neighborhood. Making use of the triangle inequality, we obtain:
    \begin{align}
    \nonumber&\sum_{i=1}^n\|h(x_i') - h(x_i'')\|_2\\
    \nonumber&=\sum_{i=1}^n\|(h(x_i') - h(x_i)) - (h(x_i'') - h(x_i))\|_2\\
    %&=\sum_{i=1}^n\|(g(x_i + \delta_{\psi,i}) - g(x_i)) - (g (x) - g(x_i + \delta_{\omega,o}))\|_2\\
    \nonumber&\leq\sum_{i=1}^n\|h(x_i') - h(x_i)\|_2 + \|h(x_i'') - h(x_i)\|_2\\
    &\leq\sum_{i=1}^n\max_{x' \in C_1(x_i)}\|h(x') - h(x_i)\|_2 + \max_{x'' \in C_2(x_i)}\|h(x'') - h(x_i)\|_2.
    \end{align}
    We then achieve our final result, recalling the assumption that $\mathcal{L}_1(h) \geq \mathcal{L}_2(h)$:
    \begin{align}
        \nonumber\mathcal{L}_1(h) - \mathcal{L}_2(h) &= \left| \mathcal{L}_1(h) - \mathcal{L}_2(h)\right|  \\
        &\leq M_1\frac{1}{n}\sum_{i=1}^n\Bigl(\max_{x' \in C_1(x_i)}\|h(x') - h(x_i)\|_2 +  \max_{x'' \in C_2(x_i)}\|h(x'') - h(x_i)\|_2\Bigr) + D,
    \end{align}
    where $D = M_2\sqrt{\frac{\log(\rho/2)}{-2n}}$.
\end{proof}

\noindent\subsection{Proof of Corollary~\ref{thm:corollary}}
\begin{proof}
    Define independent random variables $D_1,\ldots,D_n$ as
    \[D_i = \max_{x_i' \in C_1(x_i) \cup C_2(x_i)}\ell(h(x_i'),y_i) - \ell(h(x_i),y_i),\]
    based on independently drawn data points with probability distribution $\mathcal{P}(\mathcal{X})$. Using Hoeffding’s inequality, we get
    \begin{align*}
    &\mathbb{P}\left(\left| \sum_{i=1}^n D_i - n\mathbb{E}[D] \right| \geq t\right) \leq 2\cdot\exp\left( \frac{-2t^2}{nM_2^2} \right) \\ 
    \implies &\mathbb{P}\left(\left| \frac{1}{n}\sum_{i=1}^n D_i - \mathbb{E}[D] \right| \leq M_2\sqrt{\frac{\log(\rho/2)}{-2n}}\right) \geq 1 - \rho. \end{align*}
    Thus, with probability at least $1-\rho$ it holds that
    \begin{align}
    \nonumber\mathbb{E}[D] &= \left| \mathcal{L}_{1,2}(h) - \mathcal{L}(h) \right| \\
    \nonumber&= \left| \mathbb{E}_{(x,y)}\left[\max_{x' \in C_1(x) \cup C_2(x)}\ell(h(x'),y) - \ell(h(x),y)\right]\right|\\ 
    &\leq \left| \frac{1}{n}\sum_{i=1}^n\max_{x' \in C_1(x_i) \cup C_2(x_i)}\ell(h(x'),y_i) - \ell(h(x_i),y_i)\right| + M_2\sqrt{\frac{\log(\rho/2)}{-2n}}.\label{eq1}\end{align}
    We can further bound the first term on the right hand side, since the loss function $\ell(r,y)$ is $M_1$-Lipschitz in $\|\cdot\|_2$ for $r \in h(\mathcal{X})$:
    \begin{align}
    \nonumber&\left|\frac{1}{n}\sum_{i=1}^{n}\max_{x' \in C_1(x_i) \cup C_2(x_i)}\ell(h(x'),y_i) - \ell(h(x_i),y_i)\right|\\
    \nonumber&= \left|\frac{1}{n}\sum_{i=1}^{n}|\ell(h(x_i'),y_i) - \ell(h(x_i),y_i)|\right| \\
    &\leq M_1 \frac{1}{n} \sum_{i=1}^{n}\|h(x_i') - h(x_i)\|_2, \label{eq2}
    \end{align}
    where $x_1',\ldots,x_n'$ with $x_i' \in C_1(x_i) \cup C_2(x_i)$ are chosen to maximize $\ell(h(\cdot),y_i)$ for each $i$. 
    The perturbed samples represented in this inequality might not maximize the distance between the logits, but that distance can be bounded by the maximally distant perturbations within each neighborhood. 
    \begin{align}
    \nonumber&\sum_{i=1}^n\|h(x_i') - h(x_i)\|_2\\
    \nonumber&\leq\sum_{i=1}^n \max_{x'\in C_1(x_i) \cup C_2(x_i)} \|h(x') - h(x_i)\|_2\\
    \nonumber&=\sum_{i=1}^n \max_{C \in \{C_1,C_2\}}\max_{x'\in C(x_i)} \|h(x') - h(x_i)\|_2\\
    \end{align}
    We then achieve our final result :
    \begin{align}
        \nonumber\mathcal{L}_{1,2}(h) - \mathcal{L}(h) &= \left| \mathcal{L}_{1,2}(h) - \mathcal{L}(h)\right|\\
        &\leq M_1\frac{1}{n}\sum_{i=1}^n\Bigl(\max_{x' \in C_1(x_i)}\|h(x') - h(x_i)\|_2 +  \max_{x'' \in C_2(x_i)}\|h(x'') - h(x_i)\|_2\Bigr) + D,
    \end{align}
    where $D = M_2\sqrt{\frac{\log(\rho/2)}{-2n}}$.
\end{proof}
%\begin{proof}
%    Since $C_1(x) \subset C_1(x) \cup C_2(x)$, therefore $\max_{x' \in C_1(x) \cup C_2(x)} \ell(h(x'),y) \geq \max_{x' \in C_1(x)} \ell(h(x'),y)$ for all $(x,y) \in \mathcal{D}$, so $\mathcal{L}_{1,2}(h) \geq \mathcal{L}_1(h)$. Similarly, since $x \in C_2(x)$, we have that $\mathcal{L}_2(h) \geq \mathcal{L}(h)$. It follows that $\mathcal{L}_{1,2}(h) - \mathcal{L}(h) \leq \mathcal{L}_1(h) - \mathcal{L}_2(h)$, so 
%    \begin{align*}
%        \mathcal{L}_{1,2}(h) - \mathcal{L}(h) \leq M_1 \frac{1}{n}\sum_{i=1}^n\biggl(\max_{x' \in C_1(x_i)}\|h(x') - h(x_i)\|_2 + \max_{x' \in C_2(x_i)}\|h(x') - h(x_i)\|_2\biggl) + D.
%    \end{align*}
%\end{proof}

\noindent\subsection{Relating the Loss Gap to Internal Representations}

While our results bound the robust loss gap in terms of the distance between logits of samples perturbed with different attacks, similar results hold for the distance between internal activations. To show how our results can apply to common transfer learning settings (such as that of \citet{nern2023transfer}), we prove the following corollary:
\begin{corollary} \label{corollary: deeper_reps}
    Let $h: \mathbb{R}^d \rightarrow \mathbb{R}^c$ be a $c$ class neural network classification model with a final linear layer (i.e. $h(c) = Wg(x)$, where $g: \mathbb{R}^d \rightarrow \mathbb{R}^r$, and $W \in \mathbb{R}^{c \times r}$). 
    Assume that loss $\ell(\hat{y},y)$ is $M_1$-Lipschitz in $\|\cdot\|_\alpha$ for $\alpha \in \{1,2,\infty\}$, for $\hat{y} \in h(X)$ with $M_1 > 0$ and bounded by $M_2 > 0$, i.e. $0 \leq \ell(\hat{y},y) \leq M_2$ $\forall \hat{y} \in h(X)$. Then, for a subset $\mathbb{X} = \{x_i\}_{i=1}^n$ independently drawn from $\mathcal{D}$, the following holds with probability at least $1-\rho$:
    \begin{align*}
        \mathcal{L}_1(h) - \mathcal{L}_2(h) \leq L_\alpha(W)M_1 \frac{1}{n}\sum_{i=1}^n\biggl(\max_{x' \in C_1(x_i)}\|g(x') - g(x_i)\|_2 + \max_{x' \in C_2(x_i)}\|g(x') - g(x_i)\|_2\biggl) + D,
    \end{align*}

    where $D = M_2\sqrt{\frac{\log(\rho/2)}{-2n}}$ and
    \[
    L_\alpha(W) := \begin{cases}
    \|W\|_2 &, \text{if } \|\cdot\|_\alpha = \|\cdot\|_2, \\
    \sum_i\|W_i\|_2 &, \text{if } \|\cdot\|_\alpha = \|\cdot\|_1, \\
    \max_i\|W_i\|_2 &, \text{if } \|\cdot\|_\alpha = \|\cdot\|_\infty. 
    \end{cases}
    \]
\end{corollary}

\begin{proof}
    From (\ref{eq1}), (\ref{eq2}), and the definition of $h$, we have that 
    \begin{align*}
        \left|\mathcal{L}_1 - \mathcal{L}_2\right| \leq M_1 \frac{1}{n}\sum_{i=1}^n \|Wg(x_i') - h(x_i'')\|_2 + M_2\sqrt{\frac{\log{\rho/2}}{-2n}}.
    \end{align*}
    We then apply Lemma 2 from \citet{nern2023transfer} and the definition of $L_\alpha$:
    \begin{align*}
    M_1 \frac{1}{n} &\sum_{i=1}^{n}\|Wg(x_i') - Wg(x_i'')\|_\alpha \\
    &\leq L_\alpha(W)M_1 \frac{1}{n} \sum_{i=1}^{n}\|g(x_i') - g(x_i'')\|_2.
    \end{align*}
    As in the proof for Theorem~\ref{thm:robustness}, the perturbed samples represented in this inequality might not maximize the distance between the representations, but that distance can be bounded by the maximally distant perturbations within each neighborhood. Making use of the triangle inequality, we obtain:
    \begin{align*}
    &\sum_{i=1}^n\|g(x_i') - g(x_i'')\|_2\\
    &=\sum_{i=1}^n\|(g(x_i') - g(x_i)) - (g(x_i'') - g(x_i))\|_2\\
    %&=\sum_{i=1}^n\|(g(x_i + \delta_{\psi,i}) - g(x_i)) - (g (x) - g(x_i + \delta_{\omega,o}))\|_2\\
    &\leq\sum_{i=1}^n\|g(x_i') - g(x_i)\|_2 + \|g(x_i'') - g(x_i)\|_2\\
    &\leq\sum_{i=1}^n\max_{x' \in C_1(x_i)}\|g(x') - g(x_i)\|_2 + \max_{x'' \in C_2(x_i)}\|g(x'') - g(x_i)\|_2.
    \end{align*}
    We then achieve our final result:
    \begin{align*}
        \mathcal{L}_1(h) - \mathcal{L}_2(h) &= \left| \mathcal{L}_1(h) - \mathcal{L}_2(h)\right|\\
        &\leq L_\alpha(W)M_1 \frac{1}{n}\sum_{i=1}^n\Bigl(\max_{x' \in C_1(x_i)}\|g(x') - g(x_i)\|_2 +  \max_{x'' \in C_2(x_i)}\|g(x'') - g(x_i)\|_2\Bigr) + D,
    \end{align*}
    where $D = M_2\sqrt{\frac{\log(\rho/2)}{-2n}}$.
\end{proof}

\section{Connection Between Adversarial $\ell_2$ Regularization and Variation Regularization} \label{app:alr_vr}
In this section, we will show the relationship between adversarial $\ell_2$ regularization (ALR) and variation regularization (VR) \citep{dai2022formulating}.  To begin, we first revisit the definitions of ALR and VR:
$$R_{\text{ALR}}(h, K(t)) = \frac{1}{m} \sum_{i=1}^m \max_{x' \in C(x_i)} \|\edit{h}(x') - \edit{h}(x_i)\|_2$$
$$R_{\text{VR}}(h, K(t)) = \frac{1}{m} \sum_{i=1}^m  \max_{x', x'' \in C(x_i)} \|\edit{h}(x') - \edit{h}(x'')\|_2$$
Since VR optimizes over 2 perturbations $x'$ and $x''$ for each example while ALR optimizes only for $x'$, it is clear that $R_{\text{ALR}} \le R_{\text{VR}}$.  Additionally, we note that:
$$R_{\text{VR}}(h, K(t)) = \frac{1}{m} \sum_{i=1}^m \max_{x', x'' \in C(x_i)} \|\edit{h}(x') -h(x) + h(x) - \edit{h}(x'')\|_2$$
$$\le  \frac{1}{m} \sum_{i=1}^m \max_{x', x'' \in C(x_i)} \|h(x') -h(x)\|_2 + \|h(x) - h(x'')\|_2$$
$$= \frac{2}{m}\sum_{i=1}^m \max_{x'\in C(x_i)}\|h(x') -h(x)\|_2$$
$$= 2 R_\text{ALR}$$
Thus, ALR and VR are related in the sense that $R_{\text{ALR}} \le R_{\text{VR}} \le 2R_{\text{ALR}}$.

\section{Experimental Verification of Theoretical Results} 
\label{appsec:experimental_verification_of_theory}
\begin{figure*}
    \centering
    \includegraphics[width=\textwidth]{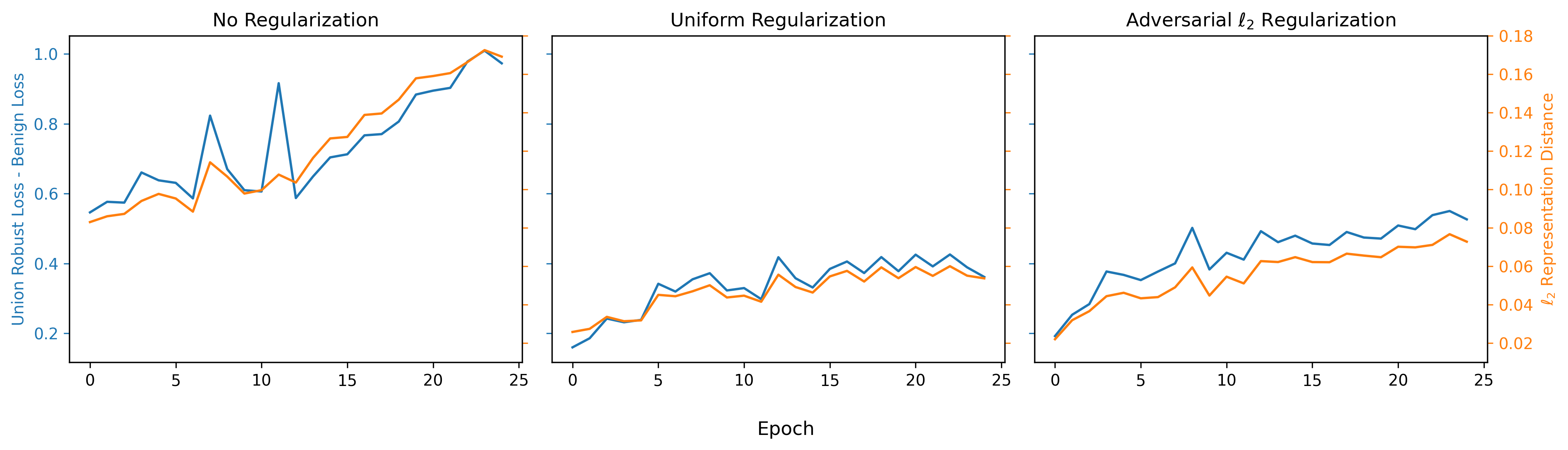}
    \vspace{-15pt}
    \caption{Adversarial loss gap ($\mathcal{L}_{1,2}(h) - \mathcal{L}(h)$) and average $\ell_2$ distance between \edit{logits} of $\ell_2$ ($\epsilon = 0.5$, representing $P_{C_1}$) and StAdv ($\epsilon = 0.05$, representing $P_{C_2}$) attacked samples over 25 epochs of fine-tuning using \cite{croce2022adversarial}'s fine-tuning method, both with and without regularization. Each model is fine-tuned starting from a model that is adversarially trained against an $\ell_2$ adversary, as described in Section~\ref{sec:exp_setup}. In all training scenarios, there is a visible correlation between the loss gap and the \edit{logit} distance, aligning with the theoretical result in Corollary~\ref{thm:corollary}. }
    \label{fig:loss_gap}
    \vspace{-10pt}
\end{figure*}

We now briefly demonstrate that our chosen regularization terms align with our theoretical results. In Figure~\ref{fig:loss_gap}, we start with WRN-28-10 models that were adversarially trained to be robust against $\ell_2$-bounded attacks, and fine-tune them to increase their robustness against StAdv attacks using either no regularization, uniform regularization, or adversarial $\ell_2$ regularization. We observe a number of trends:

\noindent
\textbf{Sensitivity correlates with loss gap.} Whether or not regularization is used, there is a clear correlation between total adversarial sensitivity across both attacks (i.e. $\max_{x'\in C_1(x)}\|\edit{h}(x')-\edit{h}(x)\| + \max_{x'\in C_2(x)}\|\edit{h}(x')-\edit{h}(x)\|$) and the loss gap between the union robust loss and the benign loss (i.e. $\mathcal{L}_{1,2}(h) - \mathcal{L}(h)$). 

\noindent
\textbf{Regularization reduces sensitivity and loss gap.} Both metrics are significantly lower throughout fine-tuning when regularization is used, indicating that regularization is successfully targeting our theoretical bounds.

\noindent
\textbf{Loss gap increases over time.} Across all three models there is an increase in both loss gap and adversarial sensitivity over the course of fine-tuning. While this may seem like a failure of regularization, the benefit is more apparent when further analyzing what is causing the loss gap to increase. In the regularized fine-tuning runs, both benign and robust losses are decreasing, with benign loss decreasing more quickly. This is likely influenced by an initial increase in benign loss at the very beginning of fine-tuning which is not captured in Figure~\ref{fig:loss_gap}. However, without regularization, benign loss decreases while union robust loss increases. This shows us that despite theoretically targeting the gap between union robust loss and benign loss, the use of regularization still aids in individually reducing both losses in absolute terms.

\section{Additional Experimental Setup Details}
\label{app:exp_setup}
\textbf{Additional training details. }For initial training, we start with a learning rate of 0.1 and then use the multistep learning rate scheduling proposed by \citet{gowal2020uncovering}; specifically, we scale the learning rate down by a factor of 10 halfway and 3/4 of the way through initial training or fine-tuning.  For fine-tuning, we maintain a learning rate of 0.001.  We train with SGD with momentum of 0.9 and weight decay of 0.0005.

\textbf{Additional Attack parameters in training. } Following other works on adversarial robustness, we use a step size of 0.075 for $\ell_2$ attacks on CIFAR-10, 0.15 for $\ell_2$ attacks on ImageNette, and $\frac{2}{255}$ for $\ell_\infty$ attacks.  For other attacks, we use $\frac{\epsilon}{8}$ where $\epsilon$ is the attack strength as the step size during training.

\textbf{Model selection. }In the main paper, we stated that we perform evaluation using the epoch at which the model has the best performance measured across known attack types.  Specifically, after each epoch of training, we evaluate the performance of each model against the attacks used during training (with the same attack parameters as used during training).  For training with AVG, we use the best performing model with respect to the AVG objective (which is the model with the best performance measured as an average over individual attack accuracies).  Meanwhile for MAX and FT MAX, we use the best performing model with respect to the MAX objective (which is the best performing model across the union of all attacks).  For procedures that only use a single attack per batch during training (Random, FT Single, FT Croce, and our procedure), we use the best performing model measured by sampling attacks per batch randomly.

\textbf{Regularization setup. }\edit{We note that all attacks used in this paper use a gradient based optimization scheme for finding the attack.  In order to compute regularization for non-$\ell_p$ threat models, we follow the same optimization scheme used by the attack \citep{XiaoZ0HLS18, LaidlawF19, kaufmann2019testing} but replace the classification loss portion of the optimization objective to be the $\ell_2$ distance between features/logits between the perturbed and unperturbed input. For fine-tuning with regularization, }since \citet{croce2022adversarial}'s fine-tuning approach only uses a single attack per batch, we structure the regularization to mimic \citet{croce2022adversarial}'s fine-tuning procedure.  Specifically, for each batch, the regularization is for a single attack type (the same one which is selected to use with adversarial training by \citet{croce2022adversarial}'s fine-tuning approach).  This helps to reduce the overhead from regularization.

\section{Additional Experimental Results for CAR}
\label{app:exp_seq}
We present additional results for CAR on CIFAR-10 in Tables \ref{tab:main_results_cifar_epochs} and  \ref{app:main_results_cifar_epochs_2}, results for CAR on Imagenette in Table \ref{app:main_results_imagenette_epochs_1} and \ref{app:main_results_imagenette_epochs_2} and results for CIFAR-100 in Tables \ref{app:main_results_cifar100_full} and \ref{app:main_results_cifar100_full_2}.  We also compare different fine-tuning approaches in the absense of regularization.

%\subsection{Continual Robust Training for \probname}
%\label{sec:crt_vs_from_scratch}

\noindent\textbf{Training time and robust performance.} We find that fine-tuning with MAX objective (FT MAX) or \citet{croce2022adversarial} (FT Croce) can generally achieve robustness across previous attacks and the new attack in the sequence comparable to training from scratch.  For example, in Table \ref{app:main_results_cifar_epochs_2}, when fine-tuning to gain robustness against StAdv attack starting from a model initially trained with adversarial training on $\ell_\infty$ attacks on CIFAR-10, we find that FT MAX achieves 50.75\% average robustness across the two attacks and 41.57\% union robustness across the two attacks, and FT Croce achieves 49.48\% average robustness and 29.69\% union robustness.  These values lie within (or even above) the range obtained through training from scratch (42.23\%-49.61\% average 
robustness and 28.03\%-40.8\% union robustness).  We find that this trend generally holds as well across time steps when new attacks are introduced, when using a different sequence ordering (Table \ref{tab:main_results_cifar_epochs}).

Of these two techniques, we find that FT MAX generally achieves higher average and union accuracies across the set of known attacks, but is less efficient when used in fine-tuning.  For example, In Table \ref{app:main_results_cifar_epochs_2}, FT MAX takes 3.99 hours for 10 epochs of fine-tuning from an $\ell_{\infty}$ robust model while FT Croce takes 2.31 hours.  The time complexity of FT MAX also scales as the number of attacks increases, leading to 7.9 hours of fine-tuning for 10 epochs when there are 4 known attacks while FT Croce maintains approximately the same training time.

In comparison to naively training from scratch, we also find that these fine-tuning techniques can be much more efficient.  For example, a model robust to a sequence of 4 attacks in Table~\ref{tab:main_results_cifar_epochs} can be found in roughly 17 hours using CRT, but training from scratch each time would require 44 hours cumulatively.

\noindent\textbf{Importance of replay. }We find that replay of previous attacks is important for achieving good robustness across the set of known attacks when training with CRT. Fine-tuning with only the new attack (FT Single) usually leads to rapid forgetting of the previous attack.  For example, in Table \ref{app:main_results_cifar_epochs_2} we observe that the accuracy of robustness on the initial attack ($\ell_\infty$) drops to 31.14\% robust accuracy at time step 1 (from the initial accuracy of 51.49\% at time step 0) and then further drops to 25.27\% at time step 2 when the third attack (Recolor) is introduced.  This forgetting is independent from tradeoffs between attacks as we find that training from scratch and FT MAX and FT Croce techniques can all achieve at least 40\% $\ell_{\infty}$ accuracy at time step 1 and at least 35\% $\ell_{\infty}$ accuracy at time step 2.  The forgetting of previous attacks is also analogous to catastrophic forgetting of previous tasks in continual learning \citep{wang2023continual, MCCLOSKEY1989109}.  We note however that forgetting is less of a limitation in CAR than in continual learning since the defender's knowledge set only grows over time; they do not forget the formulation of previous attacks and can thus can always use methods such as replay.

\textbf{ALR applied on logits vs features. }In Table \ref{tab:main_results_cifar_epochs}, we also provide results for using regularization based on distances in the feature space (before the final linear layer), which are labelled with ``+ ALR feature".  Overall we observe that using regularization in the feature space can also help improve performance on average and union robustness across known attacks as well as improve unforeseen robustness over baselines.  However, we observe that feature space regularization leads to larger tradeoffs in clean accuracy than regularization on the logits (``+ ALR" rows) while robust performance is comparable to regularization applied on the logits.

\textbf{Training durations. } Across all tables we also provide experiments for fine-tuning with 25 epochs (as opposed to 10 epochs reported in the main body).  We find that increasing the number of fine-tuning epochs can help methods such as FT Croce achieve robustness closer to that of training from scratch, but at the cost of increased time for updating the model.

\textbf{Performance on other datasets. }We find that the gain in performance through using ALR varies across datasets.  For Imagenette the gain in performance is generally much smaller than on CIFAR-10 (ALR closes the gap between fine-tuning based updates and training from scratch rather than surpassing training from scratch as in CIFAR-10.  On CIFAR-100 ALR generally does not improve performance over fine-tuning.  We believe that this is because achieving robustness on multiple attacks is quite hard on CIFAR-10; clean accuracy is between 60-70\% and robust accuracies are even lower with StAdv and $\ell_{\infty}$ robustness only achieving up to 32\% robust accuracy and 25\% robust accuracy respectively.

\begin{table*}[ht]
\centering
\scalebox{0.68}{
\begin{tabular}{|c|l|l|c|cccc|cc|cc|c|}
\hline
\multicolumn{1}{|c|}{\begin{tabular}[c]{@{}c@{}}Time\\ Step \end{tabular}} &Procedure & Threat Models & \multicolumn{1}{c|}{Clean} & \multicolumn{1}{c}{$\ell_2$} & \multicolumn{1}{c}{StAdv} & \multicolumn{1}{c}{$\ell_\infty$} & \multicolumn{1}{c|}{Recolor} & \multicolumn{1}{c}{\begin{tabular}[c]{@{}c@{}}Avg\\ (known)\end{tabular}} & \multicolumn{1}{c|}{\begin{tabular}[c]{@{}c@{}}Union\\ (known)\end{tabular}} & \multicolumn{1}{c}{\begin{tabular}[c]{@{}c@{}}Avg\\ (all)\end{tabular}} & \multicolumn{1}{c|}{\begin{tabular}[c]{@{}c@{}}Union\\ (all)\end{tabular}} & \multicolumn{1}{c|}{\begin{tabular}[c]{@{}c@{}}Time\\ (hrs)\end{tabular}} \\ \hline
\multirow{ 3}{*}{0} & AT & $\ell_2$ & \textbf{91.17} & \cellcolor[HTML]{B7E1CD}69.7 & 2.08 & 28.41 & 44.94 & 69.7 & 69.7 & 36.28 & 1.24 & 8.35 \\ %16.69
& AT + ALR ($\lambda=1$) & $\ell_2$ & 89.43 & \cellcolor[HTML]{B7E1CD}\textbf{69.84} & \textbf{48.23} & \textbf{34.00} & \textbf{65.46} & \textbf{69.84} & \textbf{69.84} & \textbf{54.38} & \textbf{31.27} & 11.15\\ %22.29
%& AT + ALR feature ($\lambda = 1$) & $\ell_2$ & 84.53 & \cellcolor[HTML]{B7E1CD}63.22 & 7.78 & 25.74 & 51.76 & 63.22 & 63.22 & 37.12 & 6.07\\
%& AT + ALR feature ($\lambda=2$) & $\ell_2$ & 84.53 & \cellcolor[HTML]{B7E1CD}64.22 & 14.64 & 30.55 &  57.16 & 64.22 & 64.22 & 41.64 & 11.4\\
& AT + ALR feature ($\lambda=5$) & $\ell_2$ & 83.7 & \cellcolor[HTML]{B7E1CD}63.1 & 26.57 & 31.6 & 62.53  & 63.1 & 63.1 & 45.95 & 20.16 & 11.13\\

\hline
\multirow{ 17}{*}{1} & AVG & $\ell_2$, StAdv & 87.74 & \cellcolor[HTML]{B7E1CD} 62.17 & \cellcolor[HTML]{B7E1CD}50.92 & 17.17 & 45.47 & 56.55 & 47.55 & 43.93 & 15.92 &23.72 \\ %20.26
& MAX & $\ell_2$, StAdv & 86.18 & \cellcolor[HTML]{B7E1CD}58.65 & \cellcolor[HTML]{B7E1CD}57.21 & 11.21 & 43.07 & 57.93 & 51.72 & 42.54 & 11.03 & 23.69 \\ %20.23
& Random & $\ell_2$, StAdv & 84.91 & \cellcolor[HTML]{B7E1CD}57.77 & \cellcolor[HTML]{B7E1CD}59.74 & 14.05 & 44.88 & 58.76 & 52.15 & 44.11 & 13.68 & 10.92 \\ %9.36
\cdashline{2-13}
 & FT MAX (10 ep) &  $\ell_2$, StAdv & 83.73 & \cellcolor[HTML]{B7E1CD}57.07 & \cellcolor[HTML]{B7E1CD}58.67 & 12.51 & 49.03 & 57.87 & 51.32 & 44.32 & 12.36 & 4 \\
 & FT MAX (25 ep) & $\ell_2$, StAdv & 84.85 & \cellcolor[HTML]{B7E1CD}56.44 & \cellcolor[HTML]{B7E1CD} 61.34 & 10.35 & 48.08 & 58.89 & 52.52 & 44.05 & 10.24 & 10 \\ %
 & FT Croce (10 ep) & $\ell_2$, StAdv & 84.7 & \cellcolor[HTML]{B7E1CD}57.88 & \cellcolor[HTML]{B7E1CD}54.27 & 14.38 & 51.08 & 56.07 & 48.13 & 44.4 & 13.8 & 2.4\\
  & FT Croce (25 ep) & $\ell_2$, StAdv & 86.24 & \cellcolor[HTML]{B7E1CD}58.94 & \cellcolor[HTML]{B7E1CD}57.37 & 13.26 & 50.36 & 58.16 & 50.89 & 44.98 & 13 & 5.98 \\ %
   & FT Single (10 ep) & $\ell_2$, StAdv &  80.89
& \cellcolor[HTML]{B7E1CD}45.45& \cellcolor[HTML]{B7E1CD}54.5
& 6.09 & 41.98 & 49.98 & 41.05 & 37 & 5.87 & 2.78 \\
 & FT Single (25 ep)  & $\ell_2$, StAdv & 81.21 & \cellcolor[HTML]{B7E1CD}44.17 & \cellcolor[HTML]{B7E1CD}54.6 & 5.56 & 40.95 & 49.38 & 39.76 & 36.32 & 5.36 & 6.92 \\ %
    & FT Single + ALR (10 ep) & $\ell_2$, StAdv & 87.24 & \cellcolor[HTML]{B7E1CD}62.22& \cellcolor[HTML]{B7E1CD}61.5
& 21.4 & \textbf{\underline{70.87}} & 61.86 & 55.04 & 54 & 21.14 & 4.24 \\
 & FT Single + ALR (25 ep) & $\ell_2$, StAdv & 87.54 & \cellcolor[HTML]{B7E1CD}61.21 & \cellcolor[HTML]{B7E1CD}60.38 & 20.81 & 69.49 & 60.8 & 54.22 & 52.97 & 20.48 & 8.77 \\ %
& FT Single + ALR feature ($\lambda=2$, 10 ep) &  $\ell_2$, StAdv & 81.79 & \cellcolor[HTML]{B7E1CD}56.98 & \cellcolor[HTML]{B7E1CD}60.28 & 20.59 & 63.64 & 58.63 & 51.65 & 50.37 & 20.21 & 3.52 \\
  & FT Single + ALR feature ($\lambda=5$, 10 ep) &  $\ell_2$, StAdv & 81.26 & \cellcolor[HTML]{B7E1CD}60.43 & \cellcolor[HTML]{B7E1CD}57.61 & \textbf{\underline{28.17}} & 67.95 & 59.02 & 51.99 & 53.54 & \textbf{\underline{27.24}} & 3.53 \\
   & FT Croce + ALR (10 ep) & $\ell_2$, StAdv & 86.03 & \cellcolor[HTML]{B7E1CD}59.18 &\cellcolor[HTML]{B7E1CD}\textbf{\underline{65.14}} & 15.36  & 63.31 & \textbf{\underline{62.16}} & \textbf{\underline{55.83}} & 50.75 & 15.29 & 3.47 \\
 & FT Croce + ALR (25 ep) & $\ell_2$, StAdv & \underline{\textbf{88.5}} & \cellcolor[HTML]{B7E1CD}\textbf{\underline{64.88}} & \cellcolor[HTML]{B7E1CD}58.98 & 23.9 & 70.79 & 61.93 & 55.03 & \underline{\textbf{54.64}} & 23.33 & 7.96 \\ %
 & FT Croce + ALR feature ($\lambda=2$, 10 ep) & $\ell_2$, StAdv &  83.19 & \cellcolor[HTML]{B7E1CD}61.28 &\cellcolor[HTML]{B7E1CD} 59.04 & 23.98 & 62.69 & 60.16 & 53.25 & 51.75 & 23.2 & 2.97
\\
 & FT Croce + ALR feature ($\lambda=5$, 10 ep) & $\ell_2$, StAdv & 83.51 & \cellcolor[HTML]{B7E1CD}61.69 &\cellcolor[HTML]{B7E1CD} 61.76 & 23.33 & 62.48 & 61.73 & 55.25 & 52.31 & 22.77 & 3.13\\

\hline

\multirow{ 17}{*}{2}&AVG & $\ell_2$, StAdv, $\ell_\infty$ & 85.98 & \cellcolor[HTML]{B7E1CD}67.60 & \cellcolor[HTML]{B7E1CD}45.81 & \cellcolor[HTML]{B7E1CD}42.39 & 62.43 & 51.93 & 34.05 & 54.56 & 33.39 & 33.12 \\
&MAX & $\ell_2$, StAdv, $\ell_\infty$ & 84.54 & \cellcolor[HTML]{B7E1CD}54.87 & \cellcolor[HTML]{B7E1CD}52.33 & \cellcolor[HTML]{B7E1CD}38.23 & 55.90 & 48.48 & 35.25 & 50.33 & 34.08 & 79.04 \\
&Random & $\ell_2$, StAdv, $\ell_\infty$ & 39.52 & \cellcolor[HTML]{B7E1CD}67.46 & \cellcolor[HTML]{B7E1CD}47.35 & \cellcolor[HTML]{B7E1CD}42.12 & 63.61 & 52.31 & 35.46 & 55.13 & 34.79 & 10.92 \\ \cdashline{2-13}
 & FT MAX (10 ep) & $\ell_2$, StAdv, $\ell_\infty$ & 83.16 &  \cellcolor[HTML]{B7E1CD}65.63 &  \cellcolor[HTML]{B7E1CD}56.68 &  \cellcolor[HTML]{B7E1CD}36.9 & 65.69 & 53.07 & 35.18 & 56.23 & 34.83 & 5.62 \\
 & FT MAX (25 ep) & $\ell_2$, StAdv, $\ell_{\infty}$ & 83.99 & \cellcolor[HTML]{B7E1CD}65.69 & \cellcolor[HTML]{B7E1CD} 58.16 & \cellcolor[HTML]{B7E1CD}37.21 & 65.52 & 53.69 & 35.76 & 56.65 & 35.31 & 12.88 \\
 & FT Croce (10 ep) & $\ell_2$, StAdv, $\ell_{\infty}$ & 85.05 &  \cellcolor[HTML]{B7E1CD}67.3 &  \cellcolor[HTML]{B7E1CD}48.07 &  \cellcolor[HTML]{B7E1CD}33.38 & 62.52 & 49.58 & 28.96 & 52.82 & 28.63 & 2.27 \\
 & FT Croce (25 ep) & $\ell_2$, StAdv, $\ell_{\infty}$ & 86.14 & \cellcolor[HTML]{B7E1CD}67.3 & \cellcolor[HTML]{B7E1CD}52.47 & \cellcolor[HTML]{B7E1CD}35.86 & 63.43 & 51.88 & 32.54 & 54.77 & 32.08 & 5.01 \\
  & FT Single (10 ep) & $\ell_2$, StAdv, $\ell_{\infty}$ & 87.99 & \cellcolor[HTML]{B7E1CD}\textbf{\underline{70.53}} & \cellcolor[HTML]{B7E1CD}11.17
 & \cellcolor[HTML]{B7E1CD}41.63 & 63.46 &41.11 & 7.95 & 46.7 & 7.74  & 1.57 \\
 & FT Single (25 ep) & $\ell_2$, StAdv, $\ell_{\infty}$ & 88.67 & \cellcolor[HTML]{B7E1CD} 70.23 & \cellcolor[HTML]{B7E1CD}8.79 & \cellcolor[HTML]{B7E1CD} 43.4 & 63.03 & 40.81 & 6.19 & 46.36 & 6.05 & 3.91 \\
 & FT Single + ALR (10 ep) & $\ell_2$, StAdv, $\ell_{\infty}$ & \textbf{\underline{88.74}} & \cellcolor[HTML]{B7E1CD}69.15 & \cellcolor[HTML]{B7E1CD}47.33 & \cellcolor[HTML]{B7E1CD}42.08 & 68.62 & 52.85 & 36.66 & 56.8 & 36.62 & 2.26 \\
 & FT Single + ALR (25 ep) & $\ell_2$, StAdv, $\ell_{\infty}$ & 88.14 & \cellcolor[HTML]{B7E1CD}68.26 & \cellcolor[HTML]{B7E1CD}49.1 & \cellcolor[HTML]{B7E1CD}41.48 & 66.73 & 52.95 & \textbf{\underline{37.55}} & 56.39 & 37.5 & 5.4 \\
 & FT Single + ALR feature ($\lambda=2$, 10 ep) &  $\ell_2$, StAdv, $\ell_{\infty}$ & 85.69 & \cellcolor[HTML]{B7E1CD}67.62 & \cellcolor[HTML]{B7E1CD}29.42 & \cellcolor[HTML]{B7E1CD}43.68 & 68.75 & 46.91 & 24.44 & 52.37 & 24.38 & 2.16 \\
 & FT Single + ALR feature ($\lambda=5$, 10 ep) &  $\ell_2$, StAdv, $\ell_{\infty}$ & 84.03 & \cellcolor[HTML]{B7E1CD}67.64 & \cellcolor[HTML]{B7E1CD}42.03 & \cellcolor[HTML]{B7E1CD} \textbf{\underline{44.36}} & 71.36 & 51.34 & 32.54 & 56.35 & 32.48 & 2.29 \\
 & FT Croce + ALR (10 ep) & $\ell_2$, StAdv, $\ell_{\infty}$ & 86.57 & \cellcolor[HTML]{B7E1CD}67.99 & \cellcolor[HTML]{B7E1CD} \textbf{\underline{61.55}} & \cellcolor[HTML]{B7E1CD}36.59 & 72.16 & \textbf{\underline{55.38}} & 35.68 & 59.57 & 35.52 & 2.87 \\
 & FT Croce + ALR (25 ep) & $\ell_2$, StAdv, $\ell_{\infty}$ & 86.96 & \cellcolor[HTML]{B7E1CD}68.91 & \cellcolor[HTML]{B7E1CD}57.21 & \cellcolor[HTML]{B7E1CD}39.65 & \textbf{\underline{72.22}} & 55.26 & 37.25 & \textbf{\underline{59.5}} & 37.14 & 6.87 \\
  & FT Croce + ALR feature ($\lambda=2$, 10 ep) &  $\ell_2$, StAdv, $\ell_{\infty}$ & 83.13 & \cellcolor[HTML]{B7E1CD}66.91 & \cellcolor[HTML]{B7E1CD}56.76 & \cellcolor[HTML]{B7E1CD}38.66 & 68.57 & 54.11 & 35.95 & 57.73 & 35.76 & 2.82 \\
  & FT Croce + ALR feature ($\lambda=5$, 10 ep) &  $\ell_2$, StAdv, $\ell_{\infty}$ & 84.25 & \cellcolor[HTML]{B7E1CD}68.14 & \cellcolor[HTML]{B7E1CD}57.7 & \cellcolor[HTML]{B7E1CD}39.8 & 70.29 & 55.21& 37.4 & 58.98 & \underline{\textbf{37.21}} & 2.79
  \\
  \hline

\multirow{ 17}{*}{3} & AVG & $\ell_2$, StAdv, $\ell_\infty$, Recolor & 87.77 & \cellcolor[HTML]{B7E1CD} 68.55 & \cellcolor[HTML]{B7E1CD}39.55 & \cellcolor[HTML]{B7E1CD}\textbf{41.97} & \cellcolor[HTML]{B7E1CD}67.93 & 54.5 & 30.39 & 54.5 & 30.39 & 50.54 \\
& MAX & $\ell_2$, StAdv, $\ell_\infty$, Recolor & 84.3 & \cellcolor[HTML]{B7E1CD}57.62 & \cellcolor[HTML]{B7E1CD}52.3 & \cellcolor[HTML]{B7E1CD}41.69 & \cellcolor[HTML]{B7E1CD}65.1 & 54.18 & \textbf{37.44} & 54.18 & \textbf{37.44} & 55.54 \\
& Random & $\ell_2$, StAdv, $\ell_\infty$, Recolor & 86.32 & \cellcolor[HTML]{B7E1CD}65.87 & \cellcolor[HTML]{B7E1CD}47.82 & \cellcolor[HTML]{B7E1CD}35.04 & \cellcolor[HTML]{B7E1CD}68.35 & 54.27 & 30.76 & 54.27 & 30.76 & 12.41 \\ \cdashline{2-13}
& FT MAX (10 ep) & $\ell_2$, StAdv, $\ell_{\infty}$, Recolor & 83.64 & \cellcolor[HTML]{B7E1CD}66.21 & \cellcolor[HTML]{B7E1CD}57.53 & \cellcolor[HTML]{B7E1CD}37.77 & \cellcolor[HTML]{B7E1CD}69.32 & 57.71 & 36.02 & 57.71 & 36.02 & 8.45\\
& FT MAX (25 ep) & $\ell_2$, StAdv, $\ell_{\infty}$, Recolor & 83.9 & \cellcolor[HTML]{B7E1CD}65.72 & \cellcolor[HTML]{B7E1CD}57.84 & \cellcolor[HTML]{B7E1CD}38.37 & \cellcolor[HTML]{B7E1CD}68.84 & 57.69 & 36.87 & 57.69 & 36.87 & 21.44 \\
& FT Croce (10 ep) & $\ell_2$, StAdv, $\ell_{\infty}$, Recolor & 86.64 & \cellcolor[HTML]{B7E1CD} \textbf{\underline{68.76}} & \cellcolor[HTML]{B7E1CD}44.81 & \cellcolor[HTML]{B7E1CD}36.02 & \cellcolor[HTML]{B7E1CD}68.05 & 54.41 & 29.44 & 54.41 & 29.44 & 2.34 \\
 & FT Croce (25 ep) & $\ell_2$, StAdv, $\ell_{\infty}$, Recolor & 87.11 & \cellcolor[HTML]{B7E1CD}67.89 & \cellcolor[HTML]{B7E1CD}49.57 & \cellcolor[HTML]{B7E1CD}35.58 & \cellcolor[HTML]{B7E1CD}67.05 & 55.02 & 31.21 & 55.02 & 31.21 & 5.9 \\
   & FT Single (10 ep) & $\ell_2$, StAdv, $\ell_{\infty}$, Recolor &  90.41& \cellcolor[HTML]{B7E1CD}66.47 & \cellcolor[HTML]{B7E1CD}3.93 & \cellcolor[HTML]{B7E1CD}29.6& \cellcolor[HTML]{B7E1CD}69.03 & 42.26 & 2.49 & 42.26 & 2.49 & 3.11 \\
 & FT Single (25 ep) & $\ell_2$, StAdv, $\ell_{\infty}$, Recolor & \textbf{\underline{90.89}} & \cellcolor[HTML]{B7E1CD}65.14 & \cellcolor[HTML]{B7E1CD}3.02 & \cellcolor[HTML]{B7E1CD}30.32 & \cellcolor[HTML]{B7E1CD}68.54 & 41.75 & 1.92 & 41.75 & 1.92 & 7.41 \\
 & FT Single + ALR (10 ep) & $\ell_2$, StAdv, $\ell_{\infty}$, Recolor & 90.45 & \cellcolor[HTML]{B7E1CD}61.58 & \cellcolor[HTML]{B7E1CD}25.77 & \cellcolor[HTML]{B7E1CD}27.43 & \cellcolor[HTML]{B7E1CD}69.26 & 46.01 & 19.2 & 46.01 & 19.2 & 4.24 \\
 & FT Single + ALR (25 ep) & $\ell_2$, StAdv, $\ell_{\infty}$, Recolor & 90.4 & \cellcolor[HTML]{B7E1CD}57.07 & \cellcolor[HTML]{B7E1CD}24.91 & \cellcolor[HTML]{B7E1CD}22.91 & \cellcolor[HTML]{B7E1CD}67.39 & 43.07 & 17.21 & 43.07 & 17.21 & 9.79 \\
 & FT Single + ALR feature ($\lambda=2$, 10 ep) & $\ell_2$, StAdv, $\ell_{\infty}$, Recolor & 90.15 & \cellcolor[HTML]{B7E1CD}57.89 & \cellcolor[HTML]{B7E1CD}8.75 & \cellcolor[HTML]{B7E1CD}22.86 & \cellcolor[HTML]{B7E1CD}72.27 & 40.44 & 6.61 & 40.44 & 6.61 & 3.94\\
 & FT Single + ALR feature ($\lambda=5$, 10 ep) & $\ell_2$, StAdv, $\ell_{\infty}$, Recolor & 88.44 & \cellcolor[HTML]{B7E1CD}66.03 & \cellcolor[HTML]{B7E1CD}18.88 & \cellcolor[HTML]{B7E1CD}34.17 & \cellcolor[HTML]{B7E1CD}69.35 & 47.11 & 16.1 & 47.11 & 16.1 & 3.76 \\
  & FT Croce + ALR (10 ep) & $\ell_2$, StAdv, $\ell_{\infty}$, Recolor &87.62 & \cellcolor[HTML]{B7E1CD}68.14 & \cellcolor[HTML]{B7E1CD}58.5 & \cellcolor[HTML]{B7E1CD}36.39 & \cellcolor[HTML]{B7E1CD}72.35 & 58.85 & 34.92 & 58.85 & 34.92 & 3.35 \\
 & FT Croce + ALR (25 ep) & $\ell_2$, StAdv, $\ell_{\infty}$, Recolor & 87.05 & \cellcolor[HTML]{B7E1CD} 68.05 & \cellcolor[HTML]{B7E1CD}\underline{\textbf{59.26}} & \cellcolor[HTML]{B7E1CD} 38.38 & \cellcolor[HTML]{B7E1CD}\underline{\textbf{73.42}} & \textbf{\underline{59.78}} & 36.83 & \underline{\textbf{59.78}} & 36.83 & 7.78\\
 & FT Croce + ALR feature ($\lambda=2$, 10 ep) & $\ell_2$, StAdv, $\ell_{\infty}$, Recolor & 84.78 & \cellcolor[HTML]{B7E1CD}67.67 & \cellcolor[HTML]{B7E1CD}53.13 & \cellcolor[HTML]{B7E1CD} \underline{40.25} & \cellcolor[HTML]{B7E1CD}69.99 & 57.76 & 36.3 & 57.76 & 36.3 & 3.04\\
  & FT Croce + ALR feature ($\lambda=5$, 10 ep) & $\ell_2$, StAdv, $\ell_{\infty}$, Recolor &83.94 & \cellcolor[HTML]{B7E1CD}67.28 & \cellcolor[HTML]{B7E1CD}59.21 & \cellcolor[HTML]{B7E1CD}39.38 & \cellcolor[HTML]{B7E1CD}71.67 & 59.38 & \underline{37.15} & 59.38 & \underline{37.15} & 2.91\\
 \hline
\end{tabular}}

\caption{\textbf{Continual Robust Training on CIFAR-10 (Sequence of 4 attacks starting with $\ell_2$).} The learner initially has knowledge of $\ell_2$ attacks and over time, we are sequentially introduced to StAdv, $\ell_\infty$, and ReColor attacks. We report clean accuracy, accuracy on different attack types, and average and union accuracies.  The threat models column represents the set of attacks known to the defender and accuracies on known attacks are highlighted with in green cells.  ``FT" procedures are fine-tuning approaches starting from adversarially trained to $\ell_2$ model (AT) and then sequentially fine-tuning with new attacks for 25 epochs. AVG, MAX, and Random strategies train models from scratch with all attacks for 100 epochs. The ``Avg (known)" and ``Union (known)" columns represent average and union accuracies on attacks known to the defender while ``Avg (all)" and ``Union (all)" columns represent average and union accuracies on all four attacks. Additionally, we report training times for the procedure (non-cumulative) in the ``Time" column.  Best performance out of both training from scratch and fine-tuning are bolded, while best performance when only comparing fine-tuning approaches is underlined.}
\label{tab:main_results_cifar_epochs}
\end{table*}

\begin{table*}[ht]
\centering
\scalebox{0.75}{
\begin{tabular}{|c|l|l|c|cccc|cc|cc|c|}
\hline
 \multicolumn{1}{|c|}{\begin{tabular}[c]{@{}c@{}}Time\\ Step \end{tabular}} &Procedure & Threat Models & Clean & $\ell_\infty$ & StAdv & Recolor & $\ell_2$ & \begin{tabular}[c]{@{}c@{}}Avg\\ (known)\end{tabular} & \begin{tabular}[c]{@{}c@{}}Union\\ (known)\end{tabular} & \begin{tabular}[c]{@{}c@{}}Avg\\ (all)\end{tabular} & \begin{tabular}[c]{@{}c@{}}Union\\ (all)\end{tabular} & Time \\ \hline
\multirow{ 2}{*}{0} & AT & $\ell_\infty$ & \textbf{85.93} & \cellcolor[HTML]{B7E1CD}51.44 & 14.87 & 62.48 & \textbf{59.48} & 51.44 & 51.44 & 47.07 & 11.9 & 7.52 \\
& AT + ALR & $\ell_\infty$ & 83.18 & \cellcolor[HTML]{B7E1CD}\textbf{51.49} & \textbf{34.78} & \textbf{58.15} & 58.15 & \textbf{51.49} & \textbf{51.49} & \textbf{53.27} & \textbf{29.87} & 11.12 \\ \hline

\multirow{ 13}{*}{1}& AVG & $\ell_\infty$, StAdv & \textbf{86.44} & \cellcolor[HTML]{B7E1CD}30.05 & \cellcolor[HTML]{B7E1CD}54.4 & 46.71 & 52.1 & 42.23 & 28.03 & 45.81 & 26.75 & 23.68 \\
& MAX & $\ell_\infty$, StAdv & 82.62 & \cellcolor[HTML]{B7E1CD}44.96 & \cellcolor[HTML]{B7E1CD}53.68 & 64.24 & \textbf{60.85} & 49.32 & 40.8 & 55.93 & 39.81 & 23.68 \\
& Random & $\ell_\infty$, StAdv & 83.15 & \cellcolor[HTML]{B7E1CD}40.86 & \cellcolor[HTML]{B7E1CD}58.37 & 60.53 & 58.17 & 49.61 & 38.95 & 54.48 & 37.64 & 11.70 \\ \cdashline{2-13}
 & FT MAX (10 ep) & $\ell_{\infty}$, StAdv & 81.63 &\cellcolor[HTML]{B7E1CD}44.13 & \cellcolor[HTML]{B7E1CD}57.38 & 66.66 & 60.27 & 50.75 & 41.57 & 57.11 & 40.96 & 3.99 \\
 & FT MAX (25 ep) & $\ell_{\infty}$, StAdv & 81.99 & \cellcolor[HTML]{B7E1CD}44.32 & \cellcolor[HTML]{B7E1CD}57.8 & 66.25 & 60.29 & 51.06 & 41.98 & 57.16 & 41.25 & 9.93 \\
 & FT Croce (10 ep) & $\ell_{\infty}$, StAdv & 82.66 & \cellcolor[HTML]{B7E1CD}44.75 & \cellcolor[HTML]{B7E1CD}54.2 & 65.99 & 60.27 & 49.48 & 39.69 & 56.3 & 39.01 & 2.31 \\

 & FT Croce (25 ep) & $\ell_{\infty}$, StAdv & \underline{83.55} & \cellcolor[HTML]{B7E1CD}45.12 & \cellcolor[HTML]{B7E1CD}53.25 & 66.44 & \underline{60.65} & 49.19 & 39.43 & 56.36 & 38.74 & 5.44 \\
 & FT Single (10 ep) & $\ell_{\infty}$, StAdv & 80.39 & \cellcolor[HTML]{B7E1CD}31.14 & \cellcolor[HTML]{B7E1CD}55.88 & 59.13 & 51.58 & 43.51 & 29.01 & 49.43 & 28.67 & 2.77\\
 & FT Single (25 ep) & $\ell_{\infty}$, StAdv & 79.85 & \cellcolor[HTML]{B7E1CD}31.34 & \cellcolor[HTML]{B7E1CD}54.86 & 58.69 & 51.43 & 43.1 & 29.01 & 49.08 & 28.66 & 6.6 \\
 & FT Single + ALR (10 ep) & $\ell_{\infty}$, StAdv & 82.77 & \cellcolor[HTML]{B7E1CD}35.67 & \cellcolor[HTML]{B7E1CD}57.92 & 68.38 & 54.91 & 46.8 & 33.69 & 54.22 & 33.65 & 3.51 \\
 & FT Single + ALR (25 ep) & $\ell_{\infty}$, StAdv & 81.81 & \cellcolor[HTML]{B7E1CD}35.4 & \cellcolor[HTML]{B7E1CD}59.47 & 68.63 & 54.34 & 47.44 & 33.72 & 54.46 & 33.66 & 8.73 \\
  & FT Croce + ALR (10 ep) & $\ell_{\infty}$, StAdv & 82.94
& \cellcolor[HTML]{B7E1CD}\textbf{\underline{46.39}} & \cellcolor[HTML]{B7E1CD}\textbf{\underline{64.13}} & \textbf{\underline{73.58}} & 59.41 & \textbf{\underline{55.26}} & \textbf{\underline{44.47}} & \underline{\textbf{60.88}} & \underline{\textbf{44.03}}&2.99

 \\
 & FT Croce + ALR (25 ep) & $\ell_{\infty}$, StAdv & 82.3 & \cellcolor[HTML]{B7E1CD}45.89 & \cellcolor[HTML]{B7E1CD}63.76 & 72.8 & 59.56 & 54.82 & 44 & 60.5 & 43.54 & 7.5 \\ \hline

\multirow{ 13}{*}{2} & AVG & $\ell_\infty$, StAdv, Recolor & 88.67 & \cellcolor[HTML]{B7E1CD}39.46 & \cellcolor[HTML]{B7E1CD}47.1 & \cellcolor[HTML]{B7E1CD}66.87 & 57.16 & 51.14 & 32.61 & 52.65 & 32.55 & 39.72 \\
& MAX & $\ell_\infty$, StAdv, Recolor & 83.42 & \cellcolor[HTML]{B7E1CD}44.54 & \cellcolor[HTML]{B7E1CD}53.06 & \cellcolor[HTML]{B7E1CD}67.56 & 60.71 & 55.05 & 40.23 & 56.47 & 40.17 & 47.21 \\
& Random & $\ell_\infty$, StAdv, Recolor & 83.23 & \cellcolor[HTML]{B7E1CD}35.01 & \cellcolor[HTML]{B7E1CD}54.7 & \cellcolor[HTML]{B7E1CD}68.68 & \textbf{62.92} & 52.8 & 32.83 & 55.33 & 32.83 & 13.81 \\ \cdashline{2-13}
& FT MAX (10 ep) & $\ell_{\infty}$, StAdv, Recolor & 81.97 & \cellcolor[HTML]{B7E1CD}44.1 & \cellcolor[HTML]{B7E1CD}57.36 & \cellcolor[HTML]{B7E1CD}68.68 & 60.37 & 56.71 & 41.21 & 57.63 & 41.2 & 6.72\\
& FT MAX (25 ep) & $\ell_{\infty}$, StAdv, Recolor & 82.24 & \cellcolor[HTML]{B7E1CD}44.36 & \cellcolor[HTML]{B7E1CD}58.52 & \cellcolor[HTML]{B7E1CD}68.87 & 60.23 & 57.25 & \textbf{\underline{41.73}} & 57.99 & \textbf{\underline{41.67}} & 16.69 \\
& FT Croce (10 ep) & $\ell_{\infty}$, StAdv, Recolor & 84.98
&\cellcolor[HTML]{B7E1CD}43.32 & \cellcolor[HTML]{B7E1CD}52.45 & \cellcolor[HTML]{B7E1CD}69.46 & 61.04 & 55.08 & 37.05 & 56.57 & 37.01 & 2.53 \\
 & FT Croce (25 ep) & $\ell_{\infty}$, StAdv, Recolor & 84.89 & \cellcolor[HTML]{B7E1CD}44.66 & \cellcolor[HTML]{B7E1CD}51.6 & \cellcolor[HTML]{B7E1CD}68.86 & 61.59 & 55.04 & 38.02 & 56.68 & 37.96 & 6.3 \\
 & FT Single (10 ep) & $\ell_{\infty}$, StAdv, Recolor & 90.55 & \cellcolor[HTML]{B7E1CD}25.27 & \cellcolor[HTML]{B7E1CD}12.77 & \cellcolor[HTML]{B7E1CD}\underline{\textbf{74.01}} & 48.99 & 37.35 & 10.85 & 40.26 & 10.85 & 4.35\\
 & FT Single (25 ep) & $\ell_{\infty}$, StAdv, Recolor & \textbf{90.24} & \cellcolor[HTML]{B7E1CD}33.94 & \cellcolor[HTML]{B7E1CD}13.43 & \cellcolor[HTML]{B7E1CD}73.23 & 53.51 & 40.2 & 10.67 & 43.53 & 10.64 & 7.83 \\
 & FT Single + ALR (10 ep) & $\ell_{\infty}$, StAdv, Recolor & 88.38 & \cellcolor[HTML]{B7E1CD}38.62 & \cellcolor[HTML]{B7E1CD}24.87 & \cellcolor[HTML]{B7E1CD}72.69 & 56.66 & 45.39 & 19.2 & 48.21 & 19.19 & 3.41 \\
 & FT Single + ALR (25 ep) & $\ell_{\infty}$, StAdv, Recolor & 89.38 & \cellcolor[HTML]{B7E1CD}33.64 & \cellcolor[HTML]{B7E1CD}20.91 & \cellcolor[HTML]{B7E1CD}73.52 & 53.36 & 42.69 & 17.39 & 45.36 & 17.38 & 9.87 \\
  & FT Croce + ALR (10 ep) & $\ell_{\infty}$, StAdv, Recolor& 84.3 & \cellcolor[HTML]{B7E1CD}44.39 & \cellcolor[HTML]{B7E1CD}58.86 & \cellcolor[HTML]{B7E1CD}71.67 & 60.42 & 58.31 & 40.82 & 58.84 & 40.69 & 3.52 \\
 & FT Croce + ALR (25 ep) & $\ell_{\infty}$, StAdv, Recolor & 84.69 & \cellcolor[HTML]{B7E1CD}\textbf{\underline{44.96}} & \cellcolor[HTML]{B7E1CD}\textbf{\underline{59.53}} & \cellcolor[HTML]{B7E1CD}73.54 & \underline{61.73} & \underline{\textbf{59.34}} & 41.39 & \textbf{\underline{59.94}} & 41.22 & 8.21 \\ \hline

\multirow{ 13}{*}{3}& AVG & $\ell_\infty$, StAdv, Recolor, $\ell_2$ & 87.77 & \cellcolor[HTML]{B7E1CD}41.97 & \cellcolor[HTML]{B7E1CD}39.55 & \cellcolor[HTML]{B7E1CD}67.93 & \cellcolor[HTML]{B7E1CD}68.55 & 54.5 & 30.39 & 54.5 & 30.39 & 50.54 \\
& MAX & $\ell_\infty$, StAdv, Recolor, $\ell_2$ & 84.3 & \cellcolor[HTML]{B7E1CD}41.69 & \cellcolor[HTML]{B7E1CD}52.3 & \cellcolor[HTML]{B7E1CD}65.1 & \cellcolor[HTML]{B7E1CD}57.62 & 54.18 & 37.44 & 54.18 & 37.44 & 55.54 \\
& Random & $\ell_\infty$, StAdv, Recolor, $\ell_2$ & 86.32 & \cellcolor[HTML]{B7E1CD}35.04 & \cellcolor[HTML]{B7E1CD}47.82 & \cellcolor[HTML]{B7E1CD}68.35 & \cellcolor[HTML]{B7E1CD}65.87 & 54.27 & 30.76 & 54.27 & 30.76 & 12.41 \\ \cdashline{2-13}
 & FT MAX (10 ep) & $\ell_{\infty}$, StAdv, Recolor, $\ell_{2}$ & 82.27
 & \cellcolor[HTML]{B7E1CD}44.21 & \cellcolor[HTML]{B7E1CD}58.13 & \cellcolor[HTML]{B7E1CD}69.08 & \cellcolor[HTML]{B7E1CD}60.7 & 58.03 & \textbf{\underline{41.48}} & 58.03 & \textbf{\underline{41.48}} & 7.9\\
 & FT MAX (25 ep) & $\ell_{\infty}$, StAdv, Recolor, $\ell_{2}$ & 82.6 & \cellcolor[HTML]{B7E1CD}43.84 & \cellcolor[HTML]{B7E1CD}57.75 & \cellcolor[HTML]{B7E1CD}68.84 & \cellcolor[HTML]{B7E1CD}60.23 & 57.66 & 41.19 & 57.66 & 41.19 & 19.74 \\
   & FT Croce (10 ep) & $\ell_{\infty}$, StAdv, Recolor, $\ell_{2}$ & 85.11 & \cellcolor[HTML]{B7E1CD}44.71 & \cellcolor[HTML]{B7E1CD}50.32 & \cellcolor[HTML]{B7E1CD}68.39 & \cellcolor[HTML]{B7E1CD}63.29 & 56.68 & 37.23 & 56.68 & 37.23 & 2.37 \\

 & FT Croce (25 ep) & $\ell_{\infty}$, StAdv, Recolor, $\ell_{2}$ & 85.33 & \cellcolor[HTML]{B7E1CD}43.8 & \cellcolor[HTML]{B7E1CD}50.28 & \cellcolor[HTML]{B7E1CD}68.77 & \cellcolor[HTML]{B7E1CD}63.17 & 56.51 & 36.77 & 56.51 & 36.77 & 5.95 \\ 
 & FT Single (10 ep) & $\ell_{\infty}$, StAdv, Recolor, $\ell_{2}$ &88.49 & \cellcolor[HTML]{B7E1CD}\underline{\textbf{44.93}} & \cellcolor[HTML]{B7E1CD}18.06 & \cellcolor[HTML]{B7E1CD}65.96 & \cellcolor[HTML]{B7E1CD}67.56 & 49.13 & 15.78 & 49.13 & 15.78 & 1.63 \\
 & FT Single (25 ep) & $\ell_{\infty}$, StAdv, Recolor, $\ell_{2}$ & \textbf{\underline{89.3}} & \cellcolor[HTML]{B7E1CD}42.72 & \cellcolor[HTML]{B7E1CD}11.85 & \cellcolor[HTML]{B7E1CD}60.27 & \cellcolor[HTML]{B7E1CD}\underline{\textbf{69.12}} & 45.99 & 10.71 & 45.99 & 10.71 & 4.07 \\
 & FT Single + ALR (10 ep) & $\ell_{\infty}$, StAdv, Recolor, $\ell_{2}$ & 88.14 & \cellcolor[HTML]{B7E1CD}41.52 & \cellcolor[HTML]{B7E1CD}26.06 & \cellcolor[HTML]{B7E1CD}61.97 & \cellcolor[HTML]{B7E1CD}68.77 & 49.58 & 24.19 & 49.58 & 24.19 & 2.52\\
 & FT Single + ALR (25 ep) & $\ell_{\infty}$, StAdv, Recolor, $\ell_{2}$ & 87.8 & \cellcolor[HTML]{B7E1CD}40.78 & \cellcolor[HTML]{B7E1CD}28.34 & \cellcolor[HTML]{B7E1CD}59.47 & \cellcolor[HTML]{B7E1CD}68.32 & 49.23 & 25.92 & 49.23 & 25.92 & 5.89 \\
  & FT Croce + ALR (10 ep) & $\ell_{\infty}$, StAdv, Recolor, $\ell_{2}$ &  84.56
 & \cellcolor[HTML]{B7E1CD}42.19 & \cellcolor[HTML]{B7E1CD}55.55 & \cellcolor[HTML]{B7E1CD}69.95 & \cellcolor[HTML]{B7E1CD}60.69 & 57.1 & 38.24 & 57.1 & 38.24 & 3.4 \\
 & FT Croce + ALR (25 ep) & $\ell_{\infty}$, StAdv, Recolor, $\ell_{2}$ & 84.1 & \cellcolor[HTML]{B7E1CD}43.32 & \cellcolor[HTML]{B7E1CD}\textbf{\underline{58.2}} & \cellcolor[HTML]{B7E1CD}\textbf{\underline{72.09}} & \cellcolor[HTML]{B7E1CD}61.96 & \textbf{\underline{58.89}} & 39.97 & \textbf{\underline{58.89}} & 39.97 & 8.28 \\ \hline
\end{tabular}}
\caption{\textbf{Continual Robust Training on CIFAR-10 (Sequence of 4 attacks starting with $\ell_\infty$).} The learner initially has knowledge of $\ell_\infty$ attacks and over time, we are sequentially introduced to StAdv, ReColor, and $\ell_2$ attacks. We report clean accuracy, accuracy on different attack types, and average and union accuracies.  The threat models column represents the set of attacks known to the defender and accuracies on known attacks are highlighted with in green cells.  ``FT" procedures are fine-tuning approaches starting from adversarially trained to $\ell_\infty$ model (AT) and then sequentially fine-tuning with new attacks for 25 epochs. AVG, MAX, and Random strategies train models from scratch with all attacks for 100 epochs. The ``Avg (known)" and ``Union (known)" columns represent average and union accuracies on attacks known to the defender while ``Avg (all)" and ``Union (all)" columns represent average and union accuracies on all four attacks. Additionally, we report training times for the procedure (non-cumulative) in the ``Time" column.  Best performance out of both training from scratch and fine-tuning are bolded, while best performance when only comparing fine-tuning approaches is underlined.}
\label{app:main_results_cifar_epochs_2}
\end{table*}

\begin{table*}[]
\scalebox{0.75}{
\begin{tabular}{|c|l|l|c|cccc|cc|cc|c|}
\hline
\begin{tabular}[c]{@{}c@{}}Time\\ Step\end{tabular} & Procedure & Threat Models & Clean & $\ell_2$ & StAdv & $\ell_\infty$ & Recolor & \begin{tabular}[c]{@{}c@{}}Avg\\ (known)\end{tabular} & \begin{tabular}[c]{@{}c@{}}Union\\ (known)\end{tabular} & \begin{tabular}[c]{@{}c@{}}Avg\\ (all)\end{tabular} & \begin{tabular}[c]{@{}c@{}}Union\\ (all)\end{tabular} & \begin{tabular}[c]{@{}c@{}}Time\\ (hrs)\end{tabular} \\ \hline
\multirow{2}{*}{0} & AT & $\ell_2$ & \textbf{90.22} & \cellcolor[HTML]{B7E1CD}83.95 & 10.65 & 7.67 & 49.22 & 83.95 & 83.95 & 37.87 & 3.16 & 1.71 \\
 & AT + ALR & $\ell_2$ & 89.76 & \cellcolor[HTML]{B7E1CD}\textbf{84.41} & \textbf{28.23} & \textbf{25.22} & \textbf{54.70} & \textbf{84.41} & \textbf{84.41} & \textbf{48.14} & \textbf{18.01} & 2.15 \\ \hline
 \multirow{13}{*}{1}& AVG & $\ell_2$, StAdv & 84.56 & \cellcolor[HTML]{B7E1CD}77.68 & \cellcolor[HTML]{B7E1CD}74.32 & 7.57 & 31.33 & 76 & 73.68 & 47.73 & 7.44 & 3.58 \\
 & MAX & $\ell_2$, StAdv & 85.22 & \cellcolor[HTML]{B7E1CD}76.87 & \cellcolor[HTML]{B7E1CD}\textbf{77.63} & 4.94 & 27.61 & \textbf{77.25} & \textbf{75.57} & 46.76 & 4.76 & 3.52 \\
 & Random & $\ell_2$, StAdv & 85.71 & \cellcolor[HTML]{B7E1CD}77.55 & \cellcolor[HTML]{B7E1CD}74.32 & 5.78 & 29.61 & 75.94 & 73.55 & 46.82 & 5.53 & 2.58 \\\cdashline{2-13}
  & FT MAX (10 ep) & $\ell_{2}$, StAdv & 83.92 & \cellcolor[HTML]{B7E1CD}77.5 & \cellcolor[HTML]{B7E1CD}69.02 & 10.78 & 35.77 & 73.26 & 68.89 & 48.27 & 10.45 & 0.61 \\
  & FT MAX (25 ep) & $\ell_{2}$, StAdv & 84.56 &\cellcolor[HTML]{B7E1CD}77.73 & \cellcolor[HTML]{B7E1CD}69.35 & 9.76 & 36.15 & 73.54 & 69.1 & 48.25 & 9.43 & 1.44 \\
  & FT Croce (10 ep) & $\ell_{2}$, StAdv & 85.07 & \cellcolor[HTML]{B7E1CD}78.62 & \cellcolor[HTML]{B7E1CD}67.52 & 10.57 & 38.34 & 73.07 & 67.31 & 48.76 & 10.29 & 0.4  \\
 & FT Croce (25 ep) & $\ell_{2}$, StAdv & \textbf{\underline{86.37}} & \cellcolor[HTML]{B7E1CD}\underline{\textbf{79.67}} & \cellcolor[HTML]{B7E1CD}69.32 & 9.81 & 38.27 & 74.5 & 69.17 & 49.27 & 9.63 & 0.98 \\
 & FT Single (10 ep) & $\ell_2$, StAdv & 84.08 & \cellcolor[HTML]{B7E1CD}77.86 & \cellcolor[HTML]{B7E1CD}68.31 & 10.83 & 36.97 & 73.08 & 68.13 & 48.49 & 10.45 & 0.51 \\
 & FT Single (25 ep) & $\ell_{2}$, StAdv & 85.63 & \cellcolor[HTML]{B7E1CD}78.39 & \cellcolor[HTML]{B7E1CD}\underline{72.31} & 7.57 & 35.31 & \underline{75.35} & \underline{72.08} & 48.39 & 7.36 & 1.15 \\
 & FT Single + ALR (10 ep) & $\ell_{2}$, StAdv & 83.8 & \cellcolor[HTML]{B7E1CD}77.94 & \cellcolor[HTML]{B7E1CD}71.62 & \underline{\textbf{20.71}} & 43.13 & 74.78 & 71.34 & \underline{\textbf{53.35}} & \underline{\textbf{20.13}} & 0.58 \\
 & FT Single + ALR (25 ep) & $\ell_{2}$, StAdv & 83.9 & \cellcolor[HTML]{B7E1CD}77.78 & \cellcolor[HTML]{B7E1CD}71.97 & 17.35 & 38.39 & 74.88 & 71.59 & 51.38 & 16.76 & 1.44  \\
 & FT Croce + ALR (10 ep) & $\ell_{2}$, StAdv & 85.04 & \cellcolor[HTML]{B7E1CD}79.54 &\cellcolor[HTML]{B7E1CD}69.99 & 18.68 & 42.93 & 74.76 & 69.89 & 52.78 & 18.09 & 0.51 \\
 & FT Croce + ALR (25 ep) & $\ell_{2}$, StAdv & 85.07 & \cellcolor[HTML]{B7E1CD}79.39 & \cellcolor[HTML]{B7E1CD}68 & 19.57 & \underline{\textbf{43.67}} & 73.69 & 67.97 & 52.66 & 19.16 & 1.24 \\ \hline
\multirow{13}{*}{2} & AVG & $\ell_2$, StAdv, $\ell_\infty$ & \textbf{86.62} & \cellcolor[HTML]{B7E1CD}\textbf{84.92} & \cellcolor[HTML]{B7E1CD}68.89 & \cellcolor[HTML]{B7E1CD}50.57 & 66.98 & \textbf{68.13} & 49.17 & \textbf{67.84} & 47.82 & 10.51 \\
 & MAX & $\ell_2$, StAdv, $\ell_\infty$ & 80.36 & \cellcolor[HTML]{B7E1CD}78.09 & \cellcolor[HTML]{B7E1CD}68.38 & \cellcolor[HTML]{B7E1CD}\textbf{52.61} & 67.29 & 66.36 & \textbf{51.77} & 66.59 & \textbf{50.37} & 11.96 \\
 & Random & $\ell_2$, StAdv, $\ell_\infty$ & 84.92 & \cellcolor[HTML]{B7E1CD}83.06 & \cellcolor[HTML]{B7E1CD}68.76 & \cellcolor[HTML]{B7E1CD}49.50 & 66.11 & 67.11 & 48.15 & 66.86 & 46.60 & 4.29 \\\cdashline{2-13}
 & FT MAX (10 ep) & $\ell_{2}$, StAdv, $\ell_{\infty}$ & 81.76 & \cellcolor[HTML]{B7E1CD}76.69 & \cellcolor[HTML]{B7E1CD}71.03 & \cellcolor[HTML]{B7E1CD}28.31 & 54.32 & 58.68 & 28.31 & 57.59 & 27.69 & 0.67\\
 & FT MAX (25 ep) & $\ell_{2}$, StAdv, $\ell_{\infty}$ & 82.04 & \cellcolor[HTML]{B7E1CD}77.86 & \cellcolor[HTML]{B7E1CD}69.02 & \cellcolor[HTML]{B7E1CD}42.83 & 66.9 & 63.24 & 42.37 & 64.15 & 41.86 & 1.71 \\
  & FT Croce (10 ep) & $\ell_{2}$, StAdv, $\ell_{\infty}$ & 83.59 & \cellcolor[HTML]{B7E1CD}78.8 & \cellcolor[HTML]{B7E1CD}69.53 & \cellcolor[HTML]{B7E1CD}34.17 & 61.5 & 60.83 & 34.06 & 61 & 33.61 & 0.3\\
 & FT Croce (25 ep) & $\ell_{2}$, StAdv, $\ell_{\infty}$ & \underline{85.22} & \cellcolor[HTML]{B7E1CD}\underline{81.02} & \cellcolor[HTML]{B7E1CD}69.58 & \cellcolor[HTML]{B7E1CD}39.92 & 64.79 & 63.51 & 39.59 & 63.83 & 39.03 & 0.73 \\
 & FT Single (10 ep) & $\ell_{2}$, StAdv, $\ell_{\infty}$ & 82.06 & \cellcolor[HTML]{B7E1CD}77.25 & \cellcolor[HTML]{B7E1CD}73.1 & \cellcolor[HTML]{B7E1CD}27.21 & 57.4 & 59.18 & 27.21 & 58.74 & 26.9 & 0.22\\
 & FT Single (25 ep) & $\ell_{2}$, StAdv, $\ell_{\infty}$ & 82.04 & \cellcolor[HTML]{B7E1CD}77.96 & \cellcolor[HTML]{B7E1CD}70.42 & \cellcolor[HTML]{B7E1CD}41.15 & 66.09 & 63.18 & 40.92 & 63.9 & 40.46 & 0.54 \\
  & FT Single + ALR (10 ep)& $\ell_{2}$, StAdv, $\ell_{\infty}$ & 81.38 & \cellcolor[HTML]{B7E1CD}77.89 & \cellcolor[HTML]{B7E1CD}71.8 & \cellcolor[HTML]{B7E1CD}46.68 & \underline{\textbf{72.13}} & \underline{65.45} & 46.5 & \underline{67.12} & 46.14 & 0.31 \\
 & FT Single + ALR (25 ep)& $\ell_{2}$, StAdv, $\ell_{\infty}$ & 80.92 & \cellcolor[HTML]{B7E1CD}77.43 & \cellcolor[HTML]{B7E1CD}70.78 & \cellcolor[HTML]{B7E1CD}\underline{47.16} & 70.6 & 65.12 & \underline{46.96} & 66.49 & \underline{46.62} & 0.79 \\
 & FT Croce + ALR (10 ep) & $\ell_{2}$, StAdv, $\ell_{\infty}$ & 83.95  & \cellcolor[HTML]{B7E1CD}79.57 &  \cellcolor[HTML]{B7E1CD}69.22 & \cellcolor[HTML]{B7E1CD}37.96 & 59.77 & 62.25 & 37.86 &  61.63 & 36.99 & 0.40\\
 & FT Croce + ALR (25 ep) & $\ell_{2}$, StAdv, $\ell_{\infty}$ & 83.11 & \cellcolor[HTML]{B7E1CD}79.24 & \cellcolor[HTML]{B7E1CD}\underline{\textbf{72.38}} & \cellcolor[HTML]{B7E1CD}36.61 & 60.18 & 62.74 & 36.59 & 62.1 & 36.15 & 1.01 \\
  \hline
\multirow{13}{*}{3} & AVG & $\ell_2$, StAdv, $\ell_\infty$, Recolor & \textbf{87.67} & \cellcolor[HTML]{B7E1CD}\textbf{85.66} & \cellcolor[HTML]{B7E1CD}66.06 & \cellcolor[HTML]{B7E1CD}50.42 & \cellcolor[HTML]{B7E1CD}75.90 & 69.51 & 47.90 & 69.51 & 47.90 & 13.79 \\
 & MAX & $\ell_2$, StAdv, $\ell_\infty$, Recolor & 83.26 & \cellcolor[HTML]{B7E1CD}81.22 & \cellcolor[HTML]{B7E1CD}70.70 & \cellcolor[HTML]{B7E1CD}\textbf{56.94} & \cellcolor[HTML]{B7E1CD}74.80 & \textbf{70.92} & \textbf{55.31} & \textbf{70.92} & \textbf{55.31} & 14.60 \\
 & Random & $\ell_2$, StAdv, $\ell_\infty$, Recolor & 86.55 & \cellcolor[HTML]{B7E1CD}84.64 & \cellcolor[HTML]{B7E1CD}66.52 & \cellcolor[HTML]{B7E1CD}47.29 & \cellcolor[HTML]{B7E1CD}74.93 & 68.34 & 45.71 & 68.34 & 45.71 & 9.61 \\\cdashline{2-13}
 & FT MAX (10 ep) & $\ell_{2}$, StAdv, $\ell_{\infty}$, Recolor & 81.99 & \cellcolor[HTML]{B7E1CD}77.78 & \cellcolor[HTML]{B7E1CD}68.28 & \cellcolor[HTML]{B7E1CD}41.83 & \cellcolor[HTML]{B7E1CD}69.91 & 64.45 & 41.4 & 64.45 & 41.4 & 1.31
 \\
  & FT MAX (25 ep) & $\ell_{2}$, StAdv, $\ell_{\infty}$, Recolor & 82.78 & \cellcolor[HTML]{B7E1CD}79.21 & \cellcolor[HTML]{B7E1CD}\textbf{\underline{70.83}} & \cellcolor[HTML]{B7E1CD}45.15 & \cellcolor[HTML]{B7E1CD}71.39 & 66.64 & \underline{44.76} & 66.64 & \underline{44.76} & 3.6 \\
   & FT Croce (10 ep) & $\ell_{2}$, StAdv, $\ell_{\infty}$, Recolor & 84.87 & \cellcolor[HTML]{B7E1CD}80.38 & \cellcolor[HTML]{B7E1CD}66.68 & \cellcolor[HTML]{B7E1CD}36.82 &\cellcolor[HTML]{B7E1CD}68.61 & 63.12 & 36.31 & 63.12 & 36.31 & 0.45\\ 
 & FT Croce (25 ep) & $\ell_{2}$, StAdv, $\ell_{\infty}$, Recolor & 86.32 & \cellcolor[HTML]{B7E1CD}82.11 & \cellcolor[HTML]{B7E1CD}68.79 & \cellcolor[HTML]{B7E1CD}41.27 & \cellcolor[HTML]{B7E1CD}72.41 & 66.15 & 40.69 & 66.15 & 40.69 & 1.2 \\
 & FT Single (10 ep)& $\ell_{2}$, StAdv, $\ell_{\infty}$, Recolor & 86.27 & \cellcolor[HTML]{B7E1CD}81.35 & \cellcolor[HTML]{B7E1CD}54.73 & \cellcolor[HTML]{B7E1CD}23.59 & \cellcolor[HTML]{B7E1CD}70.17 & 57.46 & 22.55 & 57.46 & 22.55 & 0.71\\
 & FT Single (25 ep)& $\ell_{2}$, StAdv, $\ell_{\infty}$, Recolor & 85.1 & \cellcolor[HTML]{B7E1CD}80.48 & \cellcolor[HTML]{B7E1CD}58.17 & \cellcolor[HTML]{B7E1CD}36.38 & \cellcolor[HTML]{B7E1CD}70.62 & 61.41 & 34.45 & 61.41 & 34.45 & 2.03 \\
 & FT Single + ALR (10 ep) &  $\ell_{2}$, StAdv, $\ell_{\infty}$, Recolor & 85.3 & \cellcolor[HTML]{B7E1CD}81.04 & \cellcolor[HTML]{B7E1CD}49.35 & \cellcolor[HTML]{B7E1CD}40.48 & \cellcolor[HTML]{B7E1CD}74.8 & 61.42 & 35.62 & 61.42 & 35.62 & 0.85\\
 & FT Single + ALR (25 ep)& $\ell_{2}$, StAdv, $\ell_{\infty}$, Recolor & \underline{86.78} & \cellcolor[HTML]{B7E1CD}\underline{82.8} & \cellcolor[HTML]{B7E1CD}47.82 & \cellcolor[HTML]{B7E1CD}33.12 & \cellcolor[HTML]{B7E1CD}\textbf{\underline{77.58}} & 60.33 & 29.17 & 60.33 & 29.17 & 2.38 \\
 & FT Croce + ALR (10 ep) & $\ell_{2}$, StAdv, $\ell_{\infty}$, Recolor & 85.3 & \cellcolor[HTML]{B7E1CD}81.3 & \cellcolor[HTML]{B7E1CD}69.35 & \cellcolor[HTML]{B7E1CD}43.13 & \cellcolor[HTML]{B7E1CD}70.85 & 66.16 & 42.62 & 66.16 & 42.62 & 0.53
\\
 & FT Croce + ALR (25 ep) & $\ell_{2}$, StAdv, $\ell_{\infty}$, Recolor & 85.81 & \cellcolor[HTML]{B7E1CD}81.76 & \cellcolor[HTML]{B7E1CD}67.13 & \cellcolor[HTML]{B7E1CD}\underline{45.38} & \cellcolor[HTML]{B7E1CD}73.02 & \underline{66.82} & 44.56 & \underline{66.82} & 44.56 & 1.36
 \\ \hline
\end{tabular}
}\caption{\textbf{Continual Robust Training on ImageNette (Sequence of 4 attacks starting with $\ell_2$).}}
\label{app:main_results_imagenette_epochs_1}
\end{table*}

\begin{table*}[]
\scalebox{0.75}{
\begin{tabular}{|c|l|l|c|cccc|cc|cc|c|}
\hline
\begin{tabular}[c]{@{}c@{}}Time\\ Step\end{tabular} & Procedure & Threat Models & Clean & $\ell_\infty$ & StAdv & Recolor & $\ell_2$ & \begin{tabular}[c]{@{}c@{}}Avg\\ (known)\end{tabular} & \begin{tabular}[c]{@{}c@{}}Union\\ (known)\end{tabular} & \begin{tabular}[c]{@{}c@{}}Avg\\ (all)\end{tabular} & \begin{tabular}[c]{@{}c@{}}Union\\ (all)\end{tabular} & \begin{tabular}[c]{@{}c@{}}Time\\ (hrs)\end{tabular} \\ \hline
 & AT & $\ell_\infty$ & \textbf{82.52} & \cellcolor[HTML]{B7E1CD}56.94 & \textbf{61.32} & 71.62 & \textbf{78.39} & 56.94 & 56.94 & \textbf{67.07} & 50.32 & 1.70 \\
\multirow{-2}{*}{0} & AT + ALR & $\ell_\infty$ & 81.52 & \cellcolor[HTML]{B7E1CD}\textbf{59.62} & 60.51 & \textbf{73.50} & 72.69 & \textbf{59.62} & \textbf{59.62} & 66.58 & \textbf{52.92} & 2.67 \\ \hline
 \multirow{13}{*}{1} & AVG & $\ell_\infty$, StAdv & \textbf{85.78} & \cellcolor[HTML]{B7E1CD}53.30 & \cellcolor[HTML]{B7E1CD}\textbf{75.69} & 67.69 & \textbf{81.96} & 64.5 & 53.02 & 69.66 & 51.11 & 5.87 \\
 & MAX & $\ell_\infty$, StAdv & 83.77 & \cellcolor[HTML]{B7E1CD}\textbf{58.04 }& \cellcolor[HTML]{B7E1CD}70.04 & \textbf{72.76} & 80.38 & 64.04 & \textbf{56.54} & \textbf{70.31} & 55.26 & 6.11 \\
 & Random & $\ell_\infty$, StAdv & 83.34 & \cellcolor[HTML]{B7E1CD}52.23 & \cellcolor[HTML]{B7E1CD}73.76 & 67.85 & 79.87 & 62.99 & 51.77 & 68.43 & 50.39 & 2.44 \\\cdashline{2-13}
 & FT MAX (10 ep) & $\ell_{\infty}$, StAdv & 82.27 & \cellcolor[HTML]{B7E1CD}55.03 & \cellcolor[HTML]{B7E1CD}70.52 & 69.78 & 78.27 & 62.78 & 54.17 & 68.4 & 52.66 & 0.62 \\
 & FT MAX (25 ep) & $\ell_{\infty}$, StAdv & 82.57 & \cellcolor[HTML]{B7E1CD}55.46 & \cellcolor[HTML]{B7E1CD}71.75 & 69.94 & 78.73 & 63.61 & 54.85 & 68.97 & 53.22 & 1.48 \\
 & FT Croce (10 ep)& $\ell_{\infty}$, StAdv & 82.29 & \cellcolor[HTML]{B7E1CD}54.62 & \cellcolor[HTML]{B7E1CD}69.2 & 68.87 & 78.32 & 61.91 & 53.35 & 67.75 & 51.9 & 0.37\\
 & FT Croce (25 ep) & $\ell_{\infty}$, StAdv & 83.67 & \cellcolor[HTML]{B7E1CD}54.27 & \cellcolor[HTML]{B7E1CD}71.57 & 69.07 & \underline{79.54} & 62.92 & 53.58 & 68.61 & 52.2 & 0.86 \\
 & FT Single (10 ep) & $\ell_{\infty}$, StAdv & 83.06 & \cellcolor[HTML]{B7E1CD}50.52 & \cellcolor[HTML]{B7E1CD}71.52 & 65.43 & 78.78 & 61.02 & 49.96 & 66.56 & 48.18 & 0.51 \\
 & FT Single (25 ep) & $\ell_{\infty}$, StAdv & \underline{84.00} & \cellcolor[HTML]{B7E1CD}43.59 & \cellcolor[HTML]{B7E1CD}\underline{73.68} & 58.14 & 78.85 & 58.64 & 43.46 & 63.57 & 41.27 & 1.16 \\
 & FT Single + ALR (10 ep) & $\ell_{\infty}$, StAdv & 82.19 & \cellcolor[HTML]{B7E1CD}42.7 & \cellcolor[HTML]{B7E1CD}73.17 & 60.97 & 77.55 & 57.94 & 42.6 & 63.6 & 41.27 & 0.58\\
 & FT Single + ALR (25 ep) & $\ell_{\infty}$, StAdv & 81.4 & \cellcolor[HTML]{B7E1CD}51.8 & \cellcolor[HTML]{B7E1CD}69.35 & 66.32 & 76.54 & 60.57 & 51.08 & 66 & 50.04 & 1.47 \\
 & FT Croce + ALR (10 ep) & $\ell_{\infty}$, StAdv & 82.62 & \cellcolor[HTML]{B7E1CD}\underline{57.71} & \cellcolor[HTML]{B7E1CD}70.11 & \underline{72.41} & 77.89 & 63.91 & \underline{\textbf{56.54}} & 69.53 & 55.62 & 0.49\\
 & FT Croce + ALR (25 ep) & $\ell_{\infty}$, StAdv & 82.37 & \cellcolor[HTML]{B7E1CD}57.55 & \cellcolor[HTML]{B7E1CD}71.64 & 70.93 & 78.14 & \underline{\textbf{64.6}} & \underline{\textbf{56.54}} & \underline{69.57} & \underline{\textbf{55.64}} & 1.11 \\ \hline
 \multirow{13}{*}{2}& AVG & $\ell_\infty$, StAdv, Recolor & \textbf{86.39} & \cellcolor[HTML]{B7E1CD}51.80 & \cellcolor[HTML]{B7E1CD}\textbf{73.81} & \cellcolor[HTML]{B7E1CD}\textbf{77.99} & \textbf{83.13} & \textbf{67.86} & 51.31 & \textbf{71.68} & 51.31 & 11.57 \\
 & MAX & $\ell_\infty$, StAdv, Recolor & 81.20 & \cellcolor[HTML]{B7E1CD}54.55& \cellcolor[HTML]{B7E1CD}68.64 & \cellcolor[HTML]{B7E1CD}72.82 & 78.01 & 65.33 & 53.12 & 68.5 & 53.12 & 13.55 \\
 & Random & $\ell_\infty$, StAdv, Recolor & 86.29 & \cellcolor[HTML]{B7E1CD}50.96 & \cellcolor[HTML]{B7E1CD}72.28 & \cellcolor[HTML]{B7E1CD}76.59 & 82.96 & 66.61 & 50.27 & 70.69 & 50.27 & 4.90 \\\cdashline{2-13}
 & FT MAX (10 ep) & $\ell_{\infty}$, StAdv, Recolor & 82.34 & \cellcolor[HTML]{B7E1CD}55.34 & \cellcolor[HTML]{B7E1CD}71.34 & \cellcolor[HTML]{B7E1CD}72.87 & 78.22 & 66.51 & 54.04 & 69.44 & 54.04 &1.21 \\
& FT MAX (25 ep) & $\ell_{\infty}$, StAdv, Recolor & 83.75 & \cellcolor[HTML]{B7E1CD}55.29 & \cellcolor[HTML]{B7E1CD}\underline{72.56} & \cellcolor[HTML]{B7E1CD}74.83 & 79.97 & \underline{67.56} & 54.27 & \underline{70.66} & 54.27 & 3.03 \\
& FT Croce (10 ep) & $\ell_{\infty}$, StAdv, Recolor & 84.28 & \cellcolor[HTML]{B7E1CD}54.01 & \cellcolor[HTML]{B7E1CD}69.96 & \cellcolor[HTML]{B7E1CD}72.56 & 79.72 & 65.51 & 52.13 & 69.06 & 52.13 & 0.5\\
 & FT Croce (25 ep) & $\ell_{\infty}$, StAdv, Recolor & 84.05 & \cellcolor[HTML]{B7E1CD}52.99 & \cellcolor[HTML]{B7E1CD}70.47 & \cellcolor[HTML]{B7E1CD}73.07 & 80.33 & 65.51 & 51.92 & 69.22 & 51.92 & 1.23 \\
 & FT Single (10 ep) & $\ell_{\infty}$, StAdv, Recolor & 83.77 & \cellcolor[HTML]{B7E1CD}53.61 & \cellcolor[HTML]{B7E1CD}65.38 & \cellcolor[HTML]{B7E1CD}73.35 & 79.36 & 64.11 & 50.7 & 67.92 & 50.7 & 0.75\\
 & FT Single (25 ep) & $\ell_{\infty}$, StAdv, Recolor & \underline{85.07} & \cellcolor[HTML]{B7E1CD}48.2 & \cellcolor[HTML]{B7E1CD}65.86 & \cellcolor[HTML]{B7E1CD}\underline{75.41} & \underline{80.41} & 63.16 & 46.34 & 67.47 & 46.34 & 1.88 \\
 & FT Single + ALR (10 ep) & $\ell_{\infty}$, StAdv, Recolor & 84.94	& \cellcolor[HTML]{B7E1CD}50.8 & \cellcolor[HTML]{B7E1CD}65.71 & \cellcolor[HTML]{B7E1CD}76.33 & 80.15 & 64.28 & 48.31 & 68.25 & 48.31 & 0.88\\
 & FT Single + ALR (25 ep) & $\ell_{\infty}$, StAdv, Recolor & 82.09 & \cellcolor[HTML]{B7E1CD}55.46 & \cellcolor[HTML]{B7E1CD}66.27 & \cellcolor[HTML]{B7E1CD}73.63 & 77.76 & 65.12 & 52.89 & 68.28 & 52.89 & 2.14 \\
 & FT Croce + ALR (10 ep) & $\ell_{\infty}$, StAdv, Recolor & 82.42 & \cellcolor[HTML]{B7E1CD}\underline{\textbf{55.75}} & \cellcolor[HTML]{B7E1CD}65.83 & \cellcolor[HTML]{B7E1CD}73.3 & 78.06 & 64.96 & 52.94 & 68.24 & 52.94& 0.61 \\
   & FT Croce + ALR (25 ep) & $\ell_{\infty}$, StAdv, Recolor & 83.9 & \cellcolor[HTML]{B7E1CD}55.52 & \cellcolor[HTML]{B7E1CD}71.21 & \cellcolor[HTML]{B7E1CD}75.29 & 79.82 & 67.34 & \underline{\textbf{54.32}} & 70.46 & \underline{\textbf{54.32}} & 1.49 \\ \hline
 \multirow{13}{*}{3}& AVG & $\ell_\infty$, StAdv, Recolor, $\ell_2$ & \textbf{87.67} & \cellcolor[HTML]{B7E1CD}50.42 & \cellcolor[HTML]{B7E1CD}66.06 & \cellcolor[HTML]{B7E1CD}\textbf{75.90} & \cellcolor[HTML]{B7E1CD}\textbf{85.66} & 69.51 & 47.90 & 69.51 & 47.90 & 13.79 \\
 & MAX & $\ell_\infty$, StAdv, Recolor, $\ell_2$ & 83.26 & \cellcolor[HTML]{B7E1CD}56.94 & \cellcolor[HTML]{B7E1CD}\textbf{70.70} & \cellcolor[HTML]{B7E1CD}74.80 & \cellcolor[HTML]{B7E1CD}81.22 & \textbf{70.92} & \textbf{55.31} & \textbf{70.92} & \textbf{55.31} & 14.60 \\
 & Random & $\ell_\infty$, StAdv, Recolor, $\ell_2$ & 86.55 & \cellcolor[HTML]{B7E1CD}47.29 & \cellcolor[HTML]{B7E1CD}66.52 & \cellcolor[HTML]{B7E1CD}74.93 & \cellcolor[HTML]{B7E1CD}84.64 & 68.34 & 45.71 & 68.34 & 45.71 & 4.58 \\\cdashline{2-13}
 & FT MAX (10 ep) & $\ell_{\infty}$, StAdv, Recolor, $\ell_{2}$ & 82.73 & \cellcolor[HTML]{B7E1CD}\underline{\textbf{55.08}} & \cellcolor[HTML]{B7E1CD}71.36 & \cellcolor[HTML]{B7E1CD}73.5 & \cellcolor[HTML]{B7E1CD}78.98 & 69.73 & \underline{54.19} & 69.73 & \underline{54.19} & 1.3\\
  & FT MAX (25 ep) & $\ell_{\infty}$, StAdv, Recolor, $\ell_{2}$ & 83.72 & \cellcolor[HTML]{B7E1CD}54.93 & \cellcolor[HTML]{B7E1CD}\underline{72.18} & \cellcolor[HTML]{B7E1CD}74.42 & \cellcolor[HTML]{B7E1CD}79.9 & \underline{70.36} & 53.94 & \underline{70.36} & 53.94 & 3.26 \\
  & FT Croce (10 ep) & $\ell_{\infty}$, StAdv, Recolor, $\ell_{2}$ & 84.33 & \cellcolor[HTML]{B7E1CD}52.18 & \cellcolor[HTML]{B7E1CD}69.5 & \cellcolor[HTML]{B7E1CD}72.74 & \cellcolor[HTML]{B7E1CD}79.97 & 68.6 & 50.55 & 68.6 & 50.55 & 0.44 \\
 & FT Croce (25 ep) & $\ell_{\infty}$, StAdv, Recolor, $\ell_{2}$ & 84.84 & \cellcolor[HTML]{B7E1CD}52.59 & \cellcolor[HTML]{B7E1CD}68.74 & \cellcolor[HTML]{B7E1CD}73.27 & \cellcolor[HTML]{B7E1CD}\underline{81.45} & 69.01 & 50.96 & 69.01 & 50.96 & 1.1 \\
  & FT Single (10 ep) & $\ell_{\infty}$, StAdv, Recolor, $\ell_{2}$ & 84.79 & \cellcolor[HTML]{B7E1CD}53.02 & \cellcolor[HTML]{B7E1CD}65.43 & \cellcolor[HTML]{B7E1CD}72.82 & \cellcolor[HTML]{B7E1CD}80.08 & 67.83 & 50.29 & 67.83 & 50.29 & 0.26\\
 & FT Single (25 ep) & $\ell_{\infty}$, StAdv, Recolor, $\ell_{2}$ & \underline{85.5} & \cellcolor[HTML]{B7E1CD}49.12 & \cellcolor[HTML]{B7E1CD}64.79 & \cellcolor[HTML]{B7E1CD}74.27 & \cellcolor[HTML]{B7E1CD}80.76 & 67.24 & 47.21 & 67.24 & 47.21 & 0.63 \\
 & FT Single + ALR (10 ep)& $\ell_{\infty}$, StAdv, Recolor, $\ell_{2}$ & 85.1 & \cellcolor[HTML]{B7E1CD}45.4 & \cellcolor[HTML]{B7E1CD}63.44 & \cellcolor[HTML]{B7E1CD}67.06 & \cellcolor[HTML]{B7E1CD}80.59 & 64.12 & 42.96 & 64.12 & 42.96 & 0.35\\
 & FT Single + ALR (25 ep)& $\ell_{\infty}$, StAdv, Recolor, $\ell_{2}$ & 83.64 & \cellcolor[HTML]{B7E1CD}54.37 & \cellcolor[HTML]{B7E1CD}64.48 & \cellcolor[HTML]{B7E1CD}72.41 & \cellcolor[HTML]{B7E1CD}79.69 & 67.74 & 50.96 & 67.74 & 50.96 & 0.92 \\
 & FT Croce + ALR (10 ep) & $\ell_{\infty}$, StAdv, Recolor, $\ell_{2}$ & 83.03 & \cellcolor[HTML]{B7E1CD}53.96 & \cellcolor[HTML]{B7E1CD}67.9 & \cellcolor[HTML]{B7E1CD}72.38 & \cellcolor[HTML]{B7E1CD}79.08 & 68.33 & 51.95 & 68.33 & 51.95 & 0.55 \\
  & FT Croce + ALR (25 ep)& $\ell_{\infty}$, StAdv, Recolor, $\ell_{2}$ & 84.84 & \cellcolor[HTML]{B7E1CD}53.94 & \cellcolor[HTML]{B7E1CD}68.51 & \cellcolor[HTML]{B7E1CD}\underline{75.11} & \cellcolor[HTML]{B7E1CD}81.32 & 69.72 & 52.23 & 69.72 & 52.23 & 1.36 \\ \hline
\end{tabular}
}\caption{\textbf{Continual Robust Training on ImageNette (Sequence of 4 attacks starting with $\ell_\infty$).}}
\label{app:main_results_imagenette_epochs_2}
\end{table*}

% Please add the following required packages to your document preamble:
% \usepackage[table,xcdraw]{xcolor}
% Beamer presentation requires \usepackage{colortbl} instead of \usepackage[table,xcdraw]{xcolor}
\begin{table*}[]
\scalebox{0.75}{\begin{tabular}{|c|l|l|c|cccc|cc|cc|c|}
\hline
 \multicolumn{1}{|c|}{\begin{tabular}[c]{@{}c@{}}Time\\ Step \end{tabular}} & \multicolumn{1}{|c|}{Procedure} & \multicolumn{1}{c|}{Threat Models} & Clean & $\ell_2$ & StAdv & $\ell_\infty$ & Recolor & \begin{tabular}[c]{@{}c@{}}Avg\\ (known)\end{tabular} & \begin{tabular}[c]{@{}c@{}}Union\\ (known)\end{tabular} & \begin{tabular}[c]{@{}c@{}}Avg\\ (all)\end{tabular} & \begin{tabular}[c]{@{}c@{}}Union\\ (all)\end{tabular} & \begin{tabular}[c]{@{}c@{}}Time\\ (hrs)\end{tabular} \\ \hline
\multirow{ 2}{*}{0}& AT & $\ell_2$ & \textbf{67.75} & \cellcolor[HTML]{B7E1CD}41.65 & 4.21 & 14.28 & 22.46 & 41.65 & 41.65 & 20.65 & 2.34 & 14.85 \\
& Ours & $\ell_2$ & 63.53 & \cellcolor[HTML]{B7E1CD}\textbf{42.88} & \textbf{6.16} & \textbf{19.8} & \textbf{23.36} & \textbf{42.88} & \textbf{42.88} & \textbf{23.05} & \textbf{4.37} & 21.98 \\ \hline
\multirow{13}{*}{1}&AVG & $\ell_2$, StAdv & \textbf{64.48} & \cellcolor[HTML]{B7E1CD}\textbf{35.72} & \cellcolor[HTML]{B7E1CD}28.57 & 9.18 & 19.51 & 32.15 & 25.25 & \textbf{23.24} & 7.53 & 47.24 \\
&MAX & $\ell_2$, StAdv & 61.80 & \cellcolor[HTML]{B7E1CD}33.82 & \cellcolor[HTML]{B7E1CD}31.77 & 7.58 & 17.76 & \textbf{32.79} & \textbf{27.68} & 22.73 & 6.51 & 47.24 \\
&Random & $\ell_2$, StAdv & 62.7 & \cellcolor[HTML]{B7E1CD}31.75 & \cellcolor[HTML]{B7E1CD}\textbf{32.25} & 6.50 & 17.62 & 32.00 & 26.28 & 22.03 & 5.79 & 23.56 \\ \cdashline{2-13}
& FT MAX (10 ep) & $\ell_{2}$, StAdv & 60.3 & \cellcolor[HTML]{B7E1CD}30.9 & \cellcolor[HTML]{B7E1CD}29.32 & 5.96 & 17.04 & 30.11 & 23.97 & 20.8 & 5.15 & 4\\
 & FT MAX (25 ep) & $\ell_{2}$, StAdv & 61.14 & \cellcolor[HTML]{B7E1CD}31.69 & \cellcolor[HTML]{B7E1CD}29.93 & 6.21 & 17.84 & \underline{30.81} & \underline{24.84} & 21.42 & 5.3 & 10.71 \\
 & FT Croce (10 ep) & $\ell_{2}$, StAdv & 62.68 & \cellcolor[HTML]{B7E1CD}35.2 & \cellcolor[HTML]{B7E1CD}23.7 & 8.9 & \textbf{\underline{19.93}} & 29.45 & 21.02 & 21.93 & 6.77 & 2.6\\
 & FT Croce (25 ep) & $\ell_{2}$, StAdv & \underline{63.39} & \cellcolor[HTML]{B7E1CD}31.38 & \cellcolor[HTML]{B7E1CD}28.95 & 5.7 & 17.98 & 30.16 & 23.87 & 21 & 5.04 & 5.96 \\
  & FT Single (10 ep) & $\ell_{2}$, StAdv & 60.48 & \cellcolor[HTML]{B7E1CD}18.29 & \cellcolor[HTML]{B7E1CD}30.83 & 2.24 & 13.28 & 24.56 & 16.19 & 16.16 & 1.84 & 2.77\\
 & FT Single (25 ep) & $\ell_{2}$, StAdv & 60.78 & \cellcolor[HTML]{B7E1CD}18.22 & \cellcolor[HTML]{B7E1CD}\underline{31.79} & 1.99 & 13.32 & 25 & 16.34 & 16.33 & 1.6 & 6.91 \\
 & FT Single + ALR (10 ep) & $\ell_{2}$, StAdv & 53 & \cellcolor[HTML]{B7E1CD}29.23 & \cellcolor[HTML]{B7E1CD}24 & 8.45 & 17.18 & 26.61 & 20.12 & 19.71 & 6.71 & 2.77\\
 & FT Single + ALR (25 ep) & $\ell_{2}$, StAdv & 54.44 & \cellcolor[HTML]{B7E1CD}22.03 & \cellcolor[HTML]{B7E1CD}29.1 & 4.46 & 13.6 & 25.56 & 19.21 & 17.3 & 3.78 & 8.73 \\
 & FT Croce + ALR (10 ep) & $\ell_{2}$, StAdv & 59.09 & \cellcolor[HTML]{B7E1CD}34.09 & \cellcolor[HTML]{B7E1CD}26.59 & 9.84 & 19.2 & 30.34 & 23.62 & 22.43 & 8.28 & 3.11 \\
 & FT Croce + ALR (25 ep) & $\ell_{2}$, StAdv & 58.59 & \cellcolor[HTML]{B7E1CD}\underline{34.65} & \cellcolor[HTML]{B7E1CD}26.45 & \textbf{\underline{10.55}} & 19.65 & 30.55 & 23.67 & \underline{22.82} & \underline{\textbf{8.64}} & 7.67 \\ \hline
\multirow{13}{*}{2}& AVG & $\ell_2$, StAdv, $\ell_\infty$ & 62.09 & \cellcolor[HTML]{B7E1CD}40.34 & \cellcolor[HTML]{B7E1CD}25.13 & \cellcolor[HTML]{B7E1CD}20.88 & 28.84 & \textbf{28.78} & 16.18 & 28.80 & 14.71 & 69.49 \\
& MAX & $\ell_2$, StAdv, $\ell_\infty$ & 56.94 & \cellcolor[HTML]{B7E1CD}34.74 & \cellcolor[HTML]{B7E1CD}\textbf{28.46} & \cellcolor[HTML]{B7E1CD}22.72 & \textbf{30.35} & 28.64 & \textbf{19.66} & \textbf{29.04} & \textbf{17.55} & 69.36 \\
& Random & $\ell_2$, StAdv, $\ell_\infty$ & 60.98 & \cellcolor[HTML]{B7E1CD}38.01 & \cellcolor[HTML]{B7E1CD}25.21 & \cellcolor[HTML]{B7E1CD}17.25 & 25.76 & 26.82 & 14.26 & 26.56 & 12.97 & 19.58 \\\cdashline{2-13} 
& FT MAX (10 ep) & $\ell_{2}$, StAdv, $\ell_{\infty}$ & 59.9 & \cellcolor[HTML]{B7E1CD}38.94 & \cellcolor[HTML]{B7E1CD}26.55 & \cellcolor[HTML]{B7E1CD}16.95 & 25.72 & 27.48 & 14.78 & 27.04 & 13.07 & 5.2\\
& FT MAX (25 ep) & $\ell_{2}$, StAdv, $\ell_{\infty}$ & 61.01 & \cellcolor[HTML]{B7E1CD}38.56 & \cellcolor[HTML]{B7E1CD}\underline{27.54} & \cellcolor[HTML]{B7E1CD}17.05 & 25.91 & \underline{27.72} & \underline{15.34} & \underline{27.27} & \underline{13.66} & 12.25 \\
 & FT Croce (10 ep) & $\ell_{2}$, StAdv, $\ell_{\infty}$ & 62.78 & \cellcolor[HTML]{B7E1CD}39.87 & \cellcolor[HTML]{B7E1CD}20.17 & \cellcolor[HTML]{B7E1CD}14.59 & 23.7 & 24.88 & 10.46 & 24.58 & 9.2 & 2.29 \\
 & FT Croce (25 ep) & $\ell_{2}$, StAdv, $\ell_{\infty}$ & 65.85 & \cellcolor[HTML]{B7E1CD}42.51 & \cellcolor[HTML]{B7E1CD}11.05 & \cellcolor[HTML]{B7E1CD}19.67 & 26.55 & 24.41 & 8.06 & 24.95 & 7.3 & 5.06 \\
 & FT Single (10 ep) & $\ell_{2}$, StAdv, $\ell_{\infty}$ & 65.62 & \cellcolor[HTML]{B7E1CD}42.95 & \cellcolor[HTML]{B7E1CD}7.55 & \cellcolor[HTML]{B7E1CD}22.03 & 28.06 & 24.18 & 5.57 & 25.15 & 4.98 & 1.49\\
 & FT Single (25 ep) & $\ell_{2}$, StAdv, $\ell_{\infty}$ & \underline{\textbf{65.93}} & \cellcolor[HTML]{B7E1CD}42.82 & \cellcolor[HTML]{B7E1CD}7.62 & \cellcolor[HTML]{B7E1CD}21.83 & \underline{28.2} & 24.09 & 5.65 & 25.12 & 5.12 & 3.71 \\
  & FT Single + ALR (10 ep) & $\ell_{2}$, StAdv, $\ell_{\infty}$ & 62.35 & \cellcolor[HTML]{B7E1CD}\underline{\textbf{43.56}} & \cellcolor[HTML]{B7E1CD}8.76 & \cellcolor[HTML]{B7E1CD}23.72 & 27.77 & 25.35 & 6.98 & 25.95 & 6.29 & 2.16
\\
 & FT Single + ALR (25 ep) & $\ell_{2}$, StAdv, $\ell_{\infty}$ & 62.56 & \cellcolor[HTML]{B7E1CD}42.33 & \cellcolor[HTML]{B7E1CD}7.7 & \cellcolor[HTML]{B7E1CD}\underline{\textbf{25.06}} & 26.57 & 25.03 & 6.67 & 25.41 & 6.02 & 5.39 \\
 & FT Croce + ALR (10 ep) & $\ell_{2}$, StAdv, $\ell_{\infty}$ & 60.67 & \cellcolor[HTML]{B7E1CD}42.06 & \cellcolor[HTML]{B7E1CD}16.66 & \cellcolor[HTML]{B7E1CD}21.18 & 25.89 & 26.63 & 12.59 & 26.45 & 11.21 & 3.05\\
 & FT Croce + ALR (25 ep) & $\ell_{2}$, StAdv, $\ell_{\infty}$ & 63.43 & \cellcolor[HTML]{B7E1CD}42.92 & \cellcolor[HTML]{B7E1CD}10.14 & \cellcolor[HTML]{B7E1CD}23.16 & 26.37 & 25.41 & 8.29 & 25.65 & 7.64 & 6.7 \\
  \hline
\multirow{ 13}{*}{3}&AVG & $\ell_2$, StAdv, $\ell_\infty$, Recolor & 65.61 & \cellcolor[HTML]{B7E1CD}40.86 & \cellcolor[HTML]{B7E1CD}22.4 & \cellcolor[HTML]{B7E1CD}20.45 & \cellcolor[HTML]{B7E1CD}37.27 & 30.25 & 14.09 & 30.25 & 14.09 & 101.43 \\
& MAX & $\ell_2$, StAdv, $\ell_\infty$, Recolor & 59.12 & \cellcolor[HTML]{B7E1CD}33.89 & \cellcolor[HTML]{B7E1CD}\textbf{28.02} & \cellcolor[HTML]{B7E1CD}\textbf{22.20} & \cellcolor[HTML]{B7E1CD}35.00 & 29.78 & \textbf{18.74} & 29.78 & \textbf{18.74} & 101.43 \\
& Random & $\ell_2$, StAdv, $\ell_\infty$, Recolor & 63.1 & \cellcolor[HTML]{B7E1CD}39.47 & \cellcolor[HTML]{B7E1CD}24.79 & \cellcolor[HTML]{B7E1CD}19.04 & \cellcolor[HTML]{B7E1CD}38.15 & \textbf{30.36} & 14.57 & \textbf{30.36} & 14.57 & 22.87 \\\cdashline{2-13}
& FT MAX (10 ep) & $\ell_{2}$, StAdv, $\ell_{\infty}$, Recolor & 61.5 & \cellcolor[HTML]{B7E1CD}39.34 & \cellcolor[HTML]{B7E1CD}26.97 & \cellcolor[HTML]{B7E1CD}17.25 & \cellcolor[HTML]{B7E1CD}33.56 & \underline{29.28} & 14.55 & \underline{29.28} & 14.55 & 8.61\\
& FT MAX (25 ep) & $\ell_{2}$, StAdv, $\ell_{\infty}$, Recolor & 62.14 & \cellcolor[HTML]{B7E1CD}38.68 & \cellcolor[HTML]{B7E1CD}\underline{27.51} & \cellcolor[HTML]{B7E1CD}17.13 & \cellcolor[HTML]{B7E1CD}33.06 & 29.09 & \underline{14.84} & 29.09 & \underline{14.84} & 19.67 \\
& FT Croce (10 ep) & $\ell_{2}$, StAdv, $\ell_{\infty}$, Recolor & 64.82 & \cellcolor[HTML]{B7E1CD}41.09 & \cellcolor[HTML]{B7E1CD}19.78 & \cellcolor[HTML]{B7E1CD}16.55 & \cellcolor[HTML]{B7E1CD}32.26 & 27.42 & 10.57 & 27.42 & 10.57 & 2.42\\
 & FT Croce (25 ep) & $\ell_{2}$, StAdv, $\ell_{\infty}$, Recolor & 66.31 & \cellcolor[HTML]{B7E1CD}41.02 & \cellcolor[HTML]{B7E1CD}13.42 & \cellcolor[HTML]{B7E1CD}17.34 & \cellcolor[HTML]{B7E1CD}31.02 & 25.7 & 8.4 & 25.7 & 8.4 & 6.03 \\
 & FT Single (10 ep) & $\ell_{2}$, StAdv, $\ell_{\infty}$, Recolor & \underline{\textbf{69.82}} & \cellcolor[HTML]{B7E1CD}32.63 & \cellcolor[HTML]{B7E1CD}4.07 & \cellcolor[HTML]{B7E1CD}9.38 & \cellcolor[HTML]{B7E1CD}40.07 & 21.54 & 1.42 & 21.54 & 1.42 & 3.06\\
 & FT Single (25 ep) & $\ell_{2}$, StAdv, $\ell_{\infty}$, Recolor & 68.63 & \cellcolor[HTML]{B7E1CD}37.06 & \cellcolor[HTML]{B7E1CD}5.57 & \cellcolor[HTML]{B7E1CD}13.28 & \cellcolor[HTML]{B7E1CD}37.66 & 23.39 & 2.76 & 23.39 & 2.76 & 7.78 \\
 & FT Single + ALR (10 ep) & $\ell_{2}$, StAdv, $\ell_{\infty}$, Recolor & 66.58 & \cellcolor[HTML]{B7E1CD}37.98 & \cellcolor[HTML]{B7E1CD}6.65 & \cellcolor[HTML]{B7E1CD}16.37 & \cellcolor[HTML]{B7E1CD}39.23 & 25.06 & 3.83 & 25.06 & 3.83 & 3.91\\
 & FT Single + ALR (25 ep) & $\ell_{2}$, StAdv, $\ell_{\infty}$, Recolor & 68.15 & \cellcolor[HTML]{B7E1CD}32.05 & \cellcolor[HTML]{B7E1CD}5.08 & \cellcolor[HTML]{B7E1CD}11.82 & \cellcolor[HTML]{B7E1CD}\underline{\textbf{41.5}} & 22.61 & 2.46 & 22.61 & 2.46 & 9.72 \\
  & FT Croce + ALR (10 ep) & $\ell_{2}$, StAdv, $\ell_{\infty}$, Recolor & 64.11 & \cellcolor[HTML]{B7E1CD}\textbf{\underline{42.52}} & \cellcolor[HTML]{B7E1CD}10.89 & \cellcolor[HTML]{B7E1CD}\underline{21.36} & \cellcolor[HTML]{B7E1CD}34.05 & 27.21 & 8.11 & 27.21 & 8.11 & 3.42\\
 & FT Croce + ALR (25 ep) & $\ell_{2}$, StAdv, $\ell_{\infty}$, Recolor & 65.33 & \cellcolor[HTML]{B7E1CD}39.4 & \cellcolor[HTML]{B7E1CD}11.41 & \cellcolor[HTML]{B7E1CD}16.84 & \cellcolor[HTML]{B7E1CD}34.15 & 25.45 & 7.35 & 25.45 & 7.35 & 7.41
 \\ \hline
\end{tabular}}
\caption{\textbf{Continual Robust Training on CIFAR-100 (Sequence of 4 attacks starting with $\ell_2$).}}
\label{app:main_results_cifar100_full}
\end{table*}

\begin{table*}[]
\scalebox{0.75}{\begin{tabular}{|c|l|l|c|cccc|cc|cc|c|}
\hline
\multicolumn{1}{|c|}{\begin{tabular}[c]{@{}c@{}}Time\\ Step \end{tabular}}&\multicolumn{1}{|c|}{Procedure} & Threat Models & Clean & $\ell_\infty$ & StAdv & Recolor & $\ell_2$ & \begin{tabular}[c]{@{}c@{}}Avg\\ (known)\end{tabular} & \begin{tabular}[c]{@{}c@{}}Union\\ (known)\end{tabular} & \begin{tabular}[c]{@{}c@{}}Avg\\ (all)\end{tabular} & \begin{tabular}[c]{@{}c@{}}Union\\ (all)\end{tabular} & \begin{tabular}[c]{@{}c@{}}Time\\ (hrs)\end{tabular} \\ \hline
\multirow{ 2}{*}{0}&AT & $\ell_\infty$ & \textbf{60.95} & \cellcolor[HTML]{B7E1CD}27.61 & 9.92 & \textbf{33.2} & \textbf{35.6} & 27.61 & 27.61 & \textbf{26.58} & 7.45 & 16.33 \\
& AT + ALR & $\ell_{\infty}$ & 55.36 & \cellcolor[HTML]{B7E1CD}\textbf{28.01} & \textbf{11.25} & 31.01 & 33.95 & \textbf{28.01} & \textbf{28.01} & 26.05 & 8.62 & 23.75 \\ \hline
\multirow{ 13}{*}{1}& AVG & $\ell_\infty$, StAdv & \textbf{66.09} & \cellcolor[HTML]{B7E1CD}8.65 & \cellcolor[HTML]{B7E1CD}33.19 & 18.42 & 21.58 & 20.92 & 8.18 & 20.46 & 7.47 & 47.59 \\
&MAX & $\ell_\infty$, StAdv & 57.01 & \cellcolor[HTML]{B7E1CD}22.8 & \cellcolor[HTML]{B7E1CD}29.23 & 30.54 & \textbf{33.58} & \textbf{26.02} & \textbf{20.28} & \textbf{29.04} & \textbf{17.96} & 46.97 \\
&Random & $\ell_\infty$, StAdv & 47.14 & \cellcolor[HTML]{B7E1CD}16.62 & \cellcolor[HTML]{B7E1CD}25.61 & 25.35 & 27.25 & 21.12 & 14.63 & 23.71 & 13.37 & 23.92 \\\cdashline{2-13}
& FT MAX (10 ep) & $\ell_{\infty}$, StAdv & 58.05 & \cellcolor[HTML]{B7E1CD}22.48 & \cellcolor[HTML]{B7E1CD}28.66 & 30.43 & 33.15 & 25.57 & 18.95 & 28.68 & 17.05 & 4.03\\
 & FT MAX (25 ep) & $\ell_{\infty}$, StAdv & 58.56 & \cellcolor[HTML]{B7E1CD}22.36 & \cellcolor[HTML]{B7E1CD}29.38 & \textbf{\underline{30.97}} & 33.16 & \underline{25.87} & \underline{19.38} & \underline{28.97} & \underline{17.43} & 10.02 \\
  & FT Croce (10 ep) & $\ell_{\infty}$, StAdv & 60.17 & \cellcolor[HTML]{B7E1CD}22.75 & \cellcolor[HTML]{B7E1CD}26.81 & 31.22 & 33.68 & 24.78 & 17.54 & 28.62 & 15.95 & 2.49\\
 & FT Croce (25 ep) & $\ell_{\infty}$, StAdv & \underline{60.27} & \cellcolor[HTML]{B7E1CD}22.18 & \cellcolor[HTML]{B7E1CD}28.17 & 30.38 & \underline{33.25} & 25.18 & 17.81 & 28.5 & 16.21 & 5.69 \\
 & FT Single (10 ep) & $\ell_{\infty}$, StAdv & 55.87 & \cellcolor[HTML]{B7E1CD}15.43 & \cellcolor[HTML]{B7E1CD}24.29 & 24.34 & 27.54 & 19.86 & 11.9 & 22.9 & 10.95 & 2.77\\
 & FT Single (25 ep) & $\ell_{\infty}$, StAdv & 56.11 & \cellcolor[HTML]{B7E1CD}16.09 & \cellcolor[HTML]{B7E1CD}24.46 & 24.84 & 28.65 & 20.27 & 12.5 & 23.51 & 11.33 & 6.95 \\
 & FT Single + ALR (10 ep) & $\ell_{\infty}$, StAdv & 56.03 & \cellcolor[HTML]{B7E1CD}3.81 & \cellcolor[HTML]{B7E1CD}35.21 & 18.52 & 17.31 & 19.51 & 3.77 & 18.71 & 3.51 & 75.49 \\
 & FT Single + ALR (25 ep) & $\ell_{\infty}$, StAdv & 59.65 & \cellcolor[HTML]{B7E1CD}2.6 & \cellcolor[HTML]{B7E1CD}\underline{\textbf{37.96}} & 18.51 & 13.52 & 20.28 & 2.58 & 18.15 & 2.41 & 8.34 \\
  & FT Croce + ALR (10 ep) & $\ell_{\infty}$, StAdv & 54.86 & \cellcolor[HTML]{B7E1CD}\textbf{\underline{23.33}} & \cellcolor[HTML]{B7E1CD}27.19 & 30.15 & 31.54 & 25.26 & 18.75 & 28.05 & 16.86 & 2.97\\
 & FT Croce + ALR (25 ep) & $\ell_{\infty}$, StAdv & 54.27 & \cellcolor[HTML]{B7E1CD}23.27 & \cellcolor[HTML]{B7E1CD}26.27 & 29.4 & 31.79 & 24.77 & 18.25 & 27.68 & 16.06 & 7.4 \\
  \hline
\multirow{ 13}{*}{2}&AVG & $\ell_\infty$, StAdv, Recolor & \textbf{68.19 }& \cellcolor[HTML]{B7E1CD}16.88 & \cellcolor[HTML]{B7E1CD}\textbf{29.53} & \cellcolor[HTML]{B7E1CD}38.12 & 30.75 & 28.18 & 14.14 & 28.82 & 14.11 & 79.47 \\
& MAX & $\ell_\infty$, StAdv, Recolor & 57.96 & \cellcolor[HTML]{B7E1CD}22.38 & \cellcolor[HTML]{B7E1CD}28.92 & \cellcolor[HTML]{B7E1CD}35.24 & 33.9 & 28.85 & \textbf{19.27} & 30.11 & \textbf{19.17} & 79.5 \\
& Random & $\ell_\infty$, StAdv, Recolor & 47.14 & \cellcolor[HTML]{B7E1CD}16.62 & \cellcolor[HTML]{B7E1CD}25.61 & \cellcolor[HTML]{B7E1CD}25.35 & 27.25 & 21.12 & 14.63 & 23.71 & 13.37 & 23.92 \\\cdashline{2-13}
& FT MAX (10 ep) & $\ell_{\infty}$, StAdv, Recolor & 58.96 & \cellcolor[HTML]{B7E1CD}22.04 & \cellcolor[HTML]{B7E1CD}29 & \cellcolor[HTML]{B7E1CD}36.35 & 34.55 & 29.13 & 18.37 & 30.48 & 18.31 & 6.75\\
 & FT MAX (25 ep) & $\ell_{\infty}$, StAdv, Recolor & 59.41 & \cellcolor[HTML]{B7E1CD}21.7 & \cellcolor[HTML]{B7E1CD}\underline{29.2} & \cellcolor[HTML]{B7E1CD}35.57 & 33.24 & 28.82 & \underline{18.21} & 29.93 & \underline{18.09} & 16.75 \\
  & FT Croce (10 ep) & $\ell_{\infty}$, StAdv, Recolor & 62.27 & \cellcolor[HTML]{B7E1CD}22.42 & \cellcolor[HTML]{B7E1CD}26.16 & \cellcolor[HTML]{B7E1CD}37.25 & 34.5 & 28.61 & 16.34 & 30.08 & 16.24 & 2.77 \\
 & FT Croce (25 ep) & $\ell_{\infty}$, StAdv, Recolor & 62.09 & \cellcolor[HTML]{B7E1CD}\underline{23.24} & \cellcolor[HTML]{B7E1CD}25.6 & \cellcolor[HTML]{B7E1CD}36.91 & \underline{35.89} & 28.58 & 16.63 & \underline{30.41} & 16.52 & 6.46 \\
 & FT Single (10 ep) & $\ell_{\infty}$, StAdv, Recolor & 64.02 & \cellcolor[HTML]{B7E1CD}20.72 & \cellcolor[HTML]{B7E1CD}14.72 & \cellcolor[HTML]{B7E1CD}41.7 & 32.89 & 25.71 & 9.61 & 27.51 & 9.57 & 3.01\\
 & FT Single (25 ep) & $\ell_{\infty}$, StAdv, Recolor & \underline{67.39} & \cellcolor[HTML]{B7E1CD}13.25 & \cellcolor[HTML]{B7E1CD}6.94 & \cellcolor[HTML]{B7E1CD}45.1 & 27.71 & 21.76 & 3.59 & 23.25 & 3.57 & 7.86 \\
 & FT Single + ALR (10 ep) & $\ell_{\infty}$, StAdv, Recolor & 64.86 & \cellcolor[HTML]{B7E1CD}8.76 & \cellcolor[HTML]{B7E1CD}9.34 & \cellcolor[HTML]{B7E1CD}46.4 & 25.62 & 21.5 & 3.91 & 22.53 & 3.91 & 3.8 \\
 & FT Single + ALR (25 ep) & $\ell_{\infty}$, StAdv, Recolor & 66.87 & \cellcolor[HTML]{B7E1CD}3.74 & \cellcolor[HTML]{B7E1CD}8.44 & \cellcolor[HTML]{B7E1CD}\underline{\textbf{50.81}} & 19.81 & 21 & 2.03 & 20.7 & 2.03 & 9.89 \\
 & FT Croce + ALR (10 ep) & $\ell_{\infty}$, StAdv, Recolor & 56.41 & \cellcolor[HTML]{B7E1CD}24.28 & \cellcolor[HTML]{B7E1CD}25.62 & \cellcolor[HTML]{B7E1CD}36.21 & 34.43 & 28.7 & 17.6 & 30.14 & 17.41& 3.6\\
 & FT Croce + ALR (25 ep) & $\ell_{\infty}$, StAdv, Recolor & 56.83 & \cellcolor[HTML]{B7E1CD}22.43 & \cellcolor[HTML]{B7E1CD}25.9 & \cellcolor[HTML]{B7E1CD}38.27 & 33.5 & \textbf{\underline{28.87}} & 16.98 & 30.03 & 16.76 & 8.28 \\
  \hline
\multirow{ 13}{*}{3}&AVG & $\ell_\infty$, StAdv, Recolor, $\ell_2$ & 64.8 & \cellcolor[HTML]{B7E1CD}20.9 & \cellcolor[HTML]{B7E1CD}22.46 & \cellcolor[HTML]{B7E1CD}37.27 & \cellcolor[HTML]{B7E1CD}41.05 & 30.42 & 14.56 & 30.42 & 14.56 & 101.39 \\
&MAX & $\ell_\infty$, StAdv, Recolor, $\ell_2$ & 57.9 & \cellcolor[HTML]{B7E1CD}22.39 & \cellcolor[HTML]{B7E1CD}28.72 & \cellcolor[HTML]{B7E1CD}35.96 & \cellcolor[HTML]{B7E1CD}35.65 & \textbf{30.68} & \textbf{19.15} & \textbf{30.68} & \textbf{19.15} & 101.25 \\
&Random & $\ell_\infty$, StAdv, Recolor, $\ell_2$ & 63.23 & \cellcolor[HTML]{B7E1CD}19.52 & \cellcolor[HTML]{B7E1CD}21.29 & \cellcolor[HTML]{B7E1CD}\textbf{39.71} & \cellcolor[HTML]{B7E1CD}39.95 & 30.12 & 13.6 & 30.12 & 13.6 & 24.88 \\\cdashline{2-13}
& FT MAX (10 ep) & $\ell_{\infty}$, StAdv, Recolor, $\ell_{2}$ & 59.61 & \cellcolor[HTML]{B7E1CD}22.14 & \cellcolor[HTML]{B7E1CD}29.13 & \cellcolor[HTML]{B7E1CD}36.17 & \cellcolor[HTML]{B7E1CD}34.35 & 30.45 & 18.61 & 30.45 & 18.61 & 8.58\\
 & FT MAX (25 ep) & $\ell_{\infty}$, StAdv, Recolor, $\ell_{2}$ & 59.42 & \cellcolor[HTML]{B7E1CD}22.02 & \cellcolor[HTML]{B7E1CD}\underline{\textbf{29.28}} & \cellcolor[HTML]{B7E1CD}35.96 & \cellcolor[HTML]{B7E1CD}34.45 & \underline{30.43} & \underline{18.64} & \underline{30.43} & \underline{18.64} & 21.36 \\
 & FT Croce (10 ep) & $\ell_{\infty}$, StAdv, Recolor, $\ell_{2}$ & 62.44 & \cellcolor[HTML]{B7E1CD}20.96 & \cellcolor[HTML]{B7E1CD}26.06 & \cellcolor[HTML]{B7E1CD}35.91 & \cellcolor[HTML]{B7E1CD}36.77 & 29.93 & 15.83 & 29.93 & 15.83 & 2.38 \\
 & FT Croce (25 ep) & $\ell_{\infty}$, StAdv, Recolor, $\ell_{2}$ & 62.17 & \cellcolor[HTML]{B7E1CD}21.84 & \cellcolor[HTML]{B7E1CD}26.14 & \cellcolor[HTML]{B7E1CD}\underline{36.92} & \cellcolor[HTML]{B7E1CD}36.69 & 30.4 & 16.08 & 30.4 & 16.08 & 5.81 \\
 & FT Single (10 ep) & $\ell_{\infty}$, StAdv, Recolor, $\ell_{2}$ & 63.94 & \cellcolor[HTML]{B7E1CD}23.86 & \cellcolor[HTML]{B7E1CD}13.73 & \cellcolor[HTML]{B7E1CD}37.22 & \cellcolor[HTML]{B7E1CD}41.47 & 29.07 & 9.92 & 29.07 & 9.92 & 1.61\\
 & FT Single (25 ep) & $\ell_{\infty}$, StAdv, Recolor, $\ell_{2}$ & \underline{\textbf{66.44}} & \cellcolor[HTML]{B7E1CD}21.17 & \cellcolor[HTML]{B7E1CD}7.72 & \cellcolor[HTML]{B7E1CD}31.83 & \cellcolor[HTML]{B7E1CD}\underline{\textbf{42.5}} & 25.8 & 5.67 & 25.8 & 5.67 & 4.07 \\
  & FT Single + ALR (10 ep) & $\ell_{\infty}$, StAdv, Recolor, $\ell_{2}$ & 60.76 & \cellcolor[HTML]{B7E1CD}22.36 & \cellcolor[HTML]{B7E1CD}10.35 & \cellcolor[HTML]{B7E1CD}31.33 & \cellcolor[HTML]{B7E1CD}41.91 & 26.49 & 7.99 & 26.49 & 7.99 & 2.35\\

 & FT Single + ALR (25 ep) & $\ell_{\infty}$, StAdv, Recolor, $\ell_{2}$ & 62.25 & \cellcolor[HTML]{B7E1CD}20.56 & \cellcolor[HTML]{B7E1CD}7.92 & \cellcolor[HTML]{B7E1CD}30.69 & \cellcolor[HTML]{B7E1CD}41.42 & 25.15 & 6.25 & 25.15 & 6.25 & 6.35 \\
 & FT Croce + ALR (10 ep) & $\ell_{\infty}$, StAdv, Recolor, $\ell_{2}$ & 57.56 & \cellcolor[HTML]{B7E1CD}24.64 & \cellcolor[HTML]{B7E1CD}22.52 & \cellcolor[HTML]{B7E1CD}35.77 & \cellcolor[HTML]{B7E1CD}37.55 & 30.12 & 15.96 & 30.12 & 15.96 & 3.36\\
 & FT Croce + ALR (25 ep) & $\ell_{\infty}$, StAdv, Recolor, $\ell_{2}$ & 58.14 & \cellcolor[HTML]{B7E1CD}\underline{\textbf{24.85}} & \cellcolor[HTML]{B7E1CD}18.69 & \cellcolor[HTML]{B7E1CD}36.75 & \cellcolor[HTML]{B7E1CD}39.16 & 29.86 & 13.87 & 29.86 & 13.87 & 7.64\\
 \hline
\end{tabular}}
\caption{\textbf{Continual Robust Training on CIFAR-100 (Sequence of 4 attacks starting with $\ell_\infty$).}}
\label{app:main_results_cifar100_full_2}
\end{table*}

\section{Initial Training Ablations}
\label{app:init_train}
In this section, we present some ablations across regularization strength of each regularization method on the initial training portion of our approach pipeline.  We present ablation results for CIFAR-10 and ImageNette.
\subsection{Performance across different threat models} \label{app:init_train_different_attack_types}
In this section, we perform initial training with models using different initial attacks including attacks in the UAR benchmark \citep{kaufmann2019testing} and evaluate the performance across attack types for training with single-step variation regularization, single-step adversarial $\ell_2$ regularization, uniform regularization, and gaussian regularization.

We present ablation results for CIFAR-10 (Table \ref{app:IT_ablation_VR_CIF} for variation regularization, Table \ref{app:IT_ablation_L2_CIF} for adversarial $\ell_2$ regularization, Table \ref{app:IT_ablation_UR_CIF} for uniform regularization, and Table \ref{app:IT_ablation_GR_CIF} for Gaussian regularization) and ImageNette (Table \ref{app:IT_ablation_VR_IM} for variation regularization, Table \ref{app:IT_ablation_L2_IM} for adversarial $\ell_2$ regularization, Table \ref{app:IT_ablation_UR_IM} for uniform regularization, and Table \ref{app:IT_ablation_GR_IM} for Gaussian regularization).  Overall, we find that across different starting attacks and unseen test attacks, regularization generally improves performance on unseen attacks, leading to increases in average and union accuracy across all attacks with regularization.  We find that in many cases (especially using random noise types) using regularization trades off clean accuracy.  Additionally, some threat models such as Snow are generally more difficult to gain improvement on via regularization; for many starting models, using regularization decreases accuracy on Snow attack.
% Please add the following required packages to your document preamble:
% \usepackage[table,xcdraw]{xcolor}
% Beamer presentation requires \usepackage{colortbl} instead of \usepackage[table,xcdraw]{xcolor}

\begin{table*}[]{\renewcommand{\arraystretch}{1.2}
\scalebox{0.73}{
\begin{tabular}{|l|c|c|cccccccccccc|cc|}
\hline
\begin{tabular}[c]{@{}l@{}}Train\\ Attack\end{tabular} & $\lambda$ & Clean & $\ell_2$ & $\ell_\infty$ & StAdv & ReColor & Gabor & Snow & Pixel & JPEG & Elastic & Wood & Glitch & \begin{tabular}[c]{@{}c@{}}Kaleid-\\ oscope\end{tabular} & Avg & Union \\ \hline
$\ell_2$ & 0 & \textbf{91.08} & 70.02 & 29.38 & 0.79 & 33.69 & 66.93 & 24.59 & 14.99 & 64.22 & 45.13 & 70.85 & \textbf{80.3} & 30.08 & 44.25 & 0.1 \\
$\ell_2$ & 0.1 & 90.4 & 69.7 & \cellcolor[HTML]{B7E1CD}31.78 & \cellcolor[HTML]{B7E1CD}2.27 & \cellcolor[HTML]{B7E1CD}38.16 & \cellcolor[HTML]{F4CCCC}62.99 & \cellcolor[HTML]{B7E1CD}25.77 & \cellcolor[HTML]{B7E1CD}16.91 & \cellcolor[HTML]{B7E1CD}65.35 & 45.94 & 71.05 & 79.72 & 30.17 & 44.98 & 0.74 \\
$\ell_2$ & 0.2 & \cellcolor[HTML]{F4CCCC}89.49 & \textbf{70.49} & \cellcolor[HTML]{B7E1CD}33.43 & \cellcolor[HTML]{B7E1CD}4.29 & \cellcolor[HTML]{B7E1CD}42.8 & \cellcolor[HTML]{B7E1CD}\textbf{68.03} & \cellcolor[HTML]{B7E1CD}26.85 & \cellcolor[HTML]{B7E1CD}18.79 & \cellcolor[HTML]{B7E1CD}66.04 & \cellcolor[HTML]{B7E1CD}47.27 & \cellcolor[HTML]{B7E1CD}72.21 & 79.65 & \cellcolor[HTML]{B7E1CD}\textbf{33.38} & \cellcolor[HTML]{B7E1CD}46.94 & \cellcolor[HTML]{B7E1CD}1.78 \\
$\ell_2$ & 0.5 & \cellcolor[HTML]{F4CCCC}89.57 & 70.29 & \cellcolor[HTML]{B7E1CD}\textbf{34.16} & \cellcolor[HTML]{B7E1CD}\textbf{17.44} & \cellcolor[HTML]{B7E1CD}\textbf{51.04} & \cellcolor[HTML]{F4CCCC}65.63 & \cellcolor[HTML]{B7E1CD}\textbf{28.71} & \cellcolor[HTML]{B7E1CD}\textbf{22.5} & \cellcolor[HTML]{B7E1CD}\textbf{66.76} & \cellcolor[HTML]{B7E1CD}\textbf{48.8} & \cellcolor[HTML]{B7E1CD}\textbf{73.24} & 79.66 & \cellcolor[HTML]{F4CCCC}28.83 & \cellcolor[HTML]{B7E1CD}\textbf{48.92} & \cellcolor[HTML]{B7E1CD}\textbf{5.94} \\ \hline
$\ell_\infty$ & 0 & \textbf{85.53} & \textbf{59.36} & 50.98 & 6.34 & 56.27 & \textbf{68.94} & 36.79 & 20.57 & \textbf{54.02} & 51 & \textbf{64.24} & \textbf{75.94} & \textbf{39.44} & 48.66 & 1.31 \\
$\ell_\infty$ & 0.1 & 85.06 & 58.77 & 51.44 & \cellcolor[HTML]{B7E1CD}7.43 & 55.59 & 68.33 & 37.09 & 20.11 & 53.89 & 51.84 & 64.38 & 74.96 & \cellcolor[HTML]{B7E1CD}42.43 & 48.86 & 1.95 \\
$\ell_\infty$ & 0.2 & 85.23 & \cellcolor[HTML]{F4C7C3}58.08 & \textbf{51.49} & \cellcolor[HTML]{B7E1CD}8.96 & 56.32 & 68.3 & 37.11 & 21.48 & \cellcolor[HTML]{F4C7C3}52.86 & 51.61 & 63.72 & 75.49 & 40.22 & 48.8 & \cellcolor[HTML]{B7E1CD}2.44 \\
$\ell_\infty$ & 0.5 & \cellcolor[HTML]{F4C7C3}83.18 & \cellcolor[HTML]{F4C7C3}58.21 & 51.47 & \cellcolor[HTML]{B7E1CD}\textbf{19.5} & \cellcolor[HTML]{B7E1CD}\textbf{61.02} & 68.75 & \cellcolor[HTML]{B7E1CD}\textbf{37.94} & \cellcolor[HTML]{B7E1CD}\textbf{22.78} & 53.89 & \cellcolor[HTML]{F4C7C3}\textbf{49.82} & 63.47 & \cellcolor[HTML]{F4C7C3}73.57 & 39.88 & \cellcolor[HTML]{B7E1CD}\textbf{50.02} & \cellcolor[HTML]{B7E1CD}\textbf{5.52} \\ \hline
StAdv & 0 & \textbf{87.12} & 5.48 & 0.07 & 56.22 & 5.69 & 17.62 & 57.8 & \textbf{5.93} & 11.09 & \textbf{76.02} & \textbf{77.47} & \textbf{54.04} & 43.4 & \textbf{34.24} & 0.05 \\
StAdv & 0.1 & 86.95 & 4.63 & 0.08 & 56.16 & \cellcolor[HTML]{F4C7C3}4.44 & \cellcolor[HTML]{B7E1CD}20.44 & 57.25 & 4.93 & \cellcolor[HTML]{F4C7C3}9.27 & 75.2 & 76.66 & \cellcolor[HTML]{F4C7C3}52.77 & \cellcolor[HTML]{F4C7C3}40.68 & 33.54 & 0.06 \\
StAdv & 0.2 & \cellcolor[HTML]{F4C7C3}81.39 & \textbf{5.99} & \textbf{0.1} & \cellcolor[HTML]{F4C7C3}54.98 & \cellcolor[HTML]{B7E1CD}\textbf{8.16} & \cellcolor[HTML]{F4C7C3}15.97 & \cellcolor[HTML]{F4C7C3}48.37 & 5.81 & 11.07 & \cellcolor[HTML]{F4C7C3}68.05 & \cellcolor[HTML]{F4C7C3}72.03 & \cellcolor[HTML]{F4C7C3}48.7 & \cellcolor[HTML]{B7E1CD}\textbf{45.22} & \cellcolor[HTML]{F4C7C3}32.04 & \textbf{0.09} \\
StAdv & 0.5 & \cellcolor[HTML]{F4C7C3}85.05 & 5.35 & 0.07 & \textbf{56.49} & 5.09 & \cellcolor[HTML]{B7E1CD}\textbf{23.63} & \textbf{58.09} & 5.51 & \textbf{11.29} & \cellcolor[HTML]{F4C7C3}73.77 & \cellcolor[HTML]{F4C7C3}75.86 & \cellcolor[HTML]{F4C7C3}51.32 & 42.85 & 34.11 & 0.06 \\ \hline
ReColor & 0 & 93.61 & 37.17 & 7.03 & 0.01 & 67.48 & 55.53 & 37.14 & 8.27 & 45.36 & 35.55 & 60.92 & 77.2 & \textbf{32.28} & 38.66 & 0 \\
ReColor & 0.1 & 93.79 & \cellcolor[HTML]{F4C7C3}35.12 & 6.7 & 0 & 67.12 & \cellcolor[HTML]{F4C7C3}51.54 & 37.64 & 8.69 & \cellcolor[HTML]{F4C7C3}43.64 & 36.17 & \cellcolor[HTML]{B7E1CD}\textbf{63.33} & \cellcolor[HTML]{F4C7C3}76 & \cellcolor[HTML]{F4C7C3}28.53 & 37.87 & 0 \\
ReColor & 0.2 & \textbf{93.84} & \textbf{37.7} & \textbf{7.87} & 0.01 & \cellcolor[HTML]{B7E1CD}68.67 & \textbf{55.97} & \cellcolor[HTML]{B7E1CD}38.3 & \cellcolor[HTML]{B7E1CD}9.81 & \cellcolor[HTML]{B7E1CD}46.85 & \cellcolor[HTML]{B7E1CD}\textbf{38.54} & 60.01 & \textbf{77.73} & \cellcolor[HTML]{F4C7C3}31.13 & 39.38 & 0.01 \\
ReColor & 0.5 & 94.57 & \cellcolor[HTML]{F4C7C3}32.67 & \cellcolor[HTML]{F4C7C3}5.83 & \cellcolor[HTML]{B7E1CD}\textbf{32.12} & \cellcolor[HTML]{B7E1CD}\textbf{73.79} & \cellcolor[HTML]{F4C7C3}52.74 & \cellcolor[HTML]{B7E1CD}\textbf{38.66} & \cellcolor[HTML]{B7E1CD}\textbf{20.19} & \cellcolor[HTML]{B7E1CD}\textbf{50.07} & 35.85 & \cellcolor[HTML]{B7E1CD}61.98 & \cellcolor[HTML]{F4C7C3}75.72 & \cellcolor[HTML]{F4C7C3}24.93 & \cellcolor[HTML]{B7E1CD}\textbf{42.05} & \cellcolor[HTML]{B7E1CD}\textbf{2.27} \\ \hline
Gabor & 0 & \textbf{94.08} & 0.3 & 0.01 & 0.01 & 4.43 & \textbf{92.39} & 16.96 & 8.96 & 2.08 & 2.31 & 17.99 & 41.61 & \textbf{11.87} & 16.58 & 0 \\
Gabor & 0.1 & 93.33 & 0.75 & 0.02 & 0.28 & \cellcolor[HTML]{B7E1CD}29.93 & \cellcolor[HTML]{F4C7C3}91.15 & \cellcolor[HTML]{B7E1CD}17.97 & \cellcolor[HTML]{B7E1CD}16.13 & \cellcolor[HTML]{B7E1CD}11.13 & \cellcolor[HTML]{B7E1CD}8.11 & \cellcolor[HTML]{B7E1CD}20.77 & \cellcolor[HTML]{B7E1CD}46.33 & \cellcolor[HTML]{F4C7C3}10.78 & \cellcolor[HTML]{B7E1CD}21.11 & 0 \\
Gabor & 0.2 & 93.38 & 1.17 & 0.01 & \cellcolor[HTML]{B7E1CD}9.4 & \cellcolor[HTML]{B7E1CD}54.87 & \cellcolor[HTML]{F4C7C3}91.15 & \cellcolor[HTML]{B7E1CD}\textbf{25.47} & \cellcolor[HTML]{B7E1CD}\textbf{33.47} & \cellcolor[HTML]{B7E1CD}26.37 & \cellcolor[HTML]{B7E1CD}\textbf{19.74} & \cellcolor[HTML]{B7E1CD}21.57 & \cellcolor[HTML]{B7E1CD}51.05 & \cellcolor[HTML]{F4C7C3}9.43 & \cellcolor[HTML]{B7E1CD}\textbf{28.64} & 0 \\
Gabor & 0.5 & 93.27 & \cellcolor[HTML]{B7E1CD}\textbf{1.57} & \textbf{0.03} & \cellcolor[HTML]{B7E1CD}\textbf{10.16} & \cellcolor[HTML]{B7E1CD}\textbf{56.04} & 91.52 & \cellcolor[HTML]{B7E1CD}24.9 & \cellcolor[HTML]{B7E1CD}30.66 & \cellcolor[HTML]{B7E1CD}\textbf{26.36} & \cellcolor[HTML]{B7E1CD}15.33 & \cellcolor[HTML]{B7E1CD}\textbf{24.18} & \cellcolor[HTML]{B7E1CD}\textbf{51.55} & 11.15 & \cellcolor[HTML]{B7E1CD}28.62 & \textbf{0.01} \\ \hline
Snow & 0 & \textbf{95.89} & 0.05 & 0 & 0.01 & 2.63 & \textbf{30.13} & \textbf{92.02} & \textbf{7.23} & 0.9 & 15.17 & 31.49 & \textbf{47.98} & 22.39 & 20.83 & 0 \\
Snow & 0.1 & \cellcolor[HTML]{F4C7C3}94.68 & 0.17 & 0 & 0.12 & \cellcolor[HTML]{B7E1CD}5.99 & \cellcolor[HTML]{F4C7C3}23.92 & \cellcolor[HTML]{F4C7C3}89.23 & \cellcolor[HTML]{F4C7C3}3.96 & 1.17 & \cellcolor[HTML]{B7E1CD}26.82 & \cellcolor[HTML]{B7E1CD}45.03 & \cellcolor[HTML]{F4C7C3}42.74 & \cellcolor[HTML]{B7E1CD}23.93 & \cellcolor[HTML]{B7E1CD}21.92 & 0 \\
Snow & 0.2 & \cellcolor[HTML]{F4C7C3}94.51 & 0.21 & 0 & \cellcolor[HTML]{B7E1CD}6.28 & \cellcolor[HTML]{B7E1CD}\textbf{20.53} & 29.14 & \cellcolor[HTML]{F4C7C3}89.7 & 7.12 & \cellcolor[HTML]{B7E1CD}\textbf{6.48} & \cellcolor[HTML]{B7E1CD}\textbf{40.55} & \cellcolor[HTML]{B7E1CD}51.42 & \cellcolor[HTML]{F4C7C3}45.23 & \cellcolor[HTML]{B7E1CD}\textbf{25.96} & \cellcolor[HTML]{B7E1CD}\textbf{26.89} & 0 \\
Snow & 0.5 & \cellcolor[HTML]{F4C7C3}88.84 & \textbf{0.53} & 0 & \cellcolor[HTML]{B7E1CD}\textbf{9.13} & \cellcolor[HTML]{B7E1CD}15.16 & \cellcolor[HTML]{F4C7C3}19.39 & \cellcolor[HTML]{F4C7C3}83.83 & \cellcolor[HTML]{F4C7C3}3.6 & \cellcolor[HTML]{B7E1CD}6.13 & \cellcolor[HTML]{B7E1CD}39.54 & \cellcolor[HTML]{B7E1CD}\textbf{57.62} & \cellcolor[HTML]{F4C7C3}33.58 & \cellcolor[HTML]{F4C7C3}20.73 & \cellcolor[HTML]{B7E1CD}24.1 & 0 \\ \hline
Pixel & 0 & \textbf{94.76} & 0.07 & 0 & 0.01 & 8.87 & \textbf{57.65} & 36.6 & 88.35 & 1.77 & 14.52 & 38.18 & 67.5 & 16.21 & 27.48 & 0 \\
Pixel & 0.1 & 94.47 & \textbf{0.27} & 0 & \cellcolor[HTML]{B7E1CD}1.86 & \cellcolor[HTML]{B7E1CD}31.03 & \cellcolor[HTML]{F4C7C3}52.29 & \cellcolor[HTML]{B7E1CD}41.44 & 88.01 & \cellcolor[HTML]{B7E1CD}7.59 & \cellcolor[HTML]{B7E1CD}26.96 & \cellcolor[HTML]{B7E1CD}41.24 & \textbf{68.23} & \cellcolor[HTML]{B7E1CD}19.99 & \cellcolor[HTML]{B7E1CD}31.58 & 0 \\
Pixel & 0.2 & 94.01 & \textbf{0.27} & 0 & \cellcolor[HTML]{B7E1CD}\textbf{5.57} & \cellcolor[HTML]{B7E1CD}\textbf{34.43} & \cellcolor[HTML]{F4C7C3}51.54 & \cellcolor[HTML]{B7E1CD}\textbf{43.31} & \textbf{88.53} & \cellcolor[HTML]{B7E1CD}\textbf{8.9} & \cellcolor[HTML]{B7E1CD}\textbf{28.75} & \cellcolor[HTML]{B7E1CD}43.53 & 66.73 & \cellcolor[HTML]{B7E1CD}18.36 & \cellcolor[HTML]{B7E1CD}\textbf{32.49} & 0 \\
Pixel & 0.5 & \cellcolor[HTML]{F4C7C3}92.34 & 0.06 & 0 & \cellcolor[HTML]{B7E1CD}5.38 & \cellcolor[HTML]{B7E1CD}19.85 & \cellcolor[HTML]{F4C7C3}44.07 & \cellcolor[HTML]{B7E1CD}38.11 & \cellcolor[HTML]{F4C7C3}87.21 & \cellcolor[HTML]{B7E1CD}4.85 & \cellcolor[HTML]{B7E1CD}27.22 & \cellcolor[HTML]{B7E1CD}\textbf{49.64} & \cellcolor[HTML]{F4C7C3}58.39 & \cellcolor[HTML]{B7E1CD}\textbf{20.92} & \cellcolor[HTML]{B7E1CD}29.64 & 0 \\ \hline
JPEG & 0 & \textbf{90.26} & 56.48 & 21.5 & 0.52 & 34.74 & 68.59 & 21.12 & 10.57 & 73.46 & 40 & 74.3 & \textbf{78.35} & \textbf{28.02} & 42.3 & 0.09 \\
JPEG & 0.1 & 89.41 & \cellcolor[HTML]{B7E1CD}58.2 & \cellcolor[HTML]{B7E1CD}24.43 & 1.24 & \cellcolor[HTML]{B7E1CD}37.73 & \cellcolor[HTML]{B7E1CD}\textbf{73.19} & 22 & \cellcolor[HTML]{B7E1CD}12.59 & 74.05 & 40.57 & 75.02 & 77.88 & 27.61 & \cellcolor[HTML]{B7E1CD}43.71 & 0.41 \\
JPEG & 0.2 & \cellcolor[HTML]{F4C7C3}88.56 & \cellcolor[HTML]{B7E1CD}58.55 & \cellcolor[HTML]{B7E1CD}26.21 & \cellcolor[HTML]{B7E1CD}3.19 & \cellcolor[HTML]{B7E1CD}41.12 & \cellcolor[HTML]{B7E1CD}71.49 & 22.1 & \cellcolor[HTML]{B7E1CD}14.65 & 74.23 & 40.7 & \cellcolor[HTML]{B7E1CD}75.43 & 78.17 & \cellcolor[HTML]{F4C7C3}24.93 & \cellcolor[HTML]{B7E1CD}44.23 & 1.08 \\
JPEG & 0.5 & \cellcolor[HTML]{F4C7C3}87.33 & \cellcolor[HTML]{B7E1CD}\textbf{60.43} & \cellcolor[HTML]{B7E1CD}\textbf{29.14} & \cellcolor[HTML]{B7E1CD}\textbf{11.66} & \cellcolor[HTML]{B7E1CD}\textbf{46.74} & \cellcolor[HTML]{B7E1CD}72.68 & \cellcolor[HTML]{B7E1CD}\textbf{24.34} & \cellcolor[HTML]{B7E1CD}\textbf{17.81} & \textbf{74.37} & \cellcolor[HTML]{B7E1CD}\textbf{43.52} & \cellcolor[HTML]{B7E1CD}\textbf{75.44} & \cellcolor[HTML]{F4C7C3}77.08 & \cellcolor[HTML]{F4C7C3}25.76 & \cellcolor[HTML]{B7E1CD}\textbf{46.58} & \cellcolor[HTML]{B7E1CD}\textbf{3.39} \\ \hline
Elastic & 0 & \textbf{94.06} & 1.32 & 0.02 & 7.5 & 7.92 & 25.41 & 53.68 & 9.16 & 11.2 & 79.47 & 72.94 & 50.24 & 33.1 & 29.33 & 0.01 \\
Elastic & 0.1 & 93.49 & 1.68 & \textbf{0.03} & \cellcolor[HTML]{B7E1CD}\textbf{51.42} & \cellcolor[HTML]{B7E1CD}\textbf{41.84} & \cellcolor[HTML]{B7E1CD}\textbf{28.14} & \cellcolor[HTML]{B7E1CD}\textbf{56.78} & \cellcolor[HTML]{B7E1CD}\textbf{14.2} & \cellcolor[HTML]{B7E1CD}\textbf{26.91} & 80.19 & \cellcolor[HTML]{B7E1CD}74.93 & \cellcolor[HTML]{B7E1CD}53.78 & \cellcolor[HTML]{F4C7C3}29.16 & \cellcolor[HTML]{B7E1CD}\textbf{38.25} & 0.02 \\
Elastic & 0.2 & 93.32 & 1.87 & 0.01 & \cellcolor[HTML]{B7E1CD}17.64 & \cellcolor[HTML]{B7E1CD}11.84 & 25.67 & \cellcolor[HTML]{B7E1CD}55.38 & \cellcolor[HTML]{F4C7C3}5.61 & \cellcolor[HTML]{F4C7C3}9.2 & \cellcolor[HTML]{B7E1CD}\textbf{80.66} & \cellcolor[HTML]{B7E1CD}77.2 & \cellcolor[HTML]{B7E1CD}51.52 & \cellcolor[HTML]{B7E1CD}\textbf{36.14} & \cellcolor[HTML]{B7E1CD}31.06 & 0.01 \\
Elastic & 0.5 & \cellcolor[HTML]{F4C7C3}92.62 & \cellcolor[HTML]{B7E1CD}\textbf{2.64} & 0.1 & \cellcolor[HTML]{B7E1CD}40.11 & \cellcolor[HTML]{B7E1CD}28.38 & \cellcolor[HTML]{B7E1CD}27.2 & \cellcolor[HTML]{F4C7C3}51.69 & \cellcolor[HTML]{B7E1CD}10.24 & \cellcolor[HTML]{B7E1CD}20.46 & 80.3 & \cellcolor[HTML]{B7E1CD}\textbf{77.69} & \cellcolor[HTML]{B7E1CD}\textbf{55.69} & \cellcolor[HTML]{B7E1CD}34.25 & \cellcolor[HTML]{B7E1CD}35.73 & \textbf{0.08} \\ \hline
Wood & 0 & \textbf{93.57} & 0.03 & 0 & 0.4 & 1.27 & \textbf{18.47} & 39.44 & 3.43 & 0.37 & 33.68 & \textbf{93.04} & 28.04 & 14.11 & 19.36 & 0 \\
Wood & 0.1 & 92.79 & \textbf{0.04} & 0 & \cellcolor[HTML]{B7E1CD}3.84 & \cellcolor[HTML]{B7E1CD}4.88 & \cellcolor[HTML]{F4C7C3}16.77 & \cellcolor[HTML]{B7E1CD}42.61 & 3.32 & 1.21 & \cellcolor[HTML]{B7E1CD}37.35 & 92.3 & \cellcolor[HTML]{B7E1CD}29.72 & \cellcolor[HTML]{B7E1CD}\textbf{16.04} & \cellcolor[HTML]{B7E1CD}20.67 & 0 \\
Wood & 0.2 & 92.68 & 0.02 & 0 & \cellcolor[HTML]{B7E1CD}10.98 & \cellcolor[HTML]{B7E1CD}11.85 & \cellcolor[HTML]{F4C7C3}15.74 & \cellcolor[HTML]{B7E1CD}44.47 & \cellcolor[HTML]{B7E1CD}5.55 & \cellcolor[HTML]{B7E1CD}3.76 & \cellcolor[HTML]{B7E1CD}38.77 & 92.25 & \cellcolor[HTML]{B7E1CD}32.48 & 14.83 & \cellcolor[HTML]{B7E1CD}22.56 & 0 \\
Wood & 0.5 & \cellcolor[HTML]{F4C7C3}92.03 & \textbf{0.04} & 0 & \cellcolor[HTML]{B7E1CD}\textbf{25.21} & \cellcolor[HTML]{B7E1CD}\textbf{21.39} & 17.58 & \cellcolor[HTML]{B7E1CD}\textbf{47.09} & \cellcolor[HTML]{B7E1CD}\textbf{7.52} & \cellcolor[HTML]{B7E1CD}\textbf{6.2} & \cellcolor[HTML]{B7E1CD}\textbf{43.25} & \cellcolor[HTML]{F4C7C3}91.36 & \cellcolor[HTML]{B7E1CD}\textbf{36.32} & \cellcolor[HTML]{B7E1CD}15.18 & \cellcolor[HTML]{B7E1CD}\textbf{25.93} & 0 \\ \hline
Glitch & 0 & \textbf{93.26} & 0.02 & 0 & 0 & 11.49 & 49.03 & 24.44 & 12.47 & 3.14 & 10.89 & 31.99 & \textbf{90.77} & 16.61 & 20.9 & 0 \\
Glitch & 0.1 & \cellcolor[HTML]{F4C7C3}92.06 & 0.07 & 0 & 0 & \cellcolor[HTML]{F4C7C3}8.05 & \cellcolor[HTML]{F4C7C3}45.58 & \cellcolor[HTML]{F4C7C3}20.58 & \cellcolor[HTML]{F4C7C3}5.59 & \cellcolor[HTML]{F4C7C3}1.16 & \cellcolor[HTML]{F4C7C3}9.27 & 32.45 & \cellcolor[HTML]{F4C7C3}87.47 & 16.96 & \cellcolor[HTML]{F4C7C3}18.93 & 0 \\
Glitch & 0.2 & \cellcolor[HTML]{F4C7C3}90.33 & \cellcolor[HTML]{B7E1CD}\textbf{1.09} & \textbf{0.09} & 0.01 & \cellcolor[HTML]{B7E1CD}17.43 & \cellcolor[HTML]{B7E1CD}54.43 & \cellcolor[HTML]{F4C7C3}20.15 & \cellcolor[HTML]{F4C7C3}8.3 & \cellcolor[HTML]{B7E1CD}7.71 & \cellcolor[HTML]{B7E1CD}21.52 & \cellcolor[HTML]{B7E1CD}47.55 & \cellcolor[HTML]{F4C7C3}83.74 & \cellcolor[HTML]{B7E1CD}19.34 & \cellcolor[HTML]{B7E1CD}23.45 & 0 \\
Glitch & 0.5 & \cellcolor[HTML]{F4C7C3}92.03 & 0.36 & 0.02 & \cellcolor[HTML]{B7E1CD}\textbf{16.62} & \cellcolor[HTML]{B7E1CD}\textbf{45.09} & \cellcolor[HTML]{B7E1CD}\textbf{54.52} & \textbf{25.11} & \cellcolor[HTML]{B7E1CD}\textbf{18.73} & \cellcolor[HTML]{B7E1CD}\textbf{21.68} & \cellcolor[HTML]{B7E1CD}\textbf{27.67} & \cellcolor[HTML]{B7E1CD}\textbf{47.92} & \cellcolor[HTML]{F4C7C3}87.59 & \cellcolor[HTML]{B7E1CD}\textbf{19.87} & \cellcolor[HTML]{B7E1CD}\textbf{30.43} & 0 \\ \hline
\begin{tabular}[c]{@{}l@{}}Kaleid-\\ oscope\end{tabular} & 0 & 96.03 & 0 & 0 & 0 & 0.8 & 39.49 & 40.94 & 5.75 & 0.02 & 2.4 & 43.08 & 33.71 & 91.97 & 21.51 & 0 \\
\begin{tabular}[c]{@{}l@{}}Kaleid-\\ oscope\end{tabular} & 0.1 & \textbf{96.22} & 0.03 & 0 & \cellcolor[HTML]{B7E1CD}2.15 & \cellcolor[HTML]{B7E1CD}33.97 & \cellcolor[HTML]{F4C7C3}37.41 & \cellcolor[HTML]{F4C7C3}35.7 & \cellcolor[HTML]{B7E1CD}16.06 & \cellcolor[HTML]{B7E1CD}5.21 & \cellcolor[HTML]{B7E1CD}13.03 & \cellcolor[HTML]{B7E1CD}50.03 & \cellcolor[HTML]{B7E1CD}45.1 & \cellcolor[HTML]{B7E1CD}\textbf{93.13} & \cellcolor[HTML]{B7E1CD}27.65 & 0 \\
\begin{tabular}[c]{@{}l@{}}Kaleid-\\ oscope\end{tabular} & 0.2 & 96.14 & 0.07 & 0 & \cellcolor[HTML]{B7E1CD}7.48 & \cellcolor[HTML]{B7E1CD}43.57 & 39.96 & \cellcolor[HTML]{F4C7C3}38.95 & \cellcolor[HTML]{B7E1CD}18.03 & \cellcolor[HTML]{B7E1CD}8.43 & \cellcolor[HTML]{B7E1CD}16.61 & \cellcolor[HTML]{B7E1CD}\textbf{50.22} & \cellcolor[HTML]{B7E1CD}48.81 & 92.94 & \cellcolor[HTML]{B7E1CD}30.42 & 0 \\
\begin{tabular}[c]{@{}l@{}}Kaleid-\\ oscope\end{tabular} & 0.5 & 95.71 & \textbf{0.09} & 0 & \cellcolor[HTML]{B7E1CD}\textbf{27.92} & \cellcolor[HTML]{B7E1CD}\textbf{63.33} & \cellcolor[HTML]{B7E1CD}\textbf{41.84} & \cellcolor[HTML]{B7E1CD}\textbf{43.26} & \cellcolor[HTML]{B7E1CD}\textbf{25.65} & \cellcolor[HTML]{B7E1CD}\textbf{16.73} & \cellcolor[HTML]{B7E1CD}\textbf{23.98} & \cellcolor[HTML]{B7E1CD}48.61 & \cellcolor[HTML]{B7E1CD}\textbf{52.28} & 92.43 & \cellcolor[HTML]{B7E1CD}\textbf{36.34} & 0 \\ \hline
\end{tabular}
}}
\caption{\textbf{Intial Training Ablations- adversarial $\ell_2$ regularization on CIFAR-10. }Accuracy of initially trained models on CIFAR-10 trained using different attacks as indicated in ``Train Attack" column measured across different attacks. $\ell_2$ regularization computed using single step optimization is also considered during initial training, with regularization strength $\lambda$.  Results where regularization improves over no regularization ($\lambda = 0$) by at least 1\% accuracy are highlighted in green, while results where regularization incurs at least a 1\% drop in accuracy are highlighted in red.  Best performing with respect to regularization strength are bolded.}
\label{app:IT_ablation_L2_CIF}

\end{table*}

\begin{table*}[]{\renewcommand{\arraystretch}{1.2}
\scalebox{0.73}{
\begin{tabular}{|l|c|c|cccccccccccc|cc|}
\hline
\begin{tabular}[c]{@{}l@{}}Train\\ Attack\end{tabular} & $\lambda$ & Clean & $\ell_2$ & $\ell_\infty$ & StAdv & ReColor & Gabor & Snow & Pixel & JPEG & Elastic & Wood & Glitch & \begin{tabular}[c]{@{}c@{}}Kaleid-\\ oscope\end{tabular} & Avg & Union \\ \hline
$\ell_2$ & 0 & \textbf{91.08} & 70.02 & 29.38 & 0.79 & 33.69 & 66.93 & 24.59 & 14.99 & 64.22 & 45.13 & 70.85 & \textbf{80.3} & 30.08 & 44.25 & 0.1 \\
$\ell_2$ & 0.05 & 90.15 & 69.61 & \cellcolor[HTML]{B7E1CD}32.39 & \cellcolor[HTML]{B7E1CD}1.88 & \cellcolor[HTML]{B7E1CD}40.16 & 65.99 & \cellcolor[HTML]{B7E1CD}27.04 & \cellcolor[HTML]{B7E1CD}19.33 & 64.8 & \cellcolor[HTML]{B7E1CD}47.01 & \cellcolor[HTML]{B7E1CD}72.13 & 79.82 & 30.18 & \cellcolor[HTML]{B7E1CD}45.86 & 0.72 \\
$\ell_2$ & 0.1 & \cellcolor[HTML]{F4CCCC}89.24 & 70.36 & \cellcolor[HTML]{B7E1CD}33.37 & \cellcolor[HTML]{B7E1CD}4.27 & \cellcolor[HTML]{B7E1CD}42.96 & \cellcolor[HTML]{F4CCCC}64.75 & \cellcolor[HTML]{B7E1CD}\textbf{29.09} & \cellcolor[HTML]{B7E1CD}19.85 & \cellcolor[HTML]{B7E1CD}66.28 & \cellcolor[HTML]{B7E1CD}\textbf{49.46} & \cellcolor[HTML]{B7E1CD}72.01 & 79.6 & \cellcolor[HTML]{B7E1CD}\textbf{32.48} & \cellcolor[HTML]{B7E1CD}47.04 & \cellcolor[HTML]{B7E1CD}1.92 \\
$\ell_2$ & 0.2 & \cellcolor[HTML]{F4C7C3}89.99 & \textbf{70.38} & \cellcolor[HTML]{B7E1CD}\textbf{34.56} & \cellcolor[HTML]{B7E1CD}\textbf{13.41} & \cellcolor[HTML]{B7E1CD}\textbf{48.99} & \textbf{67.64} & \cellcolor[HTML]{B7E1CD}\textbf{29.09} & \cellcolor[HTML]{B7E1CD}\textbf{22.57} & \cellcolor[HTML]{B7E1CD}\textbf{66.64} & \cellcolor[HTML]{B7E1CD}48.38 & \cellcolor[HTML]{B7E1CD}\textbf{73.31} & 80.07 & \cellcolor[HTML]{B7E1CD}32.33 & \cellcolor[HTML]{B7E1CD}\textbf{48.94} & \cellcolor[HTML]{B7E1CD}\textbf{5.4} \\ \hline
$\ell_\infty$ & 0 & \textbf{85.53} & \textbf{59.36} & 50.98 & 6.34 & 56.27 & \textbf{68.94} & 36.79 & 20.57 & 54.02 & 51 & \textbf{64.24} & \textbf{75.94} & 39.44 & 48.66 & 1.31 \\
$\ell_\infty$ & 0.05 & 84.57 & 58.68 & 51.28 & \cellcolor[HTML]{B7E1CD}7.82 & 55.74 & \cellcolor[HTML]{F4C7C3}65.31 & \cellcolor[HTML]{B7E1CD}38.58 & 21.17 & 54.41 & \cellcolor[HTML]{B7E1CD}\textbf{52.59} & 63.87 & \cellcolor[HTML]{F4C7C3}74.38 & 39.5 & 48.61 & 1.92 \\
$\ell_\infty$ & 0.1 & 84.98 & \cellcolor[HTML]{F4C7C3}57.76 & 51.52 & \cellcolor[HTML]{B7E1CD}11.45 & \cellcolor[HTML]{B7E1CD}57.69 & \cellcolor[HTML]{F4C7C3}67.39 & \cellcolor[HTML]{B7E1CD}\textbf{39.27} & \cellcolor[HTML]{B7E1CD}22.22 & 53.62 & 51.09 & \cellcolor[HTML]{F4C7C3}60.19 & \cellcolor[HTML]{F4C7C3}74.73 & \cellcolor[HTML]{B7E1CD}41.48 & 49.03 & \cellcolor[HTML]{B7E1CD}3.45 \\
$\ell_\infty$ & 0.2 & \cellcolor[HTML]{F4C7C3}82.58 & 58.36 & \textbf{51.53} & \cellcolor[HTML]{B7E1CD}\textbf{18.98} & \cellcolor[HTML]{B7E1CD}\textbf{62.12} & \cellcolor[HTML]{F4C7C3}67.18 & \cellcolor[HTML]{B7E1CD}39.22 & \cellcolor[HTML]{B7E1CD}\textbf{23.62} & \textbf{54.73} & \cellcolor[HTML]{B7E1CD}52 & 63.35 & \cellcolor[HTML]{F4C7C3}71.72 & \cellcolor[HTML]{B7E1CD}\textbf{43.18} & \cellcolor[HTML]{B7E1CD}\textbf{50.5} & \cellcolor[HTML]{B7E1CD}\textbf{5.08} \\ \hline
StAdv & 0 & \textbf{87.12} & 5.48 & 0.07 & 56.22 & 5.69 & 17.62 & \textbf{57.8} & 5.93 & 11.09 & \textbf{76.02} & \textbf{77.47} & 54.04 & \textbf{43.4} & 34.24 & 0.05 \\
StAdv & 0.05 & \cellcolor[HTML]{F4C7C3}75.6 & \cellcolor[HTML]{F4C7C3}3.52 & 0.02 & \cellcolor[HTML]{B7E1CD}69.34 & \cellcolor[HTML]{B7E1CD}23.98 & \cellcolor[HTML]{F4C7C3}15.36 & \cellcolor[HTML]{F4C7C3}37.22 & \cellcolor[HTML]{F4C7C3}3.69 & \cellcolor[HTML]{F4C7C3}9.73 & \cellcolor[HTML]{F4C7C3}55.75 & \cellcolor[HTML]{F4C7C3}66.64 & \cellcolor[HTML]{F4C7C3}37.79 & \cellcolor[HTML]{F4C7C3}26.08 & \cellcolor[HTML]{F4C7C3}29.09 & 0.02 \\
StAdv & 0.1 & \cellcolor[HTML]{F4C7C3}81.12 & \cellcolor[HTML]{B7E1CD}8.78 & \textbf{0.24} & \cellcolor[HTML]{B7E1CD}\textbf{69.82} & \cellcolor[HTML]{B7E1CD}15.75 & \cellcolor[HTML]{B7E1CD}27.34 & \cellcolor[HTML]{F4C7C3}48.4 & 5.1 & \cellcolor[HTML]{B7E1CD}\textbf{29.68} & \cellcolor[HTML]{F4C7C3}66.67 & 77.1 & 53.77 & \cellcolor[HTML]{F4C7C3}34.7 & \cellcolor[HTML]{B7E1CD}36.45 & \textbf{0.17} \\
StAdv & 0.2 & \cellcolor[HTML]{F4C7C3}84.4 & \cellcolor[HTML]{B7E1CD}\textbf{10.93} & 0.19 & \cellcolor[HTML]{B7E1CD}69.1 & \cellcolor[HTML]{B7E1CD}\textbf{28.03} & \cellcolor[HTML]{B7E1CD}\textbf{33.19} & \cellcolor[HTML]{F4C7C3}47.19 & \textbf{6.06} & \cellcolor[HTML]{B7E1CD}28.29 & \cellcolor[HTML]{F4C7C3}66.46 & \cellcolor[HTML]{F4C7C3}76.31 & \cellcolor[HTML]{B7E1CD}\textbf{58.5} & \cellcolor[HTML]{F4C7C3}40.51 & \cellcolor[HTML]{B7E1CD}\textbf{38.73} & 0.14 \\ \hline
ReColor & 0 & 93.61 & \textbf{37.17} & 7.03 & 0.01 & 67.48 & 55.53 & 37.14 & 8.27 & 45.36 & 35.55 & 60.92 & 77.2 & \textbf{32.28} & 38.66 & 0 \\
ReColor & 0.05 & \textbf{93.68} & \cellcolor[HTML]{F4C7C3}35.65 & \textbf{7.10} & 0.07 & \cellcolor[HTML]{B7E1CD}69.73 & \cellcolor[HTML]{F4C7C3}51.67 & 36.85 & 9.26 & 44.65 & \cellcolor[HTML]{B7E1CD}\textbf{38.42} & 60.68 & \textbf{77.36} & 31.85 & 38.61 & 0.02 \\
ReColor & 0.1 & 93.63 & \cellcolor[HTML]{F4C7C3}33.32 & 7.00 & \cellcolor[HTML]{B7E1CD}19.46 & \cellcolor[HTML]{B7E1CD}77.93 & \cellcolor[HTML]{B7E1CD}\textbf{57.73} & 36.18 & \cellcolor[HTML]{B7E1CD}17.35 & \cellcolor[HTML]{B7E1CD}47.63 & 34.61 & \textbf{61.9} & 76.26 & \cellcolor[HTML]{F4C7C3}27.3 & \cellcolor[HTML]{B7E1CD}41.39 & \cellcolor[HTML]{B7E1CD}\textbf{2.22} \\
ReColor & 0.2 & 92.67 & \cellcolor[HTML]{F4C7C3}29.88 & \cellcolor[HTML]{F4C7C3}5.83 & \cellcolor[HTML]{B7E1CD}\textbf{30.85} & \cellcolor[HTML]{B7E1CD}\textbf{83.95} & 55.41 & \textbf{37.24} & \cellcolor[HTML]{B7E1CD}\textbf{18.71} & \cellcolor[HTML]{B7E1CD}\textbf{48.31} & 34.9 & 60.33 & \cellcolor[HTML]{F4C7C3}74.58 & \cellcolor[HTML]{F4C7C3}27.93 & \cellcolor[HTML]{B7E1CD}\textbf{42.33} & \cellcolor[HTML]{B7E1CD}2.07 \\ \hline
Gabor & 0 & \textbf{94.08} & 0.3 & 0.01 & 0.01 & 4.43 & \textbf{92.39} & 16.96 & 8.96 & 2.08 & 2.31 & 17.99 & 41.61 & \textbf{11.87} & 16.58 & 0 \\
Gabor & 0.05 & 93.84 & 0.41 & 0.01 & 0.17 & \cellcolor[HTML]{B7E1CD}20.82 & 91.9 & 17.71 & \cellcolor[HTML]{B7E1CD}14.23 & \cellcolor[HTML]{B7E1CD}7.61 & \cellcolor[HTML]{B7E1CD}6.36 & \cellcolor[HTML]{F4C7C3}16.61 & \cellcolor[HTML]{B7E1CD}45.52 & \cellcolor[HTML]{F4C7C3}10.71 & \cellcolor[HTML]{B7E1CD}19.34 & 0 \\
Gabor & 0.1 & 93.69 & 1.14 & \textbf{0.03} & \cellcolor[HTML]{B7E1CD}4.47 & \cellcolor[HTML]{B7E1CD}45.21 & \cellcolor[HTML]{F4C7C3}91.2 & \cellcolor[HTML]{B7E1CD}22.26 & \cellcolor[HTML]{B7E1CD}27.21 & \cellcolor[HTML]{B7E1CD}22.28 & \cellcolor[HTML]{B7E1CD}19.2 & \cellcolor[HTML]{B7E1CD}19.73 & \cellcolor[HTML]{B7E1CD}47.75 & \cellcolor[HTML]{F4C7C3}9.99 & \cellcolor[HTML]{B7E1CD}25.87 & 0 \\
Gabor & 0.2 & 93.61 & \cellcolor[HTML]{B7E1CD}\textbf{1.35} & 0.02 & \cellcolor[HTML]{B7E1CD}\textbf{13.59} & \cellcolor[HTML]{B7E1CD}\textbf{57.17} & \cellcolor[HTML]{F4C7C3}90.85 & \cellcolor[HTML]{B7E1CD}\textbf{29.97} & \cellcolor[HTML]{B7E1CD}\textbf{33.81} & \cellcolor[HTML]{B7E1CD}\textbf{31.63} & \cellcolor[HTML]{B7E1CD}\textbf{25.28} & \cellcolor[HTML]{B7E1CD}\textbf{22.52} & \cellcolor[HTML]{B7E1CD}\textbf{52.4} & \cellcolor[HTML]{F4C7C3}9.08 & \cellcolor[HTML]{B7E1CD}\textbf{30.64} & 0 \\ \hline
Snow & 0 & \textbf{95.89} & 0.05 & 0 & 0.01 & 2.63 & \textbf{30.13} & \textbf{92.02} & 7.23 & 0.9 & 15.17 & 31.49 & 47.98 & 22.39 & 20.83 & 0 \\
Snow & 0.05 & \cellcolor[HTML]{F4C7C3}89.84 & 0.12 & 0 & 0.01 & \cellcolor[HTML]{F4C7C3}0.81 & \cellcolor[HTML]{F4C7C3}24.79 & \cellcolor[HTML]{F4C7C3}82.66 & \cellcolor[HTML]{F4C7C3}0.95 & 0.14 & 14.89 & \cellcolor[HTML]{F4C7C3}28.76 & \cellcolor[HTML]{F4C7C3}35.46 & \cellcolor[HTML]{F4C7C3}21.3 & \cellcolor[HTML]{F4C7C3}17.49 & 0 \\
Snow & 0.1 & \cellcolor[HTML]{F4C7C3}88.07 & \textbf{0.28} & 0 & 0.18 & \cellcolor[HTML]{F4C7C3}1.39 & \cellcolor[HTML]{F4C7C3}18.33 & \cellcolor[HTML]{F4C7C3}83.28 & \cellcolor[HTML]{F4C7C3}1.12 & 0.88 & \cellcolor[HTML]{B7E1CD}29.73 & \cellcolor[HTML]{B7E1CD}\textbf{56.23} & \cellcolor[HTML]{F4C7C3}27.04 & \cellcolor[HTML]{F4C7C3}19.81 & 19.86 & 0 \\
Snow & 0.2 & \cellcolor[HTML]{F4C7C3}94.56 & 0.21 & 0 & \cellcolor[HTML]{B7E1CD}\textbf{29.86} & \cellcolor[HTML]{B7E1CD}\textbf{42.55} & \cellcolor[HTML]{F4C7C3}23.26 & \cellcolor[HTML]{F4C7C3}90.96 & \cellcolor[HTML]{B7E1CD}\textbf{14.4} & \cellcolor[HTML]{B7E1CD}\textbf{15.56} & \cellcolor[HTML]{B7E1CD}\textbf{49.51} & \cellcolor[HTML]{B7E1CD}50.92 & \cellcolor[HTML]{B7E1CD}\textbf{50.35} & \cellcolor[HTML]{B7E1CD}\textbf{26.23} & \cellcolor[HTML]{B7E1CD}\textbf{32.82} & 0 \\ \hline
Pixel & 0 & \textbf{94.76} & \textbf{0.07} & 0 & 0.01 & 8.87 & 57.65 & 36.6 & 88.35 & 1.77 & 14.52 & 38.18 & \textbf{67.5} & 16.21 & 27.48 & 0 \\
Pixel & 0.05 & 93.81 & 0.01 & 0 & 0.03 & \cellcolor[HTML]{B7E1CD}12.65 & \cellcolor[HTML]{B7E1CD}\textbf{59.31} & 36.8 & 88.88 & \cellcolor[HTML]{B7E1CD}3.83 & 14.19 & \cellcolor[HTML]{B7E1CD}40.33 & \cellcolor[HTML]{F4C7C3}62.99 & 16.9 & 27.99 & 0 \\
Pixel & 0.1 & \cellcolor[HTML]{F4C7C3}93.54 & 0.02 & 0 & 0.22 & \cellcolor[HTML]{B7E1CD}15.32 & \cellcolor[HTML]{F4C7C3}54.00 & \cellcolor[HTML]{B7E1CD}38.37 & \cellcolor[HTML]{B7E1CD}\textbf{89.45} & \cellcolor[HTML]{B7E1CD}3.9 & 14.35 & 38.81 & \cellcolor[HTML]{F4C7C3}63.48 & \cellcolor[HTML]{B7E1CD}17.75 & 27.97 & 0 \\
Pixel & 0.2 & \cellcolor[HTML]{F4C7C3}93.06 & 0.06 & \textbf{0.01} & \cellcolor[HTML]{B7E1CD}\textbf{3.59} & \cellcolor[HTML]{B7E1CD}\textbf{22.18} & \cellcolor[HTML]{F4C7C3}50.55 & \cellcolor[HTML]{B7E1CD}\textbf{40.96} & \cellcolor[HTML]{B7E1CD}89.39 & \cellcolor[HTML]{B7E1CD}\textbf{9.85} & \cellcolor[HTML]{B7E1CD}\textbf{16.49} & \cellcolor[HTML]{B7E1CD}\textbf{40.79} & \cellcolor[HTML]{F4C7C3}61.44 & \cellcolor[HTML]{B7E1CD}\textbf{18.25} & \cellcolor[HTML]{B7E1CD}\textbf{29.46} & 0 \\ \hline
JPEG & 0 & \textbf{90.26} & 56.48 & 21.5 & 0.52 & 34.74 & 68.59 & 21.12 & 10.57 & 73.46 & 40 & 74.3 & 78.35 & \textbf{28.02} & 42.3 & 0.09 \\
JPEG & 0.05 & 89.52 & \cellcolor[HTML]{B7E1CD}57.89 & \cellcolor[HTML]{B7E1CD}23.81 & 1.19 & \cellcolor[HTML]{B7E1CD}37.2 & \cellcolor[HTML]{B7E1CD}71.84 & 21.88 & \cellcolor[HTML]{B7E1CD}13.16 & 73.9 & \cellcolor[HTML]{F4C7C3}38.82 & 74.05 & \textbf{78.61} & 27.97 & \cellcolor[HTML]{B7E1CD}43.36 & 0.46 \\
JPEG & 0.1 & \cellcolor[HTML]{F4C7C3}89.09 & \cellcolor[HTML]{B7E1CD}58.18 & \cellcolor[HTML]{B7E1CD}25.89 & \cellcolor[HTML]{B7E1CD}4.28 & \cellcolor[HTML]{B7E1CD}40.45 & \cellcolor[HTML]{B7E1CD}\textbf{74.3} & 20.85 & \cellcolor[HTML]{B7E1CD}14.35 & \textbf{74.42} & 40.82 & 75.08 & 78.07 & \cellcolor[HTML]{F4C7C3}25.35 & \cellcolor[HTML]{B7E1CD}44.34 & \cellcolor[HTML]{B7E1CD}1.42 \\
JPEG & 0.2 & \cellcolor[HTML]{F4C7C3}87.98 & \cellcolor[HTML]{B7E1CD}\textbf{58.94} & \cellcolor[HTML]{B7E1CD}\textbf{26.85} & \cellcolor[HTML]{B7E1CD}\textbf{10.19} & \cellcolor[HTML]{B7E1CD}\textbf{44.26} & \cellcolor[HTML]{B7E1CD}70.96 & \cellcolor[HTML]{B7E1CD}\textbf{22.18} & \cellcolor[HTML]{B7E1CD}\textbf{16.37} & 74.17 & \cellcolor[HTML]{B7E1CD}\textbf{41.13} & \cellcolor[HTML]{B7E1CD}\textbf{76.03} & 77.65 & \cellcolor[HTML]{F4C7C3}24.9 & \cellcolor[HTML]{B7E1CD}\textbf{45.3} & \cellcolor[HTML]{B7E1CD}\textbf{2.74} \\ \hline
Elastic & 0 & 94.06 & 1.32 & \textbf{0.02} & \textbf{7.5} & 7.92 & \textbf{25.41} & \textbf{53.68} & 9.16 & 11.2 & 79.47 & 72.94 & \textbf{50.24} & \textbf{33.1} & \textbf{29.33} & \textbf{0.01} \\
Elastic & 0.05 & 93.88 & \textbf{1.49} & 0.01 & 7.27 & \textbf{8.18} & \cellcolor[HTML]{F4C7C3}23.48 & \cellcolor[HTML]{F4C7C3}51.12 & \textbf{9.32} & \textbf{11.5} & 79.77 & 73.88 & 50.2 & \cellcolor[HTML]{F4C7C3}30.32 & 28.88 & \textbf{0.01} \\
Elastic & 0.1 & \textbf{94.13} & 1.2 & 0 & 7.25 & 7.87 & \cellcolor[HTML]{F4C7C3}20.68 & 53.62 & 8.52 & 10.93 & \textbf{79.93} & \cellcolor[HTML]{B7E1CD}\textbf{74.77} & 49.44 & \cellcolor[HTML]{F4C7C3}30.54 & 28.73 & 0 \\
Elastic & 0.2 & 93.79 & 1.23 & 0.01 & \cellcolor[HTML]{F4C7C3}6.49 & \cellcolor[HTML]{F4C7C3}6.77 & \cellcolor[HTML]{F4C7C3}21.8 & 53.29 & \cellcolor[HTML]{F4C7C3}7.64 & 10.75 & 79.76 & 73.77 & 49.72 & 32.75 & 28.67 & 0 \\ \hline
Wood & 0 & 93.57 & 0.03 & 0 & 0.4 & 1.27 & 18.47 & 39.44 & 3.43 & 0.37 & 33.68 & \textbf{93.04} & 28.04 & 14.11 & 19.36 & 0 \\
Wood & 0.05 & 92.78 & 0.01 & 0 & 0.65 & 1.65 & \cellcolor[HTML]{F4C7C3}16.47 & \cellcolor[HTML]{B7E1CD}41.6 & \cellcolor[HTML]{F4C7C3}2.41 & 0.51 & \cellcolor[HTML]{B7E1CD}35.08 & 92.36 & 27.79 & \cellcolor[HTML]{B7E1CD}\textbf{17.15} & 19.64 & 0 \\
Wood & 0.1 & 92.59 & 0.02 & 0 & \cellcolor[HTML]{B7E1CD}3.3 & \cellcolor[HTML]{B7E1CD}\textbf{5.25} & 18.41 & \cellcolor[HTML]{B7E1CD}43.33 & 4.07 & \cellcolor[HTML]{B7E1CD}1.67 & \cellcolor[HTML]{B7E1CD}37.47 & \cellcolor[HTML]{F4C7C3}92.02 & \cellcolor[HTML]{B7E1CD}29.95 & 14.92 & \cellcolor[HTML]{B7E1CD}20.87 & 0 \\
Wood & 0.2 & \cellcolor[HTML]{F4C7C3}91.92 & \textbf{0.03} & 0 & \cellcolor[HTML]{B7E1CD}\textbf{8.21} & \cellcolor[HTML]{B7E1CD}\textbf{8.76} & \cellcolor[HTML]{F4C7C3}\textbf{15.98} & \cellcolor[HTML]{B7E1CD}42.64 & \cellcolor[HTML]{B7E1CD}\textbf{5.12} & \cellcolor[HTML]{B7E1CD}\textbf{2.58} & \cellcolor[HTML]{B7E1CD}\textbf{38.15} & \cellcolor[HTML]{F4C7C3}91.5 & \cellcolor[HTML]{B7E1CD}\textbf{32.01} & 14.15 & \cellcolor[HTML]{B7E1CD}\textbf{21.59} & 0 \\ \hline
Glitch & 0 & 93.26 & 0.02 & 0 & 0 & 11.49 & 49.03 & 24.44 & 12.47 & 3.14 & 10.89 & 31.99 & 90.77 & \textbf{16.61} & 20.9 & 0 \\
Glitch & 0.05 & 92.38 & 0.01 & 0 & 0.03 & \cellcolor[HTML]{B7E1CD}16.86 & 48.73 & \cellcolor[HTML]{F4C7C3}23.12 & \cellcolor[HTML]{F4C7C3}9.96 & \cellcolor[HTML]{B7E1CD}5.16 & \cellcolor[HTML]{B7E1CD}14.01 & 31.82 & 90.8 & 15.63 & 21.34 & 0 \\
Glitch & 0.1 & \cellcolor[HTML]{F4C7C3}92.14 & 0.02 & 0 & 0.23 & \cellcolor[HTML]{B7E1CD}24.38 & 48.71 & \cellcolor[HTML]{F4C7C3}22.89 & \cellcolor[HTML]{F4C7C3}10.69 & \cellcolor[HTML]{B7E1CD}5.38 & \cellcolor[HTML]{B7E1CD}17.19 & \textbf{32.08} & \textbf{91.19} & \cellcolor[HTML]{F4C7C3}14.59 & \cellcolor[HTML]{B7E1CD}22.28 & 0 \\
Glitch & 0.2 & 92.62 & \textbf{0.04} & 0 & \cellcolor[HTML]{B7E1CD}\textbf{5.11} & \cellcolor[HTML]{B7E1CD}\textbf{39.62} & \cellcolor[HTML]{B7E1CD}\textbf{56.19} & \cellcolor[HTML]{B7E1CD}\textbf{26.18} & \cellcolor[HTML]{B7E1CD}\textbf{18.76} & \cellcolor[HTML]{B7E1CD}\textbf{14.5} & \cellcolor[HTML]{B7E1CD}\textbf{26.28} & 31.86 & 90.95 & 16.25 & \cellcolor[HTML]{B7E1CD}\textbf{27.15} & 0 \\ \hline
\begin{tabular}[c]{@{}l@{}}Kaleid-\\ oscope\end{tabular} & 0 & 96.03 & 0 & 0 & 0 & 0.8 & 39.49 & \textbf{40.94} & 5.75 & 0.02 & 2.4 & 43.08 & 33.71 & 91.97 & \textbf{21.51} & 0 \\
\begin{tabular}[c]{@{}l@{}}Kaleid-\\ oscope\end{tabular} & 0.05 & \textbf{96.31} & 0 & 0 & 0 & \textbf{1.59} & 39 & \cellcolor[HTML]{F4C7C3}35.73 & 6.29 & 0.02 & 1.88 & \cellcolor[HTML]{F4C7C3}40.31 & \cellcolor[HTML]{B7E1CD}36.53 & \cellcolor[HTML]{B7E1CD}\textbf{93.07} & 21.2 & 0 \\
\begin{tabular}[c]{@{}l@{}}Kaleid-\\ oscope\end{tabular} & 0.1 & 96.22 & 0 & 0 & 0 & 1.4 & \textbf{39.98} & \cellcolor[HTML]{F4C7C3}33.78 & \cellcolor[HTML]{B7E1CD}\textbf{6.75} & 0.01 & 1.9 & \textbf{43.17} & \cellcolor[HTML]{B7E1CD}35.1 & 92.83 & 21.24 & 0 \\
\begin{tabular}[c]{@{}l@{}}Kaleid-\\ oscope\end{tabular} & 0.2 & 96.01 & 0 & 0 & 0 & 1.33 & \cellcolor[HTML]{F4C7C3}34.92 & \cellcolor[HTML]{F4C7C3}35.41 & 6.01 & \textbf{0.1} & \textbf{2.61} & \cellcolor[HTML]{F4C7C3}38.35 & \cellcolor[HTML]{B7E1CD}\textbf{36.8} & 92.7 & 20.69 & 0 \\ \hline
\end{tabular}}}
\caption{\textbf{Intial Training Ablations- variation regularization on CIFAR-10. }Accuracy of initially trained models on CIFAR-10 trained using different attacks as indicated in ``Train Attack" column measured across different attacks.  Variation regularization computed using single step optimization is also considered during initial training, with regularization strength $\lambda$.  Results where regularization improves over no regularization ($\lambda = 0$) by at least 1\% accuracy are highlighted in green, while results where regularization incurs at least a 1\% drop in accuracy are highlighted in red.  Best performing with respect to regularization strength are bolded.}
\label{app:IT_ablation_VR_CIF}
\end{table*}

\begin{table*}[]{\renewcommand{\arraystretch}{1.2}
\scalebox{0.73}{
\begin{tabular}{|l|c|c|cccccccccccc|cc|}
\hline
\begin{tabular}[c]{@{}l@{}}Train\\ Attack\end{tabular}     & $\lambda$ & Clean   & $\ell_2$     & $\ell_\infty$   & StAdv  & ReColor& Gabor  & Snow   & Pixel  & JPEG   & Elastic& Wood   & Glitch & \begin{tabular}[c]{@{}c@{}}Kaleid-\\ oscope\end{tabular} & Avg    & Union \\ \hline
$\ell_2$ & 0  & \textbf{91.08} & 70.02  & 29.38  & 0.79   & 33.69  & 66.93  & 24.59  & 14.99  & 64.22  & 45.13  & 70.85  & \textbf{80.3}    & 30.08     & 44.25  & 0.1   \\
$\ell_2$ & 1  & 90.4    & \textbf{70.22}   & \cellcolor[HTML]{B7E1CD}\textbf{32.73} & \cellcolor[HTML]{B7E1CD}6.17    & \cellcolor[HTML]{B7E1CD}41.96   & \cellcolor[HTML]{B7E1CD}\textbf{69.81} & 25.46  & \cellcolor[HTML]{B7E1CD}18.13   & \cellcolor[HTML]{B7E1CD}65.98   & 45.71  & \cellcolor[HTML]{B7E1CD}\textbf{72.72} & 79.51  & 30.13     & \cellcolor[HTML]{B7E1CD}46.54   & \cellcolor[HTML]{B7E1CD}2.48   \\
$\ell_2$ & 2  & \cellcolor[HTML]{F4CCCC}89.45 & 69.44  & \cellcolor[HTML]{B7E1CD}31.97   & \cellcolor[HTML]{B7E1CD}12.15   & \cellcolor[HTML]{B7E1CD}51.04   & \cellcolor[HTML]{B7E1CD}69.03   & \cellcolor[HTML]{B7E1CD}25.71   & \cellcolor[HTML]{B7E1CD}19.54   & \cellcolor[HTML]{B7E1CD}\textbf{66.11} & \cellcolor[HTML]{B7E1CD}\textbf{47.5}  & 71.44  & \cellcolor[HTML]{F4CCCC}79.25   & 30.8      & \cellcolor[HTML]{B7E1CD}47.83   & \cellcolor[HTML]{B7E1CD}3.21   \\
$\ell_2$ & 5  & \cellcolor[HTML]{F4CCCC}88.34 & \cellcolor[HTML]{F4CCCC}66.66   & \cellcolor[HTML]{F4CCCC}27.41   & \cellcolor[HTML]{B7E1CD}\textbf{26.22} & \cellcolor[HTML]{B7E1CD}\textbf{60.22} & \cellcolor[HTML]{B7E1CD}69.16   & \cellcolor[HTML]{B7E1CD}\textbf{26.67} & \cellcolor[HTML]{B7E1CD}\textbf{22.57} & 64.08  & \cellcolor[HTML]{B7E1CD}46.83   & 71.14  & \cellcolor[HTML]{F4CCCC}77.6    & \cellcolor[HTML]{B7E1CD}\textbf{31.36}    & \cellcolor[HTML]{B7E1CD}\textbf{49.16} & \cellcolor[HTML]{B7E1CD}\textbf{6.23} \\ \hline
$\ell_\infty$      & 0  & \textbf{85.53} & 59.36  & \textbf{50.98}   & 6.34   & 56.27  & \textbf{68.94}   & 36.79  & 20.57  & 54.02  & 51     & 64.24  & 75.94  & \textbf{39.44}    & 48.66  & 1.31  \\
$\ell_\infty$      & 1  & 85.28   & \cellcolor[HTML]{B7E1CD}\textbf{63.66} & 50.72  & \cellcolor[HTML]{B7E1CD}10.51   & \cellcolor[HTML]{B7E1CD}60.1    & \cellcolor[HTML]{F4C7C3}66.63   & 36.67  & \cellcolor[HTML]{B7E1CD}21.66   & \cellcolor[HTML]{B7E1CD}60.69   & 51.72  & \cellcolor[HTML]{B7E1CD}\textbf{65.61} & \textbf{76.04}   & \cellcolor[HTML]{F4C7C3}38.41      & \cellcolor[HTML]{B7E1CD}50.2    & \cellcolor[HTML]{B7E1CD}3.01   \\
$\ell_\infty$      & 2  & \cellcolor[HTML]{F4C7C3}81.97 & \cellcolor[HTML]{B7E1CD}63.5    & \cellcolor[HTML]{F4C7C3}48.24   & \cellcolor[HTML]{B7E1CD}17.33   & \cellcolor[HTML]{B7E1CD}64.72   & \cellcolor[HTML]{F4C7C3}65.73   & \cellcolor[HTML]{B7E1CD}\textbf{38.7}  & \cellcolor[HTML]{B7E1CD}24.49   & \cellcolor[HTML]{B7E1CD}\textbf{61.25} & \textbf{51.96}   & 63.63  & \cellcolor[HTML]{F4C7C3}72.82   & 39.16     & \cellcolor[HTML]{B7E1CD}50.96   & \cellcolor[HTML]{B7E1CD}4.97   \\
$\ell_\infty$      & 5  & \cellcolor[HTML]{F4C7C3}78.04 & 60.28  & \cellcolor[HTML]{F4C7C3}40.59   & \cellcolor[HTML]{B7E1CD}\textbf{42.25} & \cellcolor[HTML]{B7E1CD}\textbf{70}    & \cellcolor[HTML]{F4C7C3}67.06   & \cellcolor[HTML]{F4C7C3}33.4    & \cellcolor[HTML]{B7E1CD}\textbf{26.57} & \cellcolor[HTML]{B7E1CD}60.07   & \cellcolor[HTML]{F4C7C3}49.21   & 64.61  & \cellcolor[HTML]{F4C7C3}67.08   & \cellcolor[HTML]{F4C7C3}38.43      & \cellcolor[HTML]{B7E1CD}\textbf{51.63} & \cellcolor[HTML]{B7E1CD}\textbf{8.36} \\ \hline
StAdv     & 0  & \textbf{87.12} & 5.48   & 0.07   & 56.22  & 5.69   & 17.62  & \textbf{57.8}    & 5.93   & 11.09  & \textbf{76.02}   & \textbf{77.47}   & 54.04  & \textbf{43.4}     & 34.24  & 0.05  \\
StAdv     & 1  & \cellcolor[HTML]{F4C7C3}81.36 & \cellcolor[HTML]{B7E1CD}24.44   & \cellcolor[HTML]{B7E1CD}1.29    & \cellcolor[HTML]{B7E1CD}\textbf{73.69} & \cellcolor[HTML]{B7E1CD}53.63   & \cellcolor[HTML]{B7E1CD}40.94   & \cellcolor[HTML]{F4C7C3}39.39   & \cellcolor[HTML]{B7E1CD}10.47   & \cellcolor[HTML]{B7E1CD}39.05   & \cellcolor[HTML]{F4C7C3}65.18   & \cellcolor[HTML]{F4C7C3}72.43   & \cellcolor[HTML]{B7E1CD}62.43   & \cellcolor[HTML]{F4C7C3}36.21      & \cellcolor[HTML]{B7E1CD}43.26   & 0.88  \\
StAdv     & 2  & \cellcolor[HTML]{F4C7C3}82.8  & \cellcolor[HTML]{B7E1CD}\textbf{38.42} & \cellcolor[HTML]{B7E1CD}\textbf{3.9}   & \cellcolor[HTML]{B7E1CD}71.03   & \cellcolor[HTML]{B7E1CD}\textbf{55.3}  & \cellcolor[HTML]{B7E1CD}\textbf{51.03} & \cellcolor[HTML]{F4C7C3}39.71   & \cellcolor[HTML]{B7E1CD}\textbf{12.12} & \cellcolor[HTML]{B7E1CD}\textbf{50.4}  & \cellcolor[HTML]{F4C7C3}63.57   & \cellcolor[HTML]{F4C7C3}74.24   & \cellcolor[HTML]{B7E1CD}\textbf{68.98} & \cellcolor[HTML]{F4C7C3}38.32      & \cellcolor[HTML]{B7E1CD}\textbf{47.25} & \cellcolor[HTML]{B7E1CD}\textbf{2.11} \\ \hline
ReColor   & 0  & \textbf{93.61} & 37.17  & 7.03   & 0.01   & 67.48  & 55.53  & \textbf{37.14}   & 8.27   & 45.36  & 35.55  & 60.92  & 77.2   & 32.28     & 38.66  & 0     \\
ReColor   & 1  & 92.96   & \cellcolor[HTML]{B7E1CD}51.36   & \cellcolor[HTML]{B7E1CD}13.58   & \cellcolor[HTML]{B7E1CD}5.96    & \cellcolor[HTML]{B7E1CD}73.47   & \cellcolor[HTML]{B7E1CD}63.31   & \cellcolor[HTML]{F4C7C3}31.94   & 8.73   & \cellcolor[HTML]{B7E1CD}\textbf{59.67} & \cellcolor[HTML]{B7E1CD}42.6    & \cellcolor[HTML]{B7E1CD}65.49   & \cellcolor[HTML]{B7E1CD}\textbf{78.37} & \cellcolor[HTML]{B7E1CD}34.22      & \cellcolor[HTML]{B7E1CD}44.06   & \cellcolor[HTML]{B7E1CD}1.43   \\
ReColor   & 2  & \cellcolor[HTML]{F4C7C3}92.34 & \cellcolor[HTML]{B7E1CD}53.34   & \cellcolor[HTML]{B7E1CD}14.88   & \cellcolor[HTML]{B7E1CD}17.02   & \cellcolor[HTML]{B7E1CD}78.21   & \cellcolor[HTML]{B7E1CD}63.32   & \cellcolor[HTML]{F4C7C3}34.31   & \cellcolor[HTML]{B7E1CD}13.28   & \cellcolor[HTML]{B7E1CD}58.69   & \cellcolor[HTML]{B7E1CD}44.23   & \cellcolor[HTML]{B7E1CD}65.48   & 77.02  & \cellcolor[HTML]{B7E1CD}\textbf{37.79}    & \cellcolor[HTML]{B7E1CD}46.46   & \cellcolor[HTML]{B7E1CD}3.7    \\
ReColor   & 5  & \cellcolor[HTML]{F4C7C3}86.49 & \cellcolor[HTML]{B7E1CD}\textbf{54.92} & \cellcolor[HTML]{B7E1CD}\textbf{16.5}  & \cellcolor[HTML]{B7E1CD}\textbf{44.41} & \cellcolor[HTML]{B7E1CD}\textbf{78.3}  & \cellcolor[HTML]{B7E1CD}\textbf{68.38} & \cellcolor[HTML]{F4C7C3}32.22   & \cellcolor[HTML]{B7E1CD}\textbf{27.48} & \cellcolor[HTML]{B7E1CD}56.59   & \cellcolor[HTML]{B7E1CD}\textbf{45.33} & \cellcolor[HTML]{B7E1CD}\textbf{65.82} & \cellcolor[HTML]{F4C7C3}73.44   & \cellcolor[HTML]{F4C7C3}27.4& \cellcolor[HTML]{B7E1CD}\textbf{49.23} & \cellcolor[HTML]{B7E1CD}\textbf{6.81} \\ \hline
Gabor     & 0  & \textbf{94.08} & 0.3    & 0.01   & 0.01   & 4.43   & \textbf{92.39}   & 16.96  & 8.96   & 2.08   & 2.31   & 17.99  & 41.61  & 11.87     & 16.58  & 0     \\
Gabor     & 1  & \cellcolor[HTML]{F4C7C3}91.84 & \cellcolor[HTML]{B7E1CD}\textbf{24.27} & \textbf{0.69}    & \cellcolor[HTML]{B7E1CD}\textbf{10.02} & \cellcolor[HTML]{B7E1CD}\textbf{41.71} & \cellcolor[HTML]{F4C7C3}87.58   & \cellcolor[HTML]{B7E1CD}\textbf{23.11} & \cellcolor[HTML]{B7E1CD}\textbf{13.99} & \cellcolor[HTML]{B7E1CD}\textbf{32.19} & \cellcolor[HTML]{B7E1CD}\textbf{27.95} & \cellcolor[HTML]{B7E1CD}\textbf{47.51} & \cellcolor[HTML]{B7E1CD}\textbf{62.62} & \cellcolor[HTML]{B7E1CD}\textbf{15.27}    & \cellcolor[HTML]{B7E1CD}\textbf{32.24} & \textbf{0.24}   \\
Gabor     & 2  & \cellcolor[HTML]{F4C7C3}90.58 & \cellcolor[HTML]{B7E1CD}9.65    & 0.06   & 0.17   & \cellcolor[HTML]{B7E1CD}17.39   & \cellcolor[HTML]{F4C7C3}87.02   & 16.34  & \cellcolor[HTML]{F4C7C3}5.68    & \cellcolor[HTML]{B7E1CD}9.84    & \cellcolor[HTML]{B7E1CD}15.66   & \cellcolor[HTML]{B7E1CD}42.22   & \cellcolor[HTML]{B7E1CD}52.81   & \cellcolor[HTML]{B7E1CD}14.41      & \cellcolor[HTML]{B7E1CD}22.6    & 0     \\ \hline
Snow      & 0  & \textbf{95.89} & 0.05   & 0      & 0.01   & 2.63   & 30.13  & \textbf{92.02}   & 7.23   & 0.9    & 15.17  & 31.49  & 47.98  & 22.39     & 20.83  & 0     \\
Snow      & 0.5& \cellcolor[HTML]{F4C7C3}90.18 & \cellcolor[HTML]{B7E1CD}26.34   & 0.63   & \cellcolor[HTML]{B7E1CD}4.74    & \cellcolor[HTML]{B7E1CD}22.47   & \cellcolor[HTML]{B7E1CD}\textbf{70.47} & \cellcolor[HTML]{F4C7C3}79.6    & 7.67   & \cellcolor[HTML]{B7E1CD}25.69   & \cellcolor[HTML]{B7E1CD}39.56   & \cellcolor[HTML]{B7E1CD}46.21   & \cellcolor[HTML]{B7E1CD}\textbf{69.51} & \cellcolor[HTML]{F4C7C3}14.55      & \cellcolor[HTML]{B7E1CD}33.95   & 0.13  \\
Snow      & 1  & \cellcolor[HTML]{F4C7C3}86.82 & \cellcolor[HTML]{B7E1CD}25.02   & 0.93   & \cellcolor[HTML]{B7E1CD}\textbf{14.74} & \cellcolor[HTML]{B7E1CD}\textbf{40.4}  & \cellcolor[HTML]{B7E1CD}65.19   & \cellcolor[HTML]{F4C7C3}74.04   & \cellcolor[HTML]{B7E1CD}\textbf{8.76}  & \cellcolor[HTML]{B7E1CD}31.82   & \cellcolor[HTML]{B7E1CD}\textbf{41.78} & \cellcolor[HTML]{B7E1CD}\textbf{50.21} & \cellcolor[HTML]{B7E1CD}67.76   & 22.3      & \cellcolor[HTML]{B7E1CD}\textbf{36.91} & 0.36  \\
Snow      & 2  & \cellcolor[HTML]{F4C7C3}75.97 & \cellcolor[HTML]{B7E1CD}\textbf{38.33} & \cellcolor[HTML]{B7E1CD}\textbf{8.48}  & \cellcolor[HTML]{B7E1CD}3.21    & \cellcolor[HTML]{B7E1CD}24.12   & \cellcolor[HTML]{B7E1CD}57.66   & \cellcolor[HTML]{F4C7C3}66.64   & 8.1    & \cellcolor[HTML]{B7E1CD}\textbf{34.75} & \cellcolor[HTML]{B7E1CD}39.83   & \cellcolor[HTML]{B7E1CD}47.16   & \cellcolor[HTML]{B7E1CD}60.87   & \cellcolor[HTML]{B7E1CD}\textbf{25.36}    & \cellcolor[HTML]{B7E1CD}34.54   & \textbf{0.85}   \\ \hline
Pixel     & 0  & \textbf{94.76} & 0.07   & 0      & 0.01   & 8.87   & 57.65  & \textbf{36.6}    & \textbf{88.35}   & 1.77   & 14.52  & 38.18  & 67.5   & 16.21     & 27.48  & 0     \\
Pixel     & 1  & \cellcolor[HTML]{F4C7C3}88.35 & \cellcolor[HTML]{B7E1CD}11.76   & 0.16   & 0.38   & \cellcolor[HTML]{B7E1CD}23.23   & \cellcolor[HTML]{B7E1CD}\textbf{63.75} & \cellcolor[HTML]{F4C7C3}28.4    & \cellcolor[HTML]{F4C7C3}71.23   & \cellcolor[HTML]{B7E1CD}7.2     & \cellcolor[HTML]{B7E1CD}31.88   & \cellcolor[HTML]{B7E1CD}47.73   & \cellcolor[HTML]{B7E1CD}\textbf{71.87} & \cellcolor[HTML]{B7E1CD}21.77      & \cellcolor[HTML]{B7E1CD}31.61   & 0.02  \\
Pixel     & 2  & \cellcolor[HTML]{F4C7C3}79.87 & \cellcolor[HTML]{B7E1CD}\textbf{22.8}  & \cellcolor[HTML]{B7E1CD}\textbf{3.38}  & \cellcolor[HTML]{B7E1CD}\textbf{1.75}  & \cellcolor[HTML]{B7E1CD}\textbf{34.1}  & \cellcolor[HTML]{B7E1CD}59      & \cellcolor[HTML]{F4C7C3}24      & \cellcolor[HTML]{F4C7C3}60.74   & \cellcolor[HTML]{B7E1CD}\textbf{13.63} & \cellcolor[HTML]{B7E1CD}\textbf{33.78} & \cellcolor[HTML]{B7E1CD}\textbf{52.8}  & 66.58  & \cellcolor[HTML]{B7E1CD}\textbf{31.06}    & \cellcolor[HTML]{B7E1CD}\textbf{33.63} & \textbf{0.22}   \\ \hline
JPEG      & 0  & \textbf{90.26} & 56.48  & 21.5   & 0.52   & 34.74  & 68.59  & 21.12  & 10.57  & 73.46  & 40     & 74.3   & \textbf{78.35}   & 28.02     & 42.3   & 0.09  \\
JPEG      & 1  & 89.29   & \cellcolor[HTML]{B7E1CD}60.05   & \cellcolor[HTML]{B7E1CD}26.16   & \cellcolor[HTML]{B7E1CD}4.69    & \cellcolor[HTML]{B7E1CD}42.29   & \cellcolor[HTML]{B7E1CD}\textbf{74.59} & \cellcolor[HTML]{B7E1CD}22.64   & \cellcolor[HTML]{B7E1CD}13.53   & \textbf{74.29}   & \cellcolor[HTML]{B7E1CD}43.93   & \textbf{74.64}   & 77.74  & \cellcolor[HTML]{B7E1CD}\textbf{29.07}    & \cellcolor[HTML]{B7E1CD}45.3    & \cellcolor[HTML]{B7E1CD}1.33   \\
JPEG      & 2  & \cellcolor[HTML]{F4C7C3}88.49 & \cellcolor[HTML]{B7E1CD}\textbf{61.79} & \cellcolor[HTML]{B7E1CD}\textbf{27.73} & \cellcolor[HTML]{B7E1CD}\textbf{12.37} & \cellcolor[HTML]{B7E1CD}\textbf{47.32} & \cellcolor[HTML]{B7E1CD}71.46   & \cellcolor[HTML]{B7E1CD}\textbf{25.63} & \cellcolor[HTML]{B7E1CD}\textbf{17.62} & 73.98  & \cellcolor[HTML]{B7E1CD}\textbf{45.12} & \cellcolor[HTML]{F4C7C3}72.8    & 78.03  & 27.9      & \cellcolor[HTML]{B7E1CD}\textbf{46.81} & \cellcolor[HTML]{B7E1CD}\textbf{3.87} \\ \hline
Elastic   & 0  & \textbf{94.06} & 1.32   & 0.02   & 7.5    & 7.92   & 25.41  & \textbf{53.68}   & 9.16   & 11.2   & \textbf{79.47}   & 72.94  & 50.24  & 33.1      & 29.33  & 0.01  \\
Elastic   & 1  & \cellcolor[HTML]{F4C7C3}91.85 & \cellcolor[HTML]{B7E1CD}17.14   & 0.54   & \cellcolor[HTML]{F4C7C3}5.55    & \cellcolor[HTML]{B7E1CD}14.6    & \cellcolor[HTML]{B7E1CD}40.8    & \cellcolor[HTML]{F4C7C3}35.05   & \cellcolor[HTML]{F4C7C3}5.55    & \cellcolor[HTML]{B7E1CD}31.41   & \cellcolor[HTML]{F4C7C3}66.9    & \cellcolor[HTML]{B7E1CD}\textbf{75.26} & \cellcolor[HTML]{B7E1CD}\textbf{65.89} & \cellcolor[HTML]{F4C7C3}30.44      & \cellcolor[HTML]{B7E1CD}32.43   & 0.17  \\
Elastic   & 2  & \cellcolor[HTML]{F4C7C3}67.79 & \cellcolor[HTML]{B7E1CD}\textbf{43.66} & \cellcolor[HTML]{B7E1CD}\textbf{17.81} & \cellcolor[HTML]{B7E1CD}\textbf{23}    & \cellcolor[HTML]{B7E1CD}\textbf{31.04} & \cellcolor[HTML]{B7E1CD}\textbf{56.14} & \cellcolor[HTML]{F4C7C3}28.89   & \cellcolor[HTML]{B7E1CD}\textbf{15.31} & \cellcolor[HTML]{B7E1CD}\textbf{47.71} & \cellcolor[HTML]{F4C7C3}51.31   & \cellcolor[HTML]{F4C7C3}57.9    & \cellcolor[HTML]{B7E1CD}59.33   & \cellcolor[HTML]{B7E1CD}\textbf{34.56}    & \cellcolor[HTML]{B7E1CD}\textbf{38.89} & \cellcolor[HTML]{B7E1CD}\textbf{4.6}  \\ \hline
Wood      & 0  & \textbf{93.57} & 0.03   & 0      & 0.4    & 1.27   & 18.47  & \textbf{39.44}   & 3.43   & 0.37   & 33.68  & \textbf{93.04}   & 28.04  & 14.11     & 19.36  & 0     \\
Wood      & 1  & \cellcolor[HTML]{F4C7C3}89.61 & \cellcolor[HTML]{B7E1CD}\textbf{30.94} & \cellcolor[HTML]{B7E1CD}\textbf{3.52}  & \cellcolor[HTML]{B7E1CD}\textbf{17.14} & \cellcolor[HTML]{B7E1CD}\textbf{30.98} & \cellcolor[HTML]{B7E1CD}26.2    & \cellcolor[HTML]{F4C7C3}4.9     & \cellcolor[HTML]{B7E1CD}\textbf{53.36} & \cellcolor[HTML]{B7E1CD}\textbf{45.05} & \cellcolor[HTML]{B7E1CD}\textbf{56.82} & \cellcolor[HTML]{F4C7C3}82.92   & \cellcolor[HTML]{B7E1CD}\textbf{66.69} & \cellcolor[HTML]{B7E1CD}\textbf{25.1}     & \cellcolor[HTML]{B7E1CD}\textbf{36.97} & \textbf{0.69}   \\
Wood      & 2  & \cellcolor[HTML]{F4C7C3}85.72 & \cellcolor[HTML]{B7E1CD}\textbf{30.94} & \cellcolor[HTML]{B7E1CD}\textbf{3.52}  & \cellcolor[HTML]{B7E1CD}\textbf{17.14} & \cellcolor[HTML]{B7E1CD}\textbf{30.98} & \cellcolor[HTML]{B7E1CD}\textbf{53.36} & \cellcolor[HTML]{F4C7C3}26.2    & \cellcolor[HTML]{B7E1CD}4.9     & \cellcolor[HTML]{B7E1CD}\textbf{45.05} & \cellcolor[HTML]{B7E1CD}\textbf{56.82} & \cellcolor[HTML]{F4C7C3}82.92   & \cellcolor[HTML]{B7E1CD}\textbf{66.69} & \cellcolor[HTML]{B7E1CD}\textbf{25.1}     & \cellcolor[HTML]{B7E1CD}\textbf{36.97} & \textbf{0.69}   \\ \hline
Glitch    & 0  & \textbf{93.26} & 0.02   & 0      & 0      & 11.49  & 49.03   & \textbf{24.44}   & 12.47  & 3.14   & 10.89  & 31.99  & \textbf{90.77}   & 16.61     & 20.9   & 0     \\
Glitch    & 1  & \cellcolor[HTML]{F4C7C3}90.16 & \cellcolor[HTML]{B7E1CD}34.66   & \cellcolor[HTML]{B7E1CD}2.61    & 0.28   & \cellcolor[HTML]{B7E1CD}22.95   & \cellcolor[HTML]{B7E1CD}\textbf{64.19} & \cellcolor[HTML]{F4C7C3}22.01   & 12.38  & \cellcolor[HTML]{B7E1CD}28.75   & \cellcolor[HTML]{B7E1CD}39.22   & \cellcolor[HTML]{B7E1CD}\textbf{60.2}  & \cellcolor[HTML]{F4C7C3}83.26   & \cellcolor[HTML]{B7E1CD}23.6& \cellcolor[HTML]{B7E1CD}32.84   & 0.01  \\
Glitch    & 2  & \cellcolor[HTML]{F4C7C3}79.67 & \cellcolor[HTML]{B7E1CD}\textbf{44.33} & \cellcolor[HTML]{B7E1CD}\textbf{9.64}  & \cellcolor[HTML]{B7E1CD}\textbf{1.01}  & \cellcolor[HTML]{B7E1CD}\textbf{27.31} & \cellcolor[HTML]{B7E1CD}60.07   & \cellcolor[HTML]{F4C7C3}17.26   & \cellcolor[HTML]{B7E1CD}\textbf{15.28} & \cellcolor[HTML]{B7E1CD}\textbf{38.23} & \cellcolor[HTML]{B7E1CD}\textbf{42.31} & \cellcolor[HTML]{B7E1CD}58.64   & \cellcolor[HTML]{F4C7C3}72.43   & \cellcolor[HTML]{B7E1CD}\textbf{30.1}     & \cellcolor[HTML]{B7E1CD}\textbf{34.72} & \textbf{0.36}   \\ \hline
\begin{tabular}[c]{@{}l@{}}Kaleid-\\ oscope\end{tabular} & 0  & \textbf{96.03} & 0      & 0      & 0      & 0.8    & 39.49  & \textbf{40.94}   & 5.75   & 0.02   & 2.4    & 43.08  & 33.71  & \textbf{91.97}    & 21.51  & 0     \\
\begin{tabular}[c]{@{}l@{}}Kaleid-\\ oscope\end{tabular} & 0.1& \cellcolor[HTML]{F4C7C3}93.16 & \cellcolor[HTML]{B7E1CD}5.39    & 0.02   & 0.02   & \cellcolor[HTML]{B7E1CD}4.89    & \cellcolor[HTML]{B7E1CD}\textbf{55.54} & \cellcolor[HTML]{F4C7C3}35.63   & 5.32   & \cellcolor[HTML]{B7E1CD}3.25    & \cellcolor[HTML]{B7E1CD}17.43   & \cellcolor[HTML]{B7E1CD}\textbf{52.59} & \cellcolor[HTML]{B7E1CD}60.06   & \cellcolor[HTML]{F4C7C3}86.77      & \cellcolor[HTML]{B7E1CD}27.24   & 0     \\
\begin{tabular}[c]{@{}l@{}}Kaleid-\\ oscope\end{tabular} & 0.5& \cellcolor[HTML]{F4C7C3}88.9  & \cellcolor[HTML]{B7E1CD}\textbf{16.02} & 0.22   & 0.17   & \cellcolor[HTML]{B7E1CD}9.98    & \cellcolor[HTML]{B7E1CD}54.63   & \cellcolor[HTML]{F4C7C3}29.68   & 6.35   & \cellcolor[HTML]{B7E1CD}9.75    & \cellcolor[HTML]{B7E1CD}\textbf{29.38} & \cellcolor[HTML]{B7E1CD}49.24   & \cellcolor[HTML]{B7E1CD}\textbf{64.47} & \cellcolor[HTML]{F4C7C3}78.4& \cellcolor[HTML]{B7E1CD}\textbf{29.02} & 0     \\
\begin{tabular}[c]{@{}l@{}}Kaleid-\\ oscope\end{tabular} & 1  & \cellcolor[HTML]{F4C7C3}68.25 & \cellcolor[HTML]{B7E1CD}14.05   & \textbf{0.51}    & \cellcolor[HTML]{B7E1CD}\textbf{5.33}  & \cellcolor[HTML]{B7E1CD}\textbf{42}    & \cellcolor[HTML]{B7E1CD}48.53   & \cellcolor[HTML]{F4C7C3}14.55   & \cellcolor[HTML]{B7E1CD}\textbf{10.46} & \cellcolor[HTML]{B7E1CD}\textbf{13.6}  & \cellcolor[HTML]{B7E1CD}21.92   & \cellcolor[HTML]{F4C7C3}32.66   & \cellcolor[HTML]{B7E1CD}45.02   & \cellcolor[HTML]{F4C7C3}62.09      & \cellcolor[HTML]{B7E1CD}25.89   & \textbf{0.09}   \\ \hline
\end{tabular}}}
\caption{\textbf{Intial Training Ablations- Uniform regularization on CIFAR-10. }Accuracy of initially trained models on CIFAR-10 trained using different attacks as indicated in ``Train Attack" column measured across different attacks.  Uniform regularization (with $\sigma=2$) is also considered during initial training, with regularization strength $\lambda$.  Results where regularization improves over no regularization ($\lambda = 0$) by at least 1\% accuracy are highlighted in green, while results where regularization incurs at least a 1\% drop in accuracy are highlighted in red.  Best performing with respect to regularization strength are bolded.}
\label{app:IT_ablation_UR_CIF}
\end{table*}

% Please add the following required packages to your document preamble:
% \usepackage[table,xcdraw]{xcolor}
% Beamer presentation requires \usepackage{colortbl} instead of \usepackage[table,xcdraw]{xcolor}
\begin{table*}[]
{\renewcommand{\arraystretch}{1.2}
\scalebox{0.73}{
\begin{tabular}{|l|c|c|cccccccccccc|cc|}
\hline
\begin{tabular}[c]{@{}l@{}}Train\\ Attack\end{tabular}     & $\lambda$ & Clean   & $\ell_2$     & $\ell_\infty$   & StAdv  & ReColor& Gabor  & Snow   & Pixel  & JPEG   & Elastic& Wood   & Glitch & \begin{tabular}[c]{@{}c@{}}Kaleid-\\ oscope\end{tabular} & Avg    & Union \\ \hline
$\ell_2$ & 0  & \textbf{91.08} & 70.02  & 29.38  & 0.79   & 33.69  & 66.93  & 24.59  & 14.99  & 64.22  & 45.13  & 70.85  & \textbf{80.3}    & 30.08     & 44.25  & 0.1   \\
$\ell_2$ & 0.1& 90.24   & \textbf{70.25}   & \cellcolor[HTML]{B7E1CD}31.05   & \cellcolor[HTML]{B7E1CD}3.01    & \cellcolor[HTML]{B7E1CD}40.33   & \cellcolor[HTML]{B7E1CD}70.49   & 24.6   & \cellcolor[HTML]{B7E1CD}16.78   & \cellcolor[HTML]{B7E1CD}\textbf{65.57} & 44.87  & \cellcolor[HTML]{B7E1CD}73.93   & 79.68  & 29.94     & \cellcolor[HTML]{B7E1CD}45.86   & \cellcolor[HTML]{B7E1CD}1.21   \\
$\ell_2$ & 0.2& \cellcolor[HTML]{F4CCCC}90.07 & 69.57  & \cellcolor[HTML]{B7E1CD}31.8    & \cellcolor[HTML]{B7E1CD}9.58    & \cellcolor[HTML]{B7E1CD}46.6    & \cellcolor[HTML]{B7E1CD}72.49   & 25.22  & \cellcolor[HTML]{B7E1CD}18.95   & \cellcolor[HTML]{B7E1CD}65.22   & \cellcolor[HTML]{B7E1CD}46.79   & \cellcolor[HTML]{B7E1CD}\textbf{76.23} & 79.36  & \cellcolor[HTML]{F4CCCC}28.74      & \cellcolor[HTML]{B7E1CD}47.55   & \cellcolor[HTML]{B7E1CD}3      \\
$\ell_2$ & 0.5& \cellcolor[HTML]{F4CCCC}86.89 & \cellcolor[HTML]{F4CCCC}68.19   & \cellcolor[HTML]{B7E1CD}\textbf{32.02} & \cellcolor[HTML]{B7E1CD}\textbf{16.54} & \cellcolor[HTML]{B7E1CD}\textbf{58.32} & \cellcolor[HTML]{B7E1CD}\textbf{74.85} & \cellcolor[HTML]{B7E1CD}\textbf{25.69} & \cellcolor[HTML]{B7E1CD}\textbf{21.26} & \cellcolor[HTML]{B7E1CD}65.32   & \cellcolor[HTML]{B7E1CD}\textbf{46.82} & \cellcolor[HTML]{B7E1CD}74.08   & \cellcolor[HTML]{F4CCCC}76.99   & \cellcolor[HTML]{B7E1CD}\textbf{31.93}    & \cellcolor[HTML]{B7E1CD}\textbf{49.33} & \cellcolor[HTML]{B7E1CD}\textbf{4.18} \\ \hline
$\ell_\infty$      & 0  & 85.53   & 59.36  & 50.98  & 6.34   & 56.27  & \textbf{68.94}   & 36.79  & 20.57  & 54.02  & 51     & 64.24  & \textbf{75.94}   & 39.44     & 48.66  & 1.31  \\
$\ell_\infty$      & 0.1& \textbf{86.18} & \cellcolor[HTML]{B7E1CD}60.45   & \textbf{51.52}   & 7.07   & \cellcolor[HTML]{B7E1CD}57.5    & 68.54  & \cellcolor[HTML]{B7E1CD}38.96   & 21.24  & \cellcolor[HTML]{B7E1CD}56.43   & 50.39  & 64.66  & 75.78  & 38.74     & 49.27  & 1.8   \\
$\ell_\infty$      & 0.2& 85.19   & \cellcolor[HTML]{B7E1CD}\textbf{61.63} & 50.12  & \cellcolor[HTML]{B7E1CD}17.67   & \cellcolor[HTML]{B7E1CD}67.92   & 69.6   & \cellcolor[HTML]{B7E1CD}\textbf{40.02} & \cellcolor[HTML]{B7E1CD}23.22   & \cellcolor[HTML]{B7E1CD}\textbf{57.72} & \cellcolor[HTML]{B7E1CD}\textbf{52.76} & \cellcolor[HTML]{B7E1CD}\textbf{66.28} & 75     & \cellcolor[HTML]{B7E1CD}\textbf{40.86}    & \cellcolor[HTML]{B7E1CD}51.9    & \cellcolor[HTML]{B7E1CD}3.93   \\
$\ell_\infty$      & 0.5& \cellcolor[HTML]{F4C7C3}80.65 & 59.74  & \cellcolor[HTML]{F4C7C3}46.12   & \cellcolor[HTML]{B7E1CD}\textbf{34.57} & \cellcolor[HTML]{B7E1CD}\textbf{70.49} & 68.33  & 35.8   & \cellcolor[HTML]{B7E1CD}\textbf{26.04} & \cellcolor[HTML]{B7E1CD}57.28   & 51.98  & \cellcolor[HTML]{B7E1CD}65.46   & \cellcolor[HTML]{F4C7C3}70.73   & \cellcolor[HTML]{F4C7C3}38.21      & \cellcolor[HTML]{B7E1CD}\textbf{52.06} & \cellcolor[HTML]{B7E1CD}\textbf{6.28} \\ \hline
StAdv     & 0  & \textbf{87.12} & 5.48   & 0.07   & 56.22  & 5.69   & 17.62  & \textbf{57.8}    & 5.93   & 11.09  & \textbf{76.02}   & \textbf{77.47}   & 54.04  & \textbf{43.4}     & 34.24  & 0.05  \\
StAdv     & 0.1& \cellcolor[HTML]{F4C7C3}72.19 & \cellcolor[HTML]{F4C7C3}1.16    & 0.02   & \cellcolor[HTML]{B7E1CD}62.36   & \cellcolor[HTML]{B7E1CD}42.38   & \cellcolor[HTML]{B7E1CD}44.84   & \cellcolor[HTML]{F4C7C3}35.9    & \cellcolor[HTML]{B7E1CD}10.54   & \cellcolor[HTML]{B7E1CD}16.17   & \cellcolor[HTML]{F4C7C3}52.73   & \cellcolor[HTML]{F4C7C3}65.04   & \cellcolor[HTML]{F4C7C3}50.52   & \cellcolor[HTML]{F4C7C3}34.58      & 34.69  & 0.02  \\
StAdv     & 0.2& \cellcolor[HTML]{F4C7C3}79.63 & \cellcolor[HTML]{B7E1CD}11.11   & 0.42   & \cellcolor[HTML]{B7E1CD}\textbf{72.58} & \cellcolor[HTML]{B7E1CD}57.96   & \cellcolor[HTML]{B7E1CD}43.27   & \cellcolor[HTML]{F4C7C3}40.28   & \cellcolor[HTML]{B7E1CD}11.09   & \cellcolor[HTML]{B7E1CD}25.26   & \cellcolor[HTML]{F4C7C3}60.64   & \cellcolor[HTML]{F4C7C3}70.43   & \cellcolor[HTML]{B7E1CD}57.7    & \cellcolor[HTML]{F4C7C3}36.15      & \cellcolor[HTML]{B7E1CD}40.57   & 0.34  \\
StAdv     & 0.5& \cellcolor[HTML]{F4C7C3}76.35 & \cellcolor[HTML]{B7E1CD}\textbf{30.61} & \cellcolor[HTML]{B7E1CD}\textbf{4.69}  & \cellcolor[HTML]{B7E1CD}66.84   & \cellcolor[HTML]{B7E1CD}\textbf{61.11} & \cellcolor[HTML]{B7E1CD}\textbf{50.18} & \cellcolor[HTML]{F4C7C3}36.59   & \cellcolor[HTML]{B7E1CD}\textbf{16.76} & \cellcolor[HTML]{B7E1CD}\textbf{39.22} & \cellcolor[HTML]{F4C7C3}57      & \cellcolor[HTML]{F4C7C3}67.71   & \cellcolor[HTML]{B7E1CD}\textbf{62.31} & \cellcolor[HTML]{F4C7C3}34.1& \cellcolor[HTML]{B7E1CD}\textbf{43.92} & \cellcolor[HTML]{B7E1CD}\textbf{3.13} \\ \hline
ReColor   & 0  & \textbf{93.61} & 37.17  & 7.03   & 0.01   & 67.48  & 55.53  & \textbf{37.14}   & 8.27   & 45.36  & 35.55  & 60.92  & \textbf{77.2}    & \textbf{32.28}    & 38.66  & 0     \\
ReColor   & 0.1& 93.51   & \cellcolor[HTML]{F4C7C3}36.05   & \cellcolor[HTML]{B7E1CD}8.59    & \cellcolor[HTML]{B7E1CD}24.25   & \cellcolor[HTML]{B7E1CD}79.15   & \cellcolor[HTML]{B7E1CD}65.07   & 36.75  & \cellcolor[HTML]{B7E1CD}17.11   & \cellcolor[HTML]{B7E1CD}49.12   & 35.86  & \cellcolor[HTML]{B7E1CD}65.18   & \cellcolor[HTML]{F4C7C3}76.15   & 31.97     & \cellcolor[HTML]{B7E1CD}43.77   & \cellcolor[HTML]{B7E1CD}3.03   \\
ReColor   & 0.2& \cellcolor[HTML]{F4C7C3}91.49 & \textbf{37.6}    & \cellcolor[HTML]{B7E1CD}\textbf{9.06}  & \cellcolor[HTML]{B7E1CD}37.99   & \cellcolor[HTML]{B7E1CD}\textbf{84.17} & \cellcolor[HTML]{B7E1CD}\textbf{69.1}  & \cellcolor[HTML]{F4C7C3}34.02   & \cellcolor[HTML]{B7E1CD}\textbf{18.72} & \cellcolor[HTML]{B7E1CD}\textbf{51.51} & \cellcolor[HTML]{B7E1CD}\textbf{38.12} & \cellcolor[HTML]{B7E1CD}\textbf{68.6}  & \cellcolor[HTML]{F4C7C3}73.16   & 31.63     & \cellcolor[HTML]{B7E1CD}\textbf{46.14} & \cellcolor[HTML]{B7E1CD}3.02   \\
ReColor   & 0.5& \cellcolor[HTML]{F4C7C3}74.12 & \cellcolor[HTML]{F4C7C3}34.15   & \cellcolor[HTML]{B7E1CD}8.14    & \cellcolor[HTML]{B7E1CD}\textbf{41.98} & \cellcolor[HTML]{B7E1CD}71.03   & \cellcolor[HTML]{B7E1CD}58.2    & \cellcolor[HTML]{F4C7C3}23.92   & \cellcolor[HTML]{B7E1CD}18.45   & \cellcolor[HTML]{F4C7C3}43.7    & \cellcolor[HTML]{F4C7C3}32.07   & \cellcolor[HTML]{F4C7C3}57.27   & \cellcolor[HTML]{F4C7C3}55.79   & \cellcolor[HTML]{F4C7C3}24.93      & 39.14  & \cellcolor[HTML]{B7E1CD}\textbf{3.53} \\ \hline
Gabor     & 0  & \textbf{94.08} & 0.3    & 0.01   & 0.01   & 4.43   & \textbf{92.39}   & 16.96  & 8.96   & 2.08   & 2.31   & 17.99  & 41.61  & 11.87     & 16.58  & 0     \\
Gabor     & 0.1& \cellcolor[HTML]{F4C7C3}91.52 & \cellcolor[HTML]{B7E1CD}10.06   & 0.38   & \cellcolor[HTML]{B7E1CD}\textbf{12.04} & \cellcolor[HTML]{B7E1CD}\textbf{49.59} & \cellcolor[HTML]{F4C7C3}89.29   & \cellcolor[HTML]{B7E1CD}\textbf{26.41} & \cellcolor[HTML]{B7E1CD}\textbf{19.27} & \cellcolor[HTML]{B7E1CD}\textbf{32.58} & \cellcolor[HTML]{B7E1CD}23.99   & \cellcolor[HTML]{B7E1CD}50.29   & \cellcolor[HTML]{B7E1CD}54.77   & \cellcolor[HTML]{B7E1CD}16.1& \cellcolor[HTML]{B7E1CD}32.06   & 0.13  \\
Gabor     & 0.2& \cellcolor[HTML]{F4C7C3}88.56 & \cellcolor[HTML]{B7E1CD}9.45    & 0.34   & 0.72   & \cellcolor[HTML]{B7E1CD}23.46   & \cellcolor[HTML]{F4C7C3}85.27   & \cellcolor[HTML]{B7E1CD}20.59   & 8.44   & \cellcolor[HTML]{B7E1CD}17.77   & \cellcolor[HTML]{B7E1CD}19.99   & \cellcolor[HTML]{B7E1CD}49      & \cellcolor[HTML]{B7E1CD}54.96   & \cellcolor[HTML]{B7E1CD}15.62      & \cellcolor[HTML]{B7E1CD}25.47   & 0.04  \\
Gabor     & 0.5& \cellcolor[HTML]{F4C7C3}83.09 & \cellcolor[HTML]{B7E1CD}\textbf{31.81} & \cellcolor[HTML]{B7E1CD}\textbf{4.79}  & \cellcolor[HTML]{B7E1CD}1.81    & \cellcolor[HTML]{B7E1CD}26.15   & \cellcolor[HTML]{F4C7C3}79.32   & \cellcolor[HTML]{B7E1CD}20.61   & \cellcolor[HTML]{B7E1CD}10.17   & \cellcolor[HTML]{B7E1CD}30.49   & \cellcolor[HTML]{B7E1CD}\textbf{34.34} & \cellcolor[HTML]{B7E1CD}\textbf{64.64} & \cellcolor[HTML]{B7E1CD}\textbf{64.17} & \cellcolor[HTML]{B7E1CD}\textbf{19.19}    & \cellcolor[HTML]{B7E1CD}\textbf{32.29} & \textbf{0.44}   \\ \hline
Snow      & 0  & \textbf{95.89} & 0.05   & 0      & 0.01   & 2.63   & 30.13  & \textbf{92.02}   & 7.23   & 0.9    & 15.17  & 31.49  & 47.98  & \textbf{22.39}    & 20.83  & 0     \\
Snow      & 0.1& \cellcolor[HTML]{F4C7C3}91.74 & \cellcolor[HTML]{B7E1CD}9.94    & 0.1    & 0.51   & \cellcolor[HTML]{B7E1CD}19.93   & \cellcolor[HTML]{B7E1CD}69.04   & \cellcolor[HTML]{F4C7C3}82.63   & 7.67   & \cellcolor[HTML]{B7E1CD}12.96   & \cellcolor[HTML]{B7E1CD}34.24   & \cellcolor[HTML]{B7E1CD}44.21   & \cellcolor[HTML]{B7E1CD}70.07   & \cellcolor[HTML]{F4C7C3}15.53      & \cellcolor[HTML]{B7E1CD}30.57   & 0     \\
Snow      & 0.2& \cellcolor[HTML]{F4C7C3}89.18 & \cellcolor[HTML]{B7E1CD}19.18   & 0.57   & 0.95   & \cellcolor[HTML]{B7E1CD}18.26   & \cellcolor[HTML]{B7E1CD}\textbf{73.07} & \cellcolor[HTML]{F4C7C3}76.98   & \cellcolor[HTML]{B7E1CD}9.2     & \cellcolor[HTML]{B7E1CD}19.02   & \cellcolor[HTML]{B7E1CD}36.91   & \cellcolor[HTML]{B7E1CD}\textbf{44.69} & \cellcolor[HTML]{B7E1CD}\textbf{70.51} & \cellcolor[HTML]{F4C7C3}14.24      & \cellcolor[HTML]{B7E1CD}\textbf{31.96} & 0.04  \\
Snow      & 0.5& \cellcolor[HTML]{F4C7C3}74.52 & \cellcolor[HTML]{B7E1CD}\textbf{24.82} & \cellcolor[HTML]{B7E1CD}\textbf{1.68}  & \cellcolor[HTML]{B7E1CD}\textbf{3.67}  & \cellcolor[HTML]{B7E1CD}\textbf{33.1}  & \cellcolor[HTML]{B7E1CD}57.71   & \cellcolor[HTML]{F4C7C3}65      & \cellcolor[HTML]{B7E1CD}\textbf{10.01} & \cellcolor[HTML]{B7E1CD}\textbf{22.39} & \cellcolor[HTML]{B7E1CD}\textbf{40}    & \cellcolor[HTML]{B7E1CD}40.36   & \cellcolor[HTML]{B7E1CD}61.17   & 21.48     & \cellcolor[HTML]{B7E1CD}31.78   & \textbf{0.56}   \\ \hline
Pixel     & 0  & \textbf{94.76} & 0.07   & 0      & 0.01   & 8.87   & 57.65  & \textbf{36.6}    & \textbf{88.35}   & 1.77   & 14.52  & 38.18  & 67.5   & 16.21     & 27.48  & 0     \\
Pixel     & 0.1& \cellcolor[HTML]{F4C7C3}86.83 & 0.9    & 0      & \cellcolor[HTML]{B7E1CD}1.1     & \cellcolor[HTML]{B7E1CD}\textbf{41.32} & \cellcolor[HTML]{B7E1CD}67.39   & \cellcolor[HTML]{F4C7C3}32.7    & \cellcolor[HTML]{F4C7C3}77.05   & \cellcolor[HTML]{B7E1CD}6.22    & 15.43  & \cellcolor[HTML]{B7E1CD}42.24   & \cellcolor[HTML]{F4C7C3}51.45   & \cellcolor[HTML]{B7E1CD}18.18      & \cellcolor[HTML]{B7E1CD}29.5    & 0     \\
Pixel     & 0.2& \cellcolor[HTML]{F4C7C3}89.6  & \cellcolor[HTML]{B7E1CD}4.25    & 0.01   & 0.14   & \cellcolor[HTML]{B7E1CD}20.42   & \cellcolor[HTML]{B7E1CD}\textbf{68.55} & \cellcolor[HTML]{F4C7C3}29.56   & \cellcolor[HTML]{F4C7C3}76.23   & \cellcolor[HTML]{B7E1CD}3.57    & \cellcolor[HTML]{B7E1CD}25.36   & \cellcolor[HTML]{B7E1CD}\textbf{46.03} & \textbf{68.2}    & \cellcolor[HTML]{B7E1CD}19.96      & \cellcolor[HTML]{B7E1CD}\textbf{30.19} & 0     \\
Pixel     & 0.5& \cellcolor[HTML]{F4C7C3}61.01 & \cellcolor[HTML]{B7E1CD}\textbf{23.09} & \cellcolor[HTML]{B7E1CD}\textbf{3.41}  & \cellcolor[HTML]{B7E1CD}\textbf{6.93}  & \cellcolor[HTML]{B7E1CD}34.6    & \cellcolor[HTML]{F4C7C3}49.26   & \cellcolor[HTML]{F4C7C3}19.29   & \cellcolor[HTML]{F4C7C3}44.75   & \cellcolor[HTML]{B7E1CD}\textbf{17.11} & \cellcolor[HTML]{B7E1CD}\textbf{30.67} & \cellcolor[HTML]{B7E1CD}44.62   & \cellcolor[HTML]{F4C7C3}52.42   & \cellcolor[HTML]{B7E1CD}\textbf{28.13}    & \cellcolor[HTML]{B7E1CD}29.52   & \textbf{0.89}   \\ \hline
JPEG      & 0  & \textbf{90.26} & 56.48  & 21.5   & 0.52   & 34.74  & 68.59  & 21.12  & 10.57  & 73.46  & 40     & 74.3   & \textbf{78.35}   & \textbf{28.02}    & 42.3   & 0.09  \\
JPEG      & 0.1& \cellcolor[HTML]{F4C7C3}89.2  & \cellcolor[HTML]{B7E1CD}58.97   & \cellcolor[HTML]{B7E1CD}25.22   & \cellcolor[HTML]{B7E1CD}2.37    & \cellcolor[HTML]{B7E1CD}38.03   & \cellcolor[HTML]{B7E1CD}\textbf{72.82} & 22.11  & \cellcolor[HTML]{B7E1CD}12.85   & 74.03  & \cellcolor[HTML]{B7E1CD}41.87   & \cellcolor[HTML]{B7E1CD}\textbf{76.66} & 77.74  & \cellcolor[HTML]{F4C7C3}25.58      & \cellcolor[HTML]{B7E1CD}44.02   & 0.76  \\
JPEG      & 0.2& \cellcolor[HTML]{F4C7C3}88.78 & \cellcolor[HTML]{B7E1CD}\textbf{61.28} & \cellcolor[HTML]{B7E1CD}\textbf{27.5}  & \cellcolor[HTML]{B7E1CD}8.27    & \cellcolor[HTML]{B7E1CD}43.35   & \cellcolor[HTML]{B7E1CD}71.66   & \cellcolor[HTML]{B7E1CD}22.78   & \cellcolor[HTML]{B7E1CD}14.46   & \cellcolor[HTML]{B7E1CD}\textbf{74.65} & \cellcolor[HTML]{B7E1CD}\textbf{43.63} & \cellcolor[HTML]{B7E1CD}75.89   & 78     & \cellcolor[HTML]{F4C7C3}26.09      & \cellcolor[HTML]{B7E1CD}45.63   & \cellcolor[HTML]{B7E1CD}2.01   \\
JPEG      & 0.5& \cellcolor[HTML]{F4C7C3}87.31 & \cellcolor[HTML]{B7E1CD}59.83   & \cellcolor[HTML]{B7E1CD}27.21   & \cellcolor[HTML]{B7E1CD}\textbf{15.54} & \cellcolor[HTML]{B7E1CD}\textbf{45.36} & \cellcolor[HTML]{B7E1CD}71.31   & \cellcolor[HTML]{B7E1CD}\textbf{23.82} & \cellcolor[HTML]{B7E1CD}\textbf{17.46} & 73.08  & \cellcolor[HTML]{B7E1CD}42.7    & \cellcolor[HTML]{B7E1CD}76.09   & \cellcolor[HTML]{F4C7C3}76.9    & \cellcolor[HTML]{F4C7C3}23.25      & \cellcolor[HTML]{B7E1CD}\textbf{46.05} & \cellcolor[HTML]{B7E1CD}\textbf{3.92} \\ \hline
Elastic   & 0  & \textbf{94.06} & 1.32   & 0.02   & 7.5    & 7.92   & 25.41  & \textbf{53.68}   & 9.16   & 11.2   & \textbf{79.47}   & \textbf{72.94}   & 50.24  & 33.1      & 29.33  & 0.01  \\
Elastic   & 0.1& \cellcolor[HTML]{F4C7C3}85.66 & \cellcolor[HTML]{B7E1CD}13.24   & 0.91   & \cellcolor[HTML]{F4C7C3}5.6     & \cellcolor[HTML]{B7E1CD}14.65   & \cellcolor[HTML]{B7E1CD}46.65   & \cellcolor[HTML]{F4C7C3}37.92   & \cellcolor[HTML]{F4C7C3}5.34    & \cellcolor[HTML]{B7E1CD}18.66   & \cellcolor[HTML]{F4C7C3}63.35   & \cellcolor[HTML]{F4C7C3}70.45   & \cellcolor[HTML]{B7E1CD}58.79   & \cellcolor[HTML]{F4C7C3}27.89      & 30.29  & 0.39  \\
Elastic   & 0.2& \cellcolor[HTML]{F4C7C3}85.22 & \cellcolor[HTML]{B7E1CD}28.9    & \cellcolor[HTML]{B7E1CD}2.8     & \cellcolor[HTML]{F4C7C3}6.43    & \cellcolor[HTML]{B7E1CD}18.55   & \cellcolor[HTML]{B7E1CD}53.96   & \cellcolor[HTML]{F4C7C3}37.91   & \cellcolor[HTML]{F4C7C3}6.43    & \cellcolor[HTML]{B7E1CD}35.93   & \cellcolor[HTML]{F4C7C3}63.82   & \cellcolor[HTML]{F4C7C3}69.97   & \cellcolor[HTML]{B7E1CD}65.53   & \textbf{33.61}    & \cellcolor[HTML]{B7E1CD}35.32   & 0.26  \\
Elastic   & 0.5& \cellcolor[HTML]{F4C7C3}76.85 & \cellcolor[HTML]{B7E1CD}\textbf{42.02} & \cellcolor[HTML]{B7E1CD}\textbf{10.72} & \cellcolor[HTML]{B7E1CD}\textbf{26.19} & \cellcolor[HTML]{B7E1CD}\textbf{46.27} & \cellcolor[HTML]{B7E1CD}\textbf{58.45} & \cellcolor[HTML]{F4C7C3}30.94   & \cellcolor[HTML]{B7E1CD}\textbf{11.09} & \cellcolor[HTML]{B7E1CD}\textbf{48.24} & \cellcolor[HTML]{F4C7C3}57.25   & \cellcolor[HTML]{F4C7C3}62.24   & \cellcolor[HTML]{B7E1CD}\textbf{65.81} & \cellcolor[HTML]{F4C7C3}28.98      & \cellcolor[HTML]{B7E1CD}\textbf{40.68} & \cellcolor[HTML]{B7E1CD}\textbf{2.76} \\ \hline
Wood      & 0  & \textbf{93.57} & 0.03   & 0      & 0.4    & 1.27   & 18.47  & \textbf{39.44}   & 3.43   & 0.37   & 33.68  & \textbf{93.04}   & 28.04  & 14.11     & 19.36  & 0     \\
Wood      & 0.1& \cellcolor[HTML]{F4C7C3}91.08 & \cellcolor[HTML]{B7E1CD}1.29    & 0.12   & \cellcolor[HTML]{B7E1CD}11.65   & \cellcolor[HTML]{B7E1CD}19.03   & \cellcolor[HTML]{B7E1CD}49.39   & \cellcolor[HTML]{F4C7C3}30.26   & \cellcolor[HTML]{B7E1CD}7.09    & \cellcolor[HTML]{B7E1CD}10.12   & \cellcolor[HTML]{B7E1CD}44.63   & \cellcolor[HTML]{F4C7C3}89.51   & \cellcolor[HTML]{B7E1CD}50.35   & \cellcolor[HTML]{B7E1CD}22.37      & \cellcolor[HTML]{B7E1CD}27.98   & 0.01  \\
Wood      & 0.2& \cellcolor[HTML]{F4C7C3}90.9  & \cellcolor[HTML]{B7E1CD}3.9     & 0.12   & \cellcolor[HTML]{B7E1CD}\textbf{18.97} & \cellcolor[HTML]{B7E1CD}\textbf{27.03} & \cellcolor[HTML]{B7E1CD}50.3    & \cellcolor[HTML]{F4C7C3}31.63   & \cellcolor[HTML]{B7E1CD}7.97    & \cellcolor[HTML]{B7E1CD}18.49   & \cellcolor[HTML]{B7E1CD}\textbf{50.87} & \cellcolor[HTML]{F4C7C3}89.35   & \cellcolor[HTML]{B7E1CD}58.41   & \cellcolor[HTML]{B7E1CD}22.88      & \cellcolor[HTML]{B7E1CD}31.66   & 0.05  \\
Wood      & 0.5& \cellcolor[HTML]{F4C7C3}77.59 & \cellcolor[HTML]{B7E1CD}\textbf{29.6}  & \cellcolor[HTML]{B7E1CD}\textbf{6.65}  & \cellcolor[HTML]{B7E1CD}9.33    & \cellcolor[HTML]{B7E1CD}23.37   & \cellcolor[HTML]{B7E1CD}\textbf{55.4}  & \cellcolor[HTML]{F4C7C3}19.74   & \cellcolor[HTML]{B7E1CD}\textbf{8.3}   & \cellcolor[HTML]{B7E1CD}\textbf{34.75} & \cellcolor[HTML]{B7E1CD}49.11   & \cellcolor[HTML]{F4C7C3}75.26   & \cellcolor[HTML]{B7E1CD}\textbf{59.65} & \cellcolor[HTML]{B7E1CD}\textbf{24.26}    & \cellcolor[HTML]{B7E1CD}\textbf{32.95} & \textbf{0.87}   \\ \hline
Glitch    & 0  & \textbf{93.26} & 0.02   & 0      & 0      & 11.49  & 49.03  & \textbf{24.44}   & 12.47  & 3.14   & 10.89  & 31.99  & \textbf{90.77}   & 16.61     & 20.9   & 0     \\
Glitch    & 0.1& \cellcolor[HTML]{F4C7C3}85.96 & \cellcolor[HTML]{B7E1CD}1.29    & 0.13   & 0.08   & \cellcolor[HTML]{B7E1CD}16.46   & \cellcolor[HTML]{B7E1CD}56.48   & \cellcolor[HTML]{F4C7C3}19.95   & \cellcolor[HTML]{F4C7C3}8.38    & \cellcolor[HTML]{B7E1CD}6.56    & \cellcolor[HTML]{B7E1CD}20.67   & \cellcolor[HTML]{B7E1CD}44.52   & \cellcolor[HTML]{F4C7C3}78.95   & \cellcolor[HTML]{B7E1CD}18.31      & \cellcolor[HTML]{B7E1CD}22.65   & 0     \\
Glitch    & 0.2& \cellcolor[HTML]{F4C7C3}83.4  & \cellcolor[HTML]{B7E1CD}18.14   & \cellcolor[HTML]{B7E1CD}1.94    & 0.23   & \cellcolor[HTML]{B7E1CD}24.5    & \cellcolor[HTML]{B7E1CD}\textbf{60.76} & \cellcolor[HTML]{F4C7C3}19.24   & 11.85  & \cellcolor[HTML]{B7E1CD}27.46   & \cellcolor[HTML]{B7E1CD}33.38   & \cellcolor[HTML]{B7E1CD}57.36   & \cellcolor[HTML]{F4C7C3}77.04   & \cellcolor[HTML]{B7E1CD}\textbf{23.12}    & \cellcolor[HTML]{B7E1CD}29.59   & 0.01  \\
Glitch    & 0.5& \cellcolor[HTML]{F4C7C3}75.81 & \cellcolor[HTML]{B7E1CD}\textbf{36.99} & \cellcolor[HTML]{B7E1CD}\textbf{6.88}  & \cellcolor[HTML]{B7E1CD}\textbf{1.22}  & \cellcolor[HTML]{B7E1CD}\textbf{27.29} & \cellcolor[HTML]{B7E1CD}59.99   & \cellcolor[HTML]{F4C7C3}17.29   & \cellcolor[HTML]{B7E1CD}\textbf{14.18} & \cellcolor[HTML]{B7E1CD}\textbf{34.71} & \cellcolor[HTML]{B7E1CD}\textbf{36.39} & \cellcolor[HTML]{B7E1CD}\textbf{60.77} & \cellcolor[HTML]{F4C7C3}68.21   & \cellcolor[HTML]{B7E1CD}22.75      & \cellcolor[HTML]{B7E1CD}\textbf{32.22} & \textbf{0.21}   \\ \hline
\begin{tabular}[c]{@{}l@{}}Kaleid-\\ oscope\end{tabular} & 0  & \textbf{96.03} & 0      & 0      & 0      & 0.8    & 39.49  & \textbf{40.94}   & 5.75   & 0.02   & 2.4    & 43.08  & 33.71  & \textbf{91.97}    & 21.51  & 0     \\
\begin{tabular}[c]{@{}l@{}}Kaleid-\\ oscope\end{tabular} & 0.1& \cellcolor[HTML]{F4C7C3}83.96 & \cellcolor[HTML]{B7E1CD}1.15    & 0      & 0.07   & \cellcolor[HTML]{B7E1CD}19.63   & \cellcolor[HTML]{B7E1CD}53.87   & \cellcolor[HTML]{F4C7C3}29.05   & \cellcolor[HTML]{B7E1CD}7.32    & \cellcolor[HTML]{B7E1CD}1.72    & \cellcolor[HTML]{B7E1CD}15.51   & \cellcolor[HTML]{B7E1CD}\textbf{47.8}  & \cellcolor[HTML]{B7E1CD}\textbf{52.33} & \cellcolor[HTML]{F4C7C3}72.8& \cellcolor[HTML]{B7E1CD}25.1    & 0     \\
\begin{tabular}[c]{@{}l@{}}Kaleid-\\ oscope\end{tabular} & 0.2& \cellcolor[HTML]{F4C7C3}82.48 & 0.8    & 0      & 0.07   & \cellcolor[HTML]{B7E1CD}3.81    & \cellcolor[HTML]{B7E1CD}\textbf{59.93} & \cellcolor[HTML]{F4C7C3}25.75   & 6.5    & 0.89   & \cellcolor[HTML]{B7E1CD}10.87   & \cellcolor[HTML]{F4C7C3}40.35   & \cellcolor[HTML]{B7E1CD}50.44   & \cellcolor[HTML]{F4C7C3}66.42      & 22.15  & 0     \\
\begin{tabular}[c]{@{}l@{}}Kaleid-\\ oscope\end{tabular} & 0.5& \cellcolor[HTML]{F4C7C3}46.61 & \cellcolor[HTML]{B7E1CD}\textbf{28.96} & \cellcolor[HTML]{B7E1CD}\textbf{11.27} & \cellcolor[HTML]{B7E1CD}\textbf{23.59} & \cellcolor[HTML]{B7E1CD}\textbf{23.05} & \cellcolor[HTML]{B7E1CD}41.04   & \cellcolor[HTML]{F4C7C3}15.98   & \cellcolor[HTML]{B7E1CD}\textbf{15.01} & \cellcolor[HTML]{B7E1CD}\textbf{33.14} & \cellcolor[HTML]{B7E1CD}\textbf{30.43} & 42.2   & \cellcolor[HTML]{B7E1CD}41.4    & \cellcolor[HTML]{F4C7C3}44.04      & \cellcolor[HTML]{B7E1CD}\textbf{29.18} & \cellcolor[HTML]{B7E1CD}\textbf{5.31} \\ \hline
\end{tabular}}}
\caption{\textbf{Intial Training Ablations- Gaussian regularization on CIFAR-10. }Accuracy of initially trained models on CIFAR-10 trained using different attacks as indicated in ``Train Attack" column measured across different attacks.  Gaussian regularization (with $\sigma=0.2$) is also considered during initial training, with regularization strength $\lambda$.  Results where regularization improves over no regularization ($\lambda = 0$) by at least 1\% accuracy are highlighted in green, while results where regularization incurs at least a 1\% drop in accuracy are highlighted in red.  Best performing with respect to regularization strength are bolded.}
\label{app:IT_ablation_GR_CIF}
\end{table*}

\begin{table*}[]{\renewcommand{\arraystretch}{1.2}
\scalebox{0.73}{
\begin{tabular}{|l|c|c|cccccccccccc|cc|}
\hline
\begin{tabular}[c]{@{}l@{}}Train\\ Attack\end{tabular} & $\lambda$ & Clean & $\ell_2$ & $\ell_\infty$ & StAdv & ReColor & Gabor & Snow & Pixel & JPEG & Elastic & Wood & Glitch & \begin{tabular}[c]{@{}c@{}}Kaleid-\\ oscope\end{tabular} & Avg & Union \\ \hline
$\ell_2$ & 0 & \textbf{90.04} & 83.95 & 7.57 & 5.27 & 33.91 & 65.17 & \textbf{89.3} & 28.99 & 67.52 & 62.85 & 49.22 & 45.78 & 12.76 & 46.03 & 0.51 \\
$\ell_2$ & 0.1 & \cellcolor[HTML]{F4CCCC}88.97 & 83.13 & \cellcolor[HTML]{B7E1CD}12.69 & \cellcolor[HTML]{B7E1CD}8.54 & \cellcolor[HTML]{B7E1CD}35.64 & \cellcolor[HTML]{B7E1CD}66.65 & \cellcolor[HTML]{F4CCCC}88.2 & \cellcolor[HTML]{B7E1CD}31.39 & \cellcolor[HTML]{B7E1CD}68.69 & \cellcolor[HTML]{F4CCCC}61.73 & 49.22 & 46.34 & \cellcolor[HTML]{B7E1CD}17.3 & \cellcolor[HTML]{B7E1CD}47.46 & 1.2 \\
$\ell_2$ & 0.2 & 89.68 & 84.36 & \cellcolor[HTML]{B7E1CD}16.94 & \cellcolor[HTML]{B7E1CD}9.96 & \cellcolor[HTML]{B7E1CD}39.95 & \cellcolor[HTML]{B7E1CD}66.37 & 88.89 & \cellcolor[HTML]{B7E1CD}32.48 & \cellcolor[HTML]{B7E1CD}72.87 & \textbf{63.67} & \cellcolor[HTML]{B7E1CD}52.25 & \cellcolor[HTML]{B7E1CD}48.64 & \cellcolor[HTML]{B7E1CD}17.1 & \cellcolor[HTML]{B7E1CD}49.46 & \cellcolor[HTML]{B7E1CD}1.83 \\
$\ell_2$ & 0.5 & 89.76 & \textbf{84.41} & \cellcolor[HTML]{B7E1CD}\textbf{25.32} & \cellcolor[HTML]{B7E1CD}\textbf{16.08} & \cellcolor[HTML]{B7E1CD}\textbf{44.38} & \cellcolor[HTML]{B7E1CD}\textbf{69.25} & \cellcolor[HTML]{F4CCCC}87.69 & \cellcolor[HTML]{B7E1CD}\textbf{35.46} & \cellcolor[HTML]{B7E1CD}\textbf{74.42} & 62.47 & \cellcolor[HTML]{B7E1CD}\textbf{57.66} & \cellcolor[HTML]{B7E1CD}\textbf{50.37} & \cellcolor[HTML]{B7E1CD}\textbf{20.84} & \cellcolor[HTML]{B7E1CD}\textbf{52.36} & \cellcolor[HTML]{B7E1CD}\textbf{41.27} \\ \hline
$\ell_\infty$ & 0 & \textbf{84.51} & \textbf{81.71} & 58.39 & 43.49 & 67.82 & 72.61 & \textbf{83.31} & 41.83 & \textbf{65.35} & 63.9 & \textbf{67.18} & 63.64 & 30.75 & 61.67 & 13.2 \\
$\ell_\infty$ & 0.1 & \cellcolor[HTML]{F4C7C3}82.29 & \cellcolor[HTML]{F4C7C3}80.25 & 58.5 & \cellcolor[HTML]{B7E1CD}47.85 & \cellcolor[HTML]{F4C7C3}66.8 & \cellcolor[HTML]{F4C7C3}69.55 & \cellcolor[HTML]{F4C7C3}81.5 & 41.4 & 65.53 & \cellcolor[HTML]{F4C7C3}61.45 & \cellcolor[HTML]{F4C7C3}63.13 & 64.23 & \cellcolor[HTML]{B7E1CD}33.89 & 61.17 & \cellcolor[HTML]{B7E1CD}16.15 \\
$\ell_\infty$ & 0.2 & 83.56 & 81.04 & \cellcolor[HTML]{B7E1CD}\textbf{59.75} & \cellcolor[HTML]{B7E1CD}46.04 & 67.69 & \textbf{73.12} & 82.75 & \textbf{42.27} & \cellcolor[HTML]{B7E1CD}\textbf{67.26} & \textbf{63.97} & 67.01 & 62.93 & \cellcolor[HTML]{B7E1CD}34.04 & \textbf{62.32} & \cellcolor[HTML]{B7E1CD}16.92 \\
$\ell_\infty$ & 0.5 & \cellcolor[HTML]{F4C7C3}81.53 & \cellcolor[HTML]{F4C7C3}77.91 & \cellcolor[HTML]{B7E1CD}59.62 & \cellcolor[HTML]{B7E1CD}\textbf{48.48} & \cellcolor[HTML]{B7E1CD}\textbf{69.3} & \cellcolor[HTML]{F4C7C3}67.87 & \cellcolor[HTML]{F4C7C3}80.23 & 41.86 & \cellcolor[HTML]{F4C7C3}54.85 & \cellcolor[HTML]{F4C7C3}61.71 & \cellcolor[HTML]{F4C7C3}64.18 & \cellcolor[HTML]{B7E1CD}\textbf{67.77} & \cellcolor[HTML]{B7E1CD}\textbf{38.98} & 61.06 & \cellcolor[HTML]{B7E1CD}\textbf{19.31} \\ \hline
StAdv & 0 & \textbf{83.31} & \textbf{77.58} & 1.45 & \textbf{69.81} & \textbf{13.43} & 36.66 & \textbf{81.5} & \textbf{20.56} & 49.89 & 70.32 & 60.76 & 36.15 & \textbf{24.84} & \textbf{45.25} & 1.04 \\
StAdv & 0.1 & \cellcolor[HTML]{F4C7C3}81.89 & \cellcolor[HTML]{F4C7C3}76.31 & \textbf{2.24} & 69.17 & \cellcolor[HTML]{F4C7C3}12.23 & \cellcolor[HTML]{F4C7C3}32.23 & \cellcolor[HTML]{F4C7C3}79.92 & \cellcolor[HTML]{F4C7C3}18.39 & \cellcolor[HTML]{B7E1CD}\textbf{52.25} & \textbf{70.42} & 60.46 & \cellcolor[HTML]{B7E1CD}39.29 & \cellcolor[HTML]{F4C7C3}20.59 & 44.46 & 1.25 \\
StAdv & 0.2 & \cellcolor[HTML]{F4C7C3}82.09 & 76.66 & 1.71 & 69.04 & \cellcolor[HTML]{F4C7C3}11.11 & \cellcolor[HTML]{B7E1CD}45.25 & \cellcolor[HTML]{F4C7C3}80.13 & \cellcolor[HTML]{F4C7C3}17.1 & \cellcolor[HTML]{F4C7C3}44.74 & 69.61 & \textbf{60.89} & \cellcolor[HTML]{B7E1CD}40.03 & \cellcolor[HTML]{F4C7C3}21.32 & 44.8 & 1.25 \\
StAdv & 0.5 & \cellcolor[HTML]{F4C7C3}80.66 & \cellcolor[HTML]{F4C7C3}73.86 & 1.68 & \cellcolor[HTML]{F4C7C3}67.67 & \cellcolor[HTML]{F4C7C3}10.65 & \cellcolor[HTML]{B7E1CD}\textbf{60.03} & \cellcolor[HTML]{F4C7C3}77.58 & \cellcolor[HTML]{F4C7C3}15.39 & \cellcolor[HTML]{F4C7C3}44.97 & \cellcolor[HTML]{F4C7C3}67.8 & \cellcolor[HTML]{F4C7C3}59.41 & \cellcolor[HTML]{B7E1CD}\textbf{41.86} & \cellcolor[HTML]{F4C7C3}21.78 & 45.22 & \textbf{1.35} \\ \hline
ReColor & 0 & 91.34 & 81.53 & 0.03 & 0.41 & 79.08 & 42.55 & 90.6 & 22.39 & 25.2 & 64.31 & 54.8 & 18.65 & 8.94 & 40.71 & 0 \\
ReColor & 0.1 & 91.26 & 82.5 & \textbf{0.05} & 0.48 & \cellcolor[HTML]{B7E1CD}80.74 & \cellcolor[HTML]{B7E1CD}47.03 & 91.08 & \cellcolor[HTML]{B7E1CD}25.3 & \cellcolor[HTML]{B7E1CD}27.16 & 63.54 & 55.39 & \cellcolor[HTML]{B7E1CD}20.94 & \cellcolor[HTML]{B7E1CD}11.75 & \cellcolor[HTML]{B7E1CD}42.16 & 0 \\
ReColor & 0.2 & \textbf{91.87} & \cellcolor[HTML]{B7E1CD}\textbf{83.67} & 0.03 & \textbf{0.89} & \cellcolor[HTML]{B7E1CD}\textbf{81.76} & \cellcolor[HTML]{B7E1CD}\textbf{49.63} & \textbf{91.44} & \cellcolor[HTML]{B7E1CD}23.62 & \cellcolor[HTML]{B7E1CD}\textbf{34.62} & \cellcolor[HTML]{B7E1CD}\textbf{66.37} & \cellcolor[HTML]{B7E1CD}57.61 & \cellcolor[HTML]{B7E1CD}23.64 & \cellcolor[HTML]{B7E1CD}10.93 & \cellcolor[HTML]{B7E1CD}43.68 & 0 \\
ReColor & 0.5 & \cellcolor[HTML]{F4C7C3}90.32 & 81.43 & 0 & 0.84 & \cellcolor[HTML]{B7E1CD}80.69 & \cellcolor[HTML]{B7E1CD}46.32 & \cellcolor[HTML]{F4C7C3}89.55 & \cellcolor[HTML]{B7E1CD}\textbf{30.09} & \cellcolor[HTML]{B7E1CD}32.18 & \cellcolor[HTML]{B7E1CD}65.63 & \cellcolor[HTML]{B7E1CD}\textbf{61.17} & \cellcolor[HTML]{B7E1CD}\textbf{26.42} & \cellcolor[HTML]{B7E1CD}\textbf{13.35} & \cellcolor[HTML]{B7E1CD}\textbf{43.97} & 0 \\ \hline
Gabor & 0 & \textbf{89.12} & 85.27 & 4.41 & 2.11 & 37.83 & 87.31 & \textbf{87.62} & 20.28 & 57.66 & 52.33 & 38.73 & 38.88 & 9.22 & 43.47 & 0.1 \\
Gabor & 0.1 & \cellcolor[HTML]{F4C7C3}87.44 & \cellcolor[HTML]{F4C7C3}83.77 & \cellcolor[HTML]{B7E1CD}9.45 & \cellcolor[HTML]{B7E1CD}5.96 & \cellcolor[HTML]{B7E1CD}40.05 & \cellcolor[HTML]{F4C7C3}86.01 & \cellcolor[HTML]{F4C7C3}85.5 & \cellcolor[HTML]{B7E1CD}21.94 & 58.47 & 52.41 & \cellcolor[HTML]{B7E1CD}41.61 & \cellcolor[HTML]{B7E1CD}43.57 & \cellcolor[HTML]{B7E1CD}16.13 & \cellcolor[HTML]{B7E1CD}45.41 & \cellcolor[HTML]{B7E1CD}1.25 \\
Gabor & 0.2 & 88.56 & \textbf{85.48} & \cellcolor[HTML]{B7E1CD}13.4 & \cellcolor[HTML]{B7E1CD}8.84 & \cellcolor[HTML]{B7E1CD}44.28 & \textbf{87.36} & 87.54 & \cellcolor[HTML]{B7E1CD}27.77 & \cellcolor[HTML]{B7E1CD}62.98 & \cellcolor[HTML]{B7E1CD}54.98 & \cellcolor[HTML]{B7E1CD}47.92 & \cellcolor[HTML]{B7E1CD}48.05 & \cellcolor[HTML]{B7E1CD}\textbf{16.87} & \cellcolor[HTML]{B7E1CD}48.79 & \cellcolor[HTML]{B7E1CD}1.4 \\
Gabor & 0.5 & \cellcolor[HTML]{F4C7C3}87.06 & 84.33 & \cellcolor[HTML]{B7E1CD}\textbf{21.04} & \cellcolor[HTML]{B7E1CD}\textbf{15.34} & \cellcolor[HTML]{B7E1CD}\textbf{50.45} & \cellcolor[HTML]{F4C7C3}85.38 & \cellcolor[HTML]{F4C7C3}84.18 & \cellcolor[HTML]{B7E1CD}\textbf{32.05} & \cellcolor[HTML]{B7E1CD}\textbf{64.84} & \cellcolor[HTML]{B7E1CD}\textbf{55.62} & \cellcolor[HTML]{B7E1CD}\textbf{53.61} & \cellcolor[HTML]{B7E1CD}\textbf{51.44} & \cellcolor[HTML]{B7E1CD}15.46 & \cellcolor[HTML]{B7E1CD}\textbf{51.14} & \cellcolor[HTML]{B7E1CD}\textbf{4.05} \\ \hline
Snow & 0 & 87.69 & 71.59 & \textbf{0.08} & 1.53 & 11.97 & 30.68 & 62.9 & 7.31 & 8.59 & 66.24 & 70.78 & 11.82 & 9.91 & 29.45 & \textbf{0.05} \\
Snow & 0.1 & \textbf{88.18} & \cellcolor[HTML]{F4C7C3}70.47 & \textbf{0.08} & 1.32 & 11.9 & 29.86 & \cellcolor[HTML]{F4C7C3}59.97 & 6.96 & \cellcolor[HTML]{F4C7C3}7.26 & \textbf{66.57} & 71.24 & 11.67 & \cellcolor[HTML]{F4C7C3}8.61 & 28.83 & 0.03 \\
Snow & 0.2 & 87.69 & \textbf{71.85} & 0.03 & 1.81 & 12.71 & \cellcolor[HTML]{B7E1CD}\textbf{44.33} & 62.27 & \cellcolor[HTML]{B7E1CD}8.31 & 9.22 & 66.39 & 71.08 & \cellcolor[HTML]{B7E1CD}\textbf{15.54} & \cellcolor[HTML]{B7E1CD}\textbf{13.12} & \cellcolor[HTML]{B7E1CD}\textbf{31.39} & 0.03 \\
Snow & 0.5 & 88 & 71.39 & 0.05 & \textbf{2.06} & \cellcolor[HTML]{B7E1CD}\textbf{15.77} & \cellcolor[HTML]{B7E1CD}38.88 & \textbf{62.32} & \cellcolor[HTML]{B7E1CD}\textbf{10.22} & \cellcolor[HTML]{B7E1CD}\textbf{10.24} & 66.34 & \cellcolor[HTML]{B7E1CD}\textbf{72.25} & \cellcolor[HTML]{B7E1CD}15.13 & 10.47 & \cellcolor[HTML]{B7E1CD}31.26 & 0 \\ \hline
Pixel & 0 & 88.64 & 67.24 & 0 & 0.59 & 34.96 & 30.37 & 87.18 & 78.6 & 0.25 & 61.63 & 52.08 & 49.38 & 23.97 & 40.52 & 0 \\
Pixel & 0.1 & \cellcolor[HTML]{B7E1CD}89.73 & \cellcolor[HTML]{B7E1CD}69.17 & 0 & 0.79 & \cellcolor[HTML]{B7E1CD}39.18 & \cellcolor[HTML]{B7E1CD}\textbf{37.81} & \cellcolor[HTML]{B7E1CD}88.61 & \cellcolor[HTML]{B7E1CD}81.48 & 0.51 & 62.17 & \cellcolor[HTML]{F4C7C3}49.02 & 50.11 & \cellcolor[HTML]{B7E1CD}26.32 & \cellcolor[HTML]{B7E1CD}42.1 & 0 \\
Pixel & 0.2 & \cellcolor[HTML]{B7E1CD}\textbf{90.47} & \cellcolor[HTML]{B7E1CD}\textbf{70.55} & 0 & 0.61 & \cellcolor[HTML]{B7E1CD}45.02 & \cellcolor[HTML]{B7E1CD}32.69 & \cellcolor[HTML]{B7E1CD}\textbf{89.07} & \cellcolor[HTML]{B7E1CD}\textbf{83.06} & \textbf{1.07} & \cellcolor[HTML]{B7E1CD}\textbf{65.99} & 51.59 & \cellcolor[HTML]{B7E1CD}53.32 & \cellcolor[HTML]{B7E1CD}28.23 & \cellcolor[HTML]{B7E1CD}\textbf{43.43} & 0 \\
Pixel & 0.5 & 88.2 & \cellcolor[HTML]{F4C7C3}64.38 & 0 & \textbf{1.12} & \cellcolor[HTML]{B7E1CD}\textbf{45.68} & 30.11 & 87.59 & \cellcolor[HTML]{B7E1CD}82.37 & 0.74 & \cellcolor[HTML]{B7E1CD}64.13 & \cellcolor[HTML]{B7E1CD}\textbf{53.35} & \cellcolor[HTML]{B7E1CD}\textbf{59.46} & \cellcolor[HTML]{B7E1CD}\textbf{30.27} & \cellcolor[HTML]{B7E1CD}43.27 & 0 \\ \hline
JPEG & 0 & \textbf{88.43} & 85.63 & 15.29 & 5.43 & 41.78 & 77.35 & \textbf{86.85} & 23.21 & 80.87 & 53.81 & 43.39 & 44.79 & 15.39 & 47.82 & 0.74 \\
JPEG & 0.1 & 88.23 & 85.68 & \cellcolor[HTML]{B7E1CD}22.98 & \cellcolor[HTML]{B7E1CD}9.17 & \cellcolor[HTML]{B7E1CD}45.22 & \cellcolor[HTML]{B7E1CD}84.03 & 86.17 & 24.08 & 81.76 & \cellcolor[HTML]{B7E1CD}\textbf{56.87} & \cellcolor[HTML]{B7E1CD}46.52 & \cellcolor[HTML]{B7E1CD}46.47 & 15.08 & \cellcolor[HTML]{B7E1CD}50.34 & \cellcolor[HTML]{B7E1CD}2.04 \\
JPEG & 0.2 & 88.2 & \textbf{85.71} & \cellcolor[HTML]{B7E1CD}25.68 & \cellcolor[HTML]{B7E1CD}11.82 & \cellcolor[HTML]{B7E1CD}44.31 & \cellcolor[HTML]{B7E1CD}82.85 & 86.17 & \cellcolor[HTML]{B7E1CD}24.71 & 81.86 & \cellcolor[HTML]{B7E1CD}56.25 & \cellcolor[HTML]{B7E1CD}47.67 & \cellcolor[HTML]{B7E1CD}47.29 & \cellcolor[HTML]{B7E1CD}17.66 & \cellcolor[HTML]{B7E1CD}51 & \cellcolor[HTML]{B7E1CD}2.55 \\
JPEG & 0.5 & \cellcolor[HTML]{F4C7C3}87.08 & 84.89 & \cellcolor[HTML]{B7E1CD}\textbf{31.8} & \cellcolor[HTML]{B7E1CD}\textbf{15.75} & \cellcolor[HTML]{B7E1CD}\textbf{51.36} & \cellcolor[HTML]{B7E1CD}\textbf{84.46} & \cellcolor[HTML]{F4C7C3}84.92 & \cellcolor[HTML]{B7E1CD}\textbf{28.82} & \cellcolor[HTML]{B7E1CD}\textbf{82.06} & \cellcolor[HTML]{B7E1CD}54.93 & \cellcolor[HTML]{B7E1CD}\textbf{48.18} & \cellcolor[HTML]{B7E1CD}\textbf{54.85} & \cellcolor[HTML]{B7E1CD}\textbf{18.6} & \cellcolor[HTML]{B7E1CD}\textbf{53.38} & \cellcolor[HTML]{B7E1CD}\textbf{3.8} \\ \hline
Elastic & 0 & 89.66 & 77.48 & 0 & 0.82 & 12.25 & 21.81 & 88.05 & 16.84 & 5.71 & 78.6 & 62.01 & 16.79 & 11.69 & 32.67 & 0 \\
Elastic & 0.1 & \cellcolor[HTML]{B7E1CD}\textbf{90.96} & \cellcolor[HTML]{B7E1CD}80.36 & 0 & \cellcolor[HTML]{B7E1CD}2.57 & 12.54 & \cellcolor[HTML]{B7E1CD}28.05 & \cellcolor[HTML]{B7E1CD}\textbf{89.71} & 16.18 & \cellcolor[HTML]{B7E1CD}14.04 & \cellcolor[HTML]{B7E1CD}82.62 & \cellcolor[HTML]{B7E1CD}65.58 & 16.99 & \cellcolor[HTML]{B7E1CD}\textbf{14.47} & \cellcolor[HTML]{B7E1CD}35.26 & 0 \\
Elastic & 0.2 & \cellcolor[HTML]{F4C7C3}88.41 & \cellcolor[HTML]{B7E1CD}79.18 & \textbf{0.05} & \cellcolor[HTML]{B7E1CD}5.15 & 13.04 & \cellcolor[HTML]{B7E1CD}26.42 & \cellcolor[HTML]{F4C7C3}86.7 & \cellcolor[HTML]{B7E1CD}17.89 & \cellcolor[HTML]{B7E1CD}15.24 & \cellcolor[HTML]{B7E1CD}81.27 & \cellcolor[HTML]{B7E1CD}66.11 & \cellcolor[HTML]{B7E1CD}\textbf{21.02} & \cellcolor[HTML]{B7E1CD}13.55 & \cellcolor[HTML]{B7E1CD}35.47 & \textbf{0.03} \\
Elastic & 0.5 & 89.53 & \cellcolor[HTML]{B7E1CD}\textbf{80.79} & 0.03 & \cellcolor[HTML]{B7E1CD}\textbf{8.79} & \cellcolor[HTML]{B7E1CD}\textbf{16.71} & \cellcolor[HTML]{B7E1CD}\textbf{29.66} & 87.62 & \cellcolor[HTML]{B7E1CD}\textbf{19.52} & \cellcolor[HTML]{B7E1CD}\textbf{22.29} & \cellcolor[HTML]{B7E1CD}\textbf{83.97} & \cellcolor[HTML]{B7E1CD}\textbf{66.88} & \cellcolor[HTML]{B7E1CD}20.31 & 11.21 & \cellcolor[HTML]{B7E1CD}\textbf{37.31} & 0 \\ \hline
Wood & 0 & 85.91 & 50.09 & 0 & 0.82 & 11.11 & 35.87 & 83.31 & 11.13 & 1.2 & 60.94 & 78.83 & 14.96 & \textbf{11.13} & 29.95 & 0 \\
Wood & 0.1 & \cellcolor[HTML]{B7E1CD}88.28 & \cellcolor[HTML]{B7E1CD}\textbf{75.77} & 0 & \cellcolor[HTML]{B7E1CD}2.42 & \cellcolor[HTML]{B7E1CD}\textbf{14.14} & \cellcolor[HTML]{B7E1CD}38.9 & \cellcolor[HTML]{B7E1CD}86.09 & \cellcolor[HTML]{B7E1CD}\textbf{12.97} & \cellcolor[HTML]{B7E1CD}\textbf{9.94} & \cellcolor[HTML]{B7E1CD}67.18 & \cellcolor[HTML]{B7E1CD}84.31 & \cellcolor[HTML]{B7E1CD}\textbf{22.09} & \cellcolor[HTML]{F4C7C3}9.66 & \cellcolor[HTML]{B7E1CD}\textbf{35.29} & 0 \\
Wood & 0.2 & \cellcolor[HTML]{B7E1CD}\textbf{89.45} & \cellcolor[HTML]{B7E1CD}72.61 & \textbf{0.03} & \cellcolor[HTML]{B7E1CD}1.91 & 11.44 & \cellcolor[HTML]{B7E1CD}\textbf{43.69} & \cellcolor[HTML]{B7E1CD}\textbf{87.36} & \cellcolor[HTML]{F4C7C3}8.05 & \cellcolor[HTML]{B7E1CD}9.86 & \cellcolor[HTML]{B7E1CD}\textbf{67.49} & \cellcolor[HTML]{B7E1CD}\textbf{86.88} & \cellcolor[HTML]{B7E1CD}16.79 & 10.17 & \cellcolor[HTML]{B7E1CD}34.69 & \textbf{0.03} \\
Wood & 0.5 & 85.99 & \cellcolor[HTML]{B7E1CD}67.18 & \textbf{0.03} & \cellcolor[HTML]{B7E1CD}\textbf{4.1} & \cellcolor[HTML]{B7E1CD}12.87 & \cellcolor[HTML]{B7E1CD}38.8 & 83.01 & 10.34 & \cellcolor[HTML]{B7E1CD}9.91 & \cellcolor[HTML]{B7E1CD}66.37 & \cellcolor[HTML]{B7E1CD}85.07 & \cellcolor[HTML]{B7E1CD}20.46 & \cellcolor[HTML]{F4C7C3}9.32 & \cellcolor[HTML]{B7E1CD}33.96 & \textbf{0.03} \\ \hline
Glitch & 0 & \textbf{88.51} & 36.41 & 0 & 0 & 6.7 & 18.47 & \textbf{86.96} & 17.1 & 0 & 60 & 50.93 & 84.97 & 6.37 & 30.66 & 0 \\
Glitch & 0.1 & 88.33 & \cellcolor[HTML]{F4C7C3}32.38 & 0 & 0.08 & 7.26 & \cellcolor[HTML]{B7E1CD}21.86 & 86.27 & \cellcolor[HTML]{F4C7C3}15.36 & 0.03 & \cellcolor[HTML]{B7E1CD}61.2 & 50.6 & \cellcolor[HTML]{B7E1CD}\textbf{86.68} & 6.73 & 30.7 & 0 \\
Glitch & 0.2 & 87.85 & \cellcolor[HTML]{B7E1CD}48.89 & \textbf{0.03} & 0.43 & \cellcolor[HTML]{B7E1CD}22.78 & \cellcolor[HTML]{B7E1CD}23.75 & \cellcolor[HTML]{F4C7C3}85.35 & \cellcolor[HTML]{B7E1CD}20.08 & \cellcolor[HTML]{B7E1CD}2.52 & \cellcolor[HTML]{B7E1CD}63.13 & 50.34 & \cellcolor[HTML]{B7E1CD}86.19 & \cellcolor[HTML]{B7E1CD}10.42 & \cellcolor[HTML]{B7E1CD}34.49 & 0 \\
Glitch & 0.5 & \cellcolor[HTML]{F4C7C3}86.98 & \cellcolor[HTML]{B7E1CD}\textbf{65.1} & 0 & \cellcolor[HTML]{B7E1CD}\textbf{2.57} & \cellcolor[HTML]{B7E1CD}\textbf{23.11} & \cellcolor[HTML]{B7E1CD}\textbf{32.18} & \cellcolor[HTML]{F4C7C3}84.05 & \cellcolor[HTML]{B7E1CD}\textbf{30.7} & \cellcolor[HTML]{B7E1CD}\textbf{7.11} & \cellcolor[HTML]{B7E1CD}\textbf{64.36} & \cellcolor[HTML]{B7E1CD}\textbf{58.75} & 85.35 & \cellcolor[HTML]{B7E1CD}\textbf{16.64} & \cellcolor[HTML]{B7E1CD}\textbf{39.16} & 0 \\ \hline
\begin{tabular}[c]{@{}l@{}}Kaleid-\\ oscope\end{tabular} & 0 & 88.1 & 73.5 & 0 & 0.31 & 7.03 & 28.66 & 85.91 & 18.83 & 2.22 & 62.98 & 29.78 & 21.07 & 84.89 & 34.6 & 0 \\
\begin{tabular}[c]{@{}l@{}}Kaleid-\\ oscope\end{tabular} & 0.1 & 88.38 & \cellcolor[HTML]{B7E1CD}78.98 & 0.03 & \cellcolor[HTML]{B7E1CD}2.68 & \cellcolor[HTML]{B7E1CD}13.96 & \cellcolor[HTML]{B7E1CD}\textbf{33.53} & \textbf{85.96} & \cellcolor[HTML]{B7E1CD}25.12 & \cellcolor[HTML]{B7E1CD}11.87 & \cellcolor[HTML]{B7E1CD}67.13 & \cellcolor[HTML]{B7E1CD}49.12 & \cellcolor[HTML]{B7E1CD}27.77 & 85.66 & \cellcolor[HTML]{B7E1CD}40.15 & 0 \\
\begin{tabular}[c]{@{}l@{}}Kaleid-\\ oscope\end{tabular} & 0.2 & \textbf{88.51} & \cellcolor[HTML]{B7E1CD}78.78 & 0.03 & \cellcolor[HTML]{B7E1CD}6.5 & \cellcolor[HTML]{B7E1CD}20.48 & \cellcolor[HTML]{B7E1CD}31.95 & 85.66 & \cellcolor[HTML]{B7E1CD}29.4 & \cellcolor[HTML]{B7E1CD}10.22 & \cellcolor[HTML]{B7E1CD}68.51 & \cellcolor[HTML]{B7E1CD}34.85 & \cellcolor[HTML]{B7E1CD}36.74 & \cellcolor[HTML]{B7E1CD}\textbf{86.09} & \cellcolor[HTML]{B7E1CD}40.77 & 0.03 \\
\begin{tabular}[c]{@{}l@{}}Kaleid-\\ oscope\end{tabular} & 0.5 & \cellcolor[HTML]{F4C7C3}87.01 & \cellcolor[HTML]{B7E1CD}\textbf{79.62} & \textbf{0.56} & \cellcolor[HTML]{B7E1CD}\textbf{16.38} & \cellcolor[HTML]{B7E1CD}\textbf{24.13} & \cellcolor[HTML]{B7E1CD}31.67 & \cellcolor[HTML]{F4C7C3}84.48 & \cellcolor[HTML]{B7E1CD}\textbf{32.74} & \cellcolor[HTML]{B7E1CD}\textbf{23.06} & \cellcolor[HTML]{B7E1CD}\textbf{68.61} & \cellcolor[HTML]{B7E1CD}\textbf{58.09} & \cellcolor[HTML]{B7E1CD}\textbf{37.94} & 84.2 & \cellcolor[HTML]{B7E1CD}\textbf{45.12} & \textbf{0.43} \\ \hline
\end{tabular}}}
\caption{\textbf{Intial Training Ablations- worst-case $\ell_2$ regularization on ImageNette. }Accuracy of initially trained models on ImageNette trained using different attacks as indicated in ``Train Attack" column measured across different attacks. $\ell_2$ regularization computed using single step optimization is also considered during initial training, with regularization strength $\lambda$.  Results where regularization improves over no regularization ($\lambda = 0$) by at least 1\% accuracy are highlighted in green, while results where regularization incurs at least a 1\% drop in accuracy are highlighted in red.  Best performing with respect to regularization strength are bolded.}
\label{app:IT_ablation_L2_IM}
\end{table*}

\begin{table*}[]{\renewcommand{\arraystretch}{1.2}
\scalebox{0.73}{
\begin{tabular}{|l|c|c|cccccccccccc|cc|}
\hline
\begin{tabular}[c]{@{}l@{}}Train\\ Attack\end{tabular} & $\lambda$ & Clean & $\ell_2$ & $\ell_\infty$ & StAdv & ReColor & Gabor & Snow & Pixel & JPEG & Elastic & Wood & Glitch & \begin{tabular}[c]{@{}c@{}}Kaleid-\\ oscope\end{tabular} & Avg & Union \\ \hline
$\ell_2$ & 0 & \textbf{90.04} & 83.95 & 7.57 & 5.27 & 33.91 & \textbf{65.17} & \textbf{89.3} & 28.99 & 67.52 & 62.85 & 49.22 & \textbf{45.78} & 12.76 & 46.03 & 0.51 \\
$\ell_2$ & 0.05 & \cellcolor[HTML]{F4CCCC}88.51 & \cellcolor[HTML]{F4CCCC}82.75 & \cellcolor[HTML]{B7E1CD}13.55 & \cellcolor[HTML]{B7E1CD}10.14 & 33.5 & \cellcolor[HTML]{F4CCCC}63.03 & \cellcolor[HTML]{F4CCCC}87.11 & 29.66 & \cellcolor[HTML]{B7E1CD}69.17 & \cellcolor[HTML]{B7E1CD}\textbf{63.9} & \cellcolor[HTML]{B7E1CD}53.27 & \cellcolor[HTML]{F4CCCC}41.78 & \cellcolor[HTML]{B7E1CD}17.12 & \cellcolor[HTML]{B7E1CD}47.08 & \cellcolor[HTML]{B7E1CD}2.22 \\
$\ell_2$ & 0.1 & 89.5 & 83.64 & \cellcolor[HTML]{B7E1CD}16.89 & \cellcolor[HTML]{B7E1CD}9.2 & \cellcolor[HTML]{B7E1CD}41.71 & \cellcolor[HTML]{F4CCCC}56.97 & \cellcolor[HTML]{F4CCCC}86.93 & \cellcolor[HTML]{B7E1CD}\textbf{37.91} & \cellcolor[HTML]{B7E1CD}71.52 & 62.62 & \cellcolor[HTML]{B7E1CD}\textbf{54.34} & \cellcolor[HTML]{B7E1CD}\textbf{51.85} & \cellcolor[HTML]{B7E1CD}\textbf{20.71} & \cellcolor[HTML]{B7E1CD}49.52 & 1.45 \\
$\ell_2$ & 0.2 & \cellcolor[HTML]{F4CCCC}88.87 & \textbf{84.17} & \cellcolor[HTML]{B7E1CD}\textbf{25.27} & \cellcolor[HTML]{B7E1CD}\textbf{14.04} & \cellcolor[HTML]{B7E1CD}\textbf{44.71} & \cellcolor[HTML]{B7E1CD}74.27 & \cellcolor[HTML]{F4CCCC}86.96 & \cellcolor[HTML]{B7E1CD}36.74 & \cellcolor[HTML]{B7E1CD}\textbf{74.39} & 63.36 & \cellcolor[HTML]{B7E1CD}53.07 & \cellcolor[HTML]{B7E1CD}49.45 & \cellcolor[HTML]{B7E1CD}15.34 & \cellcolor[HTML]{B7E1CD}\textbf{51.82} & \cellcolor[HTML]{B7E1CD}\textbf{2.42} \\ \hline
$\ell_\infty$ & 0 & \textbf{84.51} & \textbf{81.71} & 58.39 & 43.49 & 67.82 & \textbf{72.61} & \textbf{83.31} & 41.83 & \textbf{65.35} & \textbf{63.9} & \textbf{67.18} & 63.64 & 30.75 & \textbf{61.67} & 13.2 \\
$\ell_\infty$ & 0.05 & \cellcolor[HTML]{F4C7C3}83.13 & \cellcolor[HTML]{F4C7C3}80.15 & 58.73 & \cellcolor[HTML]{B7E1CD}44.92 & 67.44 & \cellcolor[HTML]{F4C7C3}65.53 & 83.24 & 41.53 & \cellcolor[HTML]{F4C7C3}61.58 & \cellcolor[HTML]{F4C7C3}62.42 & \cellcolor[HTML]{F4C7C3}62.98 & \cellcolor[HTML]{B7E1CD}65.07 & \cellcolor[HTML]{F4C7C3}29.63 & \cellcolor[HTML]{F4C7C3}60.26 & \cellcolor[HTML]{B7E1CD}14.55 \\
$\ell_\infty$ & 0.1 & \cellcolor[HTML]{F4C7C3}83.36 & \cellcolor[HTML]{F4C7C3}80.23 & 58.93 & 43.57 & 67.85 & \cellcolor[HTML]{F4C7C3}69.71 & 82.85 & \cellcolor[HTML]{F4C7C3}40.13 & \cellcolor[HTML]{F4C7C3}56.79 & 63.13 & \cellcolor[HTML]{F4C7C3}65.58 & \cellcolor[HTML]{B7E1CD}65.66 & \cellcolor[HTML]{B7E1CD}31.97 & \cellcolor[HTML]{F4C7C3}60.53 & \cellcolor[HTML]{B7E1CD}16.08 \\
$\ell_\infty$ & 0.2 & \cellcolor[HTML]{F4C7C3}81.22 & \cellcolor[HTML]{F4C7C3}77.83 & \textbf{59.21} & \cellcolor[HTML]{B7E1CD}\textbf{51.57} & \textbf{67.97} & \cellcolor[HTML]{F4C7C3}65.66 & \cellcolor[HTML]{F4C7C3}80.69 & \textbf{42.52} & \cellcolor[HTML]{F4C7C3}58.19 & \cellcolor[HTML]{F4C7C3}61.76 & \cellcolor[HTML]{F4C7C3}61.15 & \cellcolor[HTML]{B7E1CD}\textbf{67.44} & \cellcolor[HTML]{B7E1CD}\textbf{40.03} & 61.17 & \cellcolor[HTML]{B7E1CD}\textbf{19.13} \\ \hline
StAdv & 0 & \textbf{83.31} & \textbf{77.58} & 1.45 & 69.81 & 13.43 & 36.66 & \textbf{81.5} & 20.56 & 49.89 & \textbf{70.32} & 60.76 & 36.15 & 24.84 & 45.25 & 1.04 \\
StAdv & 0.05 & \cellcolor[HTML]{F4C7C3}82.14 & 77.17 & 2.06 & \cellcolor[HTML]{B7E1CD}73.02 & \cellcolor[HTML]{B7E1CD}24.94 & \cellcolor[HTML]{B7E1CD}45.5 & \cellcolor[HTML]{F4C7C3}78.62 & \cellcolor[HTML]{B7E1CD}\textbf{22.09} & 48.97 & 69.99 & \cellcolor[HTML]{B7E1CD}62.47 & \cellcolor[HTML]{B7E1CD}\textbf{38.73} & \cellcolor[HTML]{B7E1CD}\textbf{28.79} & \cellcolor[HTML]{B7E1CD}\textbf{47.7} & 1.63 \\
StAdv & 0.1 & \cellcolor[HTML]{F4C7C3}79.08 & \cellcolor[HTML]{F4C7C3}72.79 & 1.86 & 70.19 & \cellcolor[HTML]{B7E1CD}\textbf{35.87} & \cellcolor[HTML]{B7E1CD}48.92 & \cellcolor[HTML]{F4C7C3}74.52 & 20.25 & \cellcolor[HTML]{F4C7C3}43.9 & \cellcolor[HTML]{F4C7C3}63.29 & \cellcolor[HTML]{F4C7C3}57.96 & 35.21 & 24.05 & 45.73 & 1.55 \\
StAdv & 0.2 & \cellcolor[HTML]{F4C7C3}80.89 & \cellcolor[HTML]{F4C7C3}74.09 & \cellcolor[HTML]{B7E1CD}\textbf{3.31} & \cellcolor[HTML]{B7E1CD}\textbf{75.49} & \cellcolor[HTML]{B7E1CD}30.42 & \cellcolor[HTML]{B7E1CD}\textbf{51.21} & \cellcolor[HTML]{F4C7C3}77.61 & \cellcolor[HTML]{F4C7C3}16.56 & \cellcolor[HTML]{B7E1CD}\textbf{52.56} & \cellcolor[HTML]{F4C7C3}67.06 & \cellcolor[HTML]{B7E1CD}\textbf{62.75} & 36.28 & \cellcolor[HTML]{F4C7C3}22.39 & \cellcolor[HTML]{B7E1CD}47.48 & \cellcolor[HTML]{B7E1CD}\textbf{2.98} \\ \hline
ReColor & 0 & 91.34 & 81.53 & 0.03 & 0.41 & 79.08 & 42.55 & 90.6 & 22.39 & 25.2 & 64.31 & 54.8 & 18.65 & 8.94 & 40.71 & 0 \\
ReColor & 0.05 & 91.08 & 80.64 & 0.03 & 0.56 & 79.24 & \cellcolor[HTML]{B7E1CD}46.8 & 90.45 & \cellcolor[HTML]{B7E1CD}29.3 & \cellcolor[HTML]{B7E1CD}27.24 & 64.03 & 55.31 & \cellcolor[HTML]{B7E1CD}26.55 & \cellcolor[HTML]{B7E1CD}12.03 & \cellcolor[HTML]{B7E1CD}42.68 & 0 \\
ReColor & 0.1 & \textbf{92.18} & \cellcolor[HTML]{B7E1CD}\textbf{84.33} & 0.18 & \cellcolor[HTML]{B7E1CD}1.61 & \cellcolor[HTML]{B7E1CD}82.93 & \cellcolor[HTML]{B7E1CD}53.66 & 91.11 & \cellcolor[HTML]{B7E1CD}29.66 & \cellcolor[HTML]{B7E1CD}40.79 & \cellcolor[HTML]{B7E1CD}66.04 & \cellcolor[HTML]{B7E1CD}57.25 & \cellcolor[HTML]{B7E1CD}31.16 & \cellcolor[HTML]{B7E1CD}\textbf{15.24} & \cellcolor[HTML]{B7E1CD}46.16 & 0 \\
ReColor & 0.2 & 92.1 & \cellcolor[HTML]{B7E1CD}83.92 & \textbf{0.25} & \cellcolor[HTML]{B7E1CD}\textbf{2.45} & \cellcolor[HTML]{B7E1CD}\textbf{83.9} & \cellcolor[HTML]{B7E1CD}\textbf{54.52} & \textbf{91.54} & \cellcolor[HTML]{B7E1CD}\textbf{33.73} & \cellcolor[HTML]{B7E1CD}\textbf{43.26} & \cellcolor[HTML]{B7E1CD}\textbf{67.57} & \cellcolor[HTML]{B7E1CD}\textbf{57.83} & \cellcolor[HTML]{B7E1CD}\textbf{37.94} & \cellcolor[HTML]{B7E1CD}12.51 & \cellcolor[HTML]{B7E1CD}\textbf{47.45} & \textbf{0.1} \\ \hline
Gabor & 0 & \textbf{89.12} & \textbf{85.27} & 4.41 & 2.11 & 37.83 & 87.31 & 87.62 & 20.28 & \textbf{57.66} & 52.33 & 38.73 & 38.88 & 9.22 & 43.47 & 0.1 \\
Gabor & 0.05 & 88.76 & 84.48 & \cellcolor[HTML]{B7E1CD}5.61 & \cellcolor[HTML]{B7E1CD}3.54 & \cellcolor[HTML]{F4C7C3}35.49 & \textbf{87.54} & \textbf{88} & \cellcolor[HTML]{F4C7C3}18.73 & \cellcolor[HTML]{F4C7C3}56.61 & \cellcolor[HTML]{B7E1CD}\textbf{56.08} & \cellcolor[HTML]{B7E1CD}43.21 & \cellcolor[HTML]{F4C7C3}37.45 & \cellcolor[HTML]{B7E1CD}12.69 & 44.12 & 0.61 \\
Gabor & 0.1 & \cellcolor[HTML]{F4C7C3}87.92 & \cellcolor[HTML]{F4C7C3}84.1 & \cellcolor[HTML]{B7E1CD}8.36 & \cellcolor[HTML]{B7E1CD}4.79 & \cellcolor[HTML]{B7E1CD}41.81 & \cellcolor[HTML]{F4C7C3}85.66 & \cellcolor[HTML]{F4C7C3}84.92 & \cellcolor[HTML]{B7E1CD}26.29 & \cellcolor[HTML]{F4C7C3}54.29 & 52.97 & \cellcolor[HTML]{B7E1CD}45.17 & \cellcolor[HTML]{B7E1CD}45.2 & \cellcolor[HTML]{B7E1CD}11.08 & \cellcolor[HTML]{B7E1CD}45.39 & 0.36 \\
Gabor & 0.2 & \cellcolor[HTML]{F4C7C3}87.97 & 84.33 & \cellcolor[HTML]{B7E1CD}\textbf{12.46} & \cellcolor[HTML]{B7E1CD}\textbf{8.99} & \cellcolor[HTML]{B7E1CD}\textbf{44.36} & 86.55 & \cellcolor[HTML]{F4C7C3}85.27 & \cellcolor[HTML]{B7E1CD}\textbf{29.12} & 57.4 & \cellcolor[HTML]{B7E1CD}55.67 & \cellcolor[HTML]{B7E1CD}\textbf{51.41} & \cellcolor[HTML]{B7E1CD}\textbf{49.91} & \cellcolor[HTML]{B7E1CD}\textbf{13.25} & \cellcolor[HTML]{B7E1CD}\textbf{48.23} & \cellcolor[HTML]{B7E1CD}\textbf{1.55} \\ \hline
Snow & 0 & \textbf{87.69} & \textbf{71.59} & 0.08 & 1.53 & \textbf{11.97} & 30.68 & 62.9 & 7.31 & 8.59 & \textbf{66.24} & 70.78 & 11.82 & 9.91 & \textbf{29.45} & 0.05 \\
Snow & 0.05 & 86.85 & \cellcolor[HTML]{F4C7C3}70.34 & 0.05 & 1.48 & \cellcolor[HTML]{F4C7C3}10.01 & \cellcolor[HTML]{B7E1CD}\textbf{33.2} & \cellcolor[HTML]{B7E1CD}\textbf{64.71} & 7.52 & \textbf{9.15} & \cellcolor[HTML]{F4C7C3}64.46 & \cellcolor[HTML]{F4C7C3}69.55 & 11.41 & 8.94 & 29.24 & 0.03 \\
Snow & 0.1 & \cellcolor[HTML]{F4C7C3}86.04 & 70.96 & \textbf{0.18} & \textbf{1.86} & \cellcolor[HTML]{F4C7C3}10.57 & \cellcolor[HTML]{F4C7C3}23.77 & \cellcolor[HTML]{F4C7C3}61.43 & \cellcolor[HTML]{B7E1CD}\textbf{8.48} & \cellcolor[HTML]{F4C7C3}7.49 & 65.32 & 70.34 & \cellcolor[HTML]{B7E1CD}\textbf{15.21} & \cellcolor[HTML]{F4C7C3}7.77 & 28.62 & \textbf{0.13} \\
Snow & 0.2 & 87.62 & \cellcolor[HTML]{F4C7C3}68.89 & 0.15 & 1.61 & \textbf{11.97} & 29.91 & \cellcolor[HTML]{B7E1CD}64.38 & 7.06 & 8.03 & 66.09 & \cellcolor[HTML]{B7E1CD}\textbf{71.8} & 12 & \textbf{9.96} & 29.32 & 0.1 \\ \hline
Pixel & 0 & 88.64 & 67.24 & 0 & 0.59 & 34.96 & 30.37 & 87.18 & 78.6 & 0.25 & 61.63 & \textbf{52.08} & 49.38 & 23.97 & 40.52 & 0 \\
Pixel & 0.05 & 89.07 & \cellcolor[HTML]{B7E1CD}68.99 & 0 & \textbf{1.12} & \cellcolor[HTML]{B7E1CD}40.66 & 29.38 & 87.92 & \cellcolor[HTML]{B7E1CD}80.79 & 0.64 & 62.32 & \cellcolor[HTML]{F4C7C3}49.32 & \cellcolor[HTML]{B7E1CD}52.05 & \cellcolor[HTML]{B7E1CD}25.86 & \cellcolor[HTML]{B7E1CD}41.59 & 0 \\
Pixel & 0.1 & 89.2 & \cellcolor[HTML]{B7E1CD}\textbf{70.73} & 0 & 1.04 & \cellcolor[HTML]{B7E1CD}43.49 & \cellcolor[HTML]{B7E1CD}\textbf{33.71} & 88.1 & \cellcolor[HTML]{B7E1CD}82.55 & 0.79 & \cellcolor[HTML]{B7E1CD}\textbf{64.48} & \cellcolor[HTML]{F4C7C3}49.78 & 49.94 & \cellcolor[HTML]{B7E1CD}29.3 & \cellcolor[HTML]{B7E1CD}42.83 & 0 \\
Pixel & 0.2 & \cellcolor[HTML]{B7E1CD}\textbf{90.7} & 67.41 & 0 & 1.02 & \cellcolor[HTML]{B7E1CD}\textbf{46.45} & \cellcolor[HTML]{B7E1CD}33.25 & \cellcolor[HTML]{B7E1CD}\textbf{89.25} & \cellcolor[HTML]{B7E1CD}\textbf{83.95} & \cellcolor[HTML]{B7E1CD}\textbf{1.73} & \cellcolor[HTML]{B7E1CD}63.77 & \cellcolor[HTML]{F4C7C3}50.8 & \cellcolor[HTML]{B7E1CD}\textbf{52.89} & \cellcolor[HTML]{B7E1CD}\textbf{31.24} & \cellcolor[HTML]{B7E1CD}\textbf{43.48} & 0 \\ \hline
JPEG & 0 & \textbf{88.43} & \textbf{85.63} & 15.29 & 5.43 & 41.78 & 77.35 & \textbf{86.85} & 23.21 & \textbf{80.87} & 53.81 & 43.39 & 44.79 & 15.39 & 47.82 & 0.74 \\
JPEG & 0.05 & \cellcolor[HTML]{F4C7C3}87.31 & \cellcolor[HTML]{F4C7C3}84.54 & \cellcolor[HTML]{B7E1CD}23.72 & \cellcolor[HTML]{B7E1CD}8.51 & \cellcolor[HTML]{B7E1CD}\textbf{46.7} & \cellcolor[HTML]{B7E1CD}80.56 & 86.11 & \cellcolor[HTML]{B7E1CD}24.82 & \textbf{80.87} & 54.39 & \cellcolor[HTML]{B7E1CD}45.02 & 45.3 & \cellcolor[HTML]{B7E1CD}17.32 & \cellcolor[HTML]{B7E1CD}49.82 & \cellcolor[HTML]{B7E1CD}2.01 \\
JPEG & 0.1 & \cellcolor[HTML]{F4C7C3}86.75 & \cellcolor[HTML]{F4C7C3}84.05 & \cellcolor[HTML]{B7E1CD}24.33 & \cellcolor[HTML]{B7E1CD}11.39 & \cellcolor[HTML]{B7E1CD}43.13 & \cellcolor[HTML]{B7E1CD}79.8 & \cellcolor[HTML]{F4C7C3}84.89 & 24.15 & 80.23 & \cellcolor[HTML]{F4C7C3}52.79 & \cellcolor[HTML]{B7E1CD}47.34 & \cellcolor[HTML]{B7E1CD}47.82 & 16.36 & \cellcolor[HTML]{B7E1CD}49.69 & \cellcolor[HTML]{B7E1CD}2.7 \\
JPEG & 0.2 & \cellcolor[HTML]{F4C7C3}86.65 & \cellcolor[HTML]{F4C7C3}83.64 & \cellcolor[HTML]{B7E1CD}\textbf{27.57} & \cellcolor[HTML]{B7E1CD}\textbf{15.44} & \cellcolor[HTML]{B7E1CD}45.61 & \cellcolor[HTML]{B7E1CD}\textbf{81.17} & \cellcolor[HTML]{F4C7C3}84.74 & \cellcolor[HTML]{B7E1CD}\textbf{26.78} & 80.64 & \cellcolor[HTML]{B7E1CD}\textbf{57.68} & \cellcolor[HTML]{B7E1CD}\textbf{49.99} & \cellcolor[HTML]{B7E1CD}\textbf{50.93} & \cellcolor[HTML]{B7E1CD}\textbf{20.08} & \cellcolor[HTML]{B7E1CD}\textbf{52.02} & \cellcolor[HTML]{B7E1CD}\textbf{4.13} \\ \hline
Elastic & 0 & \textbf{89.66} & \textbf{77.48} & 0 & 0.82 & 12.25 & 21.81 & 88.05 & \textbf{16.84} & \textbf{5.71} & 78.6 & 62.01 & \textbf{16.79} & 11.69 & \textbf{32.67} & 0 \\
Elastic & 0.05 & 88.99 & \cellcolor[HTML]{F4C7C3}75.67 & 0 & 0.74 & \cellcolor[HTML]{F4C7C3}11.11 & \cellcolor[HTML]{F4C7C3}19.06 & 88.23 & \cellcolor[HTML]{F4C7C3}15.34 & 5.17 & 79.21 & \cellcolor[HTML]{F4C7C3}59.97 & \cellcolor[HTML]{F4C7C3}15.69 & \cellcolor[HTML]{F4C7C3}7.85 & \cellcolor[HTML]{F4C7C3}31.5 & 0 \\
Elastic & 0.1 & 89.43 & \textbf{77.48} & 0 & \textbf{1.12} & \cellcolor[HTML]{F4C7C3}10.14 & \cellcolor[HTML]{B7E1CD}25.07 & \textbf{88.43} & \cellcolor[HTML]{F4C7C3}12.43 & 5.53 & \textbf{79.44} & \textbf{62.09} & \cellcolor[HTML]{F4C7C3}13.38 & \cellcolor[HTML]{F4C7C3}8.82 & 31.99 & 0 \\
Elastic & 0.2 & \cellcolor[HTML]{F4C7C3}88.36 & \cellcolor[HTML]{F4C7C3}74.14 & 0 & 0.66 & \textbf{12.33} & \cellcolor[HTML]{B7E1CD}\textbf{26.93} & 87.31 & \cellcolor[HTML]{F4C7C3}15.13 & \cellcolor[HTML]{F4C7C3}3.13 & \cellcolor[HTML]{F4C7C3}77.25 & \cellcolor[HTML]{F4C7C3}58.98 & \cellcolor[HTML]{F4C7C3}14.75 & \cellcolor[HTML]{B7E1CD}\textbf{13.99} & 32.05 & 0 \\ \hline
Wood & 0 & 85.91 & 50.09 & 0 & 0.82 & 11.11 & \textbf{35.87} & 83.31 & \textbf{11.13} & 1.2 & 60.94 & 78.83 & 14.96 & \textbf{11.13} & 29.95 & 0 \\
Wood & 0.05 & \cellcolor[HTML]{B7E1CD}\textbf{88.61} & \cellcolor[HTML]{B7E1CD}72.03 & \textbf{0.03} & 0.36 & 10.27 & \cellcolor[HTML]{F4C7C3}29.43 & \cellcolor[HTML]{B7E1CD}86.11 & \cellcolor[HTML]{F4C7C3}8.43 & \cellcolor[HTML]{B7E1CD}4.1 & \cellcolor[HTML]{B7E1CD}\textbf{67.69} & \cellcolor[HTML]{B7E1CD}\textbf{85.27} & \cellcolor[HTML]{B7E1CD}17.99 & \cellcolor[HTML]{F4C7C3}9.76 & \cellcolor[HTML]{B7E1CD}32.62 & 0 \\
Wood & 0.1 & \cellcolor[HTML]{B7E1CD}87.54 & \cellcolor[HTML]{B7E1CD}70.93 & \textbf{0.03} & 0.48 & \cellcolor[HTML]{B7E1CD}\textbf{12.25} & \cellcolor[HTML]{F4C7C3}23.59 & \cellcolor[HTML]{B7E1CD}\textbf{86.62} & 10.6 & \cellcolor[HTML]{B7E1CD}5.2 & \cellcolor[HTML]{B7E1CD}65.1 & \cellcolor[HTML]{B7E1CD}83.24 & \cellcolor[HTML]{B7E1CD}\textbf{20.87} & \cellcolor[HTML]{F4C7C3}8 & \cellcolor[HTML]{B7E1CD}32.24 & \textbf{0.03} \\
Wood & 0.2 & \cellcolor[HTML]{B7E1CD}87.77 & \cellcolor[HTML]{B7E1CD}\textbf{73.89} & 0 & \cellcolor[HTML]{B7E1CD}\textbf{2.22} & 10.14 & \cellcolor[HTML]{F4C7C3}34.78 & \cellcolor[HTML]{B7E1CD}86.04 & \cellcolor[HTML]{F4C7C3}9.15 & \cellcolor[HTML]{B7E1CD}\textbf{9.27} & \cellcolor[HTML]{B7E1CD}66.14 & \cellcolor[HTML]{B7E1CD}84.08 & \cellcolor[HTML]{B7E1CD}20.28 & 11.03 & \cellcolor[HTML]{B7E1CD}\textbf{33.92} & 0 \\ \hline
Glitch & 0 & \textbf{88.51} & 36.41 & 0 & 0 & 6.7 & 18.47 & \textbf{86.96} & 17.1 & 0 & 60 & 50.93 & 84.97 & 6.37 & 30.66 & 0 \\
Glitch & 0.05 & 87.97 & \cellcolor[HTML]{F4C7C3}6.62 & 0 & 0 & \cellcolor[HTML]{F4C7C3}4.2 & \cellcolor[HTML]{B7E1CD}\textbf{29.81} & 86.5 & \cellcolor[HTML]{F4C7C3}12.1 & 0 & \cellcolor[HTML]{F4C7C3}53.71 & \cellcolor[HTML]{F4C7C3}49.12 & \cellcolor[HTML]{F4C7C3}81.91 & \cellcolor[HTML]{F4C7C3}2.96 & \cellcolor[HTML]{F4C7C3}27.24 & 0 \\
Glitch & 0.1 & 87.52 & \cellcolor[HTML]{F4C7C3}31.75 & 0 & 0.05 & \cellcolor[HTML]{B7E1CD}16.1 & \cellcolor[HTML]{B7E1CD}21.25 & 86.06 & \cellcolor[HTML]{B7E1CD}19.21 & 0.05 & \cellcolor[HTML]{B7E1CD}\textbf{62.62} & \cellcolor[HTML]{B7E1CD}\textbf{55.01} & \cellcolor[HTML]{B7E1CD}\textbf{86.24} & 6.04 & \cellcolor[HTML]{B7E1CD}32.03 & 0 \\
Glitch & 0.2 & \cellcolor[HTML]{F4C7C3}86.47 & \cellcolor[HTML]{B7E1CD}\textbf{53.61} & 0 & \textbf{0.64} & \cellcolor[HTML]{B7E1CD}\textbf{23.21} & \cellcolor[HTML]{B7E1CD}24.51 & \cellcolor[HTML]{F4C7C3}85.22 & \cellcolor[HTML]{B7E1CD}\textbf{25.55} & \cellcolor[HTML]{B7E1CD}\textbf{1.35} & \cellcolor[HTML]{B7E1CD}61.15 & \cellcolor[HTML]{B7E1CD}52.08 & 84.79 & \cellcolor[HTML]{B7E1CD}\textbf{9.86} & \cellcolor[HTML]{B7E1CD}\textbf{35.16} & 0 \\ \hline
\begin{tabular}[c]{@{}l@{}}Kaleid-\\ oscope\end{tabular} & 0 & \textbf{88.1} & 73.5 & 0 & 0.31 & 7.03 & 28.66 & 85.91 & 18.83 & 2.22 & \textbf{62.98} & 29.78 & 21.07 & 84.89 & 34.6 & 0 \\
\begin{tabular}[c]{@{}l@{}}Kaleid-\\ oscope\end{tabular} & 0.05 & \cellcolor[HTML]{F4C7C3}86.9 & \cellcolor[HTML]{F4C7C3}71.03 & \textbf{0.03} & 0.64 & 6.24 & \textbf{29.15} & 85.32 & \cellcolor[HTML]{F4C7C3}15.69 & 3.11 & \cellcolor[HTML]{F4C7C3}59.85 & \cellcolor[HTML]{B7E1CD}\textbf{44.71} & \cellcolor[HTML]{B7E1CD}23.49 & 84.23 & \textbf{35.29} & 0 \\
\begin{tabular}[c]{@{}l@{}}Kaleid-\\ oscope\end{tabular} & 0.1 & 87.64 & 74.14 & 0 & 0.56 & \cellcolor[HTML]{B7E1CD}9.15 & \cellcolor[HTML]{F4C7C3}25.71 & 85.4 & \cellcolor[HTML]{F4C7C3}17.58 & \cellcolor[HTML]{B7E1CD}\textbf{3.87} & \cellcolor[HTML]{F4C7C3}60.87 & \cellcolor[HTML]{F4C7C3}27.52 & \cellcolor[HTML]{B7E1CD}\textbf{25.81} & 85.15 & 34.65 & 0 \\
\begin{tabular}[c]{@{}l@{}}Kaleid-\\ oscope\end{tabular} & 0.2 & 87.46 & \cellcolor[HTML]{B7E1CD}\textbf{74.8} & 0 & \textbf{1.25} & \cellcolor[HTML]{B7E1CD}\textbf{9.78} & \cellcolor[HTML]{F4C7C3}24.36 & \textbf{86.06} & \cellcolor[HTML]{B7E1CD}\textbf{20.87} & \cellcolor[HTML]{B7E1CD}3.49 & 62.39 & \cellcolor[HTML]{F4C7C3}28.15 & \cellcolor[HTML]{B7E1CD}23.64 & \textbf{85.17} & 35 & 0 \\ \hline
\end{tabular}}}
\caption{\textbf{Intial Training Ablations- variation regularization on ImageNette. }Accuracy of initially trained models on ImageNette trained using different attacks as indicated in ``Train Attack" column measured across different attacks.  Variation regularization computed using single step optimization is also considered during initial training, with regularization strength $\lambda$.  Results where regularization improves over no regularization ($\lambda = 0$) by at least 1\% accuracy are highlighted in green, while results where regularization incurs at least a 1\% drop in accuracy are highlighted in red.  Best performing with respect to regularization strength are bolded.}
\label{app:IT_ablation_VR_IM}
\end{table*}

\begin{table*}[]{\renewcommand{\arraystretch}{1.2}
\scalebox{0.73}{
\begin{tabular}{|l|c|c|cccccccccccc|cc|}
\hline
\begin{tabular}[c]{@{}l@{}}Train\\ Attack\end{tabular}   & $\lambda$ & Clean  & $\ell_2$     & $\ell_\infty$   & StAdv  & ReColor& Gabor  & Snow   & Pixel  & JPEG   & Elastic& Wood   & Glitch & \begin{tabular}[c]{@{}c@{}}Kaleid\\ -oscope\end{tabular} & Avg    & Union  \\ \hline
$\ell_2$ & 0  & \textbf{90.04}   & \textbf{83.95}   & 7.57   & 5.27   & 33.91  & 65.17  & \textbf{89.3}    & 28.99  & 67.52  & 62.85  & 49.22  & \textbf{45.78}   & 12.76     & 46.03  & 0.51   \\
$\ell_2$ & 0.5& 89.68  & 83.49  & \cellcolor[HTML]{B7E1CD}8.64    & 6.17   & 34.17  & 64.59  & 88.92  & 29.1   & 68.41  & \cellcolor[HTML]{B7E1CD}\textbf{64.59} & \cellcolor[HTML]{B7E1CD}\textbf{51.49} & \cellcolor[HTML]{F4CCCC}42.14   & 13.58     & 46.27  & 0.36   \\
$\ell_2$ & 1  & 89.45  & \cellcolor[HTML]{F4CCCC}82.9    & \cellcolor[HTML]{B7E1CD}10.37   & \cellcolor[HTML]{B7E1CD}\textbf{7.49}  & 33.1   & \cellcolor[HTML]{B7E1CD}\textbf{67.95} & \cellcolor[HTML]{F4CCCC}87.59   & 28.05  & \cellcolor[HTML]{B7E1CD}\textbf{69.38} & \cellcolor[HTML]{F4CCCC}61.27   & \cellcolor[HTML]{B7E1CD}50.83   & \cellcolor[HTML]{F4CCCC}42.5    & \cellcolor[HTML]{B7E1CD}\textbf{14.9}     & \textbf{46.36}   & \textbf{0.76}    \\
$\ell_2$ & 2  & 89.76  & 83.36  & \cellcolor[HTML]{B7E1CD}\textbf{10.6}  & \cellcolor[HTML]{B7E1CD}6.73    & \cellcolor[HTML]{B7E1CD}\textbf{35.01} & \cellcolor[HTML]{F4CCCC}63.03   & 88.46  & \textbf{29.17}   & \cellcolor[HTML]{B7E1CD}68.79   & 62.83  & \cellcolor[HTML]{B7E1CD}51.11   & \cellcolor[HTML]{F4CCCC}42.22   & 12.87     & 46.18  & 0.61   \\ \hline
$\ell_\infty$      & 0  & 84.51  & 81.71  & 58.39  & 43.49  & 67.82  & \textbf{72.61}   & 83.31  & 41.83  & 65.35  & 63.9   & \textbf{67.18}   & 63.64  & 30.75     & 61.67  & 13.2   \\
$\ell_\infty$      & 0.5& 84.51  & 81.63  & 58.93  & \cellcolor[HTML]{B7E1CD}44.66   & 67.62  & 72.05  & 83.69  & \cellcolor[HTML]{B7E1CD}43.31   & \cellcolor[HTML]{F4C7C3}62.62   & \cellcolor[HTML]{F4C7C3}62.22   & \cellcolor[HTML]{F4C7C3}64.54   & \cellcolor[HTML]{B7E1CD}66.42   & \cellcolor[HTML]{B7E1CD}32.71      & 61.7   & \cellcolor[HTML]{B7E1CD}14.42   \\
$\ell_\infty$      & 1  & 83.97  & 81.22  & \textbf{59.01}   & \cellcolor[HTML]{B7E1CD}44.87   & \cellcolor[HTML]{F4C7C3}66.6    & 71.9   & \cellcolor[HTML]{F4C7C3}81.73   & \cellcolor[HTML]{B7E1CD}43.52   & \cellcolor[HTML]{B7E1CD}66.47   & 63.26  & 66.39  & \cellcolor[HTML]{B7E1CD}\textbf{66.73} & \cellcolor[HTML]{B7E1CD}\textbf{33.99}    & 62.14  & \cellcolor[HTML]{B7E1CD}\textbf{15.97} \\
$\ell_\infty$      & 2  & \textbf{85.2}    & \cellcolor[HTML]{B7E1CD}\textbf{83.36} & 58.42  & \cellcolor[HTML]{B7E1CD}\textbf{47.92} & \cellcolor[HTML]{B7E1CD}\textbf{69.27} & \cellcolor[HTML]{F4C7C3}69.99   & \textbf{84.18}   & \cellcolor[HTML]{B7E1CD}\textbf{43.97} & \cellcolor[HTML]{B7E1CD}\textbf{73.53} & \textbf{64.33}   & 66.7   & 62.75  & \cellcolor[HTML]{F4C7C3}26.75      & \textbf{62.6}    & 13.81  \\ \hline
StAdv     & 0  & 83.31  & 77.58  & 1.45   & \textbf{69.81}   & 13.43  & 36.66  & 81.5   & 20.56  & 49.89  & 70.32  & 60.76  & 36.15  & 24.84     & 45.25  & 1.04   \\
StAdv     & 0.5& 84.1   & \cellcolor[HTML]{B7E1CD}79.31   & 2.09   & \textbf{69.81}   & \cellcolor[HTML]{B7E1CD}16.28   & \cellcolor[HTML]{B7E1CD}39.52   & 81.55  & \cellcolor[HTML]{B7E1CD}\textbf{23.95} & \cellcolor[HTML]{B7E1CD}55.97   & \cellcolor[HTML]{B7E1CD}71.49   & \cellcolor[HTML]{B7E1CD}62.6    & \cellcolor[HTML]{B7E1CD}\textbf{42.42} & \cellcolor[HTML]{B7E1CD}\textbf{27.95}    & \cellcolor[HTML]{B7E1CD}\textbf{47.75} & 1.5    \\
StAdv     & 1  & \cellcolor[HTML]{B7E1CD}\textbf{84.89} & \cellcolor[HTML]{B7E1CD}\textbf{79.67} & \cellcolor[HTML]{B7E1CD}2.52    & \cellcolor[HTML]{F4C7C3}68.64   & 14.27  & \cellcolor[HTML]{B7E1CD}\textbf{41.73} & \textbf{82.27}   & 21.17  & \cellcolor[HTML]{B7E1CD}56.59   & \cellcolor[HTML]{B7E1CD}71.49   & \cellcolor[HTML]{B7E1CD}\textbf{63.64} & \cellcolor[HTML]{B7E1CD}37.81   & 24.03     & \cellcolor[HTML]{B7E1CD}46.99   & \textbf{1.81}    \\
StAdv     & 2  & 83.62  & \cellcolor[HTML]{B7E1CD}79.34   & \cellcolor[HTML]{B7E1CD}\textbf{2.96}  & 69.66  & \cellcolor[HTML]{B7E1CD}\textbf{16.59} & \cellcolor[HTML]{B7E1CD}40.08   & 81.66  & \cellcolor[HTML]{B7E1CD}22.06   & \cellcolor[HTML]{B7E1CD}\textbf{59.87} & \cellcolor[HTML]{B7E1CD}\textbf{71.54} & \cellcolor[HTML]{B7E1CD}62.01   & \cellcolor[HTML]{B7E1CD}39.18   & \cellcolor[HTML]{F4C7C3}21.66      & \cellcolor[HTML]{B7E1CD}47.22   & 1.68   \\ \hline
ReColor   & 0  & 91.34  & 81.53  & 0.03   & 0.41   & 79.08  & 42.55  & 90.6   & 22.39  & 25.2   & 64.31  & 54.8   & 18.65  & 8.94      & 40.71  & 0      \\
ReColor   & 0.5& 91.82  & \cellcolor[HTML]{B7E1CD}82.85   & 0.08   & 0.38   & 79.75  & \cellcolor[HTML]{B7E1CD}50.83   & 91.08  & 22.83  & \cellcolor[HTML]{B7E1CD}28.28   & \cellcolor[HTML]{F4C7C3}62.57   & \cellcolor[HTML]{F4C7C3}53.17   & \cellcolor[HTML]{B7E1CD}21.48   & 9.91      & \cellcolor[HTML]{B7E1CD}41.93   & 0      \\
ReColor   & 1  & 91.77  & \cellcolor[HTML]{B7E1CD}85.2    & 0.15   & 0.41   & \cellcolor[HTML]{B7E1CD}\textbf{81.2}  & \cellcolor[HTML]{B7E1CD}\textbf{54.22} & 91.39  & \cellcolor[HTML]{B7E1CD}\textbf{27.08} & \cellcolor[HTML]{B7E1CD}38.32   & \textbf{64.82}   & \textbf{55.08}   & \cellcolor[HTML]{B7E1CD}\textbf{23.06} & \cellcolor[HTML]{B7E1CD}\textbf{10.65}    & \cellcolor[HTML]{B7E1CD}\textbf{44.3}  & 0      \\
ReColor   & 2  & \textbf{92.31}   & \cellcolor[HTML]{B7E1CD}\textbf{85.53} & \textbf{0.28}    & \textbf{0.54}    & \cellcolor[HTML]{B7E1CD}80.99   & \cellcolor[HTML]{B7E1CD}52.69   & \cellcolor[HTML]{B7E1CD}\textbf{91.67} & \cellcolor[HTML]{B7E1CD}24.46   & \cellcolor[HTML]{B7E1CD}\textbf{44.38} & \textbf{64.82}   & \cellcolor[HTML]{F4C7C3}52.61   & \cellcolor[HTML]{B7E1CD}21.94   & 9.81      & \cellcolor[HTML]{B7E1CD}44.14   & \textbf{0.03}    \\ \hline
Gabor     & 0  & \textbf{89.12}   & \textbf{85.27}   & 4.41   & 2.11   & \textbf{37.83}   & \textbf{87.31}   & \textbf{87.62}   & \textbf{20.28}   & 57.66  & 52.33  & 38.73  & 38.88  & 9.22      & \textbf{43.47}   & 0.1    \\
Gabor     & 0.5& 88.33  & \cellcolor[HTML]{F4C7C3}84.1    & 4.41   & 2.9    & \cellcolor[HTML]{F4C7C3}33.4    & 86.75  & \cellcolor[HTML]{F4C7C3}86.06   & \cellcolor[HTML]{F4C7C3}18.62   & \cellcolor[HTML]{F4C7C3}54.29   & \cellcolor[HTML]{B7E1CD}\textbf{53.68} & \cellcolor[HTML]{F4C7C3}35.49   & \cellcolor[HTML]{F4C7C3}36.33   & \cellcolor[HTML]{B7E1CD}10.78      & \cellcolor[HTML]{F4C7C3}42.24   & 0.13   \\
Gabor     & 1  & 88.23  & 84.41  & \cellcolor[HTML]{B7E1CD}6.57    & \cellcolor[HTML]{B7E1CD}4.03    & \cellcolor[HTML]{F4C7C3}31.64   & \cellcolor[HTML]{F4C7C3}85.43   & 87.24  & \cellcolor[HTML]{F4C7C3}18.09   & 57.12  & \cellcolor[HTML]{B7E1CD}53.4    & \cellcolor[HTML]{B7E1CD}\textbf{40.2}  & 38.09  & \cellcolor[HTML]{B7E1CD}11.11      & 43.11  & 0.71   \\
Gabor     & 2  & \cellcolor[HTML]{F4C7C3}86.19   & \cellcolor[HTML]{F4C7C3}82.39   & \cellcolor[HTML]{B7E1CD}\textbf{9.12}  & \cellcolor[HTML]{B7E1CD}\textbf{5.91}  & \cellcolor[HTML]{F4C7C3}33.1    & \cellcolor[HTML]{F4C7C3}84.03   & \cellcolor[HTML]{F4C7C3}84.38   & \cellcolor[HTML]{F4C7C3}17.68   & \cellcolor[HTML]{B7E1CD}\textbf{63.16} & \cellcolor[HTML]{F4C7C3}44.89   & \cellcolor[HTML]{F4C7C3}33.1    & \cellcolor[HTML]{B7E1CD}\textbf{43.67} & \cellcolor[HTML]{B7E1CD}\textbf{14.32}    & 42.98  & \textbf{0.84}    \\ \hline
Snow      & 0  & 87.69  & 71.59  & 0.08   & 1.53   & 11.97  & 30.68  & \textbf{62.9}    & 7.31   & 8.59   & 66.24  & \textbf{70.78}   & 11.82  & \textbf{9.91}     & 29.45  & 0.05   \\
Snow      & 0.5& \textbf{88.46}   & \cellcolor[HTML]{B7E1CD}76.46   & 0.03   & 1.12   & \cellcolor[HTML]{B7E1CD}13.22   & \cellcolor[HTML]{B7E1CD}36.94   & \cellcolor[HTML]{F4C7C3}58.78   & \cellcolor[HTML]{B7E1CD}8.84    & \cellcolor[HTML]{B7E1CD}10.96   & 66.42  & 70.62  & \cellcolor[HTML]{B7E1CD}18.27   & \cellcolor[HTML]{F4C7C3}6.42& \cellcolor[HTML]{B7E1CD}30.67   & 0.03   \\
Snow      & 1  & \textbf{87.82}   & \cellcolor[HTML]{B7E1CD}77.04   & 0.1    & 1.48   & \cellcolor[HTML]{B7E1CD}14.14   & \cellcolor[HTML]{B7E1CD}43.21   & \cellcolor[HTML]{F4C7C3}59.06   & \cellcolor[HTML]{B7E1CD}8.64    & \cellcolor[HTML]{B7E1CD}13.48   & \cellcolor[HTML]{F4C7C3}65.22   & \cellcolor[HTML]{F4C7C3}68.89   & \cellcolor[HTML]{B7E1CD}17.76   & \cellcolor[HTML]{F4C7C3}7.67& \cellcolor[HTML]{B7E1CD}31.39   & 0.08   \\
Snow      & 2  & 88.36  & \cellcolor[HTML]{B7E1CD}\textbf{79.29} & \textbf{0.15}    & \textbf{2.34}    & \cellcolor[HTML]{B7E1CD}\textbf{16.2}  & \cellcolor[HTML]{B7E1CD}\textbf{48.99} & \cellcolor[HTML]{F4C7C3}59.69   & \cellcolor[HTML]{B7E1CD}\textbf{10.14} & \cellcolor[HTML]{B7E1CD}\textbf{21.48} & \cellcolor[HTML]{B7E1CD}\textbf{68.13} & 70.24  & \cellcolor[HTML]{B7E1CD}\textbf{19.03} & 9.58      & \cellcolor[HTML]{B7E1CD}\textbf{33.77} & \textbf{0.1}     \\ \hline
Pixel     & 0  & 88.64  & 67.24  & 0      & 0.59   & 34.96  & 30.37  & 87.18  & 78.6   & 0.25   & 61.63  & \textbf{52.08}   & 49.38  & 23.97     & 40.52  & 0      \\
Pixel     & 0.5& \cellcolor[HTML]{B7E1CD}\textbf{89.81} & \cellcolor[HTML]{B7E1CD}75.54   & 0      & 0.43   & \cellcolor[HTML]{B7E1CD}38.8    & 31.06  & 87.67  & \cellcolor[HTML]{B7E1CD}\textbf{80.48} & 1.1    & \cellcolor[HTML]{B7E1CD}\textbf{63.67} & \cellcolor[HTML]{F4C7C3}49.58   & \textbf{50.01}   & 23.26     & \cellcolor[HTML]{B7E1CD}41.8    & 0      \\
Pixel     & 1  & 89.43  & \cellcolor[HTML]{B7E1CD}\textbf{75.85} & 0      & 0.84   & \cellcolor[HTML]{B7E1CD}\textbf{39.95} & \cellcolor[HTML]{B7E1CD}\textbf{34.52} & \textbf{88}      & \cellcolor[HTML]{B7E1CD}80.25   & \cellcolor[HTML]{B7E1CD}1.27    & \cellcolor[HTML]{B7E1CD}63.18   & \cellcolor[HTML]{F4C7C3}49.27   & \cellcolor[HTML]{F4C7C3}46.96   & \cellcolor[HTML]{F4C7C3}21.78      & \cellcolor[HTML]{B7E1CD}\textbf{41.82} & 0      \\
Pixel     & 2  & \cellcolor[HTML]{F4C7C3}87.62   & \cellcolor[HTML]{B7E1CD}75.62   & 0      & \textbf{1.27}    & \cellcolor[HTML]{B7E1CD}39.03   & 30.22  & \cellcolor[HTML]{F4C7C3}86.04   & 77.86  & \cellcolor[HTML]{B7E1CD}\textbf{2.88}  & 61.15  & \textbf{52.08}   & 49.3   & \textbf{24.1}     & \cellcolor[HTML]{B7E1CD}41.63   & 0      \\ \hline
JPEG      & 0  & 88.43  & 85.63  & 15.29  & 5.43   & 41.78  & 77.35  & 86.85  & 23.21  & 80.87  & 53.81  & \textbf{43.39}   & 44.79  & 15.39     & 47.82  & 0.74   \\
JPEG      & 0.5& \textbf{89.02}   & \textbf{86.17}   & \cellcolor[HTML]{B7E1CD}\textbf{17.89} & \cellcolor[HTML]{B7E1CD}6.52    & \cellcolor[HTML]{B7E1CD}\textbf{43.29} & \cellcolor[HTML]{B7E1CD}\textbf{79.01} & \textbf{87.52}   & \cellcolor[HTML]{B7E1CD}\textbf{24.56} & \textbf{81.15}   & \cellcolor[HTML]{B7E1CD}55.62   & 42.78  & \textbf{45.55}   & \textbf{15.92}    & \cellcolor[HTML]{B7E1CD}\textbf{48.83} & 1.3    \\
JPEG      & 1  & 87.9   & 84.89  & \cellcolor[HTML]{B7E1CD}16.59   & \cellcolor[HTML]{B7E1CD}7.34    & \cellcolor[HTML]{F4C7C3}40.13   & 78.06  & 86.65  & \cellcolor[HTML]{F4C7C3}21.94   & 80.38  & \cellcolor[HTML]{B7E1CD}55.57   & 42.93  & 44.48  & 15.77     & 47.89  & \textbf{1.66}    \\
JPEG      & 2  & 88.08  & 85.1   & \cellcolor[HTML]{B7E1CD}17.68   & \cellcolor[HTML]{B7E1CD}\textbf{7.46}  & 41.58  & \cellcolor[HTML]{F4C7C3}73.02   & 86.29  & 23.49  & \cellcolor[HTML]{F4C7C3}79.85   & \cellcolor[HTML]{B7E1CD}\textbf{56.51} & 42.8   & 45.25  & \cellcolor[HTML]{F4C7C3}13.3& 47.69  & 0.97   \\ \hline
Elastic   & 0  & 89.66  & 77.48  & 0      & 0.82   & 12.25  & 21.81  & 88.05  & 16.84  & 5.71   & 78.6   & 62.01  & 16.79  & 11.69     & 32.67  & 0      \\
Elastic   & 0.5& \cellcolor[HTML]{B7E1CD}90.68   & \cellcolor[HTML]{B7E1CD}81.61   & \textbf{0.03}    & 1.66   & \cellcolor[HTML]{B7E1CD}14.11   & \cellcolor[HTML]{B7E1CD}\textbf{35.57} & \cellcolor[HTML]{B7E1CD}89.63   & \cellcolor[HTML]{B7E1CD}\textbf{18.27} & \cellcolor[HTML]{B7E1CD}11.87   & \cellcolor[HTML]{B7E1CD}80.64   & \cellcolor[HTML]{F4C7C3}60.61   & \cellcolor[HTML]{B7E1CD}18.45   & \cellcolor[HTML]{B7E1CD}\textbf{13.76}    & \cellcolor[HTML]{B7E1CD}35.52   & \textbf{0.03}    \\
Elastic   & 1  & \cellcolor[HTML]{B7E1CD}\textbf{91.13} & \cellcolor[HTML]{B7E1CD}82.93   & 0      & 1.78   & \cellcolor[HTML]{B7E1CD}14.37   & \cellcolor[HTML]{B7E1CD}31.69   & \cellcolor[HTML]{B7E1CD}\textbf{89.99} & 17.07  & \cellcolor[HTML]{B7E1CD}17.4    & \cellcolor[HTML]{B7E1CD}81.48   & \cellcolor[HTML]{B7E1CD}\textbf{65.1}  & \cellcolor[HTML]{B7E1CD}\textbf{21.53} & 10.96     & \cellcolor[HTML]{B7E1CD}36.19   & 0      \\
Elastic   & 2  & 90.06  & \cellcolor[HTML]{B7E1CD}\textbf{82.96} & \textbf{0.03}    & \cellcolor[HTML]{B7E1CD}\textbf{1.86}  & \cellcolor[HTML]{B7E1CD}\textbf{16.54} & \cellcolor[HTML]{B7E1CD}31.75   & \cellcolor[HTML]{B7E1CD}89.1    & 17.22  & \cellcolor[HTML]{B7E1CD}\textbf{20.38} & \cellcolor[HTML]{B7E1CD}\textbf{81.53} & \cellcolor[HTML]{B7E1CD}64.61   & \cellcolor[HTML]{B7E1CD}18.88   & \cellcolor[HTML]{F4C7C3}10.42      & \cellcolor[HTML]{B7E1CD}\textbf{36.27} & 0      \\ \hline
Wood      & 0  & 85.91  & 50.09  & 0      & 0.82   & 11.11  & 35.87  & 83.31  & 11.13  & 1.2    & 60.94  & 78.83  & 14.96  & 11.13     & 29.95  & 0      \\
Wood      & 0.5& \cellcolor[HTML]{B7E1CD}88.03   & \cellcolor[HTML]{B7E1CD}74.22   & 0      & 0.97   & 11.77  & \cellcolor[HTML]{B7E1CD}41.02   & \cellcolor[HTML]{B7E1CD}85.68   & \cellcolor[HTML]{F4C7C3}9.81    & \cellcolor[HTML]{B7E1CD}6.57    & \cellcolor[HTML]{B7E1CD}64.54   & \cellcolor[HTML]{B7E1CD}81.76   & \cellcolor[HTML]{B7E1CD}\textbf{18.24} & 11.01     & \cellcolor[HTML]{B7E1CD}33.8    & 0      \\
Wood      & 1  & \cellcolor[HTML]{B7E1CD}88.82   & \cellcolor[HTML]{B7E1CD}78.68   & 0.03   & 0.69   & \cellcolor[HTML]{B7E1CD}14.06   & \cellcolor[HTML]{B7E1CD}40.46   & \cellcolor[HTML]{B7E1CD}86.65   & \cellcolor[HTML]{F4C7C3}9.83    & \cellcolor[HTML]{B7E1CD}11.49   & \cellcolor[HTML]{B7E1CD}66.7    & \cellcolor[HTML]{B7E1CD}\textbf{84.94} & \cellcolor[HTML]{B7E1CD}17.02   & \cellcolor[HTML]{F4C7C3}9.76& \cellcolor[HTML]{B7E1CD}35.03   & 0      \\
Wood      & 2  & \cellcolor[HTML]{B7E1CD}\textbf{89.38} & \cellcolor[HTML]{B7E1CD}\textbf{81.71} & \textbf{0.31}    & \cellcolor[HTML]{B7E1CD}\textbf{2.62}  & \cellcolor[HTML]{B7E1CD}\textbf{15.26} & \cellcolor[HTML]{B7E1CD}\textbf{43.64} & \cellcolor[HTML]{B7E1CD}\textbf{87.64} & \textbf{11.16}   & \cellcolor[HTML]{B7E1CD}\textbf{24.64} & \cellcolor[HTML]{B7E1CD}\textbf{68.99} & \cellcolor[HTML]{B7E1CD}84.84   & \cellcolor[HTML]{B7E1CD}17.81   & \cellcolor[HTML]{B7E1CD}\textbf{12.56}    & \cellcolor[HTML]{B7E1CD}\textbf{37.6}  & \textbf{0.18}    \\ \hline
Glitch    & 0  & 88.51  & 36.41  & 0      & 0      & 6.7    & 18.47  & 86.96  & 17.1   & 0      & 60     & 50.93  & 84.97  & 6.37      & 30.66  & 0      \\
Glitch    & 0.5& \cellcolor[HTML]{B7E1CD}\textbf{90.01} & \cellcolor[HTML]{B7E1CD}74.73   & 0      & 0.54   & \cellcolor[HTML]{B7E1CD}14.96   & \cellcolor[HTML]{B7E1CD}24.38   & \cellcolor[HTML]{B7E1CD}\textbf{87.97} & \cellcolor[HTML]{B7E1CD}24.69   & \cellcolor[HTML]{B7E1CD}2.7     & \cellcolor[HTML]{B7E1CD}64.82   & \cellcolor[HTML]{B7E1CD}52.92   & \cellcolor[HTML]{B7E1CD}\textbf{86.98} & \cellcolor[HTML]{B7E1CD}8& \cellcolor[HTML]{B7E1CD}36.89   & 0      \\
Glitch    & 1  & 88.05  & \cellcolor[HTML]{B7E1CD}76.25   & \textbf{0.03}    & \cellcolor[HTML]{B7E1CD}\textbf{1.38}  & \cellcolor[HTML]{B7E1CD}14.5    & \cellcolor[HTML]{B7E1CD}\textbf{28.25} & \cellcolor[HTML]{F4C7C3}85.73   & \cellcolor[HTML]{B7E1CD}25.3    & \cellcolor[HTML]{B7E1CD}3.41    & \cellcolor[HTML]{B7E1CD}62.17   & \cellcolor[HTML]{B7E1CD}\textbf{53.1}  & 84.69  & \cellcolor[HTML]{B7E1CD}8.79& \cellcolor[HTML]{B7E1CD}36.97   & 0      \\
Glitch    & 2  & \cellcolor[HTML]{B7E1CD}89.55   & \cellcolor[HTML]{B7E1CD}\textbf{80.97} & \textbf{0.03}    & \cellcolor[HTML]{B7E1CD}1.25    & \cellcolor[HTML]{B7E1CD}\textbf{15.16} & \cellcolor[HTML]{B7E1CD}25.89   & 87.9   & \cellcolor[HTML]{B7E1CD}\textbf{27.54} & \cellcolor[HTML]{B7E1CD}\textbf{7.52}  & \cellcolor[HTML]{B7E1CD}\textbf{65.07} & \cellcolor[HTML]{B7E1CD}52.05   & \cellcolor[HTML]{B7E1CD}85.99   & \cellcolor[HTML]{B7E1CD}\textbf{10.47}    & \cellcolor[HTML]{B7E1CD}\textbf{38.32} & 0      \\ \hline
\begin{tabular}[c]{@{}l@{}}Kaleid-\\ oscope\end{tabular} & 0  & \textbf{88.1}    & 73.5   & 0      & 0.31   & 7.03   & 28.66  & 85.91  & 18.83  & 2.22   & 62.98  & 29.78  & 21.07  & \textbf{84.89}    & 34.6   & 0      \\
\begin{tabular}[c]{@{}l@{}}Kaleid-\\ oscope\end{tabular} & 0.5& \cellcolor[HTML]{F4C7C3}85.2    & \cellcolor[HTML]{B7E1CD}75.8    & 0.05   & \cellcolor[HTML]{B7E1CD}3.01    & \cellcolor[HTML]{B7E1CD}12.79   & \cellcolor[HTML]{B7E1CD}47.67   & \cellcolor[HTML]{F4C7C3}82.45   & 19.8   & \cellcolor[HTML]{B7E1CD}15.67   & \cellcolor[HTML]{F4C7C3}59.18   & \cellcolor[HTML]{B7E1CD}\textbf{41.48} & \cellcolor[HTML]{B7E1CD}29.71   & \cellcolor[HTML]{F4C7C3}81.32      & \cellcolor[HTML]{B7E1CD}39.08   & \textbf{0.05}    \\
\begin{tabular}[c]{@{}l@{}}Kaleid-\\ oscope\end{tabular} & 1  & 87.52  & \cellcolor[HTML]{B7E1CD}\textbf{80.48} & 0.03   & \cellcolor[HTML]{B7E1CD}2.52    & \cellcolor[HTML]{B7E1CD}\textbf{20.89} & \cellcolor[HTML]{B7E1CD}45.89   & \textbf{86.04}   & \cellcolor[HTML]{B7E1CD}23.85   & \cellcolor[HTML]{B7E1CD}\textbf{27.57} & \cellcolor[HTML]{B7E1CD}\textbf{64.28} & \cellcolor[HTML]{B7E1CD}33.63   & \cellcolor[HTML]{B7E1CD}27.69   & \cellcolor[HTML]{F4C7C3}83.03      & \cellcolor[HTML]{B7E1CD}\textbf{41.32} & 0      \\
\begin{tabular}[c]{@{}l@{}}Kaleid-\\ oscope\end{tabular} & 2  & \cellcolor[HTML]{F4C7C3}80.41   & \cellcolor[HTML]{F4C7C3}71.54   & \textbf{0.84}    & \cellcolor[HTML]{B7E1CD}\textbf{5.78}  & \cellcolor[HTML]{B7E1CD}19.06   & \cellcolor[HTML]{B7E1CD}\textbf{54.55} & \cellcolor[HTML]{F4C7C3}76.2    & \cellcolor[HTML]{B7E1CD}\textbf{25.02} & \cellcolor[HTML]{B7E1CD}24.41   & \cellcolor[HTML]{F4C7C3}52.56   & \cellcolor[HTML]{F4C7C3}27.82   & \cellcolor[HTML]{B7E1CD}\textbf{35.41} & \cellcolor[HTML]{F4C7C3}75.49      & \cellcolor[HTML]{B7E1CD}39.06   & 0.46   \\ \hline
\end{tabular}}}
\caption{\textbf{Intial Training Ablations- Uniform regularization on ImageNette. }Accuracy of initially trained models on ImageNette trained using different attacks as indicated in ``Train Attack" column measured across different attacks.  Uniform regularization (with $\sigma=2$) is also considered during initial training, with regularization strength $\lambda$.  Results where regularization improves over no regularization ($\lambda = 0$) by at least 1\% accuracy are highlighted in green, while results where regularization incurs at least a 1\% drop in accuracy are highlighted in red.  Best performing with respect to regularization strength are bolded.}
\label{app:IT_ablation_UR_IM}

\end{table*}

\begin{table*}[]
{\renewcommand{\arraystretch}{1.2}
\scalebox{0.73}{
\begin{tabular}{|l|c|c|cccccccccccc|cc|}
\hline
\begin{tabular}[c]{@{}l@{}}Train\\ Attack\end{tabular}   & $\lambda$ & Clean  & $\ell_2$     & $\ell_\infty$   & StAdv  & ReColor& Gabor  & Snow   & Pixel  & JPEG   & Elastic& Wood   & Glitch & \begin{tabular}[c]{@{}c@{}}Kaleid-\\ oscope\end{tabular} & Avg    & Union  \\ \hline
$\ell_2$ & 0  & \textbf{90.04}   & \textbf{83.95}   & 7.57   & 5.27   & 33.91  & 65.17  & \textbf{89.3}    & 28.99  & 67.52  & \textbf{62.85}   & 49.22  & 45.78  & 12.76     & 46.03  & 0.51   \\
$\ell_2$ & 0.1& 89.35  & 83.13  & \cellcolor[HTML]{B7E1CD}12.84   & 6.22   & \cellcolor[HTML]{B7E1CD}35.59   & \cellcolor[HTML]{B7E1CD}72.05   & \cellcolor[HTML]{F4CCCC}87.82   & 29.45  & \cellcolor[HTML]{B7E1CD}71.36   & 62.22  & 50.14  & 45.02  & \cellcolor[HTML]{B7E1CD}16.18      & \cellcolor[HTML]{B7E1CD}47.67   & 1.07   \\
$\ell_2$ & 0.2& 89.38  & 83.54  & \cellcolor[HTML]{B7E1CD}15.82   & \cellcolor[HTML]{B7E1CD}7.54    & \cellcolor[HTML]{B7E1CD}38.04   & \cellcolor[HTML]{B7E1CD}73.43   & 88.41  & \cellcolor[HTML]{B7E1CD}30.6    & \cellcolor[HTML]{B7E1CD}72.2    & \cellcolor[HTML]{F4CCCC}60.13   & \cellcolor[HTML]{B7E1CD}50.62   & 46.14  & \cellcolor[HTML]{B7E1CD}15.64      & \cellcolor[HTML]{B7E1CD}48.51   & 1.35   \\
$\ell_2$ & 0.5& \cellcolor[HTML]{F4CCCC}87.03   & \cellcolor[HTML]{F4CCCC}81.89   & \cellcolor[HTML]{B7E1CD}\textbf{22.09} & \cellcolor[HTML]{B7E1CD}\textbf{12.31} & \cellcolor[HTML]{B7E1CD}\textbf{40.36} & \cellcolor[HTML]{B7E1CD}\textbf{74.09} & \cellcolor[HTML]{F4CCCC}84.15   & \cellcolor[HTML]{B7E1CD}\textbf{34.98} & \cellcolor[HTML]{B7E1CD}\textbf{72.79} & \cellcolor[HTML]{F4CCCC}60.13   & \cellcolor[HTML]{B7E1CD}\textbf{51.49} & \cellcolor[HTML]{B7E1CD}\textbf{52.2}  & \cellcolor[HTML]{B7E1CD}\textbf{18.45}    & \cellcolor[HTML]{B7E1CD}\textbf{50.41} & \cellcolor[HTML]{B7E1CD}\textbf{2.96}  \\ \hline
$\ell_\infty$      & 0  & \textbf{84.51}   & 81.71  & 58.39  & 43.49  & 67.82  & 72.61  & \textbf{83.31}   & 41.83  & 65.35  & 63.9   & \textbf{67.18}   & \textbf{63.64}   & 30.75     & 61.67  & 13.2   \\
$\ell_\infty$      & 0.1& 83.9   & 82.09  & 57.86  & \cellcolor[HTML]{B7E1CD}47.77   & 68.38  & 73.4   & 82.96  & \cellcolor[HTML]{B7E1CD}44.92   & \cellcolor[HTML]{B7E1CD}74.55   & 63.21  & \cellcolor[HTML]{F4C7C3}64.41   & 63.18  & \cellcolor[HTML]{B7E1CD}\textbf{36.36}    & \cellcolor[HTML]{B7E1CD}63.26   & \cellcolor[HTML]{B7E1CD}17.04   \\
$\ell_\infty$      & 0.2& 84.25  & \textbf{82.39}   & 58.39  & \cellcolor[HTML]{B7E1CD}\textbf{49.38} & \textbf{68.59}   & \cellcolor[HTML]{B7E1CD}\textbf{76.33} & 82.9   & \cellcolor[HTML]{B7E1CD}45.45   & \cellcolor[HTML]{B7E1CD}76.46   & \textbf{64.56}   & 66.19  & 63.36  & 31.44     & \cellcolor[HTML]{B7E1CD}\textbf{63.79} & \cellcolor[HTML]{B7E1CD}16.25   \\
$\ell_\infty$      & 0.5& 83.87  & 82.22  & \textbf{58.42}   & \cellcolor[HTML]{B7E1CD}47.87   & 68.23  & \cellcolor[HTML]{B7E1CD}75.69   & 82.78  & \cellcolor[HTML]{B7E1CD}\textbf{46.24} & \cellcolor[HTML]{B7E1CD}\textbf{77.53} & \cellcolor[HTML]{F4C7C3}62.14   & \cellcolor[HTML]{F4C7C3}64.87   & \cellcolor[HTML]{F4C7C3}62.5    & \cellcolor[HTML]{B7E1CD}35.31      & \cellcolor[HTML]{B7E1CD}63.65   & \cellcolor[HTML]{B7E1CD}\textbf{17.66} \\ \hline
StAdv     & 0  & \textbf{83.31}   & 77.58  & 1.45   & \textbf{69.81}   & 13.43  & 36.66  & \textbf{81.5}    & 20.56  & 49.89  & \textbf{70.32}   & 60.76  & 36.15  & 24.84     & 45.25  & 1.04   \\
StAdv     & 0.1& 82.34  & 77.76  & \cellcolor[HTML]{B7E1CD}4.18    & \cellcolor[HTML]{F4C7C3}66.45   & \cellcolor[HTML]{B7E1CD}16.54   & \cellcolor[HTML]{B7E1CD}\textbf{57.91} & \cellcolor[HTML]{F4C7C3}79.52   & \cellcolor[HTML]{B7E1CD}22.14   & \cellcolor[HTML]{B7E1CD}60.43   & \cellcolor[HTML]{F4C7C3}68.56   & \cellcolor[HTML]{B7E1CD}\textbf{62.29} & \cellcolor[HTML]{B7E1CD}40.59   & \cellcolor[HTML]{F4C7C3}23.46      & \cellcolor[HTML]{B7E1CD}48.32   & \cellcolor[HTML]{B7E1CD}2.78    \\
StAdv     & 0.2& 83.11  & \cellcolor[HTML]{B7E1CD}79.21   & \cellcolor[HTML]{B7E1CD}6.5     & \cellcolor[HTML]{F4C7C3}67.64   & \cellcolor[HTML]{B7E1CD}\textbf{18.7}  & \cellcolor[HTML]{B7E1CD}45.32   & \cellcolor[HTML]{F4C7C3}80.03   & \cellcolor[HTML]{B7E1CD}26.34   & \cellcolor[HTML]{B7E1CD}64.84   & 70.27  & \cellcolor[HTML]{B7E1CD}61.94   & \cellcolor[HTML]{B7E1CD}\textbf{50.14} & \cellcolor[HTML]{B7E1CD}30.37      & \cellcolor[HTML]{B7E1CD}50.11   & \cellcolor[HTML]{B7E1CD}3.34    \\
StAdv     & 0.5& 83.13  & \cellcolor[HTML]{B7E1CD}\textbf{80.03} & \cellcolor[HTML]{B7E1CD}\textbf{9.35}  & \cellcolor[HTML]{F4C7C3}68.25   & \cellcolor[HTML]{B7E1CD}18.34   & \cellcolor[HTML]{B7E1CD}57.63   & \cellcolor[HTML]{F4C7C3}80.28   & \cellcolor[HTML]{B7E1CD}\textbf{27.44} & \cellcolor[HTML]{B7E1CD}\textbf{68.08} & 69.78  & 60.05  & \cellcolor[HTML]{B7E1CD}49.4    & \cellcolor[HTML]{B7E1CD}\textbf{32.66}    & \cellcolor[HTML]{B7E1CD}\textbf{51.77} & \cellcolor[HTML]{B7E1CD}\textbf{5.22}  \\ \hline
ReColor   & 0  & \textbf{91.34}   & 81.53  & 0.03   & 0.41   & 79.08  & 42.55  & 90.6   & 22.39  & 25.2   & \textbf{64.31}   & \textbf{54.8}    & 18.65  & 8.94      & 40.71  & 0      \\
ReColor   & 0.1& 90.83  & \cellcolor[HTML]{B7E1CD}84.97   & 0.38   & 0.92   & \textbf{79.69}   & \cellcolor[HTML]{B7E1CD}57.17   & \textbf{90.62}   & \cellcolor[HTML]{B7E1CD}29.94   & \cellcolor[HTML]{B7E1CD}48.41   & \cellcolor[HTML]{F4C7C3}60.94   & \cellcolor[HTML]{F4C7C3}53.02   & \cellcolor[HTML]{B7E1CD}33.81   & 8.54      & \cellcolor[HTML]{B7E1CD}45.7    & 0.05   \\
ReColor   & 0.2& 90.8   & \cellcolor[HTML]{B7E1CD}\textbf{85.73} & \cellcolor[HTML]{B7E1CD}1.2     & 1.17   & 78.98  & \cellcolor[HTML]{B7E1CD}62.88   & 90.5   & \cellcolor[HTML]{B7E1CD}27.87   & \cellcolor[HTML]{B7E1CD}55.34   & \cellcolor[HTML]{F4C7C3}60.74   & \cellcolor[HTML]{F4C7C3}49.07   & \cellcolor[HTML]{B7E1CD}31.57   & \cellcolor[HTML]{B7E1CD}14.29      & \cellcolor[HTML]{B7E1CD}46.61   & 0.18   \\
ReColor   & 0.5& \cellcolor[HTML]{F4C7C3}89.15   & \cellcolor[HTML]{B7E1CD}85.2    & \cellcolor[HTML]{B7E1CD}\textbf{5.81}  & \cellcolor[HTML]{B7E1CD}\textbf{2.75}  & \cellcolor[HTML]{F4C7C3}77.81   & \cellcolor[HTML]{B7E1CD}\textbf{71.26} & \cellcolor[HTML]{F4C7C3}88.46   & \cellcolor[HTML]{B7E1CD}\textbf{31.34} & \cellcolor[HTML]{B7E1CD}\textbf{65.07} & \cellcolor[HTML]{F4C7C3}59.77   & \cellcolor[HTML]{F4C7C3}52.41   & \cellcolor[HTML]{B7E1CD}\textbf{41.27} & \cellcolor[HTML]{B7E1CD}\textbf{15.39}    & \cellcolor[HTML]{B7E1CD}\textbf{49.71} & \textbf{0.33}    \\ \hline
Gabor     & 0  & \textbf{89.12}   & \textbf{85.27}   & 4.41   & 2.11   & 37.83  & \textbf{87.31}   & \textbf{87.62}   & 20.28  & 57.66  & 52.33  & 38.73  & 38.88  & 9.22      & 43.47  & 0.1    \\
Gabor     & 0.1& \cellcolor[HTML]{F4C7C3}87.29   & \cellcolor[HTML]{F4C7C3}83.77   & \cellcolor[HTML]{B7E1CD}11.26   & \cellcolor[HTML]{B7E1CD}5.53    & 38.22  & \cellcolor[HTML]{F4C7C3}85.58   & \cellcolor[HTML]{F4C7C3}84.54   & 20.64  & \cellcolor[HTML]{B7E1CD}66.11   & \cellcolor[HTML]{F4C7C3}48.82   & \cellcolor[HTML]{F4C7C3}37.02   & \cellcolor[HTML]{B7E1CD}41.96   & \cellcolor[HTML]{B7E1CD}15.06      & \cellcolor[HTML]{B7E1CD}44.87   & 0.76   \\
Gabor     & 0.2& \cellcolor[HTML]{F4C7C3}86.22   & \cellcolor[HTML]{F4C7C3}82.88   & \cellcolor[HTML]{B7E1CD}12.97   & \cellcolor[HTML]{B7E1CD}6.47    & 37.27  & \cellcolor[HTML]{F4C7C3}84.69   & \cellcolor[HTML]{F4C7C3}83.92   & \cellcolor[HTML]{B7E1CD}22.45   & \cellcolor[HTML]{B7E1CD}67.77   & \cellcolor[HTML]{F4C7C3}48.46   & \cellcolor[HTML]{F4C7C3}36.31   & \cellcolor[HTML]{B7E1CD}41.15   & \cellcolor[HTML]{B7E1CD}19.01      & \cellcolor[HTML]{B7E1CD}45.28   & 1.04   \\
Gabor     & 0.5& \cellcolor[HTML]{F4C7C3}87.29   & \cellcolor[HTML]{F4C7C3}84.13   & \cellcolor[HTML]{B7E1CD}\textbf{19.87} & \cellcolor[HTML]{B7E1CD}\textbf{9.48}  & \cellcolor[HTML]{B7E1CD}\textbf{41.35} & \cellcolor[HTML]{F4C7C3}85.4    & \cellcolor[HTML]{F4C7C3}85.07   & \cellcolor[HTML]{B7E1CD}\textbf{26.8}  & \cellcolor[HTML]{B7E1CD}\textbf{72.28} & \textbf{52.79}   & \cellcolor[HTML]{B7E1CD}\textbf{42.29} & \cellcolor[HTML]{B7E1CD}\textbf{46.19} & \cellcolor[HTML]{B7E1CD}\textbf{19.9}     & \cellcolor[HTML]{B7E1CD}\textbf{48.8}  & \cellcolor[HTML]{B7E1CD}\textbf{1.99}  \\ \hline
Snow      & 0  & \textbf{87.69}   & 71.59  & 0.08   & 1.53   & 11.97  & 30.68  & \textbf{62.9}    & 7.31   & 8.59   & 66.24  & 70.78  & 11.82  & 9.91      & 29.45  & 0.05   \\
Snow      & 0.1& \textbf{87.69}   & \cellcolor[HTML]{B7E1CD}80.25   & \cellcolor[HTML]{B7E1CD}1.35    & \cellcolor[HTML]{B7E1CD}5.81    & \cellcolor[HTML]{B7E1CD}16.41   & \cellcolor[HTML]{B7E1CD}60.87   & \cellcolor[HTML]{F4C7C3}57.02   & \cellcolor[HTML]{B7E1CD}10.09   & \cellcolor[HTML]{B7E1CD}38.24   & \cellcolor[HTML]{F4C7C3}65.17   & \textbf{71.77}   & \cellcolor[HTML]{B7E1CD}22.62   & 10.85     & \cellcolor[HTML]{B7E1CD}36.7    & 0.38   \\
Snow      & 0.2& 86.8   & \cellcolor[HTML]{B7E1CD}80.36   & \cellcolor[HTML]{B7E1CD}1.94    & \cellcolor[HTML]{B7E1CD}\textbf{23.82} & \cellcolor[HTML]{B7E1CD}\textbf{33.45} & \cellcolor[HTML]{B7E1CD}61.55   & \cellcolor[HTML]{F4C7C3}58.19   & \cellcolor[HTML]{B7E1CD}11.97   & \cellcolor[HTML]{B7E1CD}47.77   & \textbf{67.16}   & 71.54  & \cellcolor[HTML]{B7E1CD}25.68   & \cellcolor[HTML]{B7E1CD}\textbf{15.9}     & \cellcolor[HTML]{B7E1CD}41.61   & \cellcolor[HTML]{B7E1CD}1.66    \\
Snow      & 0.5& \cellcolor[HTML]{F4C7C3}85.04   & \cellcolor[HTML]{B7E1CD}\textbf{80.46} & \cellcolor[HTML]{B7E1CD}\textbf{4.31}  & \cellcolor[HTML]{B7E1CD}15.34   & \cellcolor[HTML]{B7E1CD}21.55   & \cellcolor[HTML]{B7E1CD}\textbf{66.65} & \cellcolor[HTML]{F4C7C3}55.82   & \cellcolor[HTML]{B7E1CD}\textbf{14.04} & \cellcolor[HTML]{B7E1CD}\textbf{56.33} & 65.3   & \cellcolor[HTML]{F4C7C3}69.48   & \cellcolor[HTML]{B7E1CD}\textbf{32.51} & \cellcolor[HTML]{B7E1CD}13.61      & \cellcolor[HTML]{B7E1CD}\textbf{41.28} & \cellcolor[HTML]{B7E1CD}\textbf{2.6}   \\ \hline
Pixel     & 0  & 88.64  & 67.24  & 0      & 0.59   & 34.96  & 30.37  & 87.18  & 78.6   & 0.25   & \textbf{61.63}   & \textbf{52.08}   & \textbf{49.38}   & 23.97     & 40.52  & 0      \\
Pixel     & 0.1& 88.15  & \cellcolor[HTML]{B7E1CD}76.99   & 0.03   & 0.59   & \cellcolor[HTML]{B7E1CD}41.5    & \cellcolor[HTML]{B7E1CD}42.04   & \cellcolor[HTML]{F4C7C3}85.2    & \cellcolor[HTML]{F4C7C3}77.25   & \cellcolor[HTML]{B7E1CD}8.56    & \cellcolor[HTML]{F4C7C3}56.64   & \cellcolor[HTML]{F4C7C3}46.88   & \cellcolor[HTML]{F4C7C3}43.92   & \cellcolor[HTML]{B7E1CD}28.2& \cellcolor[HTML]{B7E1CD}42.32   & 0      \\
Pixel     & 0.2& \textbf{88.89}   & \cellcolor[HTML]{B7E1CD}80.99   & 0.03   & 1.4    & \cellcolor[HTML]{B7E1CD}46.7    & \cellcolor[HTML]{B7E1CD}44.23   & \textbf{87.64}   & \textbf{78.73}   & \cellcolor[HTML]{B7E1CD}20.18   & 61.4   & \cellcolor[HTML]{F4C7C3}46.85   & 49.12  & \cellcolor[HTML]{B7E1CD}26.37      & \cellcolor[HTML]{B7E1CD}45.3    & 0      \\
Pixel     & 0.5& \cellcolor[HTML]{F4C7C3}86.6    & \cellcolor[HTML]{B7E1CD}\textbf{82.14} & \cellcolor[HTML]{B7E1CD}\textbf{1.43}  & \cellcolor[HTML]{B7E1CD}\textbf{6.62}  & \cellcolor[HTML]{B7E1CD}\textbf{46.73} & \cellcolor[HTML]{B7E1CD}\textbf{53.45} & \cellcolor[HTML]{F4C7C3}84.97   & \cellcolor[HTML]{F4C7C3}77.35   & \cellcolor[HTML]{B7E1CD}\textbf{44.99} & \cellcolor[HTML]{F4C7C3}59.08   & \cellcolor[HTML]{F4C7C3}48.28   & \cellcolor[HTML]{B7E1CD}53.73   & \cellcolor[HTML]{B7E1CD}\textbf{28.66}    & \cellcolor[HTML]{B7E1CD}\textbf{48.95} & \textbf{0.38}    \\ \hline
JPEG      & 0  & 88.43  & 85.63  & 15.29  & 5.43   & 41.78  & 77.35  & 86.85  & 23.21  & 80.87  & 53.81  & 43.39  & 44.79  & 15.39     & 47.82  & 0.74   \\
JPEG      & 0.1& 88.18  & 84.79  & \cellcolor[HTML]{B7E1CD}20.33   & \cellcolor[HTML]{B7E1CD}9.2     & 42.34  & \cellcolor[HTML]{B7E1CD}80.54   & \cellcolor[HTML]{F4C7C3}85.45   & 23.36  & 80.18  & \cellcolor[HTML]{B7E1CD}\textbf{57.38} & \cellcolor[HTML]{B7E1CD}46.55   & \cellcolor[HTML]{B7E1CD}46.78   & \cellcolor[HTML]{B7E1CD}17.96      & \cellcolor[HTML]{B7E1CD}49.57   & \cellcolor[HTML]{B7E1CD}2.22    \\
JPEG      & 0.2& \textbf{88.92}   & \textbf{86.04}   & \cellcolor[HTML]{B7E1CD}23.31   & \cellcolor[HTML]{B7E1CD}8.28    & \cellcolor[HTML]{B7E1CD}\textbf{46.88} & \cellcolor[HTML]{B7E1CD}\textbf{80.61} & \textbf{87.36}   & \cellcolor[HTML]{B7E1CD}26.73   & \textbf{81.68}   & \cellcolor[HTML]{B7E1CD}57.22   & \cellcolor[HTML]{B7E1CD}47.18   & \cellcolor[HTML]{B7E1CD}48.94   & 16.36     & \cellcolor[HTML]{B7E1CD}50.88   & \cellcolor[HTML]{B7E1CD}1.86    \\
JPEG      & 0.5& \cellcolor[HTML]{F4C7C3}87.18   & \cellcolor[HTML]{F4C7C3}84.41   & \cellcolor[HTML]{B7E1CD}\textbf{26.5}  & \cellcolor[HTML]{B7E1CD}\textbf{12.46} & \cellcolor[HTML]{B7E1CD}44.94   & \cellcolor[HTML]{B7E1CD}79.36   & \cellcolor[HTML]{F4C7C3}85.07   & \cellcolor[HTML]{B7E1CD}\textbf{28.36} & 80.54  & \cellcolor[HTML]{B7E1CD}56.59   & \cellcolor[HTML]{B7E1CD}\textbf{47.77} & \cellcolor[HTML]{B7E1CD}\textbf{52.71} & \cellcolor[HTML]{B7E1CD}\textbf{19.01}    & \cellcolor[HTML]{B7E1CD}\textbf{51.48} & \cellcolor[HTML]{B7E1CD}\textbf{3.06}  \\ \hline
Elastic   & 0  & \textbf{89.66}   & 77.48  & 0      & 0.82   & 12.25  & 21.81  & \textbf{88.05}   & 16.84  & 5.71   & 78.6   & 62.01  & 16.79  & 11.69     & 32.67  & 0      \\
Elastic   & 0.1& \cellcolor[HTML]{F4C7C3}88.59   & \cellcolor[HTML]{B7E1CD}82.65   & 0.46   & \cellcolor[HTML]{B7E1CD}4.31    & \cellcolor[HTML]{B7E1CD}17.83   & \cellcolor[HTML]{B7E1CD}45.53   & 87.44  & \cellcolor[HTML]{B7E1CD}18.8    & \cellcolor[HTML]{B7E1CD}35.85   & \cellcolor[HTML]{B7E1CD}79.85   & \cellcolor[HTML]{B7E1CD}\textbf{64.23} & \cellcolor[HTML]{B7E1CD}27.11   & 11.11     & \cellcolor[HTML]{B7E1CD}39.6    & 0.08   \\
Elastic   & 0.2& 89.35  & \cellcolor[HTML]{B7E1CD}\textbf{84.03} & \cellcolor[HTML]{B7E1CD}1.66    & \cellcolor[HTML]{B7E1CD}5.55    & \cellcolor[HTML]{B7E1CD}\textbf{20.64} & \cellcolor[HTML]{B7E1CD}53.22   & 88     & \cellcolor[HTML]{B7E1CD}22.45   & \cellcolor[HTML]{B7E1CD}50.75   & \cellcolor[HTML]{B7E1CD}\textbf{80.15} & 62.29  & \cellcolor[HTML]{B7E1CD}31.75   & \cellcolor[HTML]{F4C7C3}8.1& \cellcolor[HTML]{B7E1CD}42.38   & 0.23   \\
Elastic   & 0.5& \cellcolor[HTML]{F4C7C3}82.17   & 77.61  & \cellcolor[HTML]{B7E1CD}\textbf{6.9}   & \cellcolor[HTML]{B7E1CD}\textbf{62.17} & \cellcolor[HTML]{B7E1CD}63.8    & \cellcolor[HTML]{B7E1CD}\textbf{55.44} & \cellcolor[HTML]{F4C7C3}78.32   & \cellcolor[HTML]{B7E1CD}\textbf{29.61} & \cellcolor[HTML]{B7E1CD}\textbf{55.75} & \cellcolor[HTML]{F4C7C3}73.22   & \cellcolor[HTML]{B7E1CD}63.36   & \cellcolor[HTML]{B7E1CD}\textbf{36.66} & \cellcolor[HTML]{B7E1CD}\textbf{19.52}    & \cellcolor[HTML]{B7E1CD}\textbf{51.86} & \cellcolor[HTML]{B7E1CD}\textbf{4.15}  \\ \hline
Wood      & 0  & 85.91  & 50.09  & 0      & 0.82   & 11.11  & 35.87  & 83.31  & 11.13  & 1.2    & 60.94  & 78.83  & 14.96  & 11.13     & 29.95  & 0      \\
Wood      & 0.1& \cellcolor[HTML]{B7E1CD}88.1    & \cellcolor[HTML]{B7E1CD}80.71   & 0.87   & \cellcolor[HTML]{B7E1CD}6.68    & \cellcolor[HTML]{B7E1CD}14.19   & \cellcolor[HTML]{B7E1CD}46.55   & \cellcolor[HTML]{B7E1CD}86.22   & 11.85  & \cellcolor[HTML]{B7E1CD}33.3    & \cellcolor[HTML]{B7E1CD}66.68   & \cellcolor[HTML]{B7E1CD}82.88   & \cellcolor[HTML]{B7E1CD}30.8    & \cellcolor[HTML]{F4C7C3}8.1& \cellcolor[HTML]{B7E1CD}39.07   & 0.33   \\
Wood      & 0.2& \cellcolor[HTML]{B7E1CD}\textbf{88.99} & \cellcolor[HTML]{B7E1CD}83.03   & \cellcolor[HTML]{B7E1CD}2.14    & \cellcolor[HTML]{B7E1CD}7.67    & \cellcolor[HTML]{B7E1CD}17.17   & \cellcolor[HTML]{B7E1CD}45.25   & \cellcolor[HTML]{B7E1CD}\textbf{86.98} & \cellcolor[HTML]{B7E1CD}13.96   & \cellcolor[HTML]{B7E1CD}45.22   & \cellcolor[HTML]{B7E1CD}68.33   & \cellcolor[HTML]{B7E1CD}\textbf{83.67} & \cellcolor[HTML]{B7E1CD}30.29   & \cellcolor[HTML]{B7E1CD}14.19      & \cellcolor[HTML]{B7E1CD}41.49   & \cellcolor[HTML]{B7E1CD}1.15    \\
Wood      & 0.5& \cellcolor[HTML]{B7E1CD}88.13   & \cellcolor[HTML]{B7E1CD}\textbf{83.92} & \cellcolor[HTML]{B7E1CD}\textbf{5.53}  & \cellcolor[HTML]{B7E1CD}\textbf{13.66} & \cellcolor[HTML]{B7E1CD}\textbf{24.87} & \cellcolor[HTML]{B7E1CD}\textbf{64.15} & \cellcolor[HTML]{B7E1CD}86.68   & \cellcolor[HTML]{B7E1CD}\textbf{18.83} & \cellcolor[HTML]{B7E1CD}\textbf{60.41} & \cellcolor[HTML]{B7E1CD}\textbf{69.94} & \cellcolor[HTML]{B7E1CD}82.98   & \cellcolor[HTML]{B7E1CD}\textbf{39.9}  & \cellcolor[HTML]{B7E1CD}\textbf{15.77}    & \cellcolor[HTML]{B7E1CD}\textbf{47.22} & \cellcolor[HTML]{B7E1CD}\textbf{1.99}  \\ \hline
Glitch    & 0  & \textbf{88.51}   & 36.41  & 0      & 0      & 6.7    & 18.47  & \textbf{86.96}   & 17.1   & 0      & 60     & 50.93  & \textbf{84.97}   & 6.37      & 30.66  & 0      \\
Glitch    & 0.1& \cellcolor[HTML]{F4C7C3}87.41   & \cellcolor[HTML]{B7E1CD}80.08   & \cellcolor[HTML]{B7E1CD}1.32    & \cellcolor[HTML]{B7E1CD}4.56    & \cellcolor[HTML]{B7E1CD}27.77   & \cellcolor[HTML]{B7E1CD}48.15   & \cellcolor[HTML]{F4C7C3}85.1    & \cellcolor[HTML]{B7E1CD}26.93   & \cellcolor[HTML]{B7E1CD}46.19   & \cellcolor[HTML]{B7E1CD}61.45   & \cellcolor[HTML]{F4C7C3}48.56   & \cellcolor[HTML]{F4C7C3}82.47   & \cellcolor[HTML]{B7E1CD}12.36      & \cellcolor[HTML]{B7E1CD}43.75   & 0.33   \\
Glitch    & 0.2& \cellcolor[HTML]{F4C7C3}86.93   & \cellcolor[HTML]{B7E1CD}83.54   & \cellcolor[HTML]{B7E1CD}4.54    & \cellcolor[HTML]{B7E1CD}8.31    & \cellcolor[HTML]{B7E1CD}24.08   & \cellcolor[HTML]{B7E1CD}47.34   & \cellcolor[HTML]{F4C7C3}85.73   & \cellcolor[HTML]{B7E1CD}32.94   & \cellcolor[HTML]{B7E1CD}53.61   & \cellcolor[HTML]{B7E1CD}\textbf{62.29} & \cellcolor[HTML]{B7E1CD}\textbf{52.36} & \cellcolor[HTML]{F4C7C3}83.21   & \cellcolor[HTML]{B7E1CD}17.4& \cellcolor[HTML]{B7E1CD}46.28   & \cellcolor[HTML]{B7E1CD}1.27    \\
Glitch    & 0.5& \cellcolor[HTML]{F4C7C3}87.16   & \cellcolor[HTML]{B7E1CD}\textbf{84.38} & \cellcolor[HTML]{B7E1CD}\textbf{13.73} & \cellcolor[HTML]{B7E1CD}\textbf{14.96} & \cellcolor[HTML]{B7E1CD}\textbf{28.94} & \cellcolor[HTML]{B7E1CD}\textbf{62.11} & \cellcolor[HTML]{F4C7C3}84.89   & \cellcolor[HTML]{B7E1CD}\textbf{35.9}  & \cellcolor[HTML]{B7E1CD}\textbf{67.92} & 60.92  & 51.46  & \cellcolor[HTML]{F4C7C3}82.8    & \cellcolor[HTML]{B7E1CD}\textbf{20.48}    & \cellcolor[HTML]{B7E1CD}\textbf{50.71} & \cellcolor[HTML]{B7E1CD}\textbf{2.8}   \\ \hline
\begin{tabular}[c]{@{}l@{}}Kaleid-\\ oscope\end{tabular} & 0  & \textbf{88.1}    & \textbf{73.5}    & 0      & 0.31   & 7.03   & 28.66  & \textbf{85.91}   & 18.83  & 2.22   & \textbf{62.98}   & 29.78  & 21.07  & \textbf{84.89}    & 34.6   & 0      \\
\begin{tabular}[c]{@{}l@{}}Kaleid-\\ oscope\end{tabular} & 0.1& \cellcolor[HTML]{F4C7C3}84.05   & \cellcolor[HTML]{F4C7C3}72.25   & 0.1    & \cellcolor[HTML]{B7E1CD}2.98    & \cellcolor[HTML]{B7E1CD}14.27   & \cellcolor[HTML]{B7E1CD}40.74   & \cellcolor[HTML]{F4C7C3}80.87   & \cellcolor[HTML]{F4C7C3}14.7    & \cellcolor[HTML]{B7E1CD}14.29   & \cellcolor[HTML]{F4C7C3}55.39   & \cellcolor[HTML]{B7E1CD}\textbf{46.42} & 21.15  & \cellcolor[HTML]{F4C7C3}74.96      & \cellcolor[HTML]{B7E1CD}36.51   & 0.1    \\
\begin{tabular}[c]{@{}l@{}}Kaleid-\\ oscope\end{tabular} & 0.2& \cellcolor[HTML]{F4C7C3}69.17   & \cellcolor[HTML]{F4C7C3}62.19   & \cellcolor[HTML]{B7E1CD}1.27    & \cellcolor[HTML]{B7E1CD}10.11   & \cellcolor[HTML]{B7E1CD}18.7    & \cellcolor[HTML]{B7E1CD}55.72   & \cellcolor[HTML]{F4C7C3}66.14   & \cellcolor[HTML]{F4C7C3}17.81   & \cellcolor[HTML]{B7E1CD}30.57   & \cellcolor[HTML]{F4C7C3}41.78   & \cellcolor[HTML]{F4C7C3}21.68   & \cellcolor[HTML]{B7E1CD}32.36   & \cellcolor[HTML]{F4C7C3}66.8& 35.43  & 0.66   \\
\begin{tabular}[c]{@{}l@{}}Kaleid-\\ oscope\end{tabular} & 0.5& \cellcolor[HTML]{F4C7C3}71.95   & \cellcolor[HTML]{F4C7C3}66.85   & \cellcolor[HTML]{B7E1CD}\textbf{7.82}  & \cellcolor[HTML]{B7E1CD}\textbf{28.54} & \cellcolor[HTML]{B7E1CD}\textbf{21.91} & \cellcolor[HTML]{B7E1CD}\textbf{58.5}  & \cellcolor[HTML]{F4C7C3}68.64   & \cellcolor[HTML]{B7E1CD}\textbf{24.92} & \cellcolor[HTML]{B7E1CD}\textbf{47.44} & \cellcolor[HTML]{F4C7C3}47.77   & \cellcolor[HTML]{B7E1CD}38.32   & \cellcolor[HTML]{B7E1CD}\textbf{43.62} & \cellcolor[HTML]{F4C7C3}65.04      & \cellcolor[HTML]{B7E1CD}\textbf{43.28} & \cellcolor[HTML]{B7E1CD}\textbf{6.01}  \\ \hline
\end{tabular}}}
\caption{\textbf{Intial Training Ablations- Gaussian regularization on ImageNette. }Accuracy of initially trained models on ImageNette trained using different attacks as indicated in ``Train Attack" column measured across different attacks.  Gaussian regularization (with $\sigma=0.2$) is also considered during initial training, with regularization strength $\lambda$.  Results where regularization improves over no regularization ($\lambda = 0$) by at least 1\% accuracy are highlighted in green, while results where regularization incurs at least a 1\% drop in accuracy are highlighted in red.  Best performing with respect to regularization strength are bolded.}
\label{app:IT_ablation_GR_IM}
\end{table*}

%\subsubsection{Impact of number of iterations for computing variation regularization and adversarial $\ell_2$ regularization}
%\label{app:num_iterations_abl}

\subsection{Impact of random noise parameter $\sigma$}
\label{app:random_noise_var_abl}
To investigate the impact of the noise parameter $\sigma$, we perform initial training on CIFAR-10 with uniform and gaussian regularization at different values of $\sigma$.  We maintain a value of regularization strength $\lambda=5$ to isolate the impact of the noise variance from regularization strength.  We report results in Table \ref{tab:rand_noise_sigma}.  Overall, we find that $\sigma$ has an effect similar to the effect of increasing $\lambda$ where higher values of $\sigma$ leads to higher average and union robust accuracies at the cost of lower clean accuracy and accuracy on the initial attack.
\begin{table*}[]
\centering
\begin{tabular}{|l|c|c|cccc|cc|}
\hline
Noise type & $\sigma$ & Clean & $\ell_2$ & $\ell_\infty$ & StAdv & Recolor & Avg & Union \\ \hline
Uniform & 0.5 & 90.40 & 70.21 & 31.91 & 1.31 & 39.83 & 35.81 & 0.86 \\
Uniform & 1 & 90.76 & 69.89 & 32.58 & 1.49 & 40.14 & 36.02 & 1.15 \\
Uniform & 2 & 85.28 & 63.65 & 50.73 & 10.62 & 60.10 & 46.28 & 8.40 \\ \hline
Gaussian & 0.05 & 90.04 & 69.62 & 31.61 & 7.25 & 43.84 & 38.08 & 6.64 \\
Gaussian & 0.1 & 88.53 & 68.54 & 32.04 & 14.5 & 51.73 & 41.70 & 12.63 \\
Gaussian & 0.2 & 87.00 & 64.88 & 27.46 & 31.82 & 63.59 & 46.94 & 18.7 \\ \hline
\end{tabular}
\caption{\textbf{Impact of $\sigma$ on regularization based on random noise in initial training.} We maintain regularization strength $\lambda = 5$ and perform initial training on CIFAR-10 with $\ell_2$ attacks. We report the clean accuracy, accuracy on $\ell_2$, $\ell_\infty$, StAdv, and Recolor attacks, and the average and union accuracies on the set.}
\label{tab:rand_noise_sigma}
\end{table*}

\subsection{Impact of starting and new attack pairs}
\label{app:fine-tuning_pairs_init_train}
In order to see how much our results depend on attack choice, we experiment with starting with a model initially trained with a starting attack and then fine-tuned for robustness to a new attack for different starting and new attack pairs on Imagenette.  In this section, we ask the question: does regularization in initial training generally lead to better starting points for fine-tuning?  In the following experiments, we use adversarial training as the base initial training procedure and \citet{croce2022adversarial}'s fine-tuning approach as the base fine-tuning procedure.  We consider these approaches with and without regularization.

We compare models initially trained with regularization (and fine-tuned without regularization) to models initially trained without regularization (and fine-tuned without regularization).  We present the differences in average accuracy across the 2 attacks, union accuracy across the 2 attacks, accuracy on the starting attack, accuracy on the new attack, and clean accuracy between the 2 settings for adversarial $\ell_2$ regularization (with $\lambda=0.5$) in \cref{fig:imagenette_finetune_abl_froml2}, for variation regularization in \cref{fig:imagenette_finetune_abl_fromvarreg} (with $\lambda=0.2$), for uniform regularization (with $\sigma=2$ and $\lambda=1$) in Figure \ref{fig:imagenette_finetune_abl_fromuniform}, and for gaussian regularization (with $\sigma=0.2$ and $\lambda=0.5$) in Figure \ref{fig:imagenette_finetune_abl_fromgauss}. In these figures, we highlight gains in accuracy larger than 1\% in green and drops in accuracy larger than 1\% in red.

\textbf{Regularization in initial training generally improves performance. }Across all figures, we can see that for most pairs of attacks, regularization leads to improvements on average accuracy, union accuracy, accuracy on the initial attack, accuracy on the new attack.  We find that this improvement is more consistent across attack types when using adversarial versions of regularization such as adversarial $\ell_2$ regularization or variation regularization in comparison to random noise based regularizations.  This improvement in performance may be due to the fact that regularization improves unforeseen robustness, causing the initial accuracy on the new attack to generally be higher, and thus a better starting point for fine-tuning the model for robustness against new attacks.

\textbf{Uniform regularization in initial training can improve clean accuracy for certain starting attack types. }From Figure \ref{fig:imagenette_finetune_abl_fromuniform}e, we observe that using uniform regularization in initial training can lead to increases in clean accuracy after fine-tuning for several initial attack types: StAdv, ReColor, Pixel, Elastic, Wood, and Kaleidoscope attacks.  In comparison, Figure \ref{fig:imagenette_finetune_abl_froml2}e, demonstrates that using adversarial $\ell_2$ regularization does not improve clean accuracy for as many threat models as uniform regularization; for adversarial $\ell_2$ regularization, the most improvements in clean accuracy are when the initial attack is Elastic attack or when the new attack is $\ell_\infty$ attack. Adversarial $\ell_2$ regularization generally maintains clean accuracy for most attacks, but leads a drop in clean accuracy when the starting attack type is StAdv attack.  We find that similarly, variation regularization also maintains clean accuracy.  Gaussian regularization on the other hand either maintains or exhibits a tradeoff with clean accuracy.

\begin{figure*}[t!]
    \centering
    \begin{subfigure}[t]{0.4\textwidth}
        \centering
        \includegraphics[width=0.95\textwidth]{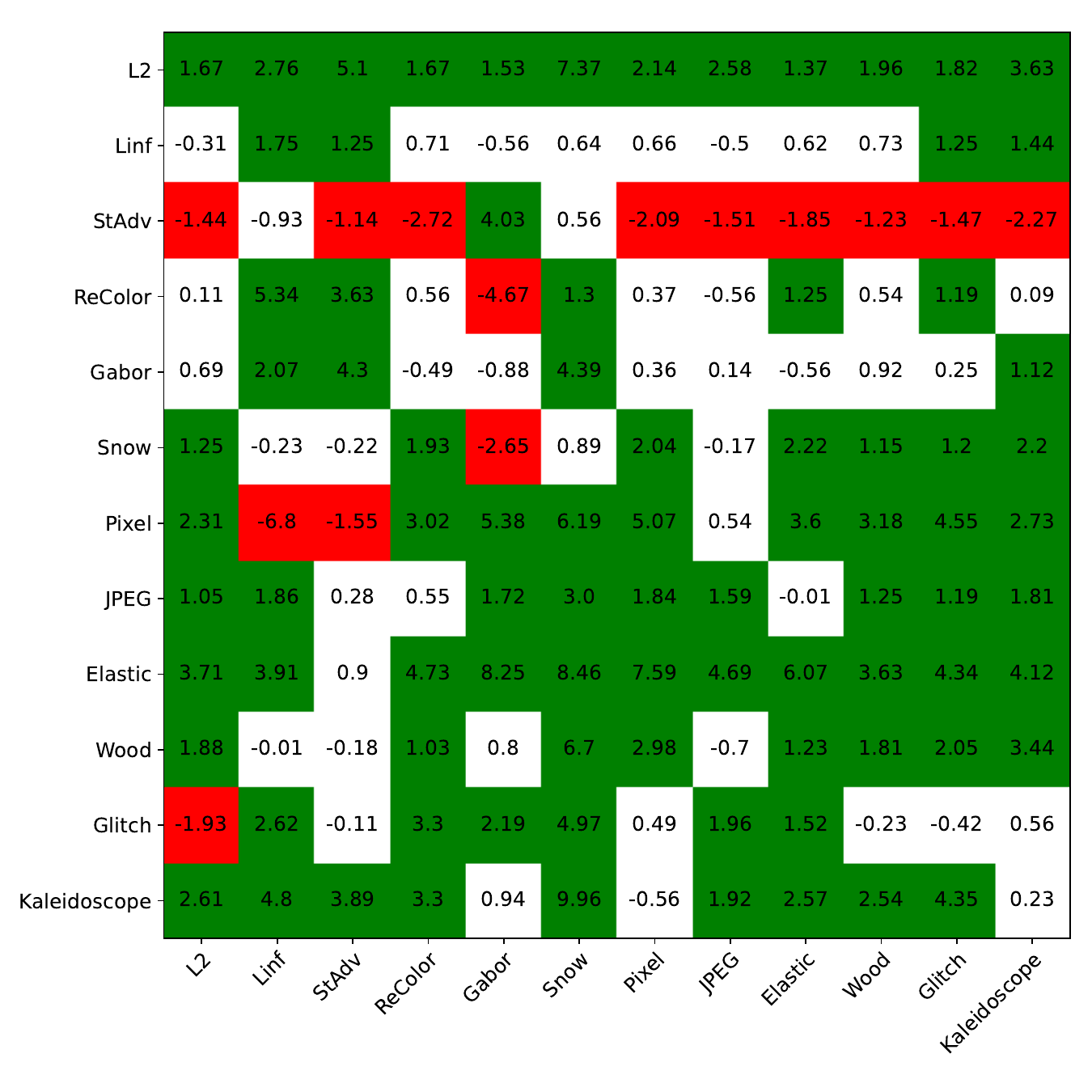}
        \vspace{-10pt}
        \caption{Difference in Avg Acc}
    \end{subfigure}%
    \begin{subfigure}[t]{0.4\textwidth}
        \centering
        \includegraphics[width=0.95\textwidth]{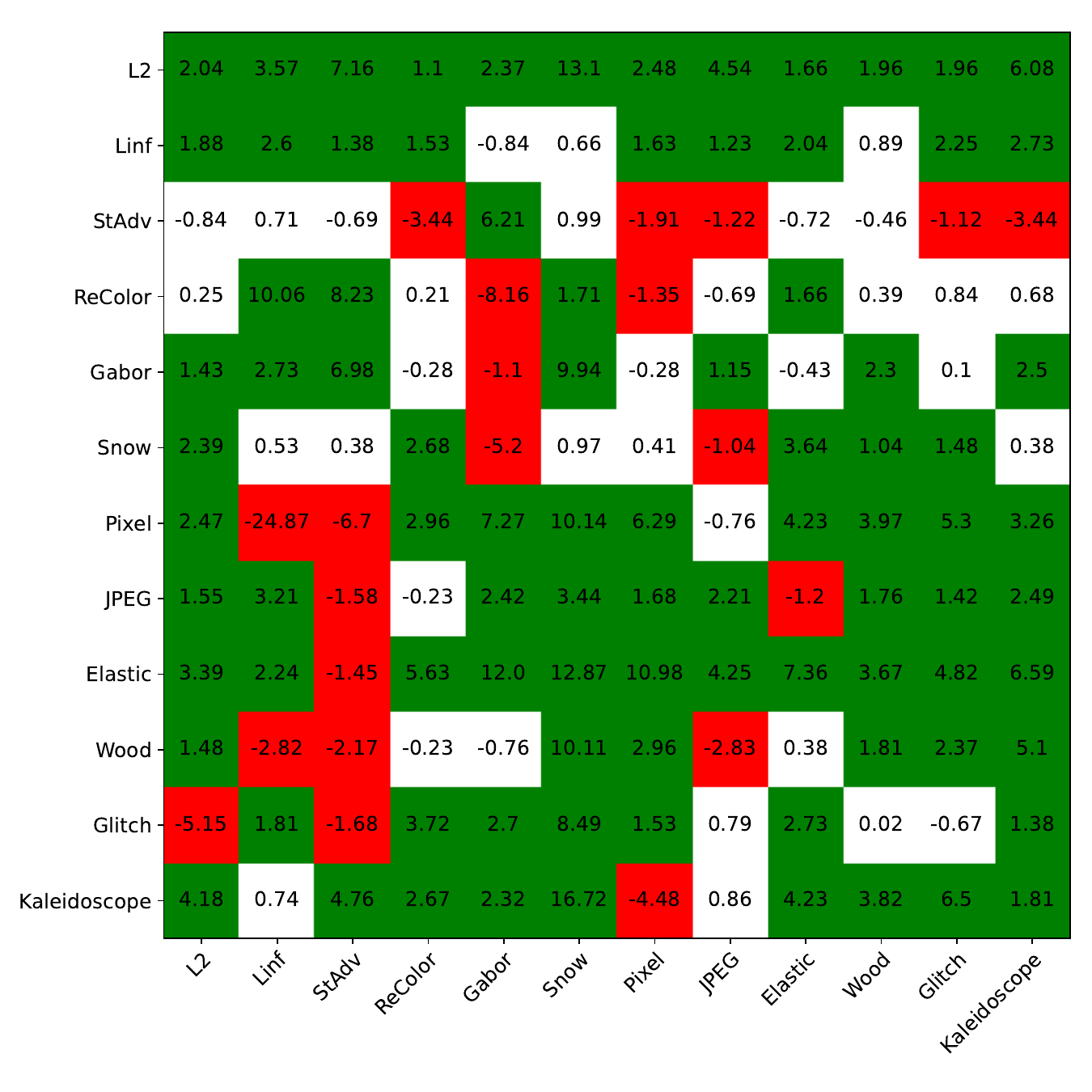}
        \vspace{-10pt}
        \caption{Difference in Union Acc}
    \end{subfigure}
    \begin{subfigure}[t]{0.4\textwidth}
        \centering
        \includegraphics[width=0.95\textwidth]{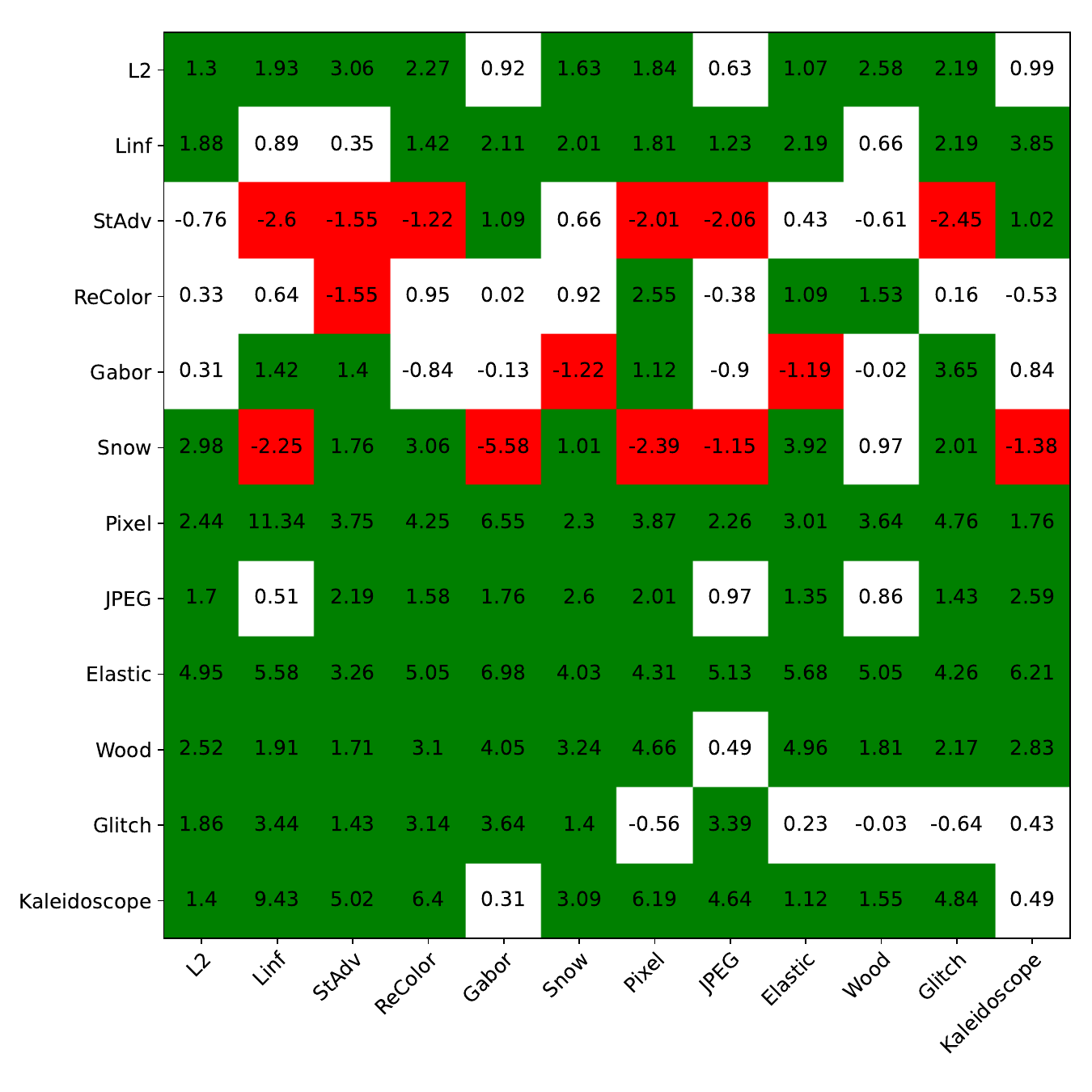}
        \vspace{-10pt}
        \caption{Difference in Initial Attack Acc}
    \end{subfigure}%
    \begin{subfigure}[t]{0.4\textwidth}
        \centering
        \includegraphics[width=0.95\textwidth]{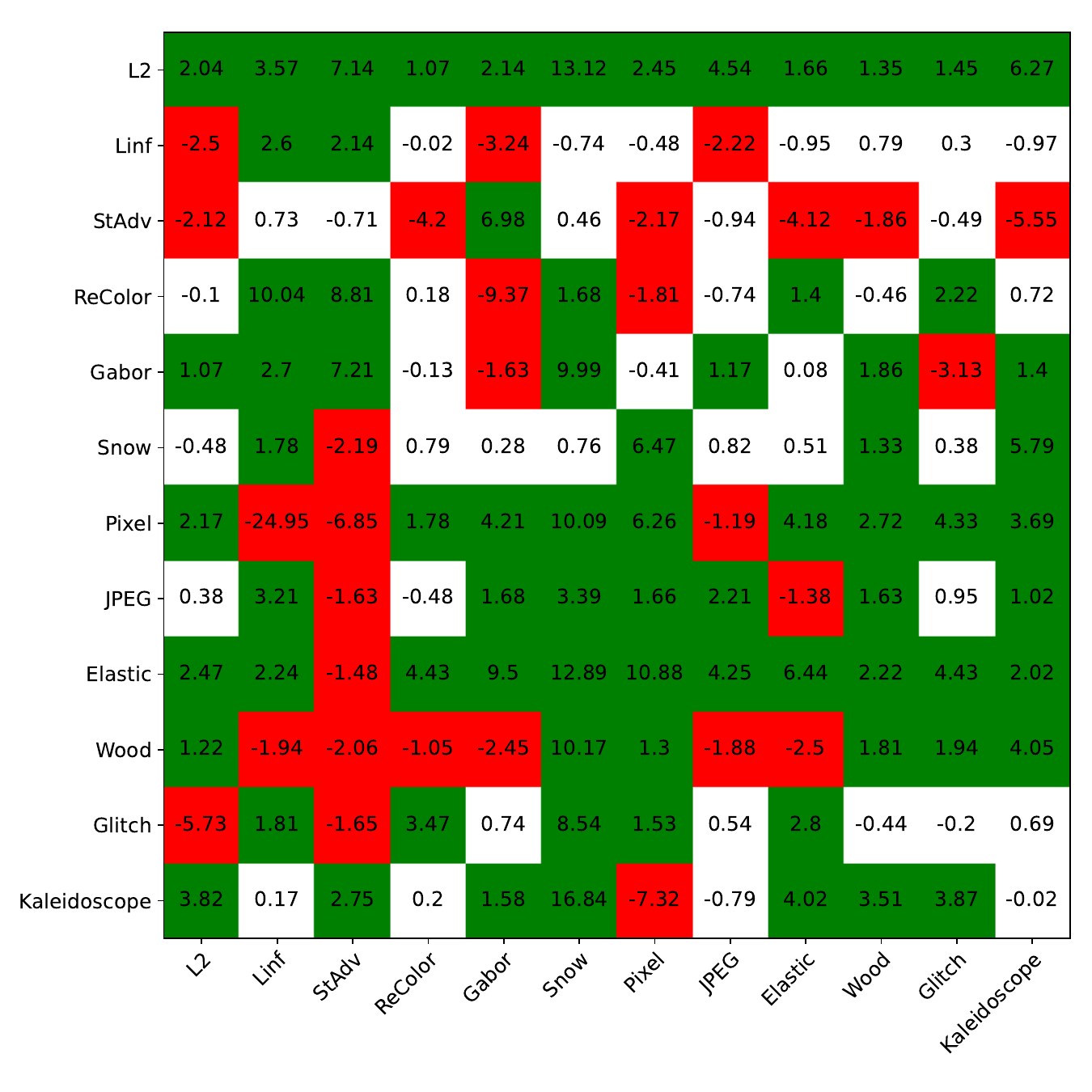}
        \vspace{-10pt}
        \caption{Difference in New Attack Acc}
    \end{subfigure}
    \begin{subfigure}[t]{0.4\textwidth}
        \centering
        \includegraphics[width=0.95\textwidth]{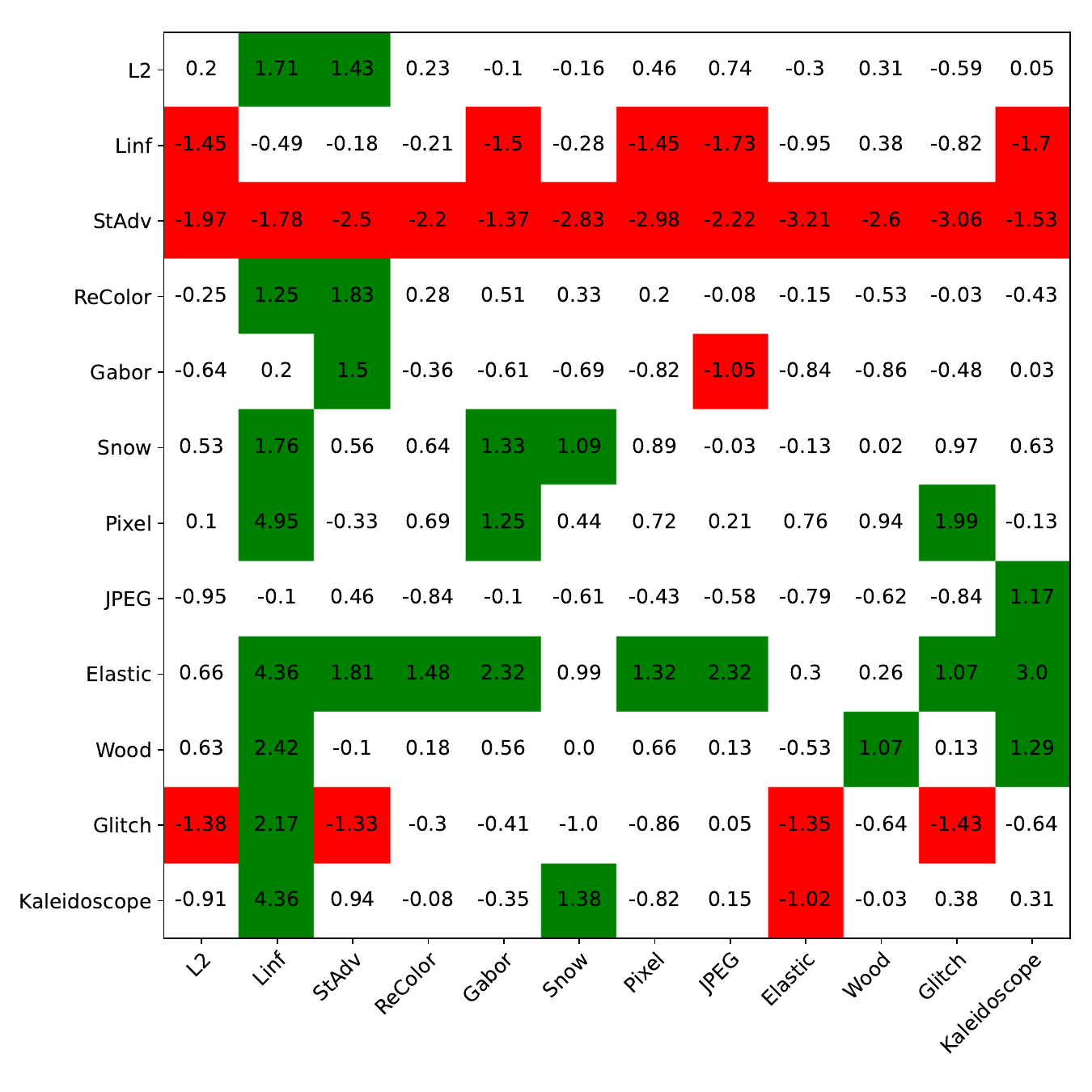}
        \vspace{-10pt}
        \caption{Difference in Clean Acc}
    \end{subfigure}
    \caption{\textbf{Change in robust accuracy after fine-tuning with models initally trained with adversarial $\ell_2$ regularization different initial attack and new attack pairs.}  We fine-tune models on Imagenette across 144 pairs of initial attack and new attack.  The initial attack corresponds to the row of each grid and new attack corresponds to each column.  Values represent differences between the accuracy measured on a model fine-tuned with and without regularization in initial training.  Gains in accuracy of at least 1\% are highlighted in green, while drops in accuracy of at least 1\% are highlighted in red. % See Appendix \ref{app:fine-tuning_pairs} for experimental setup details.
    }
    \label{fig:imagenette_finetune_abl_froml2}
\end{figure*}

\begin{figure*}[t!]
    \centering
    \begin{subfigure}[t]{0.4\textwidth}
        \centering
        \includegraphics[width=0.95\textwidth]{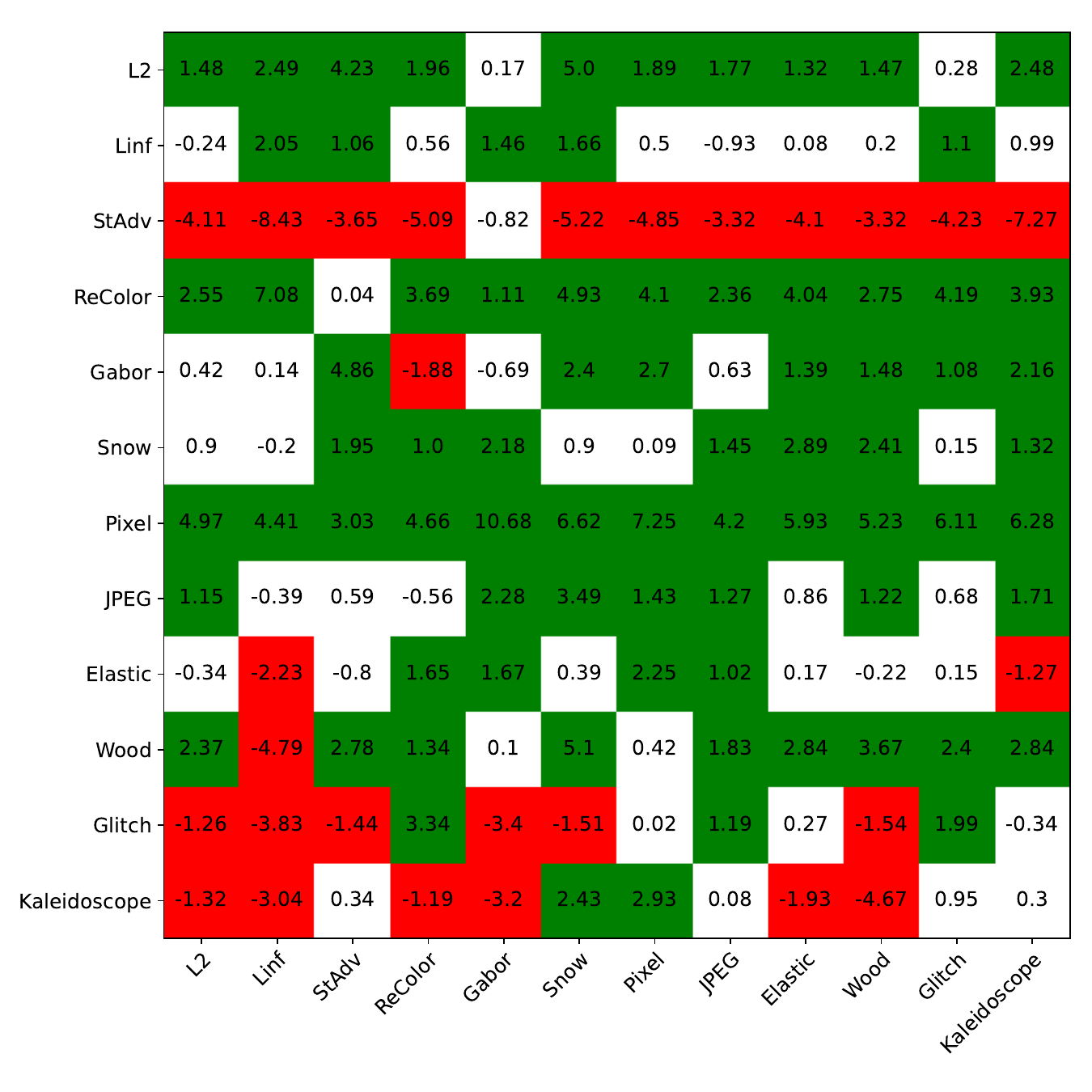}
        \vspace{-10pt}
        \caption{Difference in Avg Acc}
    \end{subfigure}%
    \begin{subfigure}[t]{0.4\textwidth}
        \centering
        \includegraphics[width=0.95\textwidth]{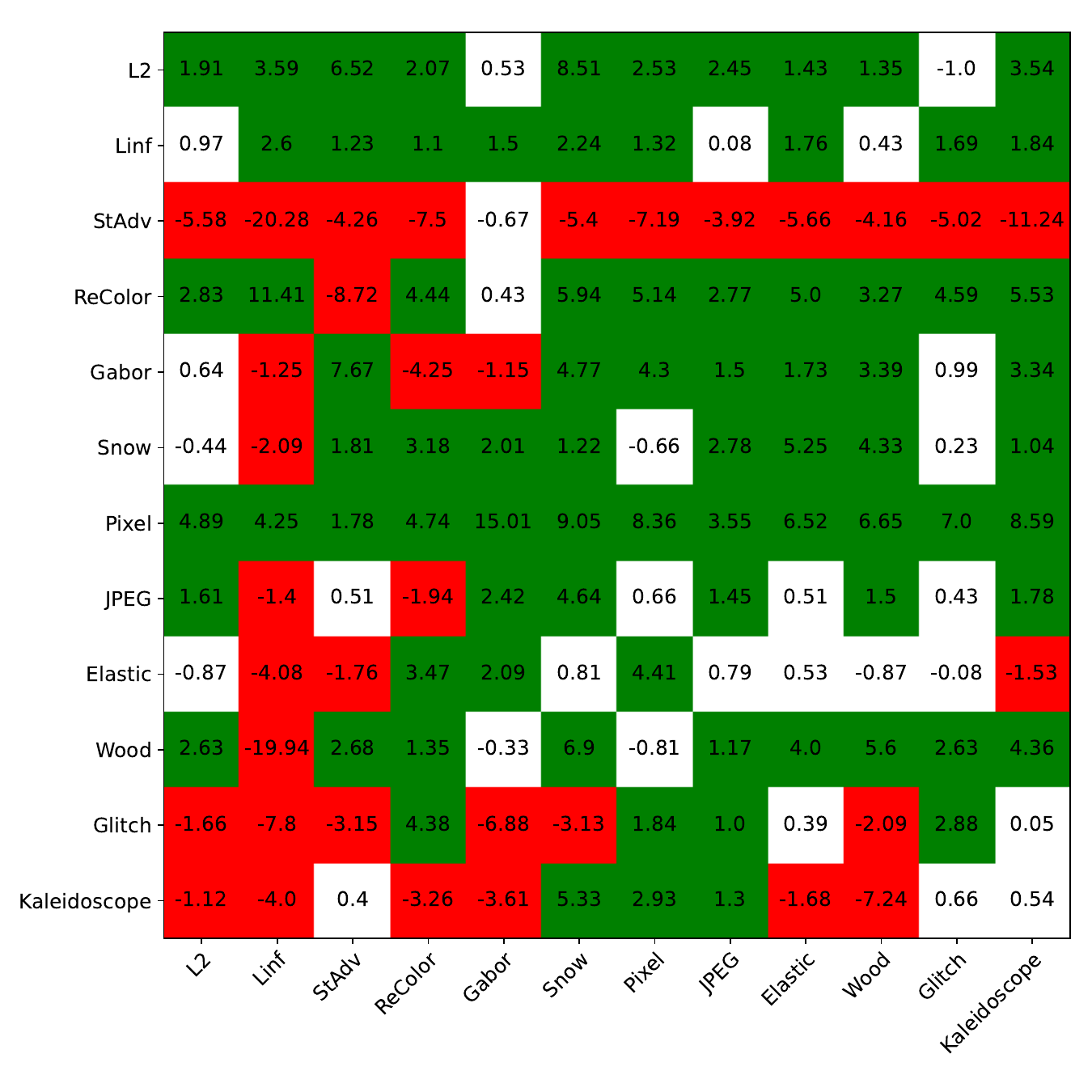}
        \vspace{-10pt}
        \caption{Difference in Union Acc}
    \end{subfigure}
    \begin{subfigure}[t]{0.4\textwidth}
        \centering
        \includegraphics[width=0.95\textwidth]{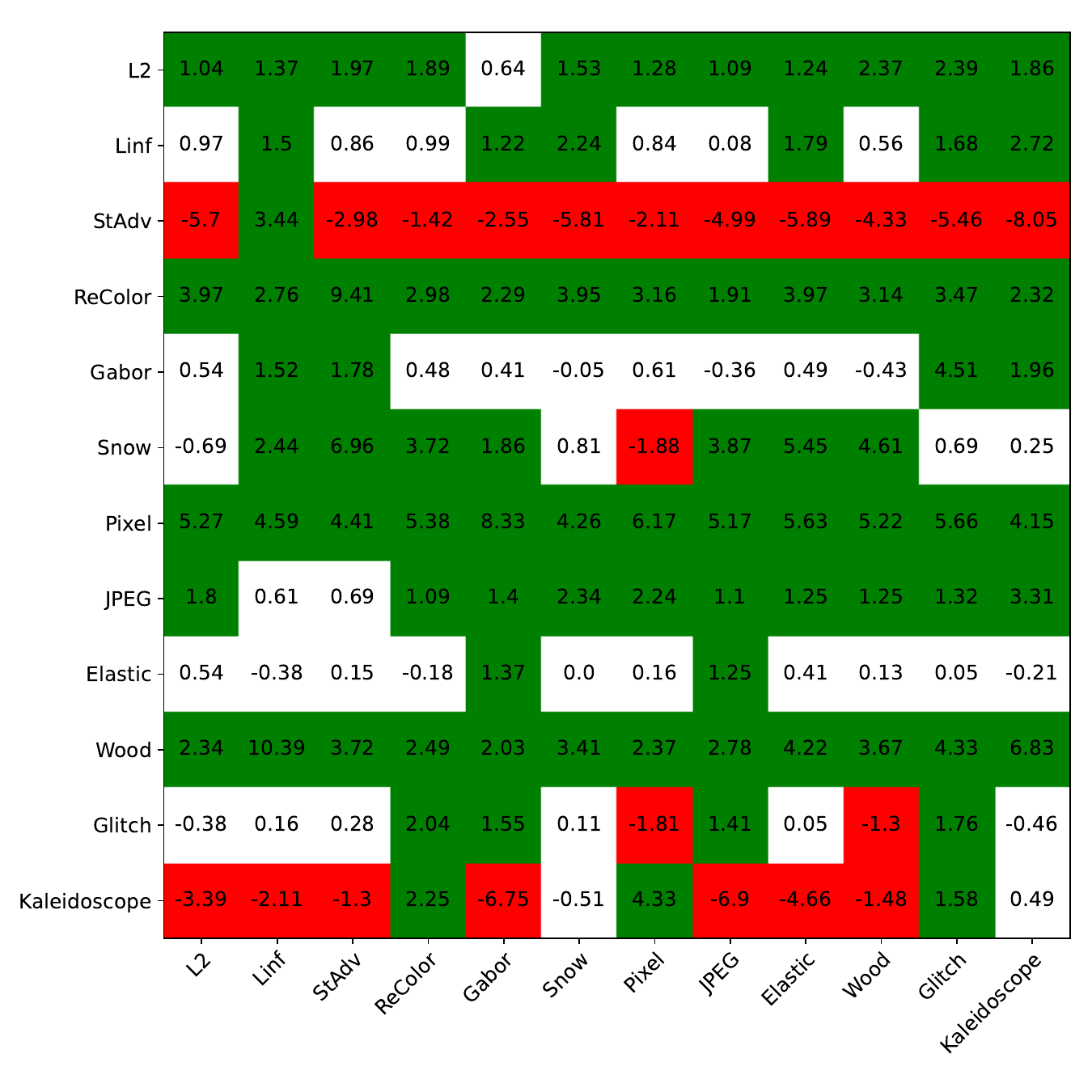}
        \vspace{-10pt}
        \caption{Difference in Initial Attack Acc}
    \end{subfigure}%
    \begin{subfigure}[t]{0.4\textwidth}
        \centering
        \includegraphics[width=0.95\textwidth]{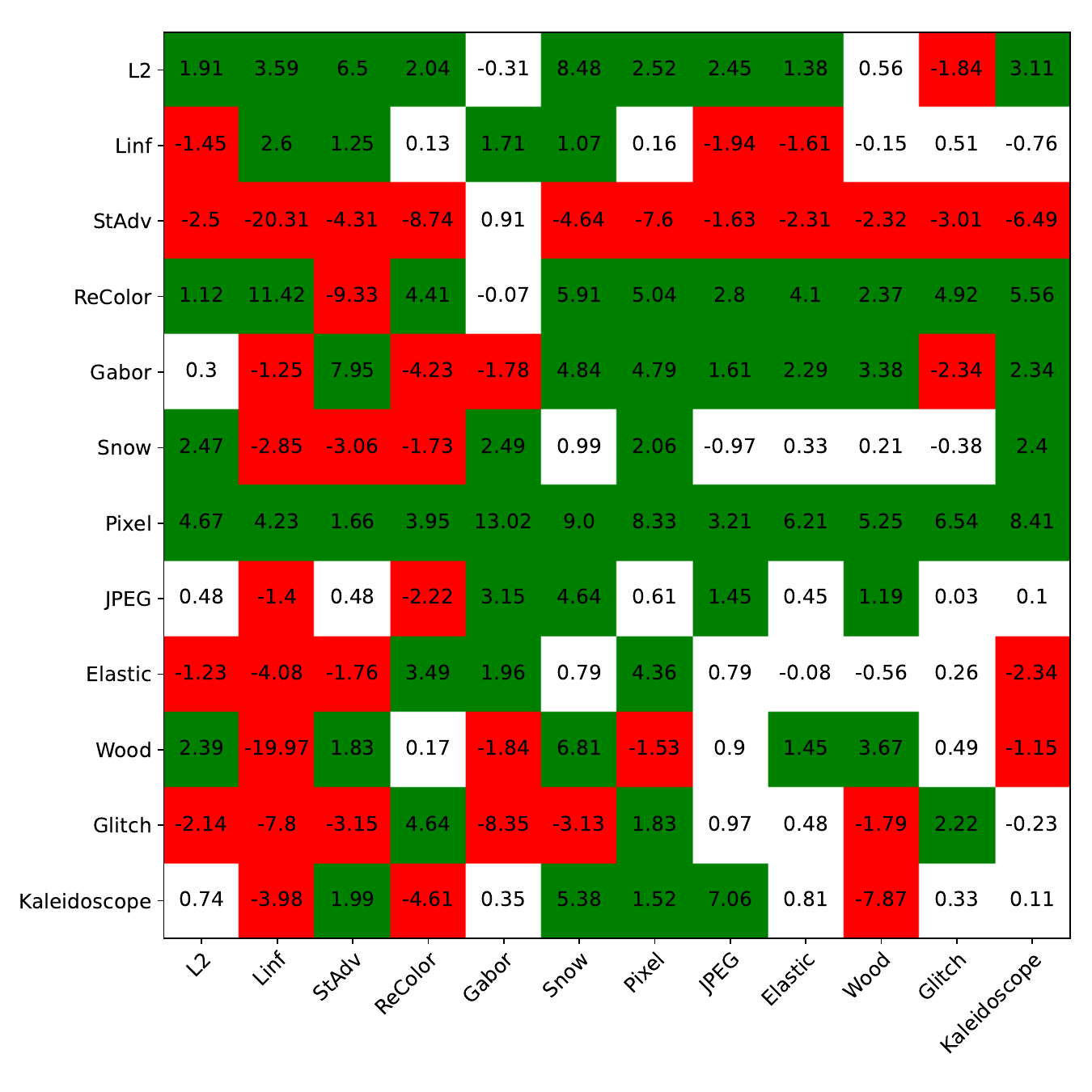}
        \vspace{-10pt}
        \caption{Difference in New Attack Acc}
    \end{subfigure}
    \begin{subfigure}[t]{0.4\textwidth}
        \centering
        \includegraphics[width=0.95\textwidth]{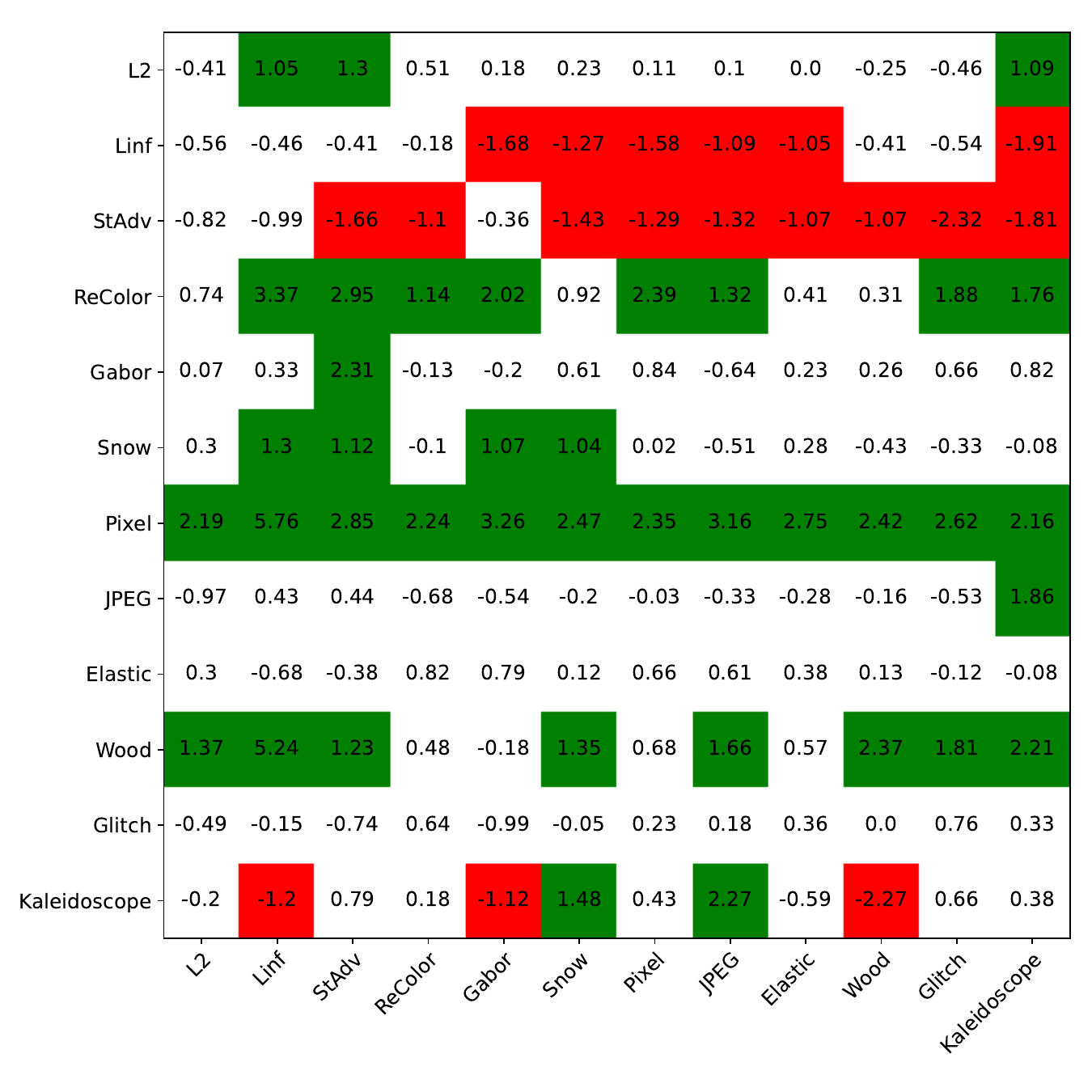}
        \vspace{-10pt}
        \caption{Difference in Clean Acc}
    \end{subfigure}
    \caption{\textbf{Change in robust accuracy after fine-tuning with models initally trained with variation regularization different initial attack and new attack pairs.}  We fine-tune models on Imagenette across 144 pairs of initial attack and new attack.  The initial attack corresponds to the row of each grid and new attack corresponds to each column.  Values represent differences between the accuracy measured on a model fine-tuned with and without regularization in initial training.  Gains in accuracy of at least 1\% are highlighted in green, while drops in accuracy of at least 1\% are highlighted in red. % See Appendix \ref{app:fine-tuning_pairs} for experimental setup details.
    }
    \label{fig:imagenette_finetune_abl_fromvarreg}
\end{figure*}

\begin{figure*}[t!]
    \centering
    \begin{subfigure}[t]{0.4\textwidth}
        \centering
        \includegraphics[width=0.95\textwidth]{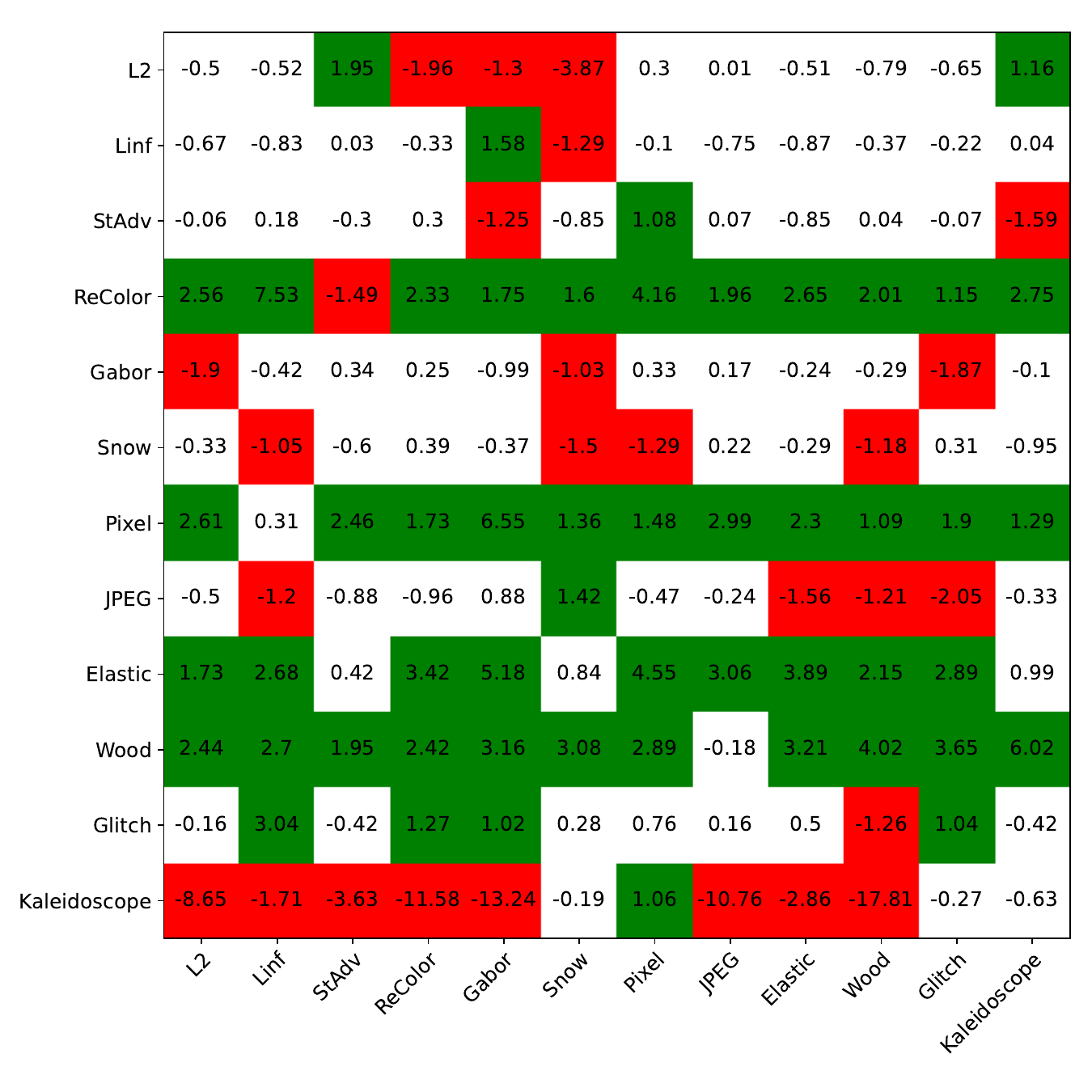}
        \vspace{-10pt}
        \caption{Difference in Avg Acc}
    \end{subfigure}%
    \begin{subfigure}[t]{0.4\textwidth}
        \centering
        \includegraphics[width=0.95\textwidth]{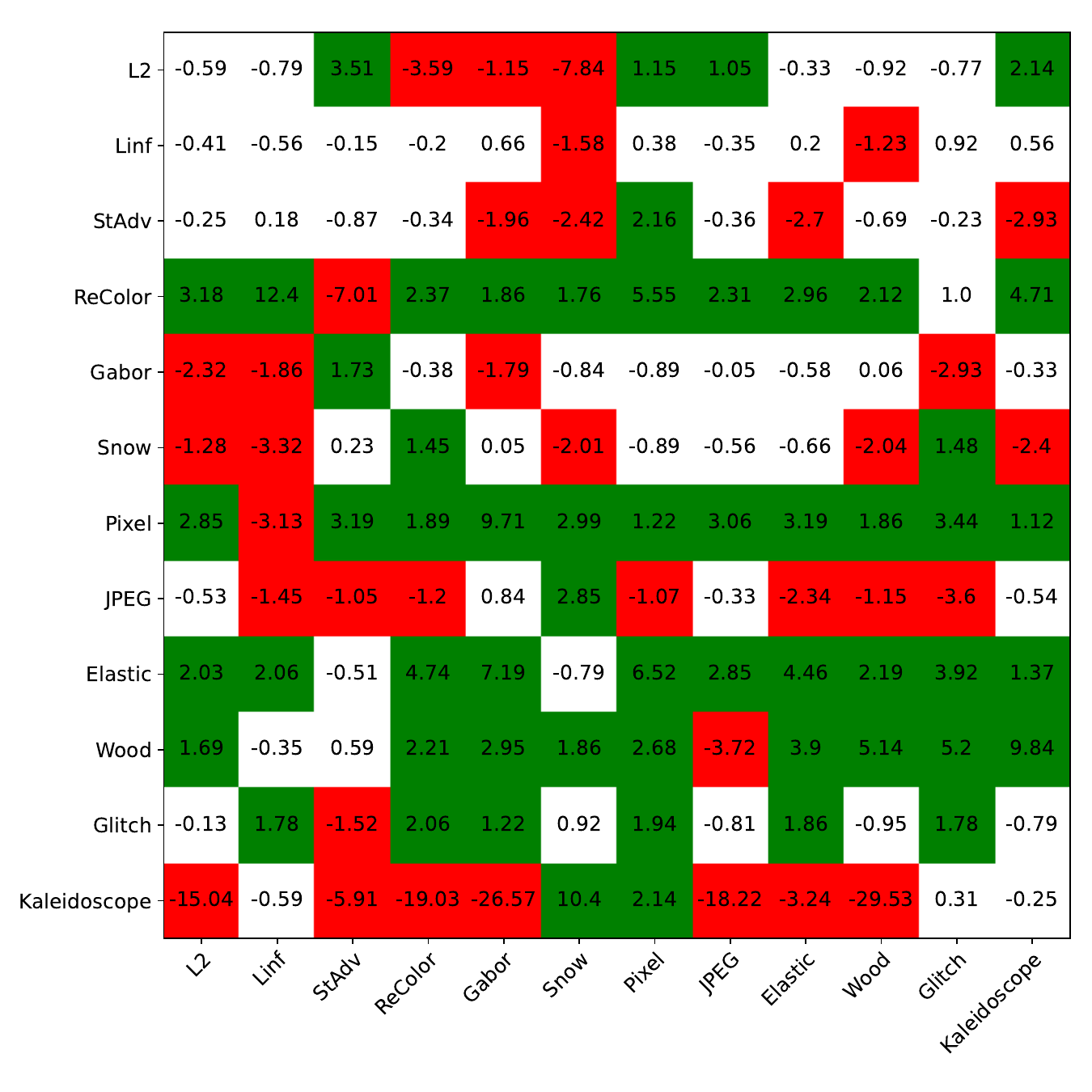}
        \vspace{-10pt}
        \caption{Difference in Union Acc}
    \end{subfigure}
    \begin{subfigure}[t]{0.4\textwidth}
        \centering
        \includegraphics[width=0.95\textwidth]{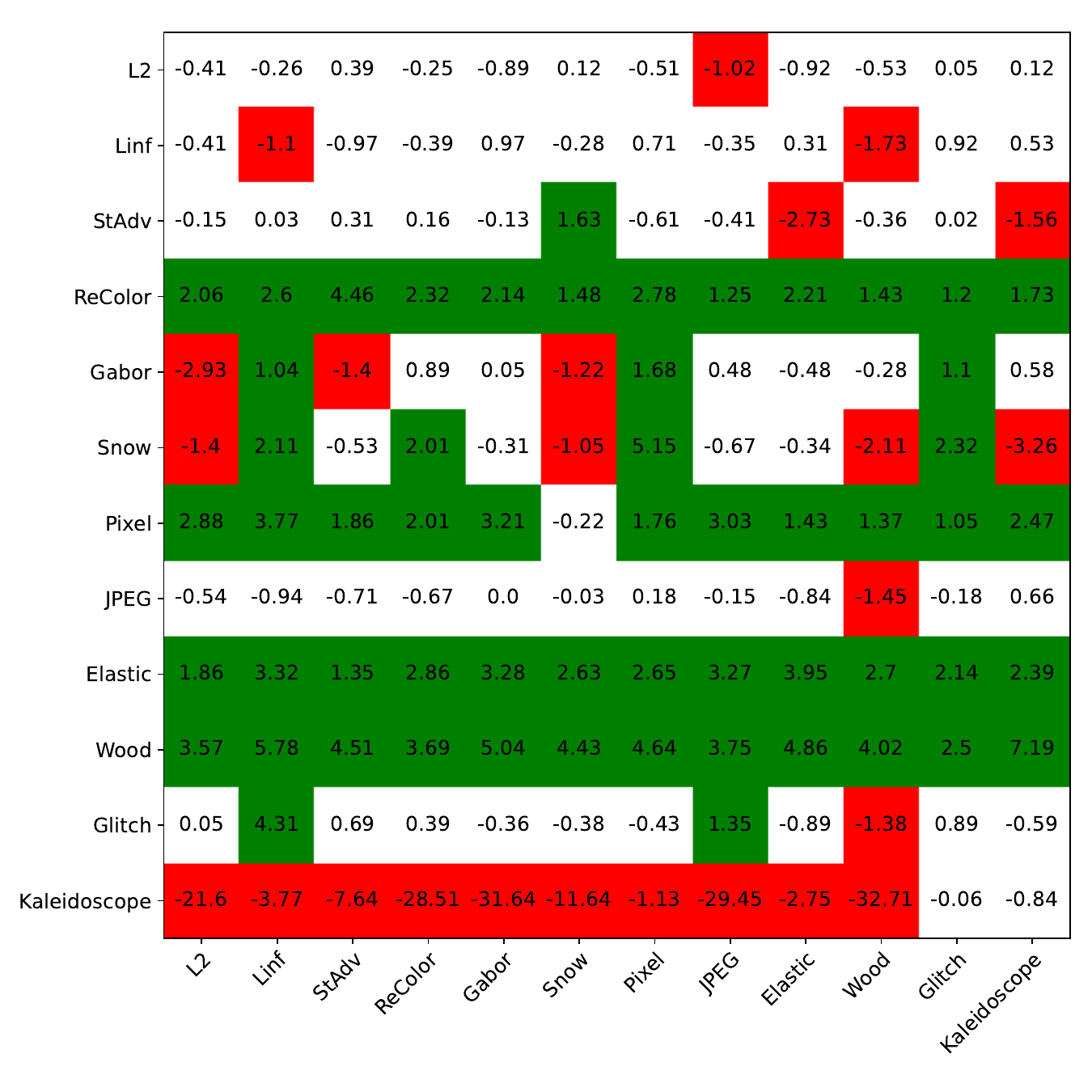}
        \vspace{-10pt}
        \caption{Difference in Initial Attack Acc}
    \end{subfigure}%
    \begin{subfigure}[t]{0.4\textwidth}
        \centering
        \includegraphics[width=0.95\textwidth]{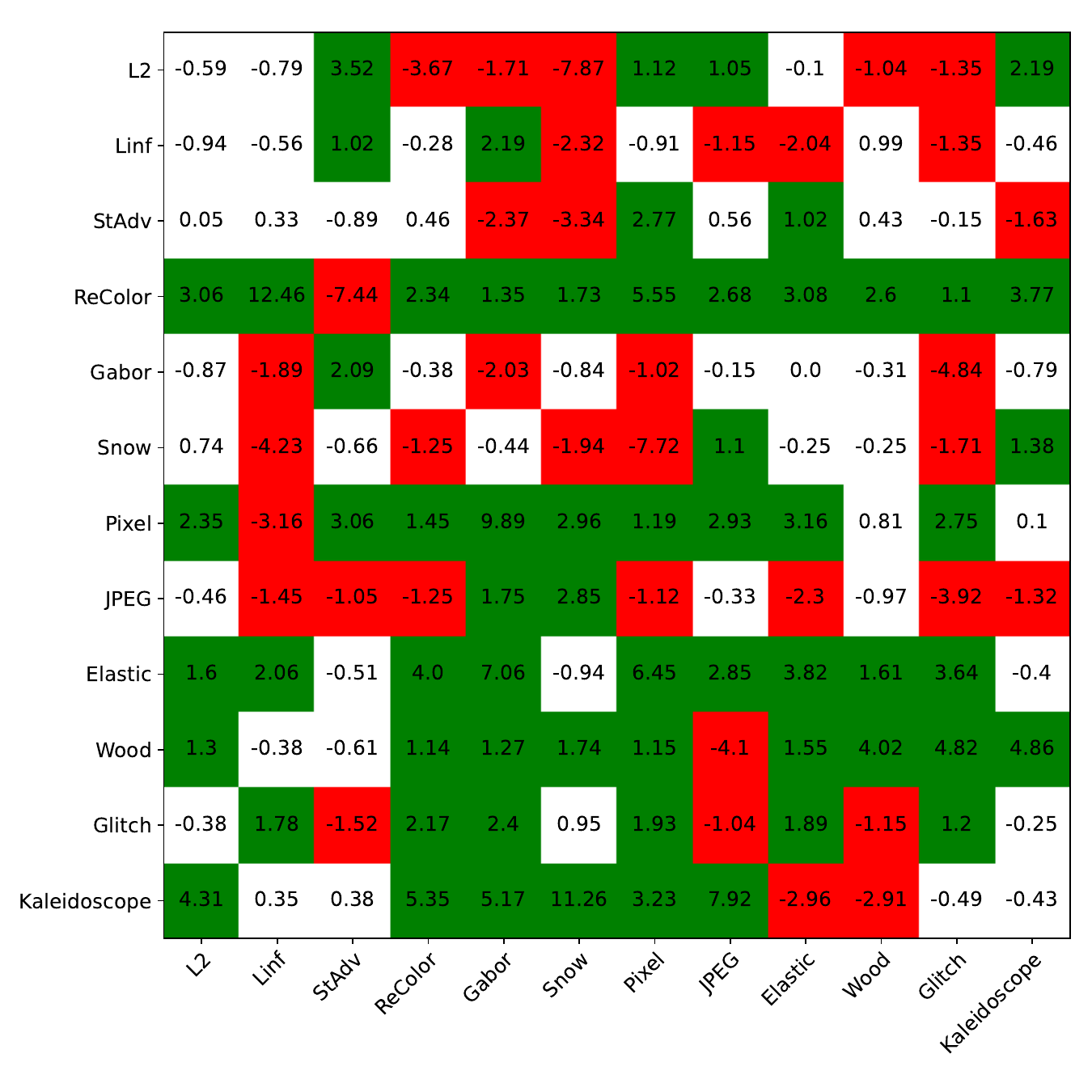}
        \vspace{-10pt}
        \caption{Difference in New Attack Acc}
    \end{subfigure}
    \begin{subfigure}[t]{0.4\textwidth}
        \centering
        \includegraphics[width=0.95\textwidth]{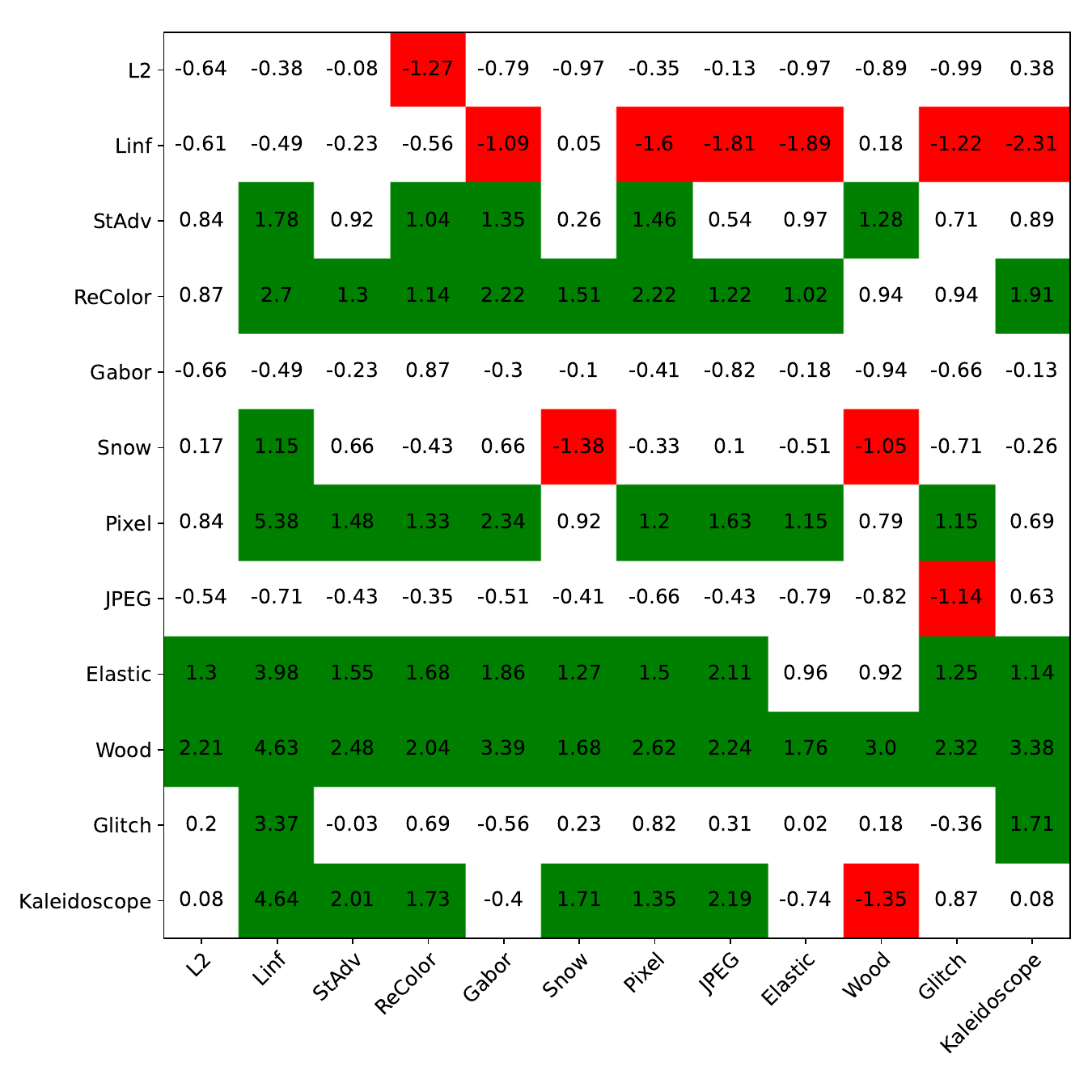}
        \vspace{-10pt}
        \caption{Difference in Clean Acc}
    \end{subfigure}
    \caption{\textbf{Change in robust accuracy after fine-tuning with models initally trained with uniform regularization different initial attack and new attack pairs.}  We fine-tune models on Imagenette across 144 pairs of initial attack and new attack.  The initial attack corresponds to the row of each grid and new attack corresponds to each column.  Values represent differences between the accuracy measured on a model fine-tuned with and without regularization in initial training.  Gains in accuracy of at least 1\% are highlighted in green, while drops in accuracy of at least 1\% are highlighted in red. % See Appendix \ref{app:fine-tuning_pairs} for experimental setup details.
    }
    \label{fig:imagenette_finetune_abl_fromuniform}
\end{figure*}

\begin{figure*}[t!]
    \centering
    \begin{subfigure}[t]{0.4\textwidth}
        \centering
        \includegraphics[width=0.95\textwidth]{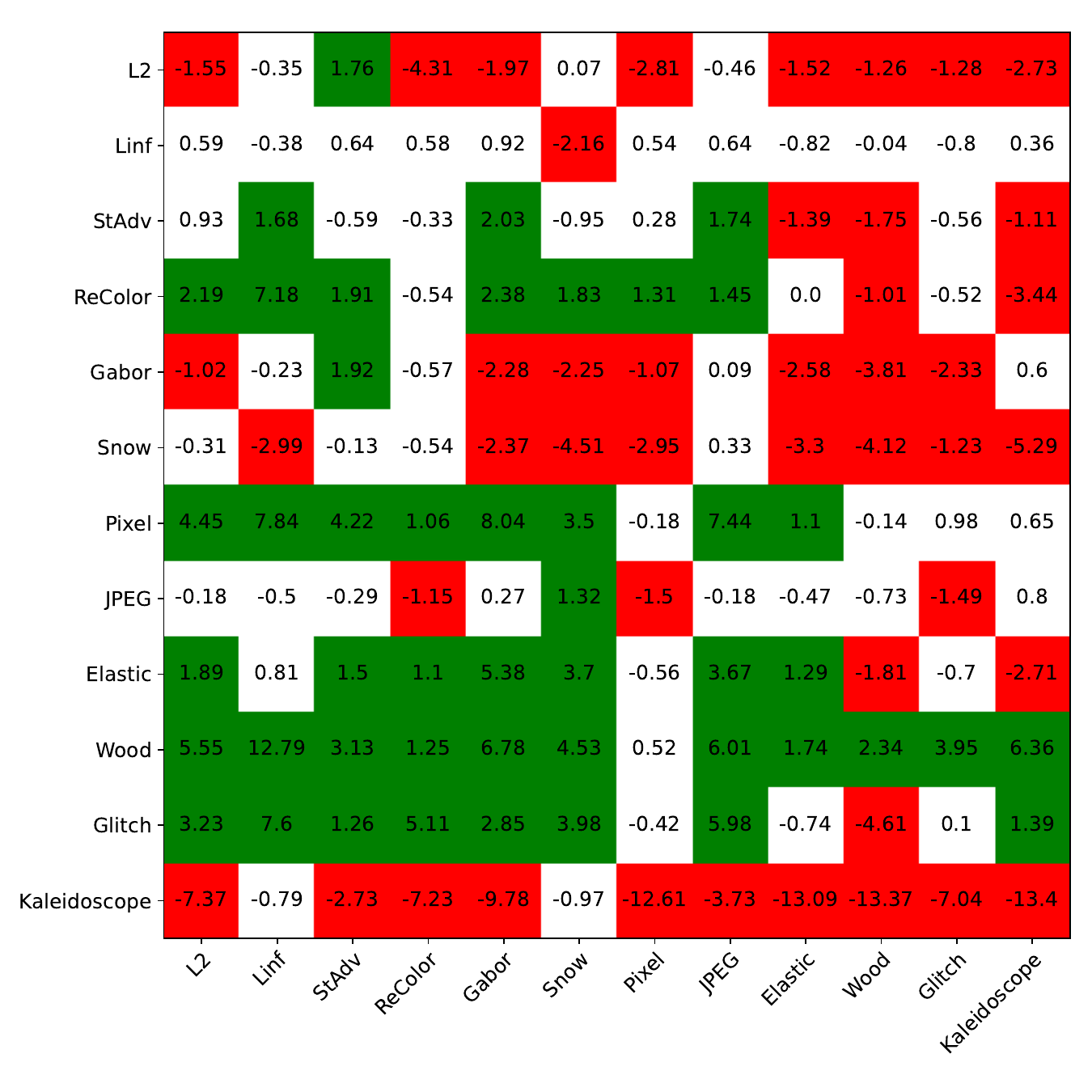}
        \vspace{-10pt}
        \caption{Difference in Avg Acc}
    \end{subfigure}%
    \begin{subfigure}[t]{0.4\textwidth}
        \centering
        \includegraphics[width=0.95\textwidth]{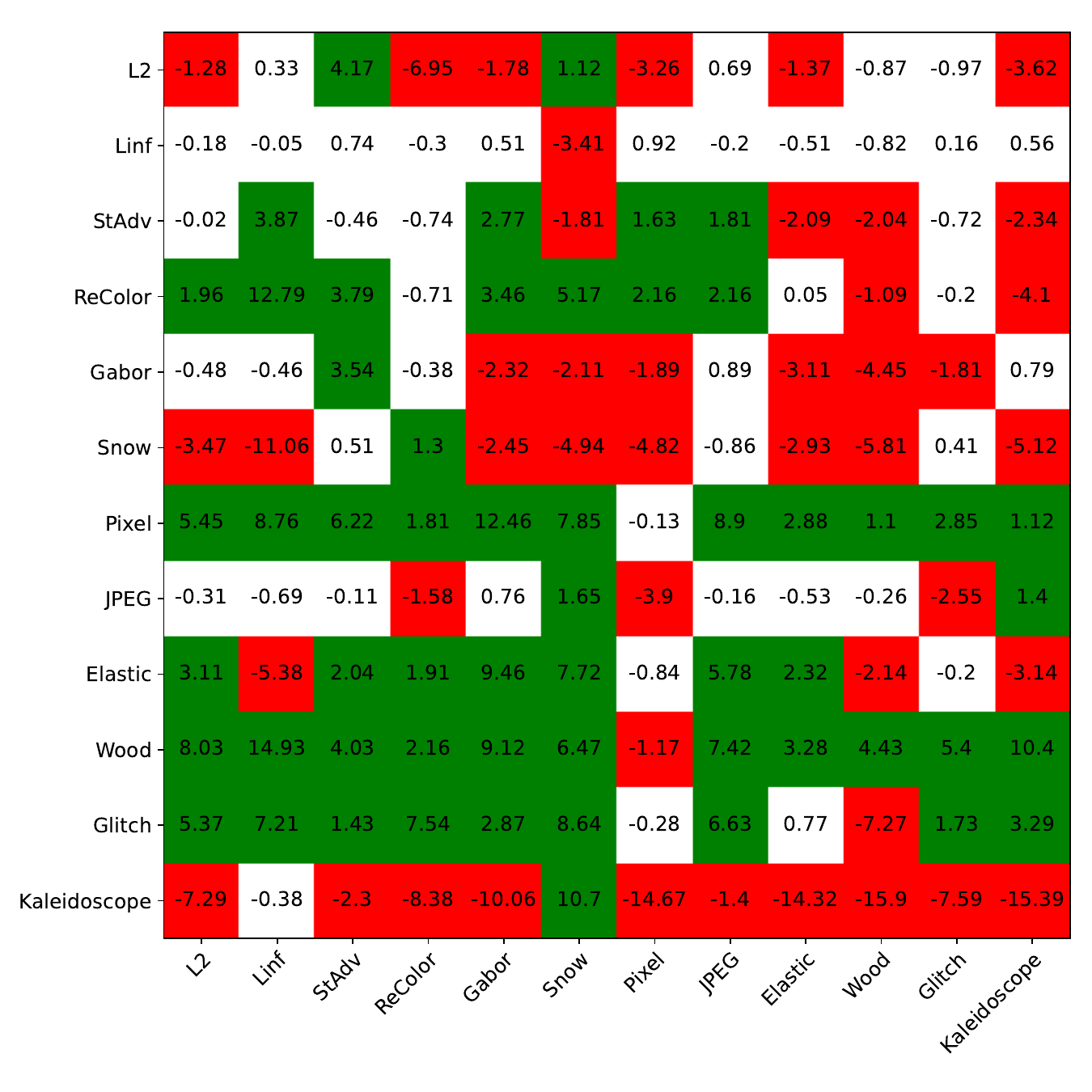}
        \vspace{-10pt}
        \caption{Difference in Union Acc}
    \end{subfigure}
    \begin{subfigure}[t]{0.4\textwidth}
        \centering
        \includegraphics[width=0.95\textwidth]{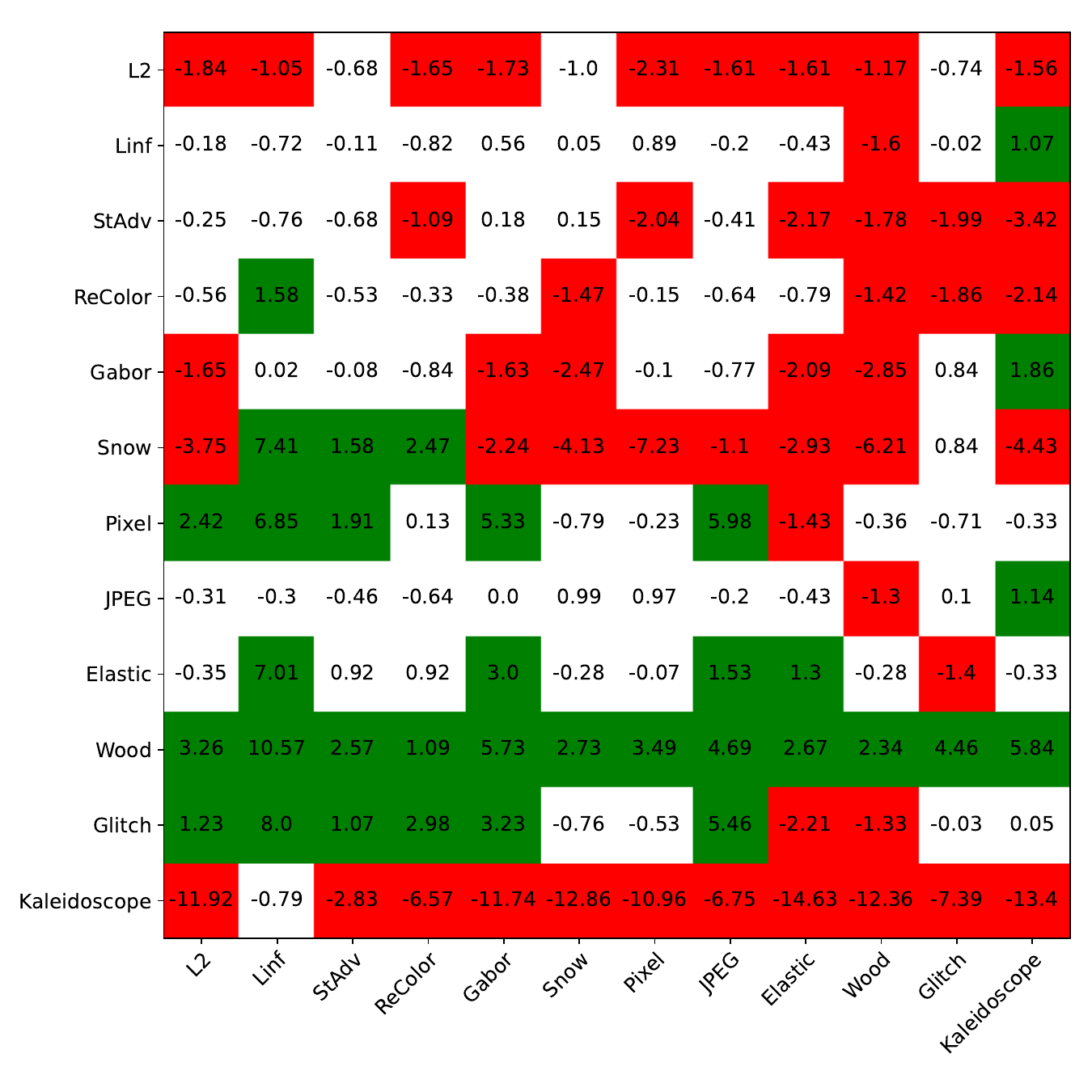}
        \vspace{-10pt}
        \caption{Difference in Initial Attack Acc}
    \end{subfigure}%
    \begin{subfigure}[t]{0.4\textwidth}
        \centering
        \includegraphics[width=0.95\textwidth]{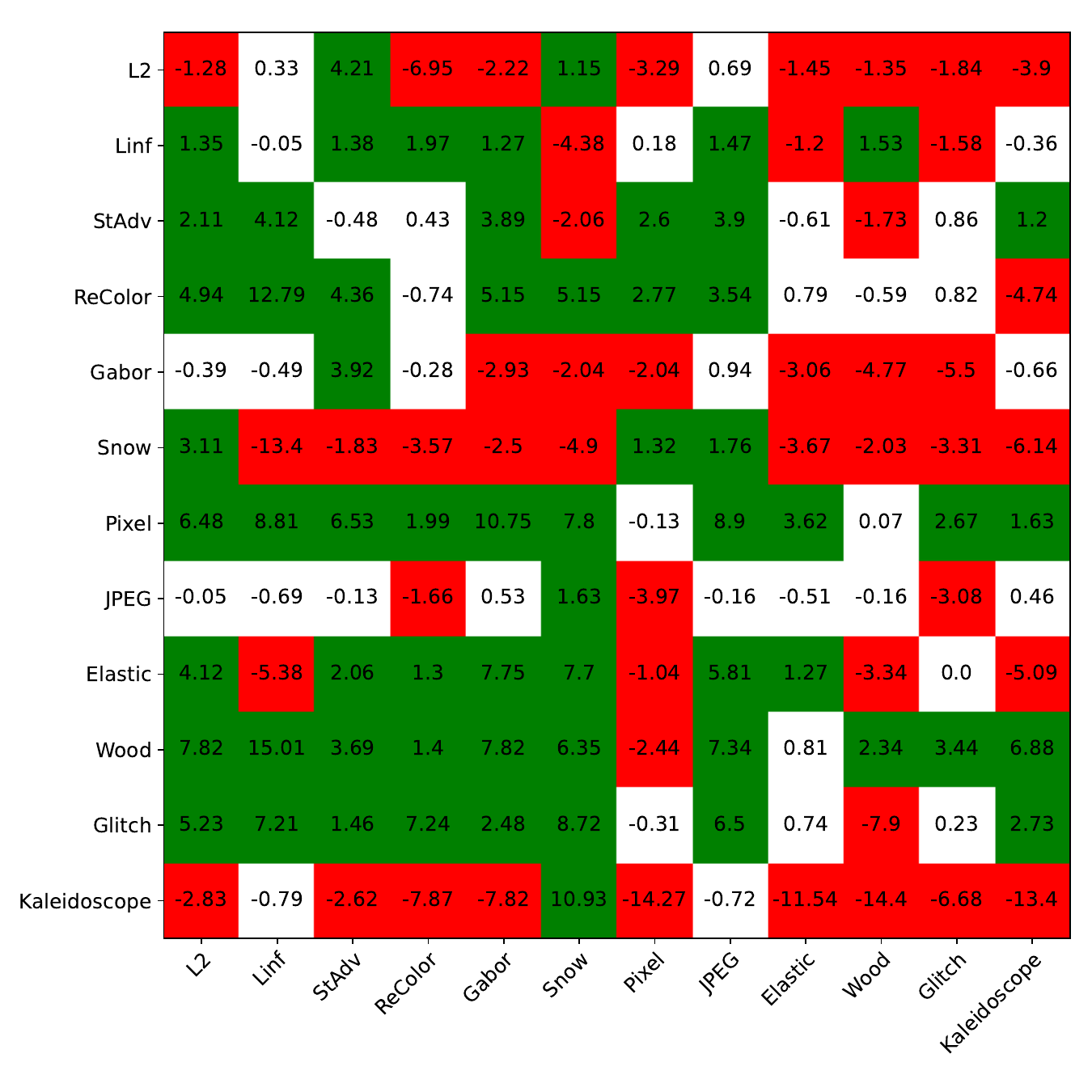}
        \vspace{-10pt}
        \caption{Difference in New Attack Acc}
    \end{subfigure}
    \begin{subfigure}[t]{0.4\textwidth}
        \centering
        \includegraphics[width=0.95\textwidth]{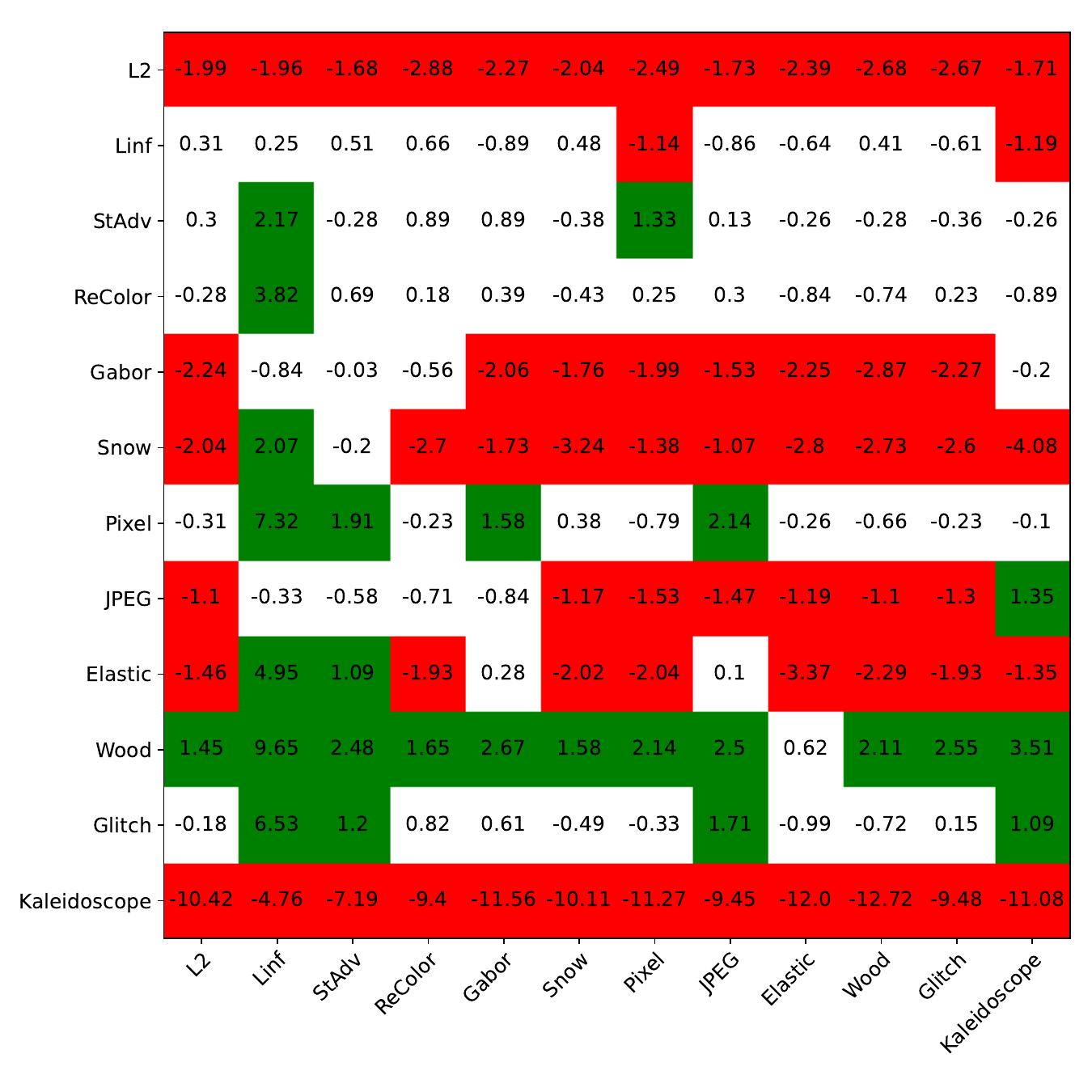}
        \vspace{-10pt}
        \caption{Difference in Clean Acc}
    \end{subfigure}
    \caption{\textbf{Change in robust accuracy after fine-tuning with models initally trained with Gaussian regularization different initial attack and new attack pairs.}  We fine-tune models on Imagenette across 144 pairs of initial attack and new attack.  The initial attack corresponds to the row of each grid and new attack corresponds to each column.  Values represent differences between the accuracy measured on a model fine-tuned with and without regularization in initial training.  Gains in accuracy of at least 1\% are highlighted in green, while drops in accuracy of at least 1\% are highlighted in red. % See Appendix \ref{app:fine-tuning_pairs} for experimental setup details.
    }
    \label{fig:imagenette_finetune_abl_fromgauss}
\end{figure*}

\section{Fine-tuning Ablations}
\label{app:fine-tuning}
\subsection{Impact of starting and new attack pairs}
\label{app:fine-tuning_pairs}
Similar to \cref{app:fine-tuning_pairs_init_train}, we ablate over starting and new attack pairs in finetuning.  In this section, we address the question: does regularization in fine-tuning generally lead to more robust models?  We follow the same setup as in \cref{app:fine-tuning_pairs_init_train} but  we compare models fine-tuned with regularization (with no regularization in pretraining) to models fine-tuned without regularization (with no regularization in pretraining).  We present the differences in average accuracy across the 2 attacks, union accuracy across the 2 attacks, accuracy on the starting attack, accuracy on the new attack, and clean accuracy between the 2 settings for adversarial $\ell_2$ regularization (with $\lambda=0.5$) in Figure \ref{fig:imagenette_finetune_abl} and for uniform regularization (with $\sigma=2$ and $\lambda=1$) in Figure \ref{fig:imagenette_finetune_abl_uniform}.  In these figures, we highlight gains in accuracy larger than 1\% in green and drops in accuracy larger than 1\% in red.

\textbf{Adversarial $\ell_2$ regularization in fine-tuning generally improves performance but trades off clean accuracy. }From Figure \ref{fig:imagenette_finetune_abl}, we can see that for many pairs of initial and new attack, regularization leads to improvements in union accuracy, average accuracy, and new attack accuracy.  However, this comes at a clear tradeoff with clean accuracy.  For accuracy on the initial attack, it is difficult to see clear trends; depending on threat models there can be gains in robustness or drops in robustness.  For example, when the new attack is $\ell_\infty$, we find that the initial attack accuracy generally drops.  We find that variation regularization can also lead to gains in performance, but these gains are much less consistent than compared to adversarial $\ell_2$ regularization.

\textbf{Random noise regularization in fine-tuning hurts overall robustness.}  Unlike adversarial $\ell_2$ regularization which can improve performance when used in both initial training and regularization, we find that uniform and Gaussian regularization generally hurts average, union, initial attack, and new attack accuracies when incorporated in fine-tuning.  This suggests that while random noise based regularization may help with initial training (and unforeseen robustness), they do not necessarily help with continual adaptive robustness through fine-tuning.

\begin{figure*}[t!]
    \centering
    \begin{subfigure}[t]{0.4\textwidth}
        \centering
        \includegraphics[width=0.95\textwidth]{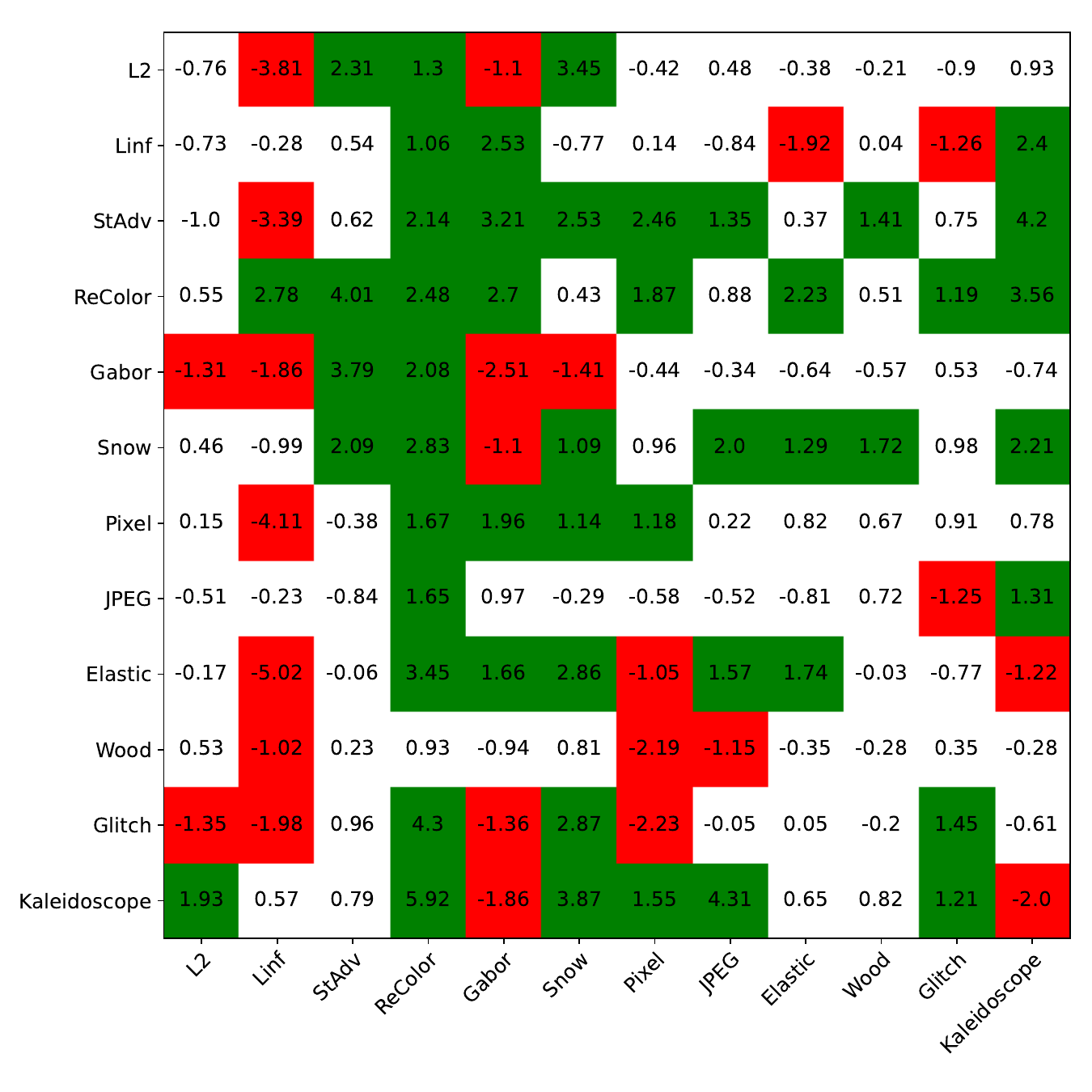}
        \vspace{-10pt}
        \caption{Difference in Avg Acc}
    \end{subfigure}%
    \begin{subfigure}[t]{0.4\textwidth}
        \centering
        \includegraphics[width=0.95\textwidth]{figures/imagenette_finetune_l2_reg_union.pdf}
        \vspace{-10pt}
        \caption{Difference in Union Acc}
    \end{subfigure}
    \begin{subfigure}[t]{0.4\textwidth}
        \centering
        \includegraphics[width=0.95\textwidth]{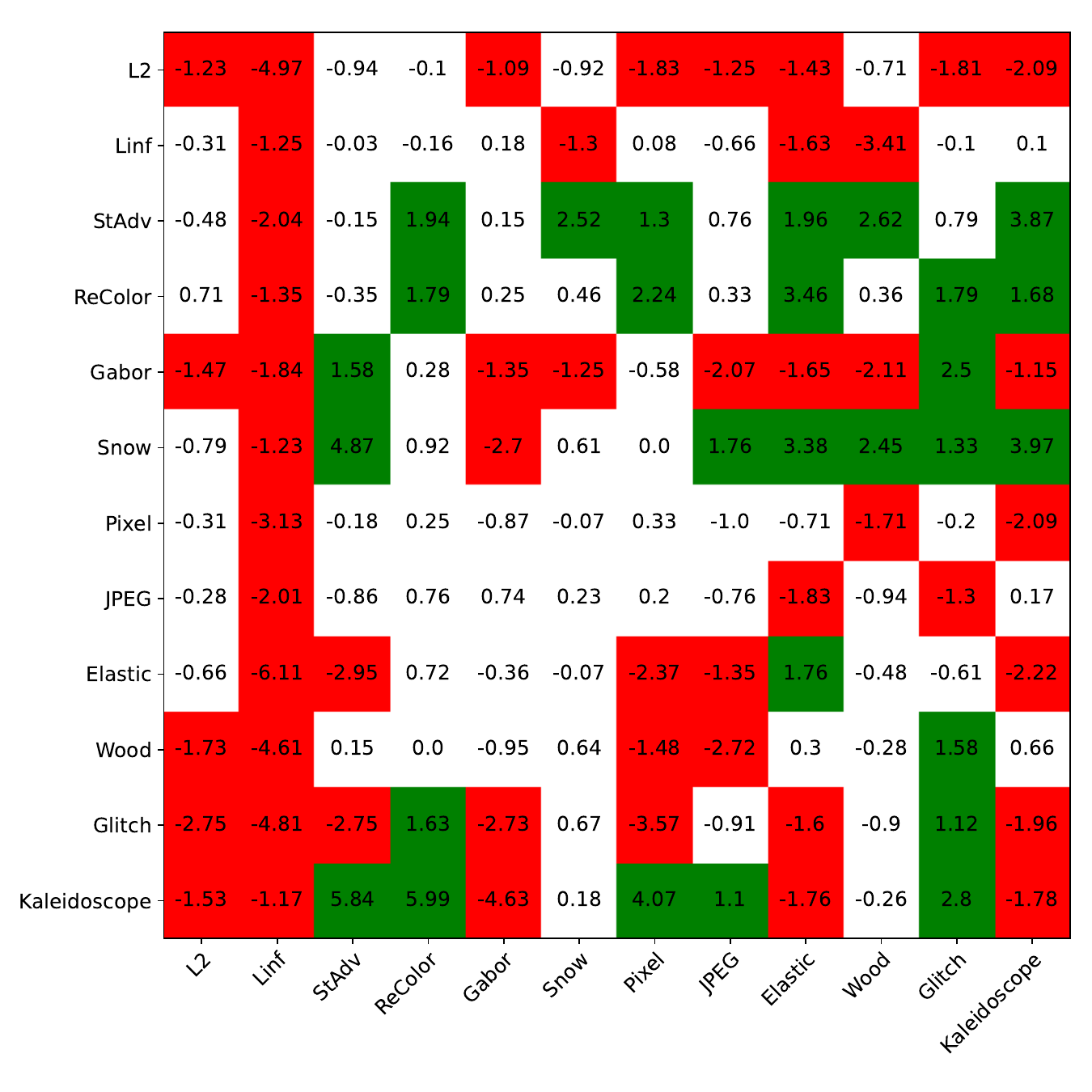}
        \vspace{-10pt}
        \caption{Difference in Initial Attack Acc}
    \end{subfigure}%
    \begin{subfigure}[t]{0.4\textwidth}
        \centering
        \includegraphics[width=0.95\textwidth]{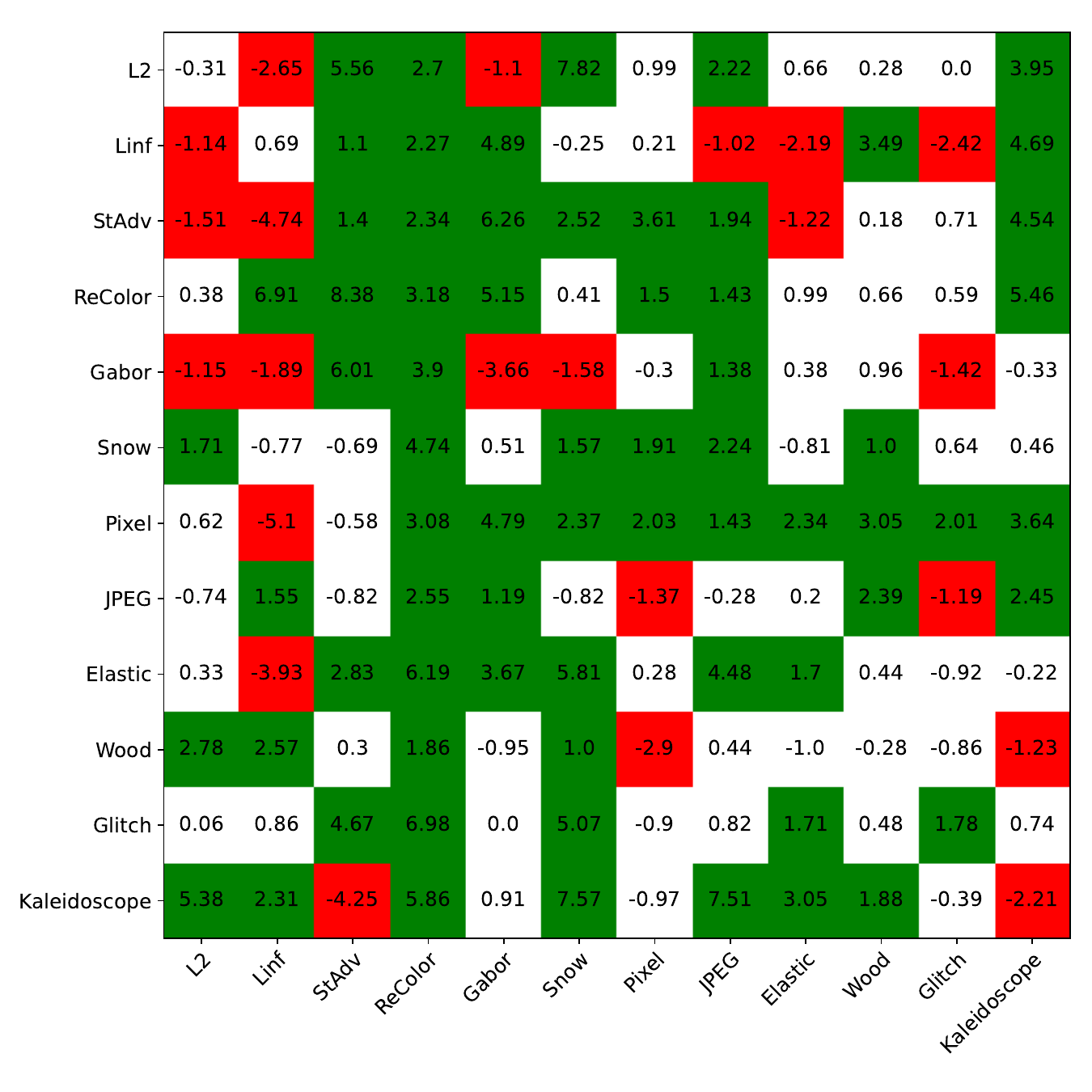}
        \vspace{-10pt}
        \caption{Difference in New Attack Acc}
    \end{subfigure}
    \begin{subfigure}[t]{0.4\textwidth}
        \centering
        \includegraphics[width=0.95\textwidth]{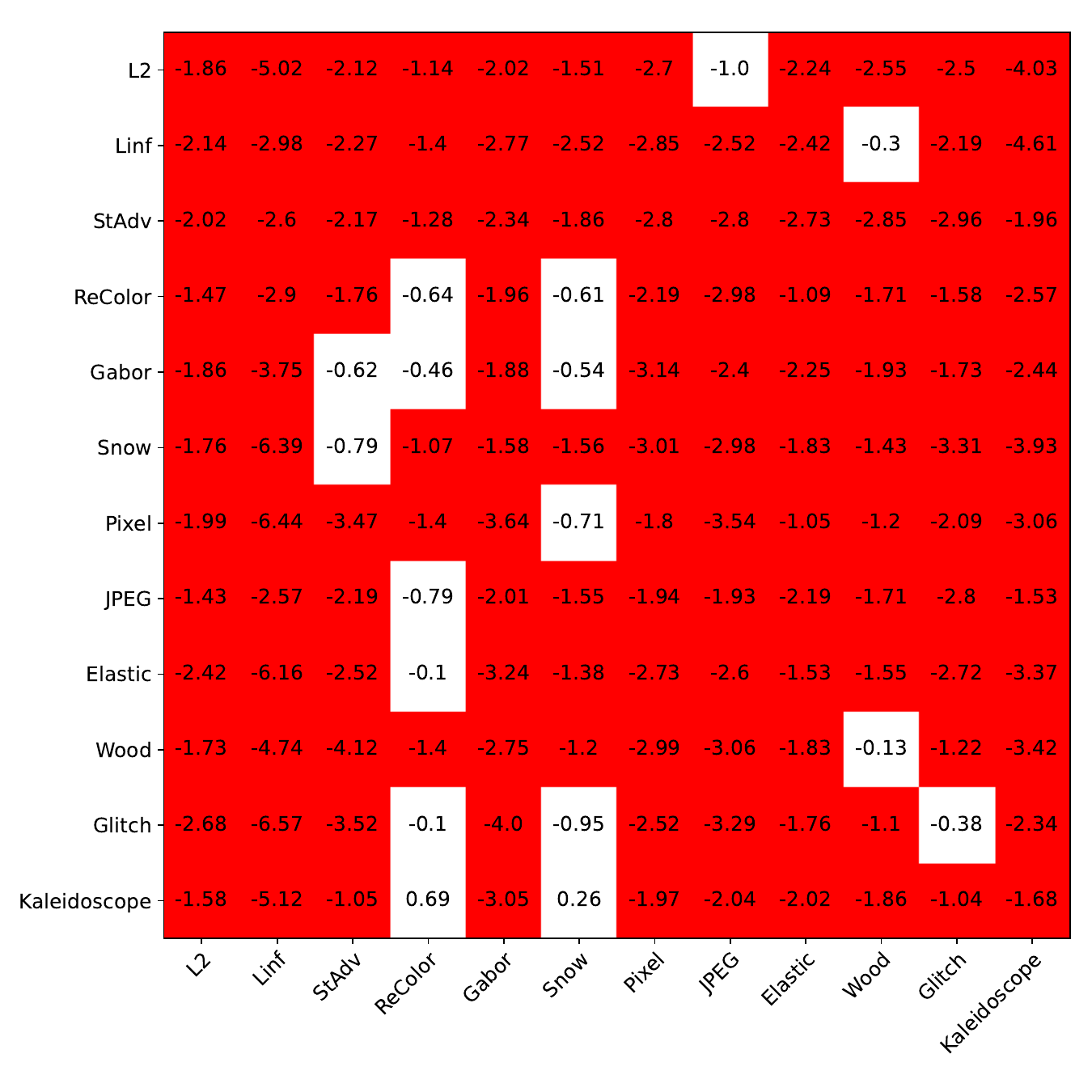}
        \vspace{-10pt}
        \caption{Difference in Clean Acc}
    \end{subfigure}
    \caption{\textbf{Change in robust accuracy after fine-tuning with adversarial $\ell_2$ regularization.}  We fine-tune models on Imagenette across 144 pairs of initial attack and new attack.  The initial attack corresponds to the row of each grid and new attack corresponds to each column.  Values represent differences between the accuracy measured on a model fine-tuned with and without regularization.  Gains in accuracy of at least 1\% are highlighted in green, while drops in accuracy of at least 1\% are highlighted in red.}
    \label{fig:imagenette_finetune_abl}
\end{figure*}

\begin{figure*}[t!]
    \centering
    \begin{subfigure}[t]{0.4\textwidth}
        \centering
        \includegraphics[width=0.95\textwidth]{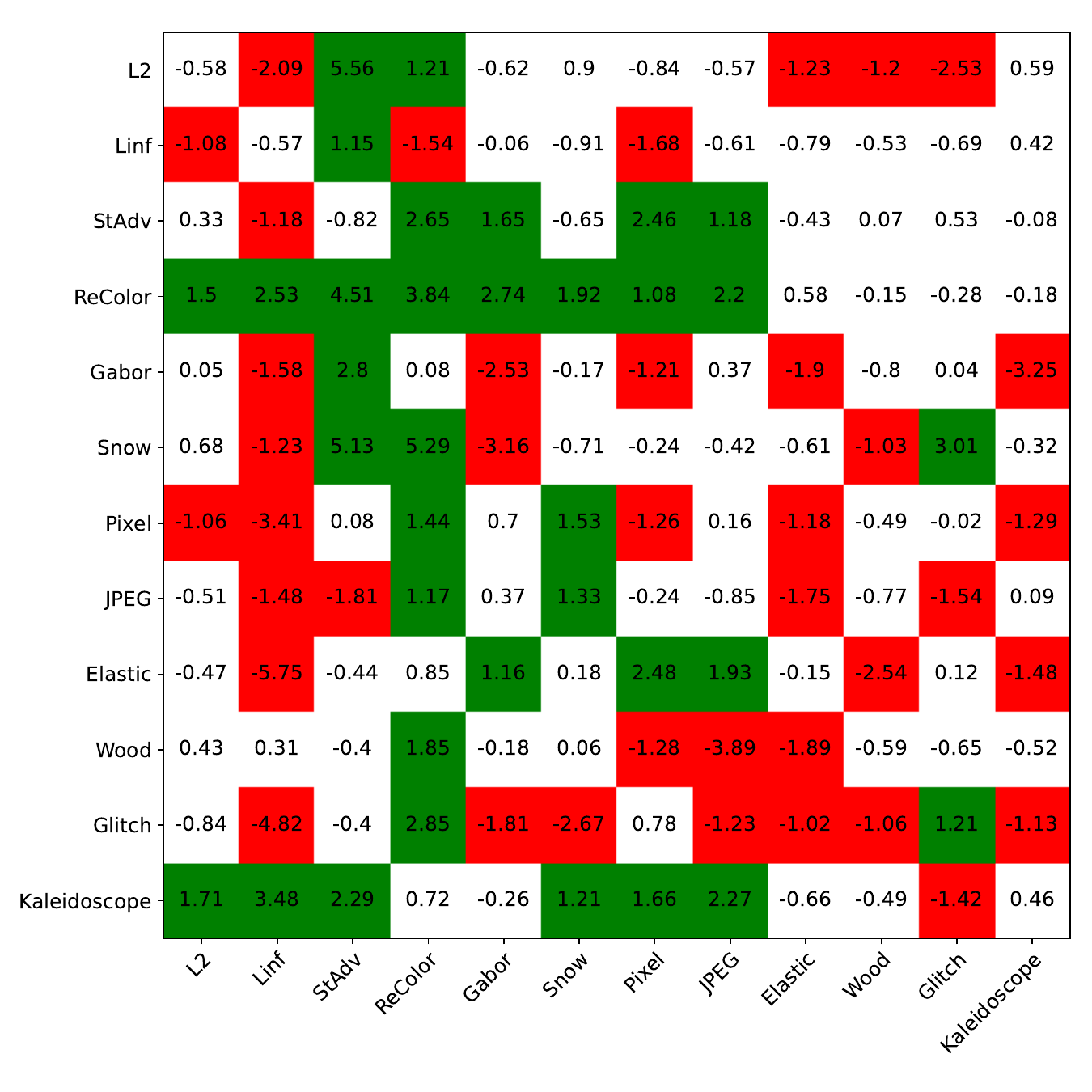}
        \vspace{-10pt}
        \caption{Difference in Avg Acc}
    \end{subfigure}%
    \begin{subfigure}[t]{0.4\textwidth}
        \centering
        \includegraphics[width=0.95\textwidth]{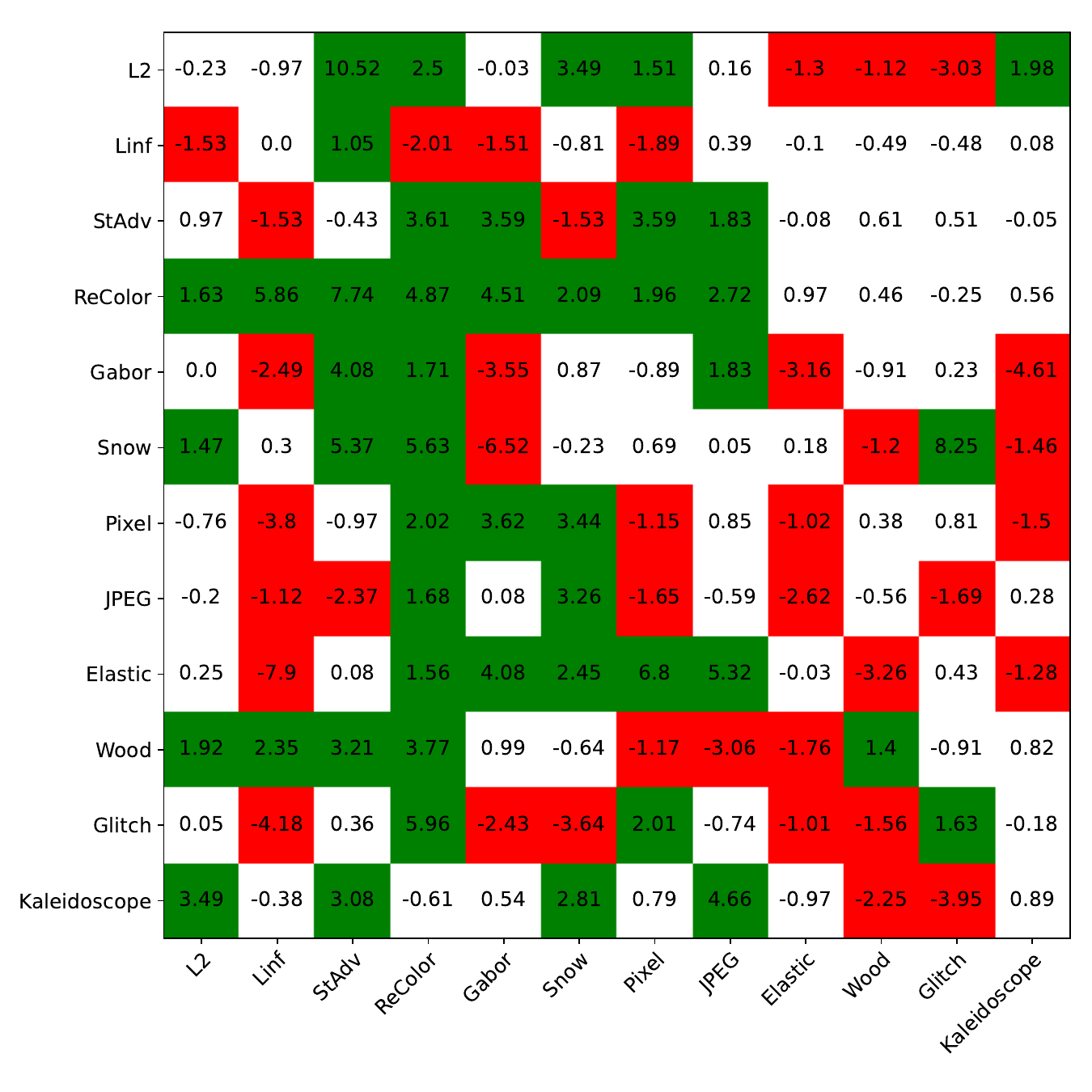}
        \vspace{-10pt}
        \caption{Difference in Union Acc}
    \end{subfigure}
    \begin{subfigure}[t]{0.4\textwidth}
        \centering
        \includegraphics[width=0.95\textwidth]{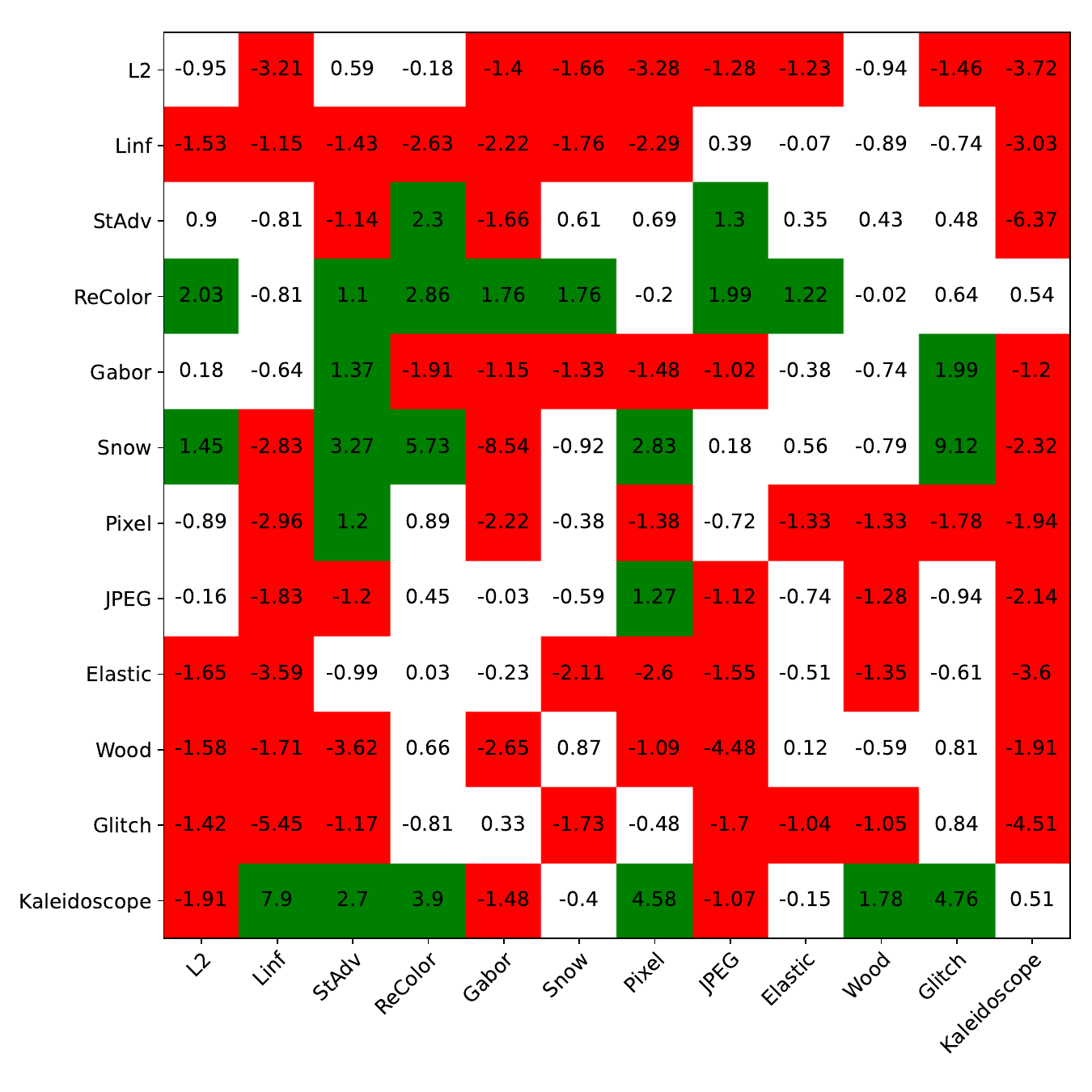}
        \vspace{-10pt}
        \caption{Difference in Initial Attack Acc}
    \end{subfigure}%
    \begin{subfigure}[t]{0.4\textwidth}
        \centering
        \includegraphics[width=0.95\textwidth]{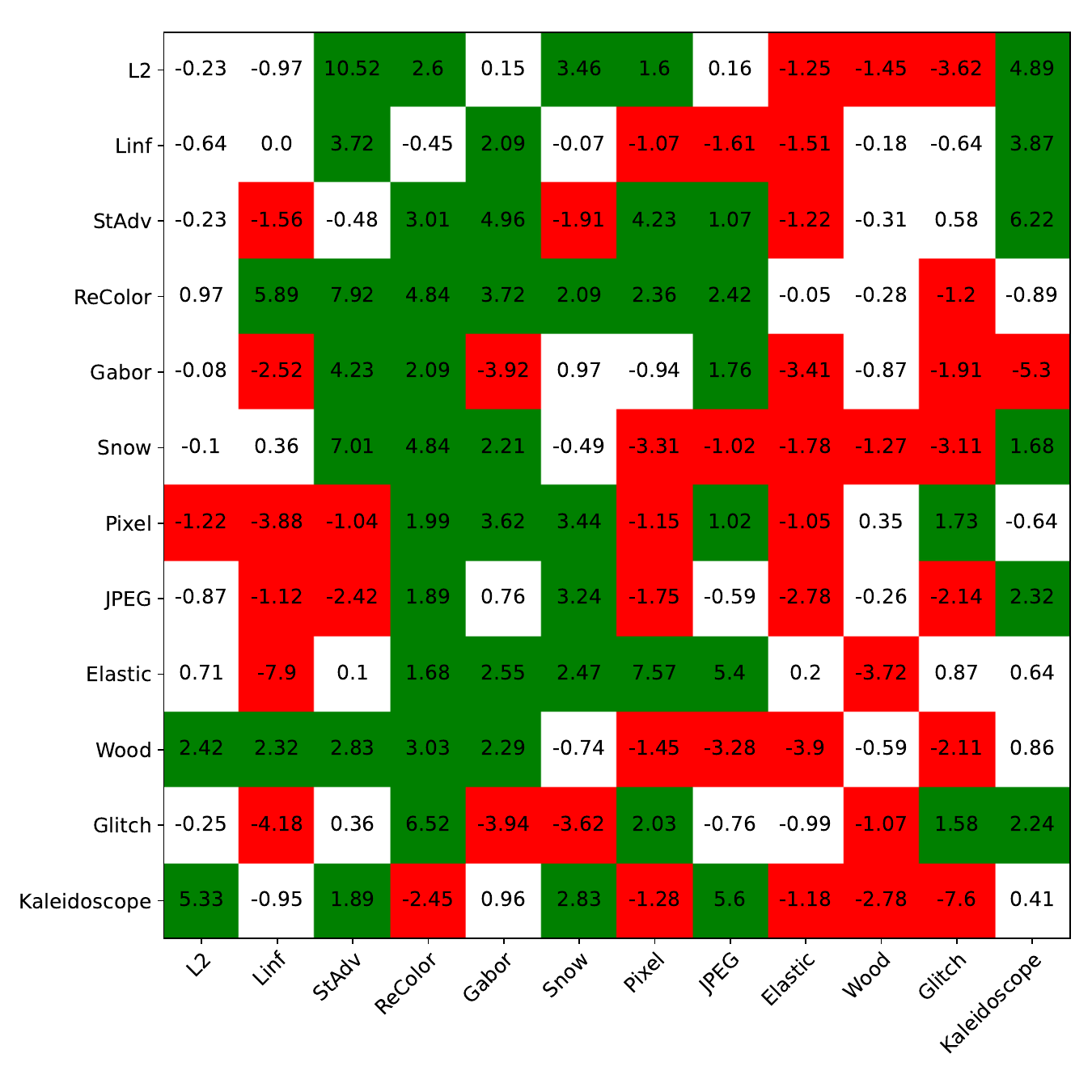}
        \vspace{-10pt}
        \caption{Difference in New Attack Acc}
    \end{subfigure}
    \begin{subfigure}[t]{0.4\textwidth}
        \centering
        \includegraphics[width=0.95\textwidth]{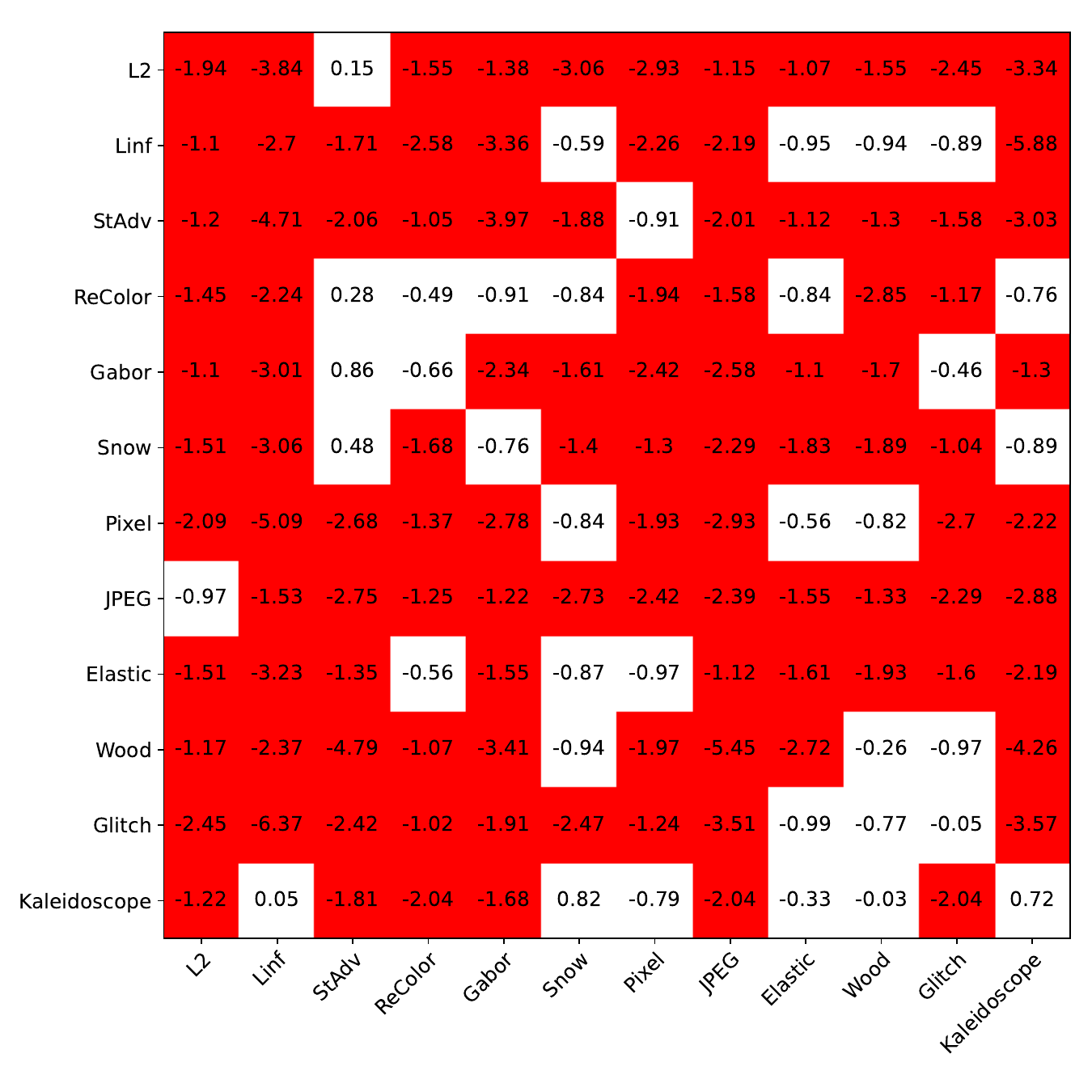}
        \vspace{-10pt}
        \caption{Difference in Clean Acc}
    \end{subfigure}
    \caption{\textbf{Change in robust accuracy after fine-tuning with variation regularization.}  We fine-tune models on Imagenette across 144 pairs of initial attack and new attack.  The initial attack corresponds to the row of each grid and new attack corresponds to each column.  Values represent differences between the accuracy measured on a model fine-tuned with and without regularization.  Gains in accuracy of at least 1\% are highlighted in green, while drops in accuracy of at least 1\% are highlighted in red.}
    \label{fig:imagenette_finetune_abl_varreg}
\end{figure*}

\begin{figure*}[t!]
    \centering
    \begin{subfigure}[t]{0.4\textwidth}
        \centering
        \includegraphics[width=0.95\textwidth]{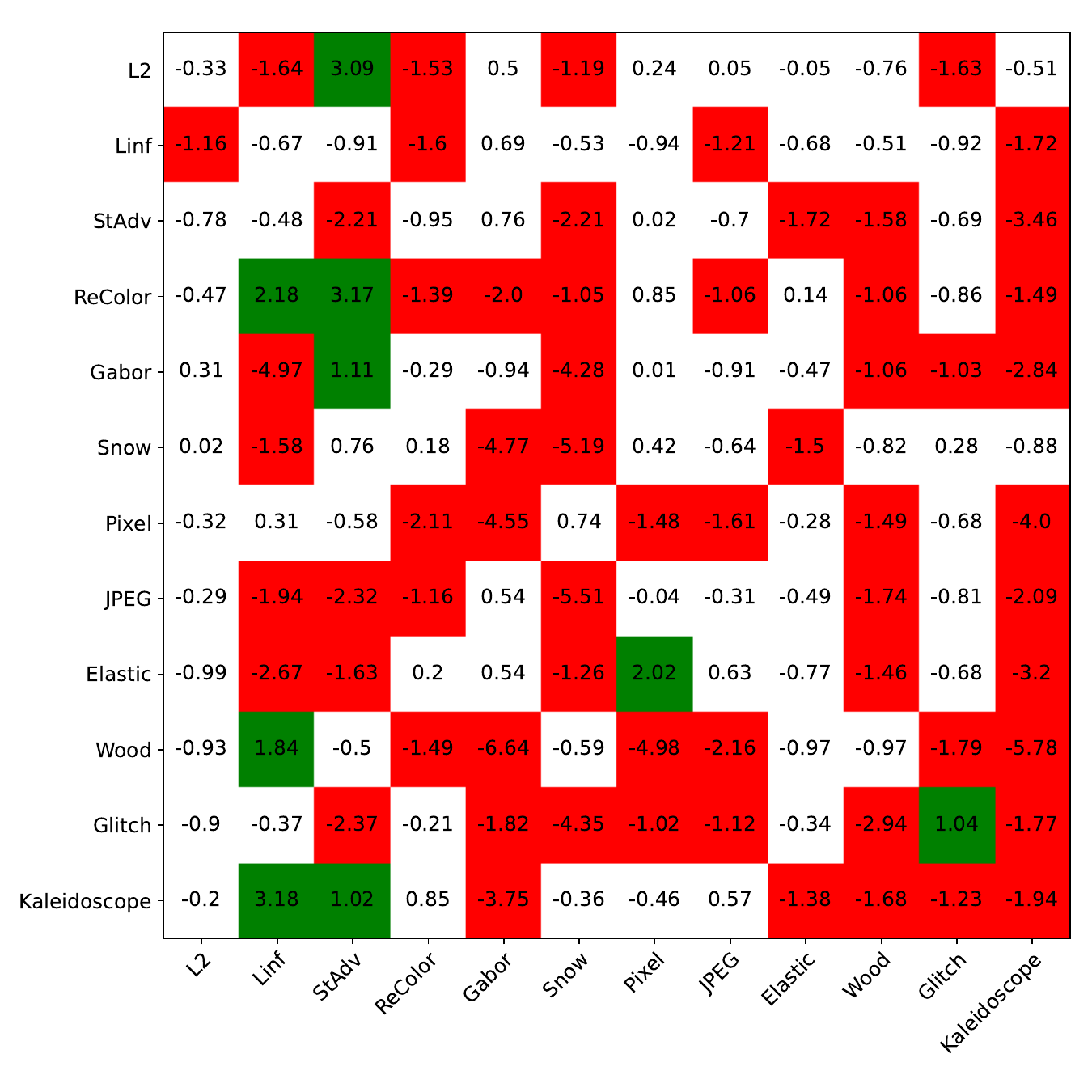}
        \vspace{-10pt}
        \caption{Difference in Avg Acc}
    \end{subfigure}%
    \begin{subfigure}[t]{0.4\textwidth}
        \centering
        \includegraphics[width=0.95\textwidth]{figures/imagenette_finetune_uniform_union.pdf}
        \vspace{-10pt}
        \caption{Difference in Union Acc}
    \end{subfigure}
    \begin{subfigure}[t]{0.4\textwidth}
        \centering
        \includegraphics[width=0.95\textwidth]{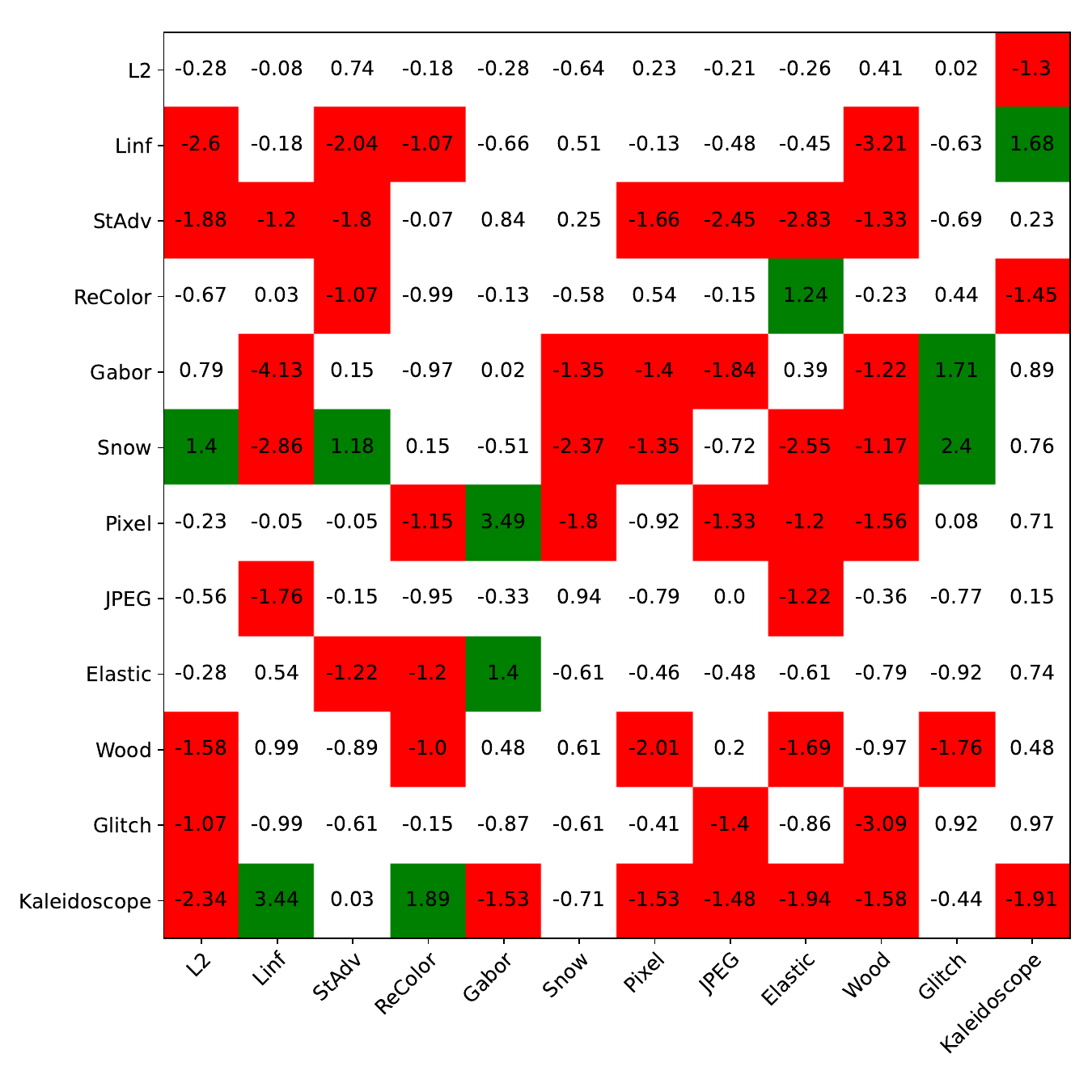}
        \vspace{-10pt}
        \caption{Difference in Initial Attack Acc}
    \end{subfigure}%
    \begin{subfigure}[t]{0.4\textwidth}
        \centering
        \includegraphics[width=0.95\textwidth]{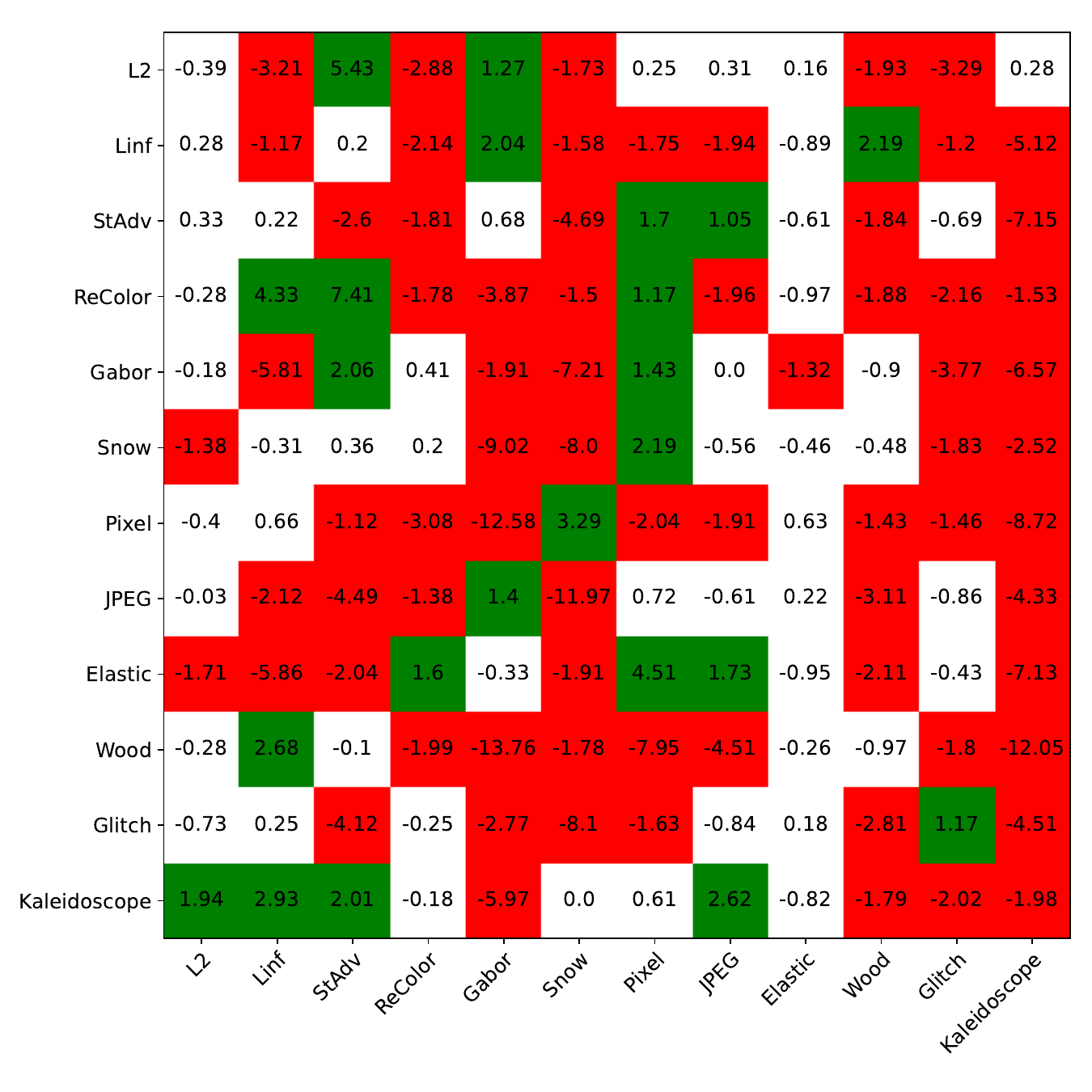}
        \vspace{-10pt}
        \caption{Difference in New Attack Acc}
    \end{subfigure}
    \begin{subfigure}[t]{0.4\textwidth}
        \centering
        \includegraphics[width=0.95\textwidth]{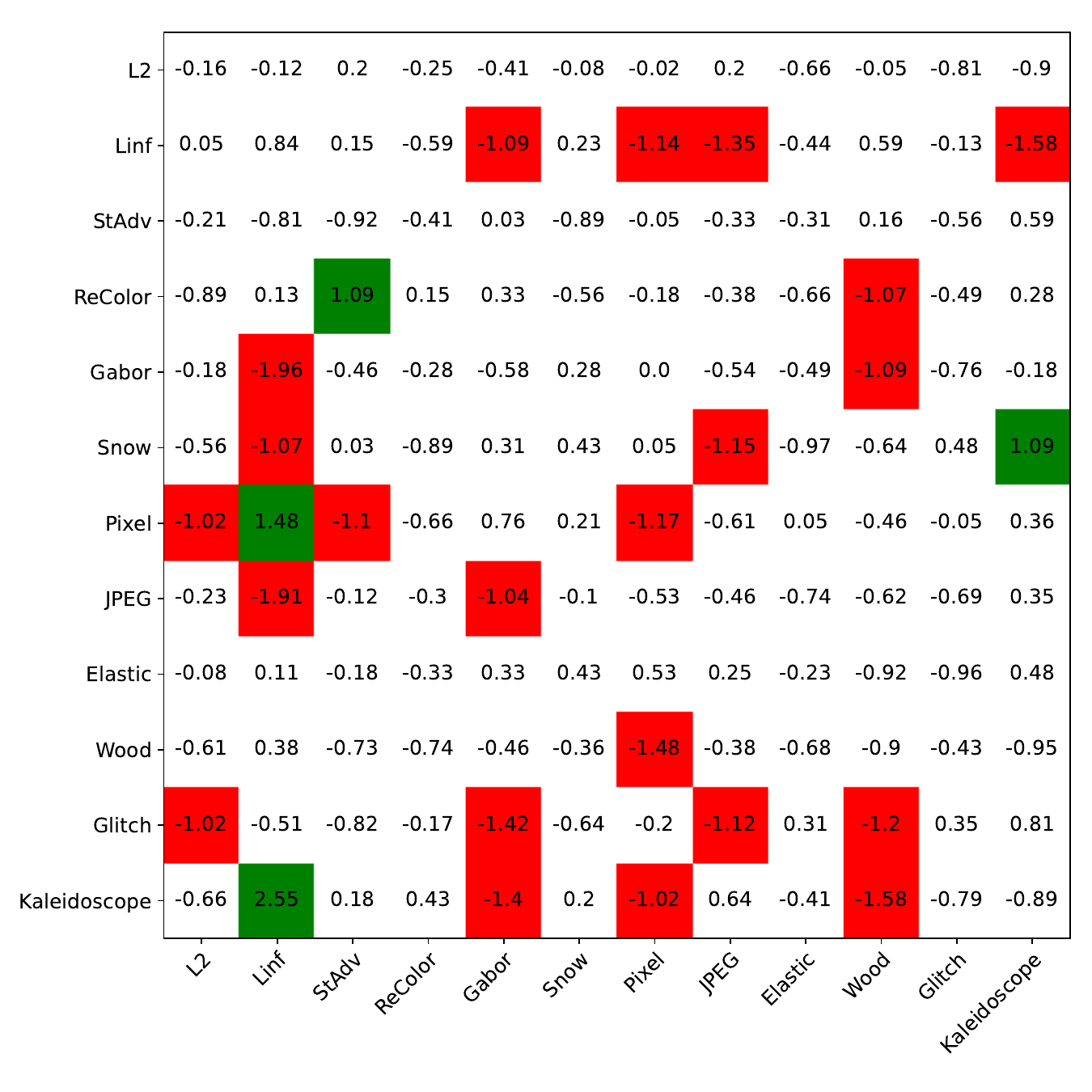}
        \vspace{-10pt}
        \caption{Difference in Clean Acc}
    \end{subfigure}
    \caption{\textbf{Change in robust accuracy after fine-tuning with uniform regularization.}  We fine-tune models on Imagenette across 144 pairs of initial attack and new attack.  The initial attack corresponds to the row of each grid and new attack corresponds to each column.  Values represent differences between the accuracy measured on a model fine-tuned with and without regularization.  Gains in accuracy of at least 1\% are highlighted in green, while drops in accuracy of at least 1\% are highlighted in red.}
    \label{fig:imagenette_finetune_abl_uniform}
\end{figure*}

\begin{figure*}[t!]
    \centering
    \begin{subfigure}[t]{0.4\textwidth}
        \centering
        \includegraphics[width=0.95\textwidth]{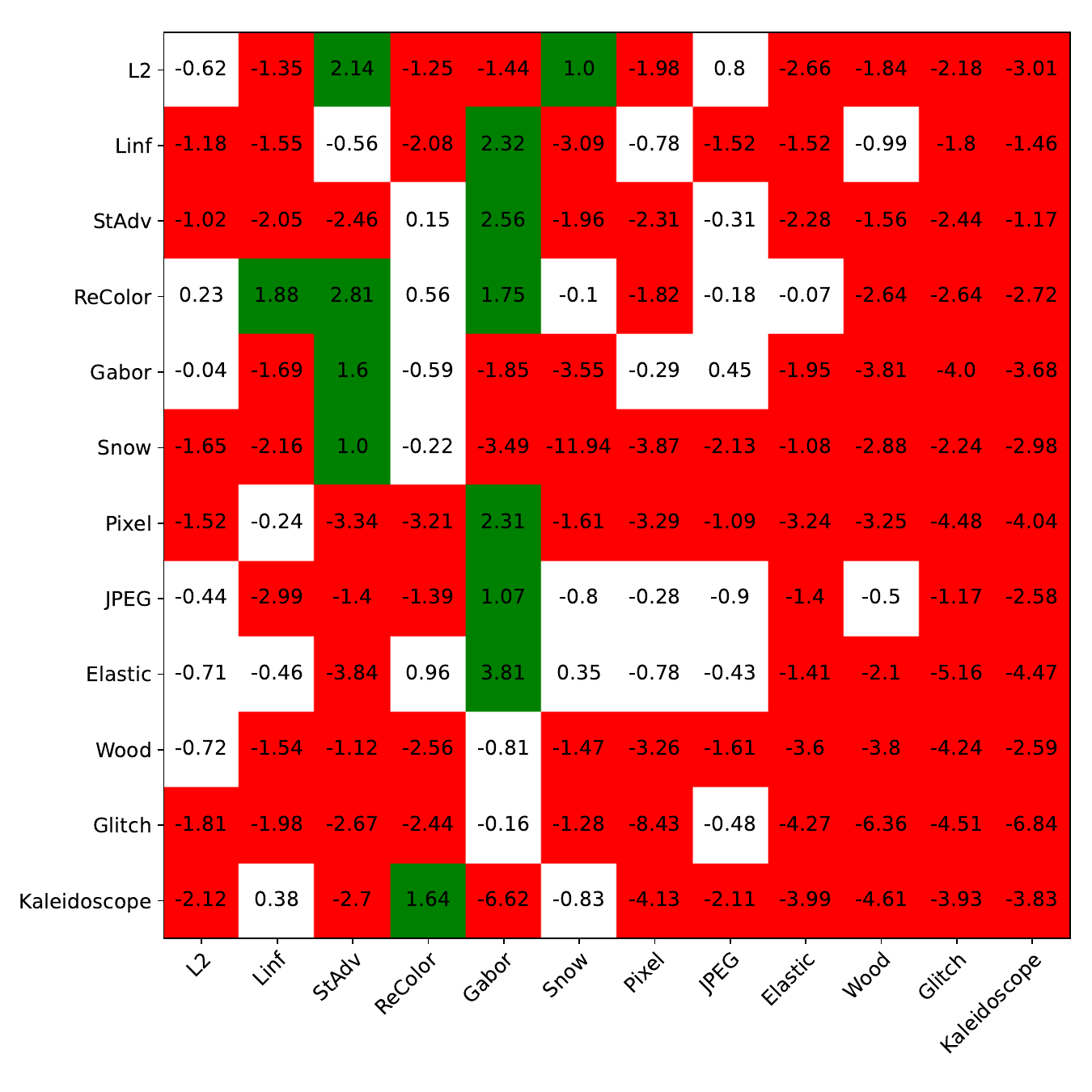}
        \vspace{-10pt}
        \caption{Difference in Avg Acc}
    \end{subfigure}%
    \begin{subfigure}[t]{0.4\textwidth}
        \centering
        \includegraphics[width=0.95\textwidth]{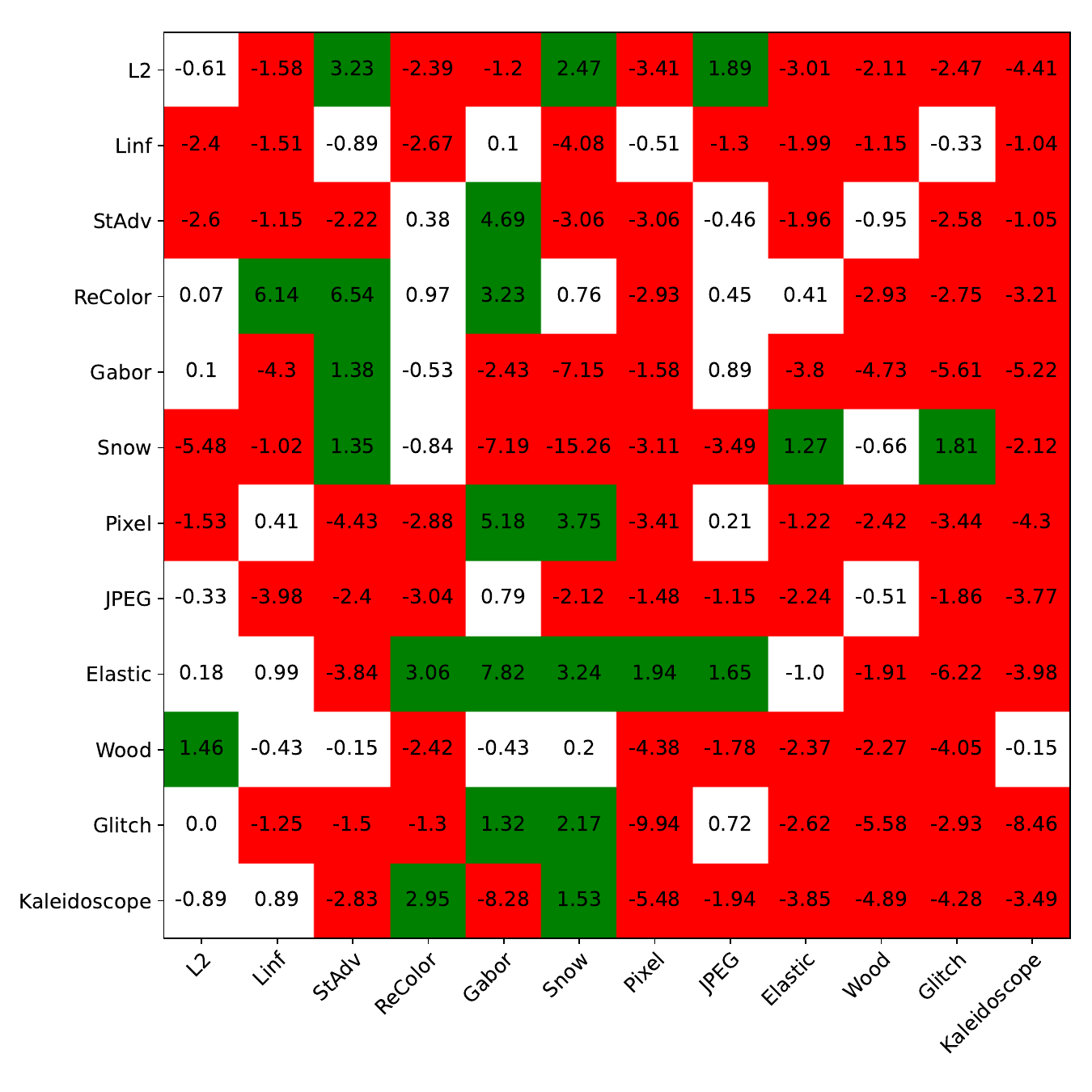}
        \vspace{-10pt}
        \caption{Difference in Union Acc}
    \end{subfigure}
    \begin{subfigure}[t]{0.4\textwidth}
        \centering
        \includegraphics[width=0.95\textwidth]{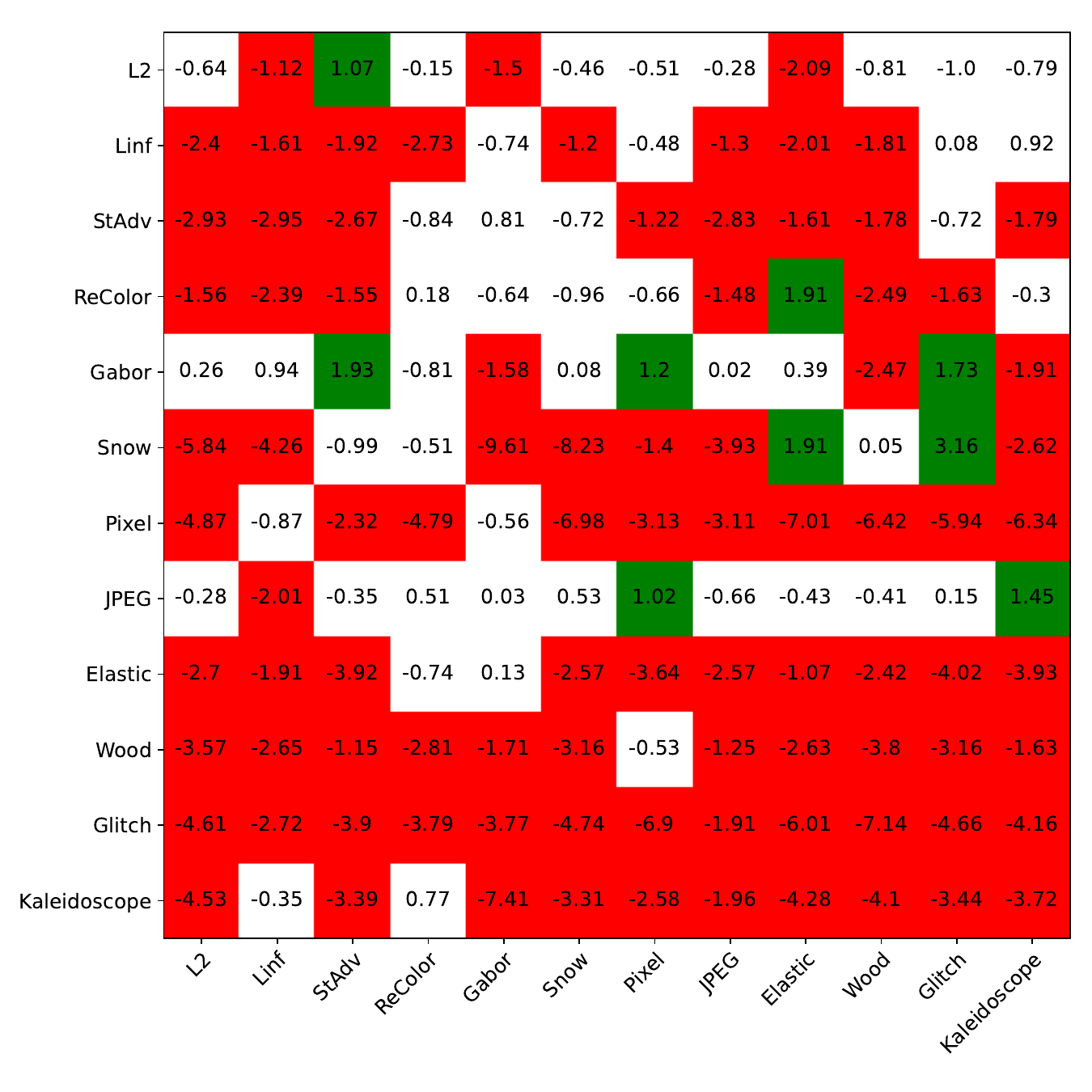}
        \vspace{-10pt}
        \caption{Difference in Initial Attack Acc}
    \end{subfigure}%
    \begin{subfigure}[t]{0.4\textwidth}
        \centering
        \includegraphics[width=0.95\textwidth]{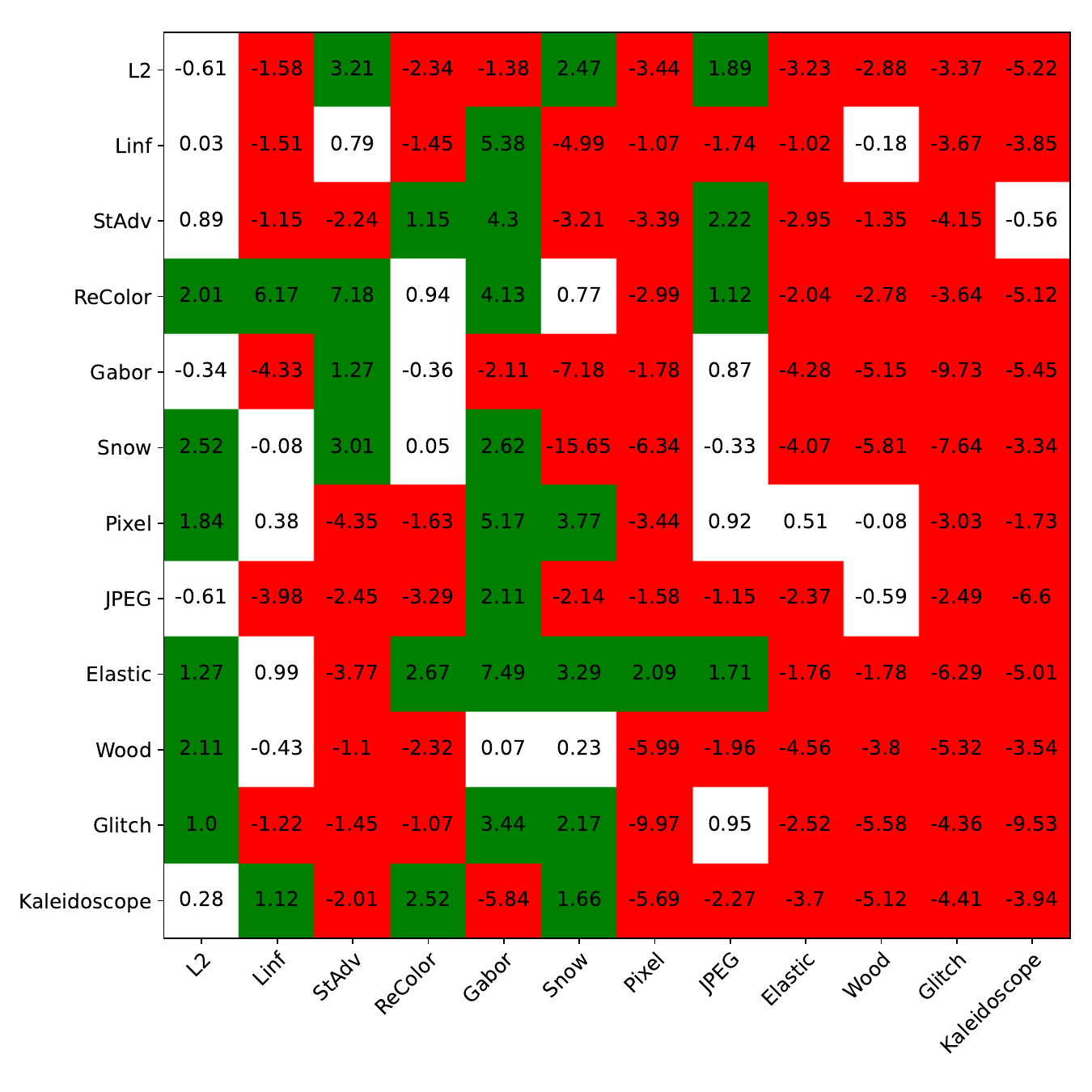}
        \vspace{-10pt}
        \caption{Difference in New Attack Acc}
    \end{subfigure}
    \begin{subfigure}[t]{0.4\textwidth}
        \centering
        \includegraphics[width=0.95\textwidth]{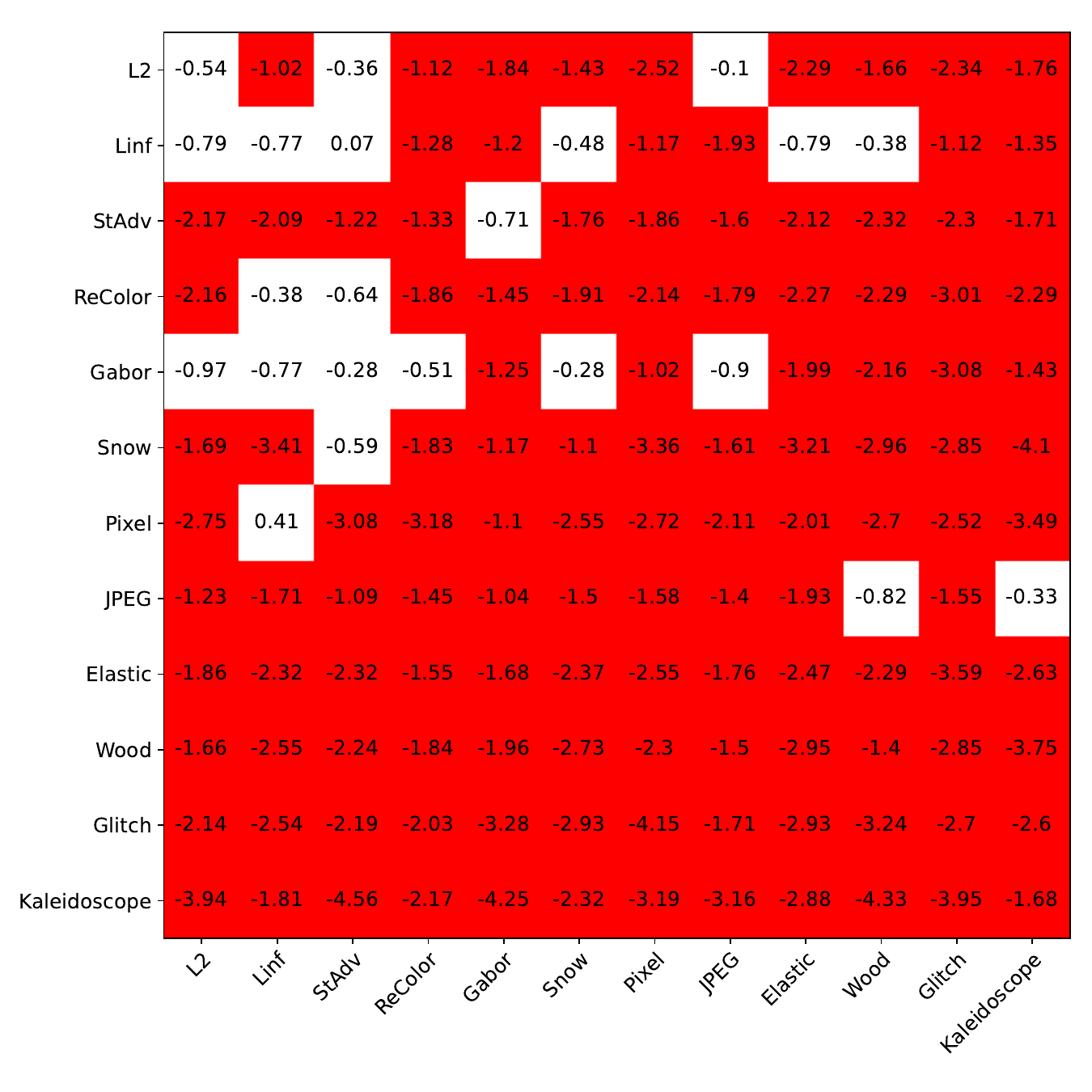}
        \vspace{-10pt}
        \caption{Difference in Clean Acc}
    \end{subfigure}
    \caption{\textbf{Change in robust accuracy after fine-tuning with Gaussian regularization.}  We fine-tune models on Imagenette across 144 pairs of initial attack and new attack.  The initial attack corresponds to the row of each grid and new attack corresponds to each column.  Values represent differences between the accuracy measured on a model fine-tuned with and without regularization.  Gains in accuracy of at least 1\% are highlighted in green, while drops in accuracy of at least 1\% are highlighted in red.}
    \label{fig:imagenette_finetune_abl_uniform}
\end{figure*}

\end{document}